\documentclass{article} 
\usepackage[margin=1in]{geometry}

\usepackage[utf8]{inputenc} 
\usepackage[T1]{fontenc}    
\usepackage{hyperref}       
\usepackage{url}            
\usepackage{booktabs}       
\usepackage{amsfonts,amsmath,amssymb}       
\usepackage{nicefrac}       
\usepackage{microtype}      
\usepackage{xcolor}         
\usepackage{txfonts}
\usepackage[inline]{enumitem}

\usepackage{graphicx}
\usepackage{subfigure}
\usepackage{wrapfig,bm,comment,color}
\usepackage{breakurl,epsfig,epsf,fmtcount,semtrans,multirow,boldline}
\usepackage{tcolorbox}
\usepackage{xcolor}
\tcbuselibrary{skins}
\usepackage{tikz}
\definecolor{darkred}{RGB}{150,0,0}
\definecolor{darkgreen}{RGB}{0,150,0}
\definecolor{darkblue}{RGB}{0,0,200}
\hypersetup{colorlinks=true, linkcolor=darkred, citecolor=darkgreen, urlcolor=darkblue}

\newtheorem{theorem}{Theorem}

\newtheorem{assumption}{Assumption}

\newtheorem{lemma}{Lemma}
\newtheorem{corollary}{Corollary}

\newtheorem{definition}{Definition}

\def \endprf{\hfill {\vrule height6pt width6pt depth0pt}\medskip}

\newenvironment{proof}{\noindent {\bf Proof.} }{\endprf\par}

\newcommand{\qed}{{\unskip\nobreak\hfil\penalty50\hskip2em\vadjust{}
           \nobreak\hfil$\Box$\parfillskip=0pt\finalhyphendemerits=0\par}}

\newcommand{\oo}{{11}}
\newcommand{\red}{\textcolor{red}}

\newcommand{\cln}[1]{\red{}}

\newcommand{\xat}{\ob}
\newcommand{\rfn}{\texttt{SVMeq}}
\newcommand{\ont}{\texttt{1token}}
\newcommand{\xast}{\xat^\st}
\newcommand{\Wf}{{\mtx{W}^\tsc{fin}}}
\newcommand\tr{{{\operatorname{trace}}}}

\newcommand{\eps}{\epsilon}
\newcommand{\epsd}{\varepsilon_\dm}

\newcommand{\hbm}{\vct{\bar{h}}}

\newcommand{\bhbg}{{\bar h}_{\tsc{gap}}}

\newcommand{\hp}{\tilde{\mtx{h}}}

\newcommand{\hb}{\vct{h}}
\newcommand{\al}{\alpha}

\newcommand{\st}{\star}
\newcommand{\dm}{{\diamond}}

\newcommand{\xa}{\x^\bal}

\newcommand{\beq}{\begin{equation}}
\newcommand{\ba}{\begin{align}}
\newcommand{\ea}{\end{align}}

\newcommand{\eeq}{\end{equation}}

\newcommand{\Rcm}{\Rcc_m}

\newcommand{\nn}{\nonumber}
\newcommand{\la}{\lambda}

\newcommand{\K}{\mtx{K}}
\newcommand{\A}{{\mtx{A}}}

\newcommand{\Ub}{{\mtx{U}}}

\newcommand{\kron}{\otimes}
\newcommand{\M}{{\mtx{M}}}

\newcommand{\name}{Att-SVM}

\newcommand{\B}{{{\mtx{B}}}}

\newcommand{\Gb}{{\mtx{G}}}

\newcommand{\diag}[1]{\text{diag}(#1)}

\newcommand{\Lc}{{\cal{L}}}

\newcommand{\Nc}{{\cal{N}}}

\newcommand{\Qb}{{\mtx{Q}}}

\newcommand{\Cb}{{\mtx{C}}}

\newcommand{\F}{{\mtx{F}}}
\newcommand{\Fa}{{\mtx{F}}^\bal}

\newcommand{\bSi}{{\boldsymbol{{\Sigma}}}}

\newcommand{\onebb}{{\mathbf{1}}}

\newcommand{\order}[1]{{\cal{O}}(#1)}

\newcommand{\z}{{\vct{z}}}

\newcommand{\sft}[1]{\mathbb{S}(#1)}
\newcommand{\sftk}[1]{\mathbb{S}_k(#1)}
\newcommand{\sftx}{\mathbb{S}}
\newcommand{\sfp}[1]{\mathbb{S}'(#1)}
\newcommand{\distd}[1]{\texttt{dist}_\dm\left(#1\right)}
\newcommand{\tn}[1]{\|{#1}\|}
\newcommand{\td}[1]{\|{#1}\|_\dm}

\newcommand{\tf}[1]{\|{#1}\|_{F}}
\newcommand{\tnuc}[1]{\|{#1}\|_{\star}}

\newcommand{\dist}[1]{\texttt{dist}\left(#1\right)}

\newcommand{\Cc}{\mathcal{C}}
\newcommand{\Rcc}{\mathcal{R}}

\newcommand{\Bal}{{\boldsymbol{\Delta}}}

\newcommand{\Rc}{\mathcal{O}}

\newcommand{\bal}{{\boldsymbol{\alpha}}}
\newcommand{\bgam}{\boldsymbol{\gamma}}
\newcommand{\bga}{\boldsymbol{\gamma}^\bal}
\newcommand{\gamb}{\bar{\gamma}}

\newcommand{\Sc}{\mathcal{S}}
\newcommand{\Scc}{\bar{\mathcal{S}}}

\newcommand{\Nn}{\mathcal{N}}

\newcommand{\vb}{\vct{v}}

\newcommand{\fb}{\vct{f}}
\newcommand{\Fb}{\vct{F}}

\newcommand{\nei}{\text{support index}\xspace}
\newcommand{\neis}{\text{support indices}\xspace}

\newcommand{\Neis}{\text{Support indices}\xspace}
\newcommand{\NEIS}{\text{Support Indices}\xspace}
\newcommand{\w}{\vct{w}}

\newcommand{\cdm}{c_\dm}
\newcommand{\cop}{c_\texttt{up}}
\newcommand{\cdn}{c_\texttt{dn}}

\newcommand{\ob}{\mtx{o}}

\newcommand{\li}{\left<}
\newcommand{\xdm}{\Xi_\dm}
\newcommand{\ri}{\right>}
\newcommand{\s}{\vct{s}}
\newcommand{\sik}{\s^{(\ik)}}
\newcommand{\sir}{\s^R}
\newcommand{\abik}{\ab^{(\ik)}}
\newcommand{\abr}{\ab^R}
\newcommand{\ab}{{\vct{a}}}

\newcommand{\bgag}{{\gamma}_{\tsc{gap}}}
\newcommand{\bgg}{\gamma^{\tsc{gap}}}
\newcommand{\bggm}{\gamma^{\tsc{gap}}_{\min}}
\newcommand{\bgm}{\bar{\gamma}^{\tsc{gap}}}

\newcommand{\bb}{\vct{b}}

\newcommand{\corr}[1]{{\texttt{corr\_coef}}(#1)}
\newcommand{\Tc}{\mathcal{T}}
\newcommand{\Tcb}{\bar{\mathcal{T}}}

\newcommand{\kb}{\vct{k}}

\newcommand{\cone}{\Sc}
\newcommand{\conb}{\bar{\Cc}}
\newcommand{\con}[1]{\texttt{cone}_{\eps}(#1)}

\newcommand{\low}{\texttt{low}^{\alpha}(\X)}
\newcommand{\high}{\texttt{high}^{\alpha}(\X)}
\newcommand{\lowi}{\texttt{low}_\ik^{\alpha}}
\newcommand{\higi}{\texttt{high}_\ik^{\alpha}}

\newcommand{\bbg}{\bgam^{\tsc{gap}}}

\newcommand{\ps}{\W^\svm}
\newcommand{\Ws}{\W^\svm}
\newcommand{\Wma}{\W^\svm_\bal}
\newcommand{\Wsf}{\W^\svm}
\newcommand{\Wsb}{\bar{\W}^\svm}
\newcommand{\Wcs}{\Wc^\svm}
\newcommand{\ik}{{ik}}
\newcommand{\itt}{{it}}
\newcommand{\ittt}{{i\tau}}
\newcommand{\ikt}{{ikt}}
\newcommand{\iktt}{{ik\tau}}
\newcommand{\ikix}{_{ik=(1,1)}^{(n,k)}}
\newcommand{\inn}[1]{\left<#1\right>}
\newcommand{\Ccd}{\Cc_{\eps,R_0}^\dm}
\newcommand{\aik}{\alpha_{ik}}

\newcommand{\RR}{\bar{R}}

\newcommand{\x}{\vct{x}}

\newcommand{\W}{\mtx{W}}

\newcommand{\Wc}{{\cal{W}}}
\newcommand{\Wcb}{{\cal{W}}}

\newcommand{\bgl}{{~\big |~}}

\definecolor{emmanuel}{RGB}{255,127,0}

\newcommand{\Kb}{{\mtx{K}}}
\newcommand{\Kbb}{{\mtx{\bar{K}}}}
\newcommand{\Qbb}{{\mtx{\bar{Q}}}}
\newcommand{\Wb}{{\mtx{\bar{W}}}}

\newcommand{\pb}{{\vct{p}}}

\newcommand{\wrb}[1]{{\vct{\bar{W}}}(#1)}
\newcommand{\wrt}[1]{{\vct{\bar{W}}}_0(#1)}

\newcommand{\R}{\mathbb{R}}

\newcommand{\Z}{\mtx{Z}}
\newcommand{\V}{\mtx{V}}

\newcommand{\vct}[1]{\bm{#1}}
\newcommand{\mtx}[1]{\bm{#1}}

\newcommand{\X}{{\mtx{X}}}
\newcommand{\Y}{{\mtx{Y}}}
\newcommand{\Vb}{{\mtx{V}}}

\newcommand{\iprod}[2]{\left\langle #1 , #2 \right\rangle}

\newcommand{\mc}{\mathcal}

\renewcommand{\mc}[1]{\ensuremath{\mathcal{#1}}} 
\newcommand{\g}{\vct{g}}

\renewcommand{\qed}{\hfill\blacksquare}

\usepackage{xspace}
\usepackage{pifont}
\newcommand{\tsc}{\textsl}

\newcommand{\svm}{\tsc{mm}}
\newcommand{\Wm}{\W^\svm}

\newcommand{\op}{\texttt{opt}}
\newcommand{\opt}{\texttt{opt}}

\newcommand{\Rcb}{\bar{\Rc}}

\usepackage[utf8]{inputenc} 
\usepackage[T1]{fontenc}    
\usepackage{hyperref}       
\usepackage{url}            
\usepackage{booktabs}       
\usepackage{amsfonts}       
\usepackage{nicefrac}       
\usepackage{microtype}      
\usepackage{xcolor}         
\usepackage{enumitem}
\usepackage{tikz}

\usepackage[toc,page,header]{appendix}
\usepackage{minitoc}

\title{
Transformers as Support Vector Machines}
\date{}
\begin{document}

\author{\\Davoud Ataee Tarzanagh$^{1\star}$\quad Yingcong Li$^{2\star}$\qquad Christos Thrampoulidis$^{3}$\qquad Samet Oymak$^{4\dagger}$}

\addtocontents{toc}{\protect\setcounter{tocdepth}{0}}
\maketitle

{\let\thefootnote\relax\footnotetext{$^1$ University of Pennsylvania, \texttt{tarzanaq@upenn.edu}. $^2$ University of California, Riverside, \texttt{yli692@ucr.edu}. $^3$ University of British Columbia, \texttt{cthrampo@ece.ubc.ca}. $^4$ University of Michigan, \texttt{oymak@umich.edu}. $^\star$ Equal contribution. $^\dagger$ Corresponding author.}}

\begin{abstract} 
Since its inception in ``Attention Is All You Need'', the transformer architecture has led to revolutionary advancements in natural language processing. The attention layer within the transformer admits a sequence of input tokens $\X$ and makes them interact through pairwise similarities computed as $\texttt{softmax}(\X\Qb\Kb^\top\X^\top)$, where $(\Kb,\Qb)$ are the trainable key-query parameters. In this work, we establish a formal equivalence between the optimization geometry of self-attention and a hard-margin SVM problem that separates optimal input tokens from non-optimal tokens using linear constraints on the outer-products of token pairs. This formalism allows us to characterize the implicit bias of 1-layer transformers optimized with gradient descent, as follows. \textbf{(1)} Optimizing the attention layer, parameterized by $(\Kb,\Qb)$, with vanishing regularization, converges in direction to an SVM solution minimizing the nuclear norm of the combined parameter $\W:=\Kb\Qb^\top$. Instead, directly parameterizing by $\W$ minimizes a Frobenius norm SVM objective. 
We  characterize this convergence, highlighting that it can occur in locally-optimal directions rather than global ones.
\textbf{(2)} Complementing this, for $\W$-parameterization, we prove the local/global directional convergence of gradient descent under suitable geometric conditions. Importantly, we show that over-parameterization catalyzes global convergence by ensuring the feasibility of the SVM problem and by guaranteeing a benign optimization landscape devoid of stationary points. \textbf{(3)} While our theory  applies primarily to linear prediction heads, we propose a more general SVM equivalence that  predicts the implicit bias of 1-layer transformers with nonlinear heads/MLPs. 
Our findings apply to general datasets, trivially extend to cross-attention layer, and their practical validity is verified via thorough numerical experiments. We also introduce open problems and future research directions. We believe these findings inspire a new perspective, interpreting multilayer transformers as a hierarchy of SVMs that separates and selects optimal tokens.
\end{abstract}
\begin{figure}
    \centering
    \begin{minipage}{.58\textwidth}
    \centering
    \hspace{-15pt}
    \subfigure[$\W$-parameterization]{
        \begin{tikzpicture}
        \node at (0,0) {\includegraphics[height=.37\columnwidth]{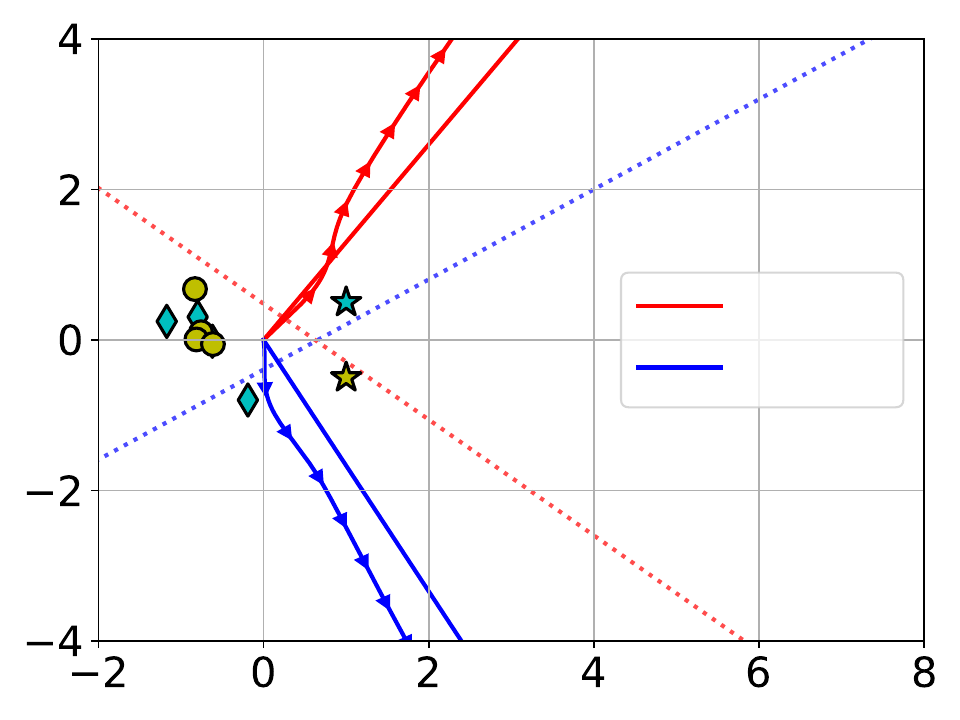}};
        \node[right] at (1.13,0.23) {\scriptsize{$\Wm\z_1$}};
        \node[right] at (1.13,-0.07) {\scriptsize{$\Wm\z_2$}};
        \end{tikzpicture}
        \label{fig path W}
    }
    \hspace{-12pt}
    \subfigure[$(\Kb,\Qb)$-parameterization]{
        \begin{tikzpicture}
        \node at (0,0) {\includegraphics[height=.37\columnwidth]{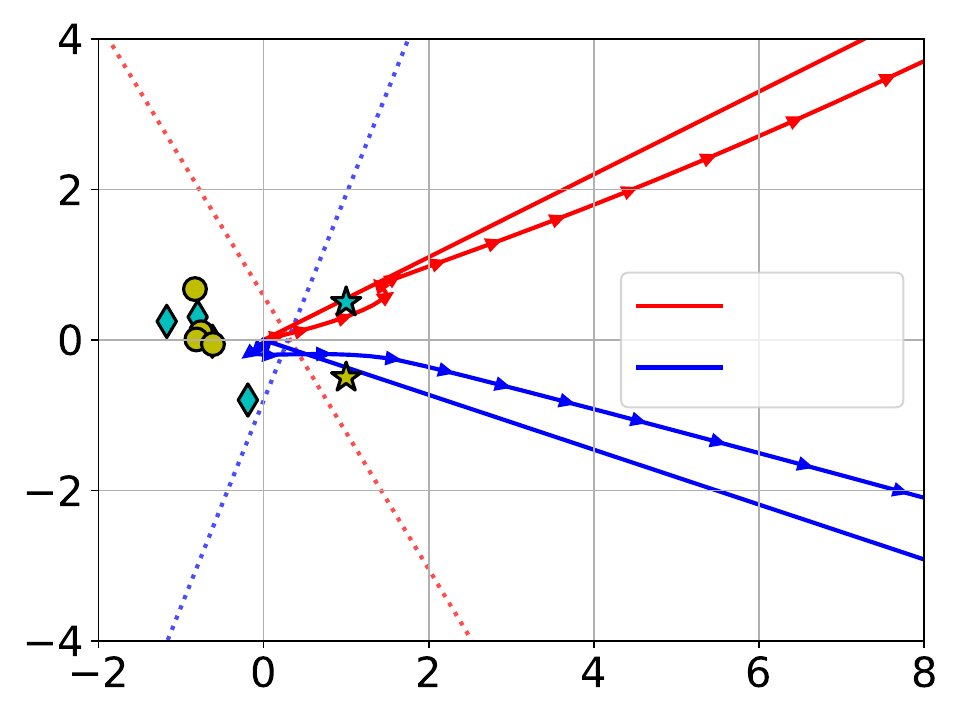}};
        \node[right] at (1.13,0.23) {\scriptsize{$\Wm_\st\z_1$}};
        \node[right] at (1.13,-0.07) {\scriptsize{$\Wm_\st\z_2$}};
        \end{tikzpicture}
        \label{fig path KQ}
    }
    \vspace{0pt}
    \caption{GD convergence during training of cross-attention weight $\W$ or $(\Kb,\Qb)$ with data. Teal and yellow markers represent tokens from $\X_1$ and $\X_2$, while stars mark optimal tokens. Solid lines in Figures {\color{blue}(a)} and {\color{blue}(b)} depict \ref{eqn:sattnsvm} and  \ref{eqn:sattnsvmst}  directions mapped to $\z_1$ (red) and $\z_2$ (blue), respectively.  Arrows illustrating GD trajectories converging towards these SVM directions. Red and blue dotted lines represent the corresponding separating hyperplanes.} 
    \label{fig path}
    \end{minipage}
    \hspace{4pt}
    \begin{minipage}{0.4\textwidth}
    \centering
    \vspace{7pt}
        \begin{tikzpicture}
        \node at (0,0) {\includegraphics[height=.47\columnwidth, trim={1.3cm 1.3cm 3.2cm 0}, clip]{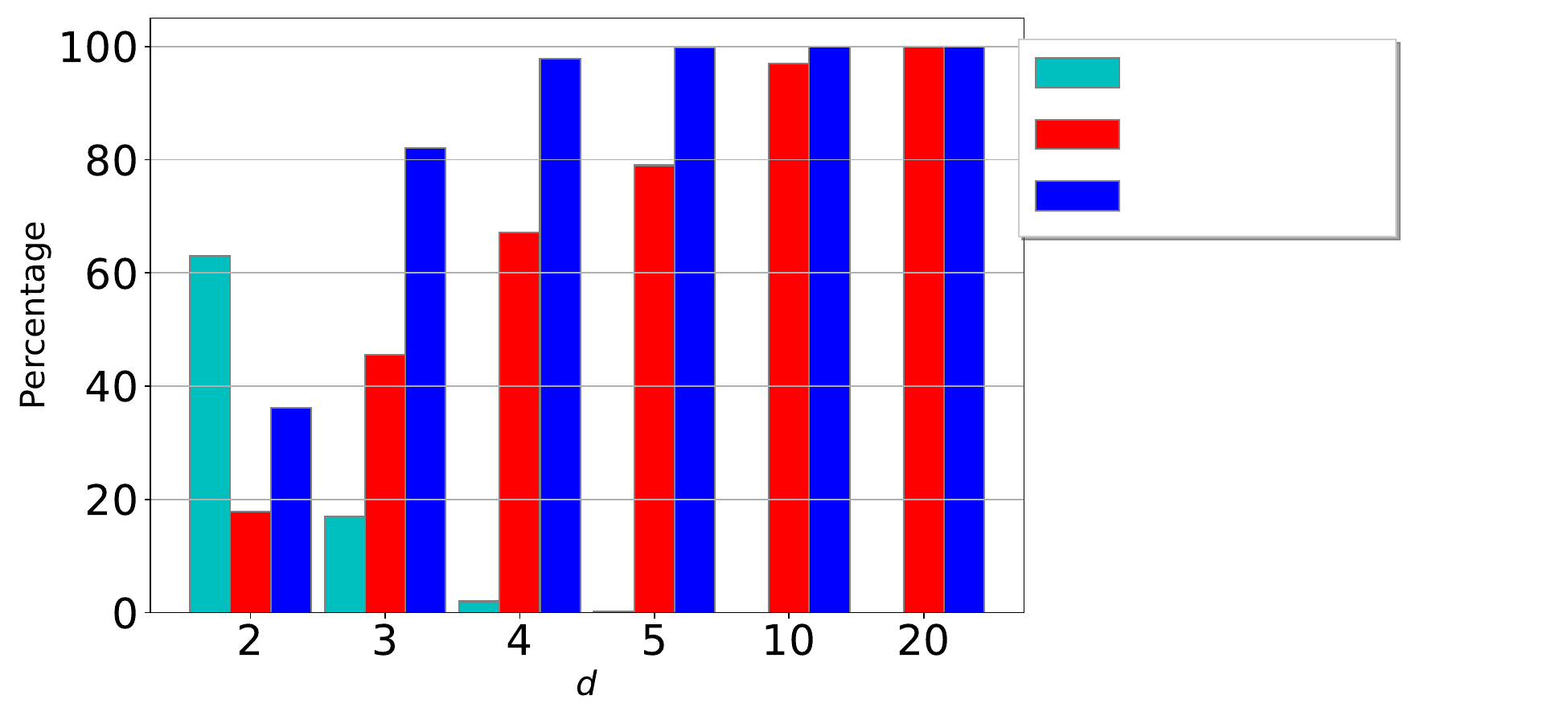}};
        \node[rotate=90] at (-3.3,0) {\footnotesize{Percentage}};
        \node at (-0.7,-1.8){\footnotesize{Varying $d$}};
        \node[right] at (1.8,1.2){\scriptsize{Not Local}};
        \node[right] at (1.8,0.92){\scriptsize{Global}};
        \node[right] at (1.8,0.64){\scriptsize{Local}};
        \end{tikzpicture}
    \vspace{-17pt}
    \caption{
Percentage of different convergence types when training cross-attention weights ($\W$) using GD and varying dimension ($d$).  Red and blue bars represent the percentages of convergence to globally-optimal and locally-optimal (including global) SVM solutions, respectively. Teal bars are complements of the blue bars. 
Larger overparameterization ($d$) increases the likelihood of global convergence.
    } 
    \label{fig overparam W bar}
    \end{minipage}
\end{figure}

\section{Introduction}
Self-attention, the central component of the transformer architecture, has revolutionized natural language processing (NLP) by empowering the model to identify complex dependencies within input sequences \cite{vaswani2017attention}. By assessing the relevance of each token to every other token, self-attention assigns varying degrees of importance to different parts of the input sequence. This mechanism has proven highly effective in capturing long-range dependencies, which is essential for applications arising in NLP ~\cite{kenton2019bert,brown2020language,raffel2020exploring}, computer vision~\cite{fan2021multiscale,liu2021swin,touvron2021training,chen2023jigsaw}, and reinforcement learning~\cite{janner2021offline,chen2021decision,wu2022flowformer}.  
Remarkable success of the self-attention mechanism and transformers has paved the way for the development of sophisticated language models such as GPT4 \cite{gpt4},
Bard \cite{bard}, LLaMA \cite{touvron2023llama}, and  ChatGPT \cite{openai_chatgpt}. 
\begin{quote}
\textbf{Q:}~Can we characterize the optimization landscape and implicit bias of transformers? 
How does the attention layer select and compose tokens when trained with gradient descent?
\end{quote}

We address these questions by rigorously connecting the optimization geometry of the attention layer and a hard max-margin SVM problem, namely \eqref{eqn:sattnsvm}, that separates and selects the optimal tokens from each input sequence. This formalism, which builds on the recent work \cite{tarzanagh2023margin}, is practically meaningful as demonstrated through experiments, and sheds light on the intricacies of self-attention. 
Throughout, given  input sequences $\X,\Z\in\R^{T\times d}$ with length $T$ and embedding dimension $d$, we study the core cross-attention and self-attention models:
\begin{subequations}\label{eqn:sa:obj}
\begin{align}
f_{\texttt{cross}}(\X,\Z)&:= \sftx(\Z \Qb\K^\top\X^\top)\X\V, 
\label{xatt eq}\\
f_{\texttt{self}}(\X)&:= \sftx(\X \Qb\K^\top\X^\top)\X\V.
\label{satt eq}
\end{align}
\end{subequations}
Here, $\Kb, \Qb  \in \R^{d\times m}$, $\V \in \R^{d\times v} $ are the trainable key, query, value matrices respectively; $\sft{\cdot}$ denotes the softmax nonlinearity, which is applied row-wise on  $\X \Qb\K^\top\X^\top$.  Note that self-attention  \eqref{satt eq} is a special instance of the cross-attention \eqref{xatt eq} by setting $\Z\gets \X$. To expose our main results, suppose the first token of $\Z$ -- denoted by $\z$ -- is used for prediction. Concretely, given a training dataset $(Y_i,\X_i, \z_i)_{i=1}^n$ with labels $Y_i\in \{-1,1\}$ and inputs $\X_i\in\R^{T\times d},\z_i\in\R^d$, we consider the empirical risk minimization with a decreasing loss function $\ell(\cdot):\R\rightarrow\R$, represented as follows:
\begin{align}\label{eq:erm:init} 
\Lc(\K,\Qb)=\frac{1}{n}\sum_{i=1}^n \ell \left(Y_i\cdot f(\X_i,\z_i)\right),\quad\text{where}~~f(\X_i,\z_i)=h\left(\X^\top_i \sftx\left(\X_i \K\Qb^\top\z_i\right)\right).
\end{align}
Here, $h(\cdot):\R^{d}\rightarrow\R$ is the prediction head that subsumes the value weights $\Vb$. In this formulation, the model $f(\cdot)$ precisely represents a one-layer transformer where an MLP follows the attention layer. Note that, we recover the  self-attention in \eqref{eq:erm:init} by setting $\z_i\gets \x_{i1}$, where $\x_{i1}$ denotes the first token of the sequence $\X_i$\footnote{Note that for simplicity, we set $\z_i = \x_{i1}$, but it can be any other row of $\X_i$.}. The softmax operation, due to its nonlinear nature, poses a significant challenge when optimizing \eqref{eq:erm:init}. The problem is nonconvex and nonlinear even when the prediction head is fixed and linear. In this study, we focus on optimizing the attention weights ($\Kb,\Qb$ or $\W$) and overcome such challenges to establish a fundamental SVM equivalence.\footnote{We fix $h(\cdot)$ and only optimize the attention weights. This is partly to avoid the degenerate case where $h(\cdot)$ can be used to achieve zero training loss (e.g.~via standard arguments like NTK \cite{jacot2018neural}) without providing any meaningful insight into the functionality of the attention mechanism.}

The paper's main contributions are as follows:

\begin{enumerate}[label=$\bullet$, wide, labelwidth=!,itemindent=!, labelindent=5pt]
\item \textbf{Implicit bias of the attention layer (Secs. \ref{sec:prelim}-\ref{sec:bias}).}  Optimizing the attention parameters $(\Kb,\Qb)$ with vanishing regularization converges in direction towards a max-margin solution of \eqref{eqn:sattnsvmst} with the nuclear norm objective of the combined parameter $\W:=\K\Qb^\top$ (Thm \ref{thm global reg path}). 
In the case of directly parameterizing cross-attention by the combined parameter $\W$, the regularization path (RP) directionally converges to \eqref{eqn:sattnsvm} solution with the Frobenius norm objective. To our knowledge, this is the first result to formally distinguish the optimization dynamics of $\W$ vs $(\Kb,\Qb)$ parameterizations, revealing the low-rank bias of the latter.  Our theory clearly characterizes the \emph{optimality} of selected tokens (Definition~\ref{score def}) and naturally extends to sequence-to-sequence or causal classification settings (see \ref{seqattnsvm} and Theorem \ref{local RP thm} in appendix).

\item \textbf{Convergence of gradient descent (Secs. \ref{provable global}-\ref{sec local}).} Gradient descent (GD) iterates for the combined key-query variable $\W$ converge in direction to a \emph{locally-optimal} solution of \eqref{eqn:sattnsvm} with appropriate initialization and a linear head $h(\cdot)$ (Sec. \ref{sec local}). For local optimality, selected tokens must have higher scores than their neighboring tokens. Locally-optimal directions are not necessarily unique and are characterized in terms of the problem geometry.  As a key contribution, we identify geometric conditions that guarantee convergence to the globally-optimal direction (Sec. \ref{provable global}). Besides these, we show that over-parameterization (i.e.~dimension $d$ being large, and equivalent conditions) catalyzes global convergence by ensuring \textbf{(1)} feasibility of \eqref{eqn:sattnsvm}, and, \textbf{(2)} benign optimization landscape, in the sense that there are no stationary points and no spurious locally-optimal directions (see Sec.~\ref{sec overparam}). These are illustrated in Figures \ref{fig path} and \ref{fig overparam W bar}. 


\item \textbf{Generality of SVM equivalence (Sec. \ref{sec:multi}).} When optimizing with linear $h(\cdot)$, the attention layer is inherently biased towards selecting a single token from each sequence (a.k.a.~hard attention). This is reflected in \eqref{eqn:sattnsvm} and arises from output tokens being convex combinations of the input tokens. In contrast, we show that nonlinear heads necessitate composing multiple tokens, highlighting their importance in the transformer's dynamics (Sec. \ref{sec when}). Using insights gathered from our theory, we propose a more general SVM equivalence. Remarkably, we demonstrate that our proposal accurately predicts the implicit bias of attention trained by gradient descent under general scenarios not covered by theory (e.g.~$h(\cdot)$ being an MLP). Specifically, our general formulae decouple attention weights into two components: A \textbf{directional component} governed by SVM which selects the tokens by applying a 0-1 mask, and a \textbf{finite component} which dictates the precise composition of the selected tokens by adjusting the softmax probabilities.
\end{enumerate}

An important feature of these findings is that they apply to arbitrary datasets (whenever SVM is feasible) and are numerically verifiable. We extensively validate the max-margin equivalence and implicit bias of transformers through enlightening experiments. We hold the view that these findings aid in understanding transformers as hierarchical max-margin token-selection mechanisms, and we hope that our outcomes will serve as a foundation for upcoming studies concerning their optimization and generalization dynamics.

\smallskip
\noindent\textbf{Overview.}~The paper is structured as follows: Section \ref{sec:prelim} introduces preliminaries on self-attention and optimization. Section~\ref{sec:bias} analyzes self-attention's optimization geometry, showing the RP of attention parameters converges to a max-margin solution. Sections \ref{provable global} and \ref{sec local} present global and local gradient descent analyses, respectively, demonstrating convergence of $\W$, the key-query variable, towards the solution of \eqref{eqn:sattnsvm}. Section \ref{sec:multi} provides our results on nonlinear prediction heads and generalized SVM equivalence. 
Section~\ref{sec:related} discusses relevant literature. Finally, Section~\ref{sec:conc} concludes the paper with open problems and future research directions inspired by our findings. All proofs are deferred to the appendix.

\section{Preliminaries}\label{sec:prelim}

\begin{wrapfigure}{r}{0.31\linewidth}\vspace{-0.5cm}
	\centering
	\includegraphics[scale=0.57]{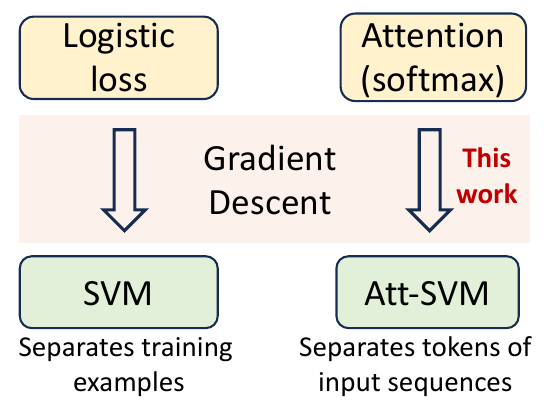}
 \label{fig:main_bias}\vspace{-0.8cm}\caption{Implicit biases of the attention layer and logistic regression.}\vspace{-0.5cm}
\end{wrapfigure}

\textbf{Unveiling the relationship between attention and linear SVMs.}  For linear classification, it is well-established that GD iterations on logistic loss and separable datasets converge towards the hard margin SVM solution, which effectively separates the two classes within the data \cite{soudry2018implicit,rosset2003margin,zhang2005boosting}. 
The softmax nonlinearity employed by the attention layer exhibits an exponentially-tailed behavior similar to the logistic loss; thus, attention may also be biased towards margin-maximizing solutions. However, the attention layer operates on tokens within an input sequence, rather than performing classification directly. Therefore, its bias is towards an SVM, specifically \eqref{eqn:sattnsvm}, which aims to separate the tokens of input sequences by selecting the relevant ones and suppressing the rest.  Nonetheless, formalizing this intuition is considerably more challenging: The presence of the highly nonlinear and nonconvex softmax operation renders the analysis of standard GD algorithms intricate. Additionally, the 1-layer transformer in \eqref{eq:erm:init} does not inherently exhibit a singular bias towards a single \eqref{eqn:sattnsvm} problem, even when using a linear head $h$. Instead, it can result in multiple locally optimal directions induced by their associated SVMs. We emphasize that \cite{tarzanagh2023margin} is the first work to make this attention$\leftrightarrow$SVM connection. Here, we augment their framework to transformers by developing the first guarantees for self/cross-attention layer, nonlinear prediction heads, and global convergence.

\smallskip
\noindent \textbf{Notation.} 
For any integer $N \geq 1$, let $[N]:=\{1, \dots, N\}$. We use lowercase and uppercase bold letters (e.g., $\ab$ and $\A$) to represent vectors and matrices, respectively. The entries of a vector $\ab$ are denoted as $\ab_i$. For a matrix $\A$, $\tn{\A}$ denotes the spectral norm, i.e.~maximum singular value, $\tnuc{\A}$ denotes the nuclear norm, i.e.~summation of all singular values, and $\tf{\A} := \sqrt{ \tr(\A^\top \A)}$ denotes the Frobenius norm. $\dist{\cdot,\cdot}$ denotes the Euclidean distance between two sets.
The minimum / maximum of two numbers $a$, $b$ is denoted as $a\wedge b$ / $a\vee b$. The big-O notation $\mc{O}(\cdot)$ hides the universal constants. 
%
%

\smallskip
\noindent\textbf{Optimization problem definition.} 
We use a linear head $h(\x)=\vb^\top\x$ for most of our theoretical exposition. Given dataset $(Y_i,\X_i,\z_i)_{i=1}^n$, we minimize the empirical risk of an 1-layer transformer using combined weights $\W\in\R^{d\times d}$ or individual weights $\Kb,\Qb\in\R^{d\times m}$ for a fixed head and decreasing loss function:\vspace{-3pt}
\begin{align}\label{eqn:erm:w}
\Lc(\W)&=\frac{1}{n}\sum_{i=1}^n \ell\left(Y_i\cdot \vb^\top\X_i^\top \sft{\X_i\W\z_{i}}\right), \tag{W-ERM}\\
\Lc(\Kb,\Qb)&=\frac{1}{n}\sum_{i=1}^n \ell\left(Y_i\cdot \vb^\top\X_i^\top \sft{\X_i\Kb\Qb^\top\z_{i}}\right). \label{eqn:erm:kq} \tag{KQ-ERM}
\end{align}
%
We can recover the self-attention model by setting $\z_i$ to be the first token of $\X_i$, i.e.,~$\z_i\gets \x_{i1}$. While the above formulation regresses a single label $Y_i$ per $(\X_i,\z_i)$, in Sections \ref{sec:multi} and  \ref{sec multioutput}, we show that our findings gracefully extend to the sequence-to-sequence classification setting where we output and classify $T$ tokens per inputs $\X_i,\Z_i\in\R^{T\times d}$. See Sections \ref{sec:multi} and \ref{sec multioutput} for results on nonlinear prediction heads.

\smallskip
\noindent\textbf{Optimization algorithms.} 
Given a parameter $R>0$, we consider an $\ell_2$-norm bound $R$, and define the regularized path solution associated with  Objectives \eqref{eqn:erm:w} and \eqref{eqn:erm:kq}, respectively as \eqref{RP-W} and \eqref{RP-QK}. These update rules allow us to find the solution within a constrained region defined by the norm bound. The RP illustrates the evolution of $\Wb_R$ as $R$ increases, capturing the essence of GD where the ridge constraint serves as an approximation for the number of iterations. Previous studies, including \cite{rosset2003margin,suggala2018connecting,ji2020gradient,tarzanagh2023margin}, have examined the implicit bias of logistic regression and established a connection between the directional convergence of the RP (i.e., $\lim_{R\rightarrow \infty} \Wb_R/R$) and GD. In \cite{tarzanagh2023margin}, the concept of a local RP was also employed to investigate implicit bias along local directions. For GD, with appropriate initialization and step size $\eta>0$, we describe the optimization process associated with \eqref{eqn:erm:w} and \eqref{eqn:erm:kq} as \eqref{GD-W} and \eqref{GD-QK}, respectively.

\smallskip
\tcbset{colback=white!5!white,colframe=black!5!black,colback=green!1!white}
\begin{tcolorbox}[height=2cm, sidebyside,righthand width=7cm]
\begin{small}
\vspace{-.25cm}
\hspace{-.2cm} Given $\W(0) \in \R^{d\times d}$, $\eta>0$, for $k\geq 0$ do:
\begin{equation}\tag{\small{W-GD}}
\W(k+1) = \W(k) -\eta \nabla \Lc(\W(k)).
\label{GD-W}    
\end{equation}
\end{small}
\tcblower
\begin{small}
\hspace{-.2cm} Given $\Qb(0), \Kb(0)  \in \R^{d\times m}$, $\eta>0$, for $k\geq 0$ do:
\begin{equation}\label{GD-QK}
\hspace{-.1cm}
\begin{bmatrix}
\Kb(k+1)   \\
 \Qb(k+1) \\
\end{bmatrix}
= \begin{bmatrix}
\Kb(k)   \\
 \Qb(k) \\
\end{bmatrix} 
-\eta 
\begin{bmatrix}
    \nabla_{\Kb} \Lc\left(\Kb(k), \Qb(k)\right)\\
    \nabla_{\Qb} \Lc\left(\Kb(k), \Qb(k)\right)
\end{bmatrix}.
\tag{\small{KQ-GD}}
\end{equation}
\end{small}
\end{tcolorbox}
\begin{tcolorbox}[height=1.8cm, sidebyside,righthand width=7cm]
\begin{small}
\hspace{-.2cm} Given $R>0$, find $d\times d$ matrix:
\begin{align}\tag{W-RP}
\Wb_R=\underset{\tf{\W}\leq R}{\arg\min}~\Lc(\W). 
\label{RP-W}
\end{align}
\end{small}
\tcblower
\begin{small}
\hspace{-.2cm} Given $R>0$, find $d\times m$ matrices:
\begin{align}\tag{KQ-RP}
(\Kbb_R,\Qbb_R)=\underset{\tf{\Kb}^2+\tf{\Qb}^2\leq 2R}{\arg\min}\Lc(\Kb,\Qb).
\label{RP-QK}
\end{align}
\end{small}
\end{tcolorbox}

%
%
\subsection{Optimal tokens and hard-margin SVM problem for cross attention} 
Given $\X_i\in\R^{T\times d},\z_i\in\R^d$,  we present a convex hard-margin SVM problem, denoted as \eqref{eqn:sattnsvm}, that aims to separate a specific token from the remaining tokens in the input sequence $\X_i$. This problem is jointly solved for all inputs, allowing us to examine the optimization properties of cross-attention. To delve deeper, we introduce the concept of optimal tokens, which are tokens that minimize the training objective under the decreasing loss function $\ell(\cdot)$ as shown in Lemma~\ref{lem min risk}. This exploration will introduce the notions of token scores and optimality, providing insights into the underlying principles of self-attention mechanisms \cite{tarzanagh2023margin}.

\begin{definition}[Token Score and Optimality]\label{score def}
Given a prediction head $\vb\in\R^d$, the score of a token $\x_{it}$ of input $\X_i$ is defined as $\bgam_{it} = Y_i \cdot \vb^\top\x_{it}$. The optimal token for each input $\X_i$ is given by the index $\op_i \in \arg\max_{t \in [T]} \bgam_{it}$ for all $i \in [n]$.
\end{definition}
By introducing token scores and identifying optimal tokens, we can better understand the importance of individual tokens and their impact on the overall objective. The token score quantifies the contribution of a token to the prediction or classification task, while the optimal token represents the token that exhibits the highest relevance within the corresponding input sequence. 

\smallskip
\noindent$\bullet$ \textbf{Hard-margin SVM for $\W$-parameterization.} Equipped with the set of optimal indices  $\opt:=(\opt_i)_{i=1}^n$ as per Definition~\ref{score def}, we introduce the following SVM formulation associated to \eqref{eqn:erm:w}:

\begin{tcolorbox}[colback=white!5!white,colframe=black!5!black,colback=green!1!white]
\vspace{5pt}
\begin{equation}\tag{\name}
\Wm=\arg\min_{\W}\tf{\W}
\quad \text{subj. to} \quad (\x_{i\op_i}-\x_\itt)^\top\W\z_i\geq 1\quad  \text{for all} \quad t \neq \op_i, \quad  i\in[n]
\label{eqn:sattnsvm}.
\end{equation}
\end{tcolorbox}

The existence of matrix $\Wm$ implies the separability of tokens $(\opt_i)_{i=1}^n$ from the others. The terms $a_{\itt}:=\x_\itt^\top\W\z_i$ represent the dot product between the key-query features before applying the softmax nonlinearity. This dot product is a crucial characteristic of self-attention and we can express the SVM constraints in \eqref{eqn:sattnsvm} as $a_{i\op_i}\geq a_\itt+1$. Thus,  \eqref{eqn:sattnsvm} finds the most efficient direction that ensures the optimal token $\x_{i\op_i}$ achieves the highest similarity with query $\z_i$ among all key embeddings. Our first result shows that \eqref{eqn:sattnsvm} is feasible under mild over-parameterization. 

\begin{theorem}\label{thm:separation} Suppose $d\geq \max(T-1,n)$. Then, almost all datasets $(Y_i,\X_i,\z_i)_{i=1}^n$ -- including the self-attention setting with $\z_i\gets\x_{i1}$ -- obey the following: 
\eqref{eqn:sattnsvm} is feasible i.e.,~$\Wm$ separates the desired tokens $\opt=(\opt_i)_{i=1}^n$.
\end{theorem}

We note that the convex formulation \eqref{eqn:sattnsvm} does not  fully capture the GD geometry on \eqref{eqn:erm:w}. In a more general sense, GD can provably converge to an SVM solution over locally-optimal tokens, as detailed in Section~\ref{local GD sec}. For deeper insight into \eqref{eqn:sattnsvm}, consider that the attention layer's output is a convex mixture of input tokens. Thus, if minimizing the training loss involves choosing the optimal token $\x_{i\op_i}$, the softmax similarities should eventually converge to a one-hot vector, precisely including $\x_{i\op_i}$ (assigned 1), while ignoring all other tokens (assigned 0s). This convergence requires the attention weights $\W$ to diverge in norm to saturate the softmax probabilities.  Due to the exponential-tailed nature of the softmax function, the weights converge directionally to the max-margin solution. This phenomenon resembles the implicit bias of logistic regression on separable data \cite{soudry2018implicit,ji2018risk}. Lemma \ref{lem min risk} formalizes this intuition and rigorously motivates optimal tokens.

\smallskip
\noindent$\bullet$ \textbf{Non-convex SVM for $(\Kb,\Qb)$-parameterization.} The objective function \eqref{eqn:erm:kq} has an extra layer of nonconvexity compared to \eqref{eqn:erm:w} as $(\Kb,\Qb)$ corresponds to a matrix factorization of $\W$. To study this, we introduce the following nonconvex SVM problem over $(\Kb,\Qb)$ akin to \eqref{eqn:sattnsvm}:

\smallskip
\begin{tcolorbox}
\vspace{-7pt}
\begin{equation}\tag{KQ-SVM}
\min_{\Kb,\Qb}~\frac{1}{2} \left(\|\Kb\|_F^2+\|\Qb\|_F^2\right)
\quad \text{subj. to} \quad (\x_{i\op_i}-\x_\itt)^\top\Kb\Qb^\top\z_i\geq 1\quad \text{for all} \quad t \neq \op_i, \quad  i\in[n]
\label{eqn:qk:svm}.
\end{equation}
\end{tcolorbox}
Even if the direction of GD is biased towards the SVM solution, it does not have to converge to the global minima of \eqref{eqn:qk:svm}. Instead, it can  converge towards a Karush–Kuhn–Tucker (KKT) point of the max-margin SVM. Such KKT convergence has been studied by \cite{lyu2019gradient,ji2020directional} in the context of other nonconvex margin maximization problems. Fortunately, \eqref{eqn:erm:kq} may not be as daunting as it may initially seem: Our experiments in Figures~\ref{fig path} and \ref{fig general} reveal that GD is indeed biased towards the global minima of \eqref{eqn:qk:svm}. This global minima is achieved by setting $\W:=\Kb\Qb^\top$ and finding the factorization of $\W$ that minimizes the quadratic objective, yielding the following $\W$-parameterized SVM with nuclear norm objective:

\smallskip
\begin{tcolorbox}[colback=white!5!white,colframe=black!5!black,colback=green!1!white]
\vspace{5pt}
\begin{equation}\tag{Att-SVM$_\star$}
\Wm_\st\in\underset{\texttt{rank}(\W)\leq m}{\arg\min}\tnuc{\W}
\quad \text{subj. to} \quad (\x_{i\op_i}-\x_\itt)^\top\W\z_i\geq 1\quad   \text{for all} \quad t \neq \op_i, \quad  i\in[n]
\label{eqn:sattnsvmst}.
\end{equation}
\vspace{-10pt}
\end{tcolorbox}

Above, the nonconvex rank constraint arises from the fact that the rank of $\W = \Kb\Qb^\top$ is at most $m$. However, under the condition of the full parameterization where $m \geq d$, the rank constraint disappears, leading to a convex nuclear norm minimization problem. Besides, the nuclear norm objective inherently encourages a low-rank solution \cite{recht2010guaranteed,fazel2002matrix,srebro2004maximum}. 
Lemma \ref{lem:rank}, presented below, demonstrates that this guarantee holds whenever $n \leq m$. This observation is further supported by our experiments (see Fig.~\ref{fig rank}). Thus, it offers a straightforward rationale for why setting $m < d$ is a reasonable practice, particularly in scenarios involving limited data.

\begin{lemma}\label{lem:rank} Any optimal solution of \eqref{eqn:sattnsvm} or \eqref{eqn:sattnsvmst} is at most rank $n$. More precisely, the  row space of $\Ws$ or $\Ws_\st$ lies within $\texttt{span}(\{\z_i\}_{i=1}^n)$.
\end{lemma}

Figure~\ref{fig rank} illustrates rank range of solutions for  \eqref{eqn:sattnsvm} and \eqref{eqn:sattnsvmst}, denoted as $\Ws$ and $\Ws_{\star}$, solved using optimal tokens $(\opt_i)_{i=1}^n$ and setting  $m=d$ (the rank constraint is eliminated). Each result is averaged over 100 trials, and for each trial, ${\x}_{it}$, ${\z}_i$, and linear head ${\vb}$ are randomly sampled from the unit sphere. In Fig.~\ref{fig rank svm n}, we fix $T=5$ and vary $n$ across $\{5,10,15\}$. Conversely, in Fig.~\ref{fig rank svm T}, we keep $n=5$ constant and alter $T$ across $\{5,10,15\}$.  Both figures confirm rank of $\Ws$ and $\Ws_\star$ are bounded by $\max(n,d)$, validating Lemma~\ref{lem:rank}.

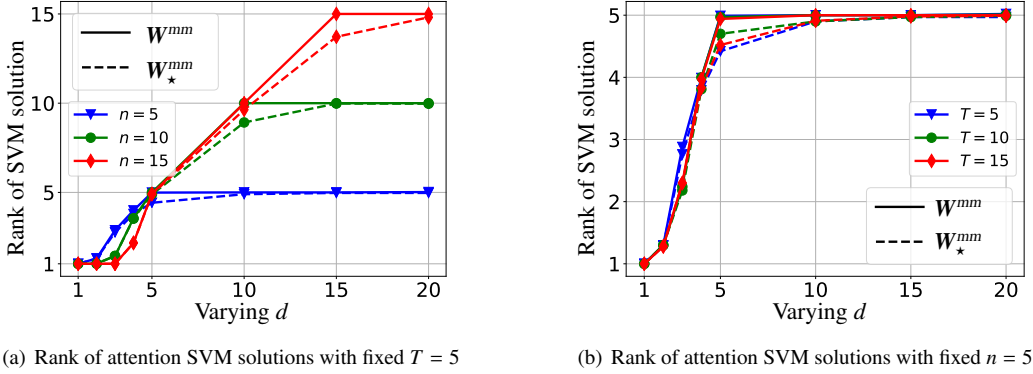
\begin{figure}
    \centering
    \hspace{-10pt}
    \subfigure[Rank of attention SVM solutions with fixed $T=5$]{
        \begin{tikzpicture}
        \node at (0,0) {\includegraphics[height=.25\columnwidth, trim={1.3cm 1.4cm 0 0}, clip]{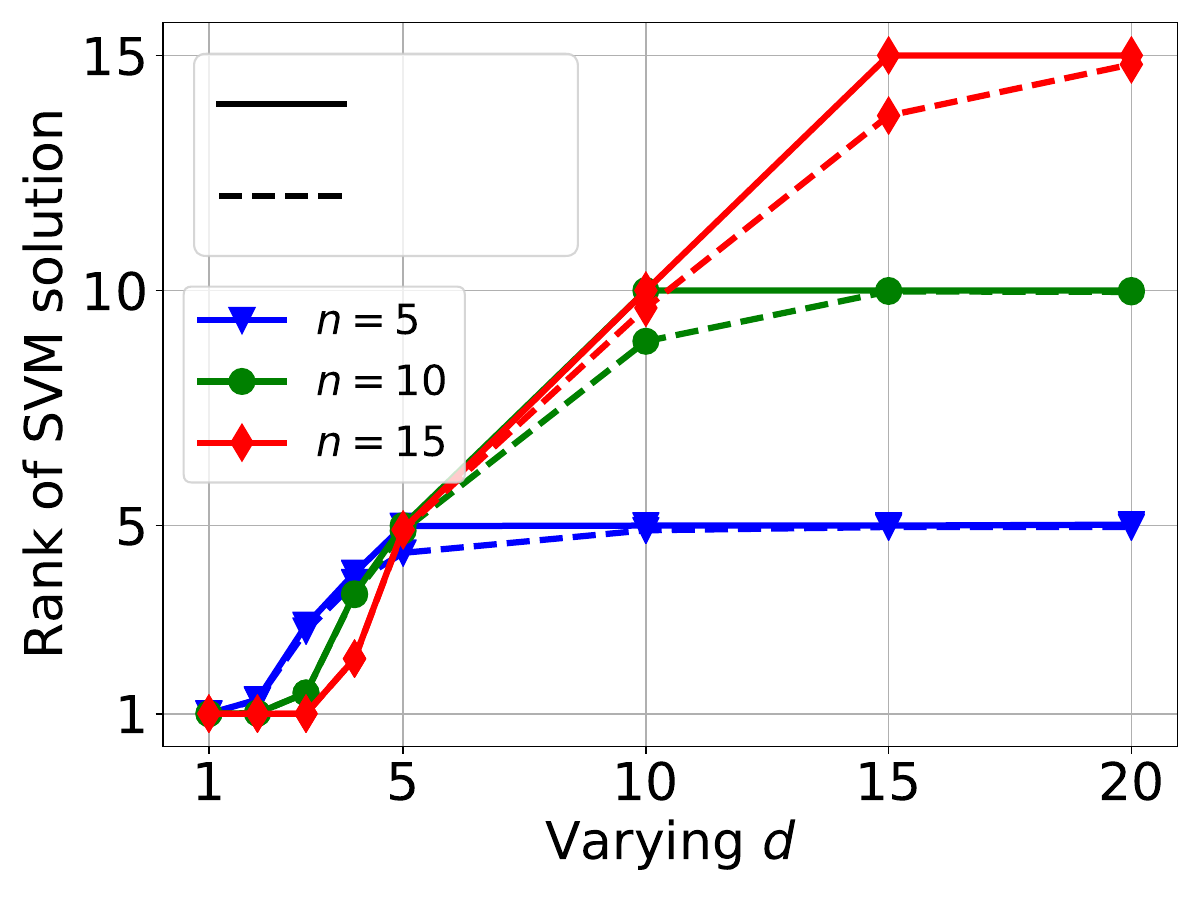}};
        \node at (-0.95,1.5) {\small{$\Ws$}};
        \node at (-0.95,1.) {\small{$\Ws_\star$}};
        \node at (0,-2.2) {\small{Varying $d$}};
        \node[rotate=90] at (-3,0) {\small{Rank of SVM solution}};
        \end{tikzpicture}
        \label{fig rank svm n}
    }
    \hspace{30pt}
    \subfigure[Rank of attention SVM solutions with fixed $n=5$]{
        \begin{tikzpicture}
        \node at (0,0) {\includegraphics[height=.25\columnwidth, trim={1.3cm 1.4cm 0 0}, clip]{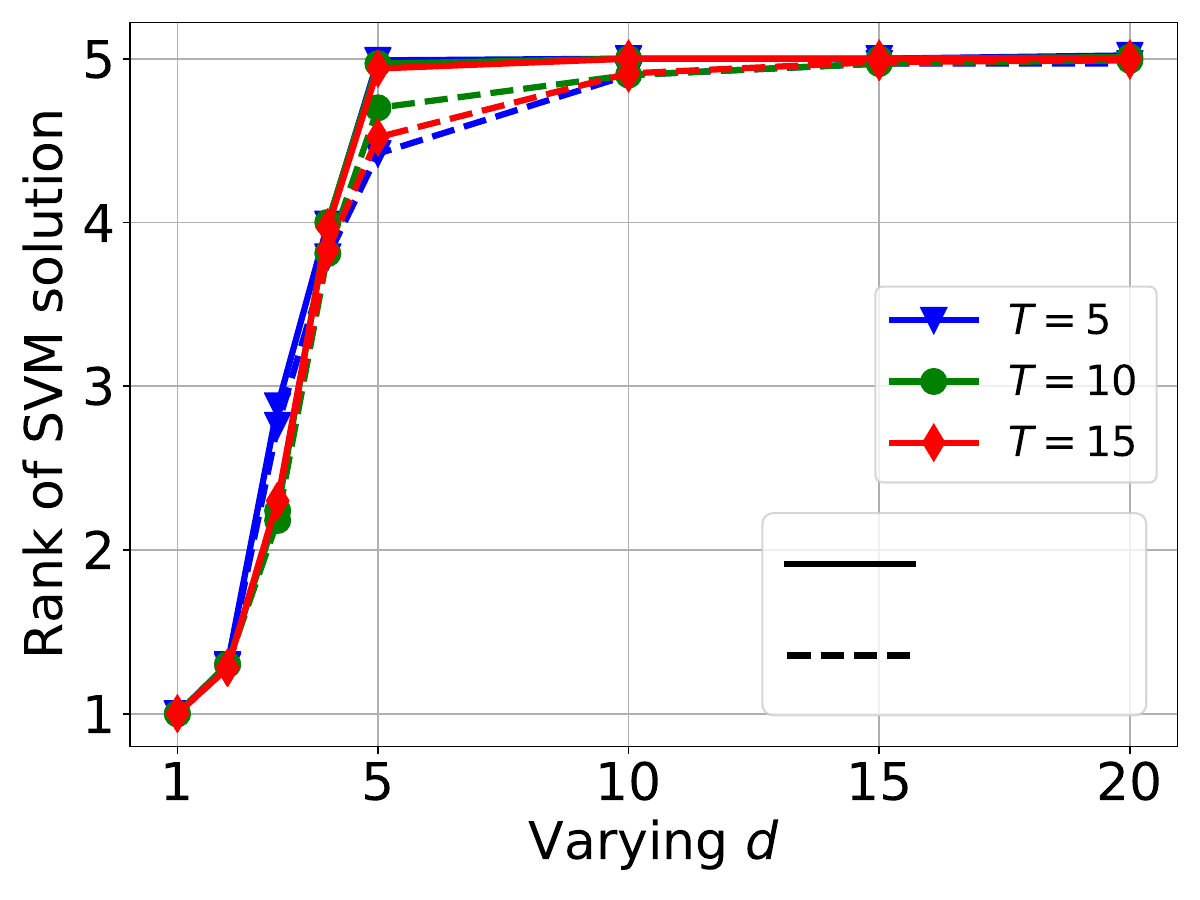}};
        \node at (1.85,-0.8) {\small{$\Ws$}};
        \node at (1.85,-1.28) {\small{$\Ws_\star$}};
        \node at (0,-2.2) {\small{Varying $d$}};
        \node[rotate=90] at (-3.1,0) {\small{Rank of SVM solution}};
        \end{tikzpicture}
        \label{fig rank svm T}
    }
    \caption{Rank range of solutions for  \eqref{eqn:sattnsvm} and \eqref{eqn:sattnsvmst}, denoted as $\Ws$ and $\Ws_{\star}$, solved using optimal tokens $(\opt_i)_{i=1}^n$ and setting  $m=d$ (the rank constraint is eliminated). Both figures confirm ranks of $\Ws$ and $\Ws_\star$ are bounded by $\max(n,d)$, validating Lemma~\ref{lem:rank}.}
        \label{fig rank}
\end{figure}
\section{Understanding the Implicit Bias of Self-Attention}\label{sec:bias}

We start by motivating the optimal token definition and establishing the global convergence of RPs which shed light on the implicit bias of attention parameterizations. Throughout, we maintain the following assumption regarding the loss function.
\begin{assumption}\label{assum:loss:prope}
Over any bounded interval $[a,b]$: (i) $\ell:\R\rightarrow\R$ is strictly decreasing; (ii) The derivative $\ell^{\prime}$ is bounded as $|\ell^{\prime}(u)|\leq M_1$; (iii) $\ell'$ is $M_0$-Lipschitz continuous.
\end{assumption}
Assumption~\ref{assum:loss:prope} encompasses many common loss functions, e.g., logistic  $\ell\left(u\right)=\log\left(1+e^{-u}\right)$, exponential  $\ell\left(u\right)=e^{-u}$, and correlation $\ell(u)=-u$ losses.  
\begin{lemma}[Optimal Tokens Minimize Training Loss]\label{lem min risk} Suppose Assumption \ref{assum:loss:prope} (i)-(ii) hold, and not all tokens are optimal per Definition~\ref{score def}. Then, training risk obeys $\Lc(\W)>\Lc_\st:=\frac{1}{n}\sum_{i=1}^n \ell(\bgam_{i\op_i})$. Additionally, suppose there are optimal indices $(\op_i)_{i=1}^n$ for which \eqref{eqn:sattnsvm} is feasible, i.e.~there exists a $\W$ separating optimal tokens. This $\W$ choice obeys $\lim_{R\rightarrow\infty}\Lc(R\cdot\W)=\Lc_\st$.
\end{lemma}
The result presented in Lemma~\ref{lem min risk} originates from the observation that the output tokens of the attention layer constitute a convex combination of the input tokens. Consequently, when subjected to a strictly decreasing loss function, attention optimization inherently leans towards the selection of a singular token, specifically, the optimal token $(\op_i)_{i=1}^n$. Our following theorem unveils the implicit bias ingrained within both attention parameterizations through RP analysis.

\begin{theorem}\label{thm global reg path}
Suppose Assumptions \ref{assum:loss:prope} holds, optimal indices $(\op_i)_{i=1}^n$ are unique, and \eqref{eqn:sattnsvm} is feasible. Let $\Wm$ be the unique solution of \eqref{eqn:sattnsvm}, and let $\Wc^\svm_\star$ be the solution set of \eqref{eqn:sattnsvmst} with nuclear norm achieving objective $C_\st$. Then, Algorithms~\ref{RP-W} and \ref{RP-QK}, respectively, satisfy:
\begin{itemize}
\item $\W$-parameterization has Frobenius norm bias: $\underset{R\rightarrow\infty}{\lim} \frac{\Wb_R}{R}=\frac{\Wm}{\tf{\Wm}}$.
\item $(\Kb,\Qb)$-parameterization has nuclear norm bias: $\underset{R\rightarrow\infty}{\lim} \dist{\frac{\Kbb_R\Qbb_R^\top}{R},\frac{\Wc^\svm_\star}{C_\st}}=0$.
\begin{itemize}
\item Setting $m=d$: \eqref{eqn:sattnsvmst} is a convex problem without rank constraints.
\end{itemize} 
\end{itemize}
\end{theorem}
Theorem~\ref{thm global reg path} demonstrates that the RP of the $(\Kb,\Qb)$-parameterization converges to a max-margin solution of \eqref{eqn:sattnsvmst} with nuclear norm objective on $\W = \Kb \Qb^\top$. When self-attention is directly parameterized by $\W$, the RP converges to the solution of  \eqref{eqn:sattnsvm} with a Frobenius norm objective. This result is the first to distinguish the optimization dynamics of $\W$ and $(\Kb,\Qb)$ parameterizations, revealing the low-rank bias of the latter. These findings also provide a clear characterization of token optimality (Definition \ref{score def}) and extend naturally to the setting with  multiple optimal tokens per sequence (Theorem \ref{local RP thm} in appendix). By definition, the RP captures the global geometry and cannot be used for the implicit bias of GD towards locally-optimal directions. Sections \ref{provable global} and \ref{sec local} accomplish this goal through gradient-descent and localized RP analysis to obtain locally-applicable SVM equivalences. Note that, this theorem requires each input sequence has a unique optimal token per Definition~\ref{score def}. Fortunately, this is a very mild condition as it holds for almost all datasets, namely, as soon as input features are slightly perturbed.

Theorem \ref{thm global reg path} establishes the implicit bias of attention from the perspective of RP analysis. This leads to the question: To what extent is this RP theory predictive of the implicit bias exhibited by GD?  To delve into this, we examine the gradient paths of $\W(k)$ or $(\Kb(k),\Qb(k))$ and present the findings in Figure~\ref{fig path}. We consider a scenario where $n=d=m=2$ and $T=5$, and employ cross-attention, where tokens $(\z_1,\z_2)$ are generated independently of the inputs $(\X_1,\X_2)$. The teal and yellow markers correspond to tokens from $\X_1$ and $\X_2$, respectively. The stars indicate the optimal token for each input. To provide a clearer view of the gradient convergence path, we illustrate the outcomes of training the attention weight $\W$ or $(\Kb,\Qb)$ in the form of $\W\z_i$ or $\Kb\Qb^\top\z_i$, where $i={1,2}$. With reference to Equations (\ref{eqn:sattnsvm}) and (\ref{eqn:sattnsvmst}), the red and blue solid lines in Fig.~\ref{fig path W} delineate the directions of $\Wsf\z_1$ and $\Wsf\z_2$, correspondingly. Conversely, the red and blue solid lines in Fig.~\ref{fig path KQ} show the directions of $\Ws_\star\z_1$ and $\Ws_\star\z_2$. The red/blue arrows denote the corresponding directions of gradient evolution with the dotted lines representing the corresponding separating hyperplanes. Figure~\ref{fig path} provides a clear depiction of the incremental alignment of $\W(k)$ and $\Kb(k)\Qb(k)^\top$ with their respective attention SVM solutions as $k$ increases. This strongly supports the assertions of Theorem~\ref{thm global reg path}.


It is worth noting that \eqref{eqn:sattnsvmst} imposes a nonconvex rank constraint, i.e., $\texttt{rank}(\W)\leq m$. Nevertheless, this constraint becomes inconsequential if the unconstrained problem, with $m$ set to be greater than or equal to $d$, admits a low-rank solution, as demonstrated in Lemma \ref{lem:rank}. Consequently, in our experimental endeavors, we have the flexibility to employ the unconstrained attention SVM for predicting the implicit bias. This concept is succinctly summarized by the following lemma.

\begin{lemma} 
Let $\Wc^\svm_\star$ be the solution set of \eqref{eqn:sattnsvmst} with nuclear norm achieving objective $C_\st$. Further let $\Wcs_{\texttt{cvx}}$ be the solution set of \eqref{eqn:sattnsvmst} with $m=d$ achieving objective $C_{\texttt{cvx}}$. If $\Wcs_\st\cap \Wcs_{\texttt{cvx}}\neq\emptyset$, then $C_\st=C_{\texttt{cvx}}$ and $\Wcs_\st\subseteq \Wcs_{\texttt{cvx}}$. Also, if the elements of $\Wcs_{\texttt{cvx}}$ have rank at most $m$, then,  $\Wcs_\st=\Wcs_{\texttt{cvx}}$.  
\end{lemma}

\section{Global Convergence of Gradient Descent}\label{provable global}

In this section, we will establish conditions that guarantee the global convergence of GD. Concretely, we will investigate when GD solution selects the \emph{optimal token within each input sequence} through the softmax nonlinearity and coincides with the solution of the RP. Section~\ref{sec local} will complement this with showing that self-attention can more generally converge to locally-optimal max-margin directions. We identify the following conditions as provable catalysts for global convergence: (i) Optimal tokens have relatively large scores; (ii) Initial gradient direction is favorable; (iii) Overparameterization, i.e.~$d$ is appropriately large. 

\subsection{Properties of optimization landscape}
We start by establishing some fundamental properties of Objectives \eqref{eqn:erm:w} and \eqref{eqn:erm:kq}.
\begin{lemma}\label{lem:lip}
Under Assumption~\ref{assum:loss:prope}, $ \nabla\Lc(\W)$,   $ \nabla_{\Kb} \Lc(\Kb,\Qb)$, and  $\nabla_{\Qb} \Lc(\Kb,\Qb)$ are $L_{\W}$,  $L_{\Kb}$, $L_{\Qb}$--Lipschitz continuous, respectively, where $a_i=\|\vb\|~\|\z_i\|^2 \|\X_i \|^3$,  $b_i= M_0\|\vb\|~\|\X_i\|+ 3  M_1 $ for all $i\in[n]$,
\begin{align}\label{eqn:lip:cons:erm}
L_{\W}:=\frac{1}{n}\sum_{i=1}^{n} a_i b_i, \quad L_{\Kb}:= \|\Qb\|L_{\W}, \quad \textnormal{and} \quad L_{\Qb}:= \|\Kb\|L_{\W}.
\end{align}
\end{lemma}

The next assumption will play an important role ensuring the attention layer has a benign optimization landscape.
\begin{assumption}\label{assum:token}
Optimal tokens' indices $(\op_i)_{i=1}^n$ are unique and one of the following conditions on the tokens holds:
\begin{enumerate}[label={\textnormal{{\textbf{B.\arabic*}}}}, wide, labelwidth=!,itemindent=!, labelindent=1pt]
\item \label{assum:token:supp} All tokens are support vectors, i.e., $ (\x_{i\op_i}-\x_{it})^\top\Ws\z_i= 1$ for all $t\neq \op_i$ and $i\in[n]$.
\item \label{assum:opt:token} The tokens' scores, as defined in Definition~\ref{score def}, satisfy $\bgam_{it}=\bgam_{i\tau}<\bgam_{i\op_i}$,
for all $t,\tau\neq \op_i$ and $i\in[n]$.
\end{enumerate}
\end{assumption}

Assumption \ref{assum:token:supp} is directly linked to overparameterization and holds practical significance. In scenarios such as classical SVM classification, where the goal is to separate labels, overparameterization leads to the situation where \emph{all training points become support vectors}. Consequently, the SVM solution aligns with the least-squares minimum-norm interpolation, a concept established in \cite{muthukumar2021classification, hsu2021proliferation} under broad statistical contexts. Assumption \ref{assum:token:supp} represents an analogous manifestation of this condition. Therefore, in cases involving realistic data distributions with sufficiently large $d$, we expect the same phenomena to persist, causing the SVM solution $\Wm$ to coincide with \eqref{eqn:sattnsvm}. 

Drawing on insights from \cite[Theorem 1]{hsu2021proliferation} and our  Theorem \ref{thm:separation}, we expect that the necessary degree of overparameterization remains moderate. Specifically, in instances where input sequences follow an independent and identically distributed (IID) pattern and tokens exhibit IID isotropic distributions, we posit that $d\gtrsim (T+n)\log(T+n)$ will suffice. More generally, the extent of required overparameterization will be contingent on the covariance of tokens \cite{bartlett2020benign, muthukumar2021classification} and the distribution characteristics of input sequences \cite{wang2022binary}. 

Assumption \ref{assum:opt:token} stipulates that non-optimal tokens possess identical scores which constitutes a relatively stringent assumption that we will subsequently relax. Under Assumption~\ref{assum:token}, we establish that when optimization problem \eqref{eqn:erm:w} is trained using GD, the norm of parameters will diverge. 

\begin{theorem}
\label{diverg:norm:w}
Suppose Assumption~\ref{assum:loss:prope} on the loss function $\ell$ and Assumption \ref{assum:token} on the tokens hold.  

\begin{itemize}
\item 
There is no $\W\in\R^{d\times d}$ satisfying $\nabla \Lc(\W)=0$.
    \item  Algorithm~\ref{GD-W} with the step size $\eta \leq 1 /L_{\W}$ and any starting point $\W(0)$ satisfies 
$\lim_{k\rightarrow\infty} \tf{\W(k)}=\infty$.
\end{itemize}
\end{theorem}
The feasibility of SVM (per Theorem \ref{thm:separation}) is a necessary condition for the convergence of GD to the $\Wm$ direction. However, it does not inform us about the optimization landscape. Two additional criteria are essential for convergence: the absence of stationary points $\nabla\Lc(\W)=0$ and divergence of parameter norm to infinity. Theorem \ref{diverg:norm:w} above precisely guarantees both of these criteria under Assumption \ref{assum:token}. 
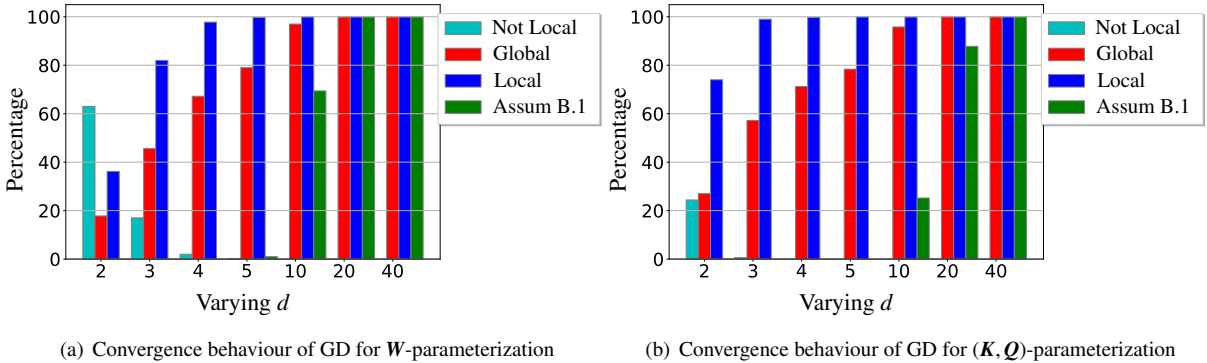
\begin{figure}[t]
    \centering
    \subfigure[Convergence behaviour of  GD for $\W$-parameterization]{
        \begin{tikzpicture}
        \node at (0,0) {\includegraphics[height=.23\columnwidth, trim={1.3cm 1.2cm 3.2cm 0}, clip]{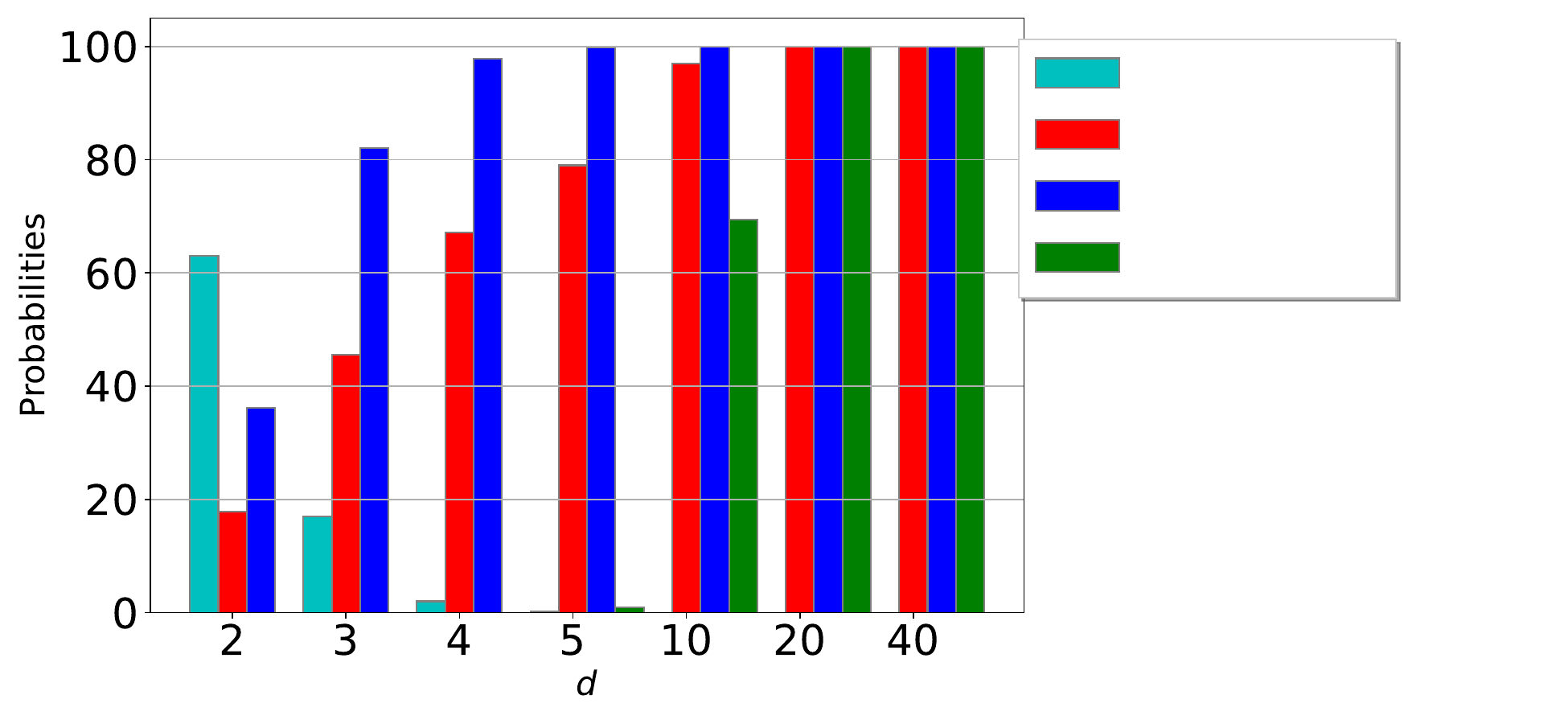}};
        \node[rotate=90] at (-4.,0) {\small{Percentage}};
        \node at (-1,-2.2){\small{Varying $d$}};
        \node at (2.88,1.5){\footnotesize{Not Local}};
        \node at (2.71,1.15){\footnotesize{Global}};
        \node at (2.66,0.8){\footnotesize{Local}};
        \node at (2.95,0.45){\footnotesize{Assum B.1}};
        \end{tikzpicture}
        \label{fig overparam W bar all}
    }
    \hspace{-16pt}
    \subfigure[Convergence  behaviour of GD for $(\Kb,\Qb)$-parameterization]{
        \begin{tikzpicture}
        \node at (0,0) {\includegraphics[height=.23\columnwidth, trim={1.3cm 1.2cm 3.2cm 0}, clip]{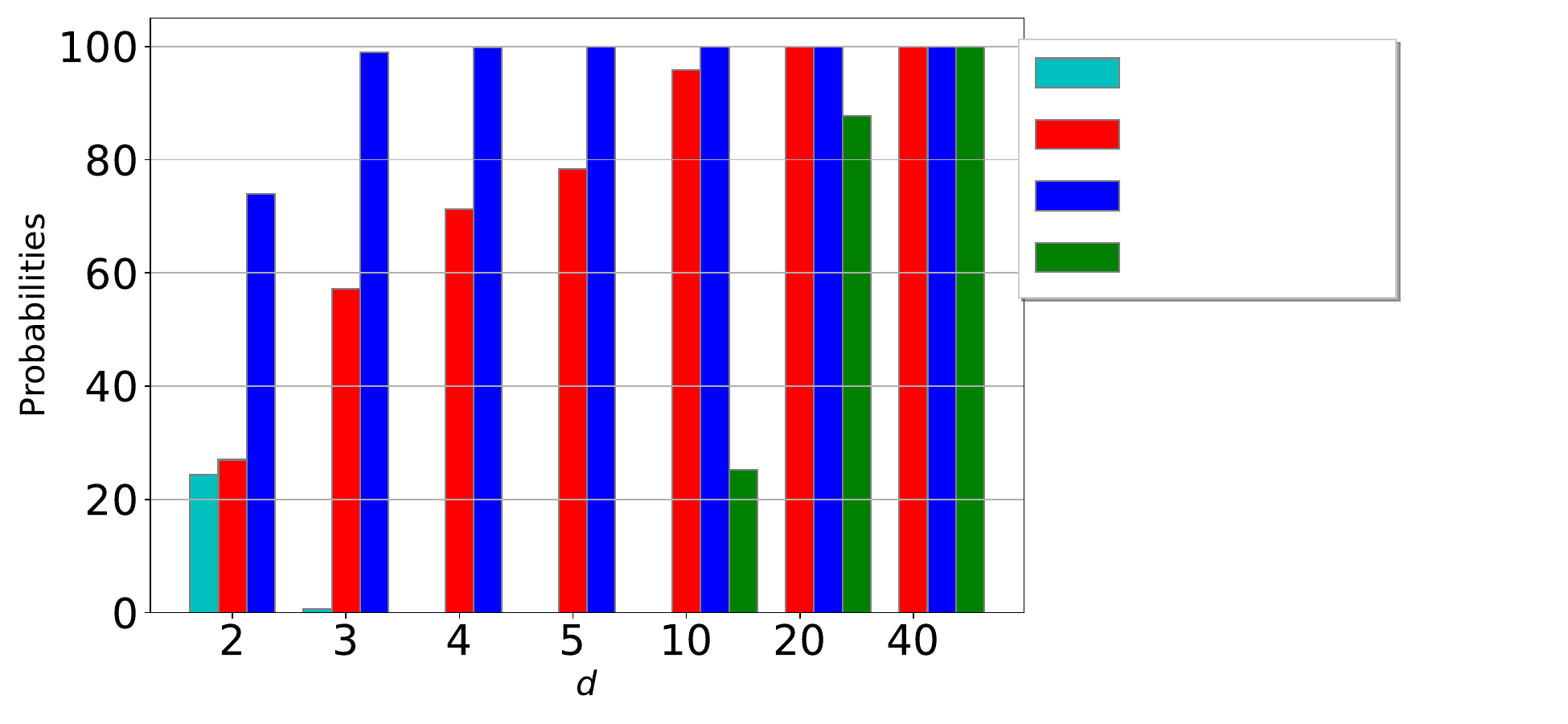}};
        \node[rotate=90] at (-4.,0) {\small{Percentage}};
        \node at (-1,-2.2){\small{Varying $d$}};
        \node at (2.88,1.5){\footnotesize{Not Local}};
        \node at (2.71,1.15){\footnotesize{Global}};
        \node at (2.66,0.8){\footnotesize{Local}};
        \node at (2.95,0.45){\footnotesize{Assum B.1}};
        \end{tikzpicture}
        \label{fig overparam KQ bar all}
    }
    \caption{Percentage of different convergence types of GD when training cross-attention weights \textbf{(a)}: $\W$ or \textbf{(b)}: $(\Kb,\Qb)$ with varying $d$. In both figures, red, blue, and teal bars represent the percentages of Global, Local (including Global), and Not Local convergence, respectively. The green bar corresponds to Assumption \ref{assum:token:supp} where all tokens act as support vectors. Larger overparameterization ($d$) relates to a higher percentage of globally-optimal SVM convergence.}
    \label{fig overparam bar}
\end{figure}
\subsection{Provable global convergence of 1-layer transformer}
In the quest for understanding the global convergence of a 1-layer transformer,  \cite[Theorem 2]{tarzanagh2023margin} provided the first global convergence analysis of the attention in a restrictive scenario where $n=1$ and under the assumption \ref{assum:opt:token}. Here, we present two new conditions for achieving global convergence towards the max-margin direction $\Wm$ based on: \textbf{(I)} the initial gradient direction, and \textbf{(II)} over-parameterization. For the first case, we provide precise theoretical guarantees. For the second, we offer strong empirical evidence, supported by Theorem \ref{diverg:norm:w}, and a formal conjecture described in Section \ref{sec overparam}. We remind the reader that we optimize the attention weights $\W$ while fixing the linear prediction head $h(\x)=\vb^\top\x$. This approach avoids trivial convergence guarantees where an over-parameterized $h(\cdot)$—whether it is a linear model or an MLP—can be used to achieve zero training loss without providing any meaningful insights into the functionality of the attention mechanism.
%
%
%
%

\paragraph{(I) Global convergence under good initial gradient.} To ensure global convergence, we identify an assumption that prevents GD from getting trapped at suboptimal tokens that offer no scoring advantage compared to other choices. To establish a foundation for providing the convergence of GD to the globally optimal solution $\Ws$, we need the following definitions.  For parameters $\mu >0$ and $R>0$, define
\begin{align}\label{eqn:con:nabla0:main}
\conb_{\mu,R}:=\Bigg\{ \tf{\W}\geq R~\Big|~  \li(\x_{i\op_i}-\x_\itt)\z_i^\top, \frac{\W}{\tf{\W}}\ri\geq \mu \quad \textnormal{for all}\quad t\neq \op_i,~  i\in[n]\Bigg\}.
\end{align}
This is the set of all $\W$s that separate the optimal tokens from the rest with margin $\mu$. We will show that, for any $\mu>0$, the optimization landscape of this set is favorable and, if the updates remain in the set, the gradient descent will maximize the margin and find $\Wm$.

\begin{assumption}[First GD Step is Separating]\label{assum:nabla0} For some $\iota>0$ and all $t\neq \op_i,~i\in[n]$: $ (\x_{it}-\x_{i\op_i})^\top\nabla\Lc(0)\z_i\geq \iota$.
\end{assumption}

\begin{theorem}\label{conv:gd:w:global:nabla0}
Suppose Assumption~\ref{assum:loss:prope} on the loss function $\ell$ and Assumption \ref{assum:nabla0} on the initial gradient hold. 
\begin{itemize}
\item  For any $\mu>0$, there exists  $R>0$ such  that   $\conb_{\mu,R}$ does not contain any  stationary points. 
\item Fix any $\mu \in  (0, \iota/\tf{\nabla \Lc(0)})$. Consider GD iterations with $\W(0)=0$, $\W(1)=-R\nabla\,\Lc(0)/\tf{\nabla\Lc(0)}$, and $\W(k+1)=\W(k)-\eta\nabla\Lc(\W(k))$ for $k\ge 1$, where $\eta\le 1/L_{\W}$ and $R$ sufficiently large. If all iterates remain within $\conb_{\mu,R}$, then $\lim_{k\rightarrow\infty} \tf{\W(k)}=\infty$ and $\lim_{k\rightarrow\infty}\frac{\W(k)}{\tf{\W(k)}}=\frac{\Wm}{\tf{\Wm}}$.
\end{itemize}
\end{theorem}
Note that the second result of Theorem~\ref{conv:gd:w:global:nabla0}, i.e.,
the divergence of parameter norm to infinity and directional convergence requires that all GD iterations remain within $\conb_{\mu,R}$ defined in \eqref{eqn:con:nabla0:main}. In Appendix~\ref{app B4}, we show that if for all $\W \in \conb_{\mu,R}(\Ws)$, $\min_{i \in [n]}\li(\x_{i\op_i}-\x_\itt)\z_i^\top, \W-\eta\nabla\Lc(\W) \ri-\min_{i \in [n]}\li(\x_{i\op_i}-\x_\itt)\z_i^\top, \W \ri$ is lower bounded by $(2\eta\mu/\tf{\Wm})\iprod{-\nabla\mc{L}(\W)}{\Wm}$, then all GD iterations $\W(k)$ remain within $\conb_{\mu,R}$. While this condition may appear complicated, it is essentially a tight requirement for updates to remain within $\conb_{\mu,R}$. Finally, it is worth mentioning that, if a stronger correlation condition between initial gradient $\nabla\Lc(0)$ and $\Wm$ holds, one can also prove that updates remain within a tighter cone around $\Wm$ through ideas developed in Theorem \ref{thm:local:gd} by landing $\W(1)$ around $\Wm$ direction. However, we opt to state here the result for the milder condition  $\conb_{\mu,R}$. 


 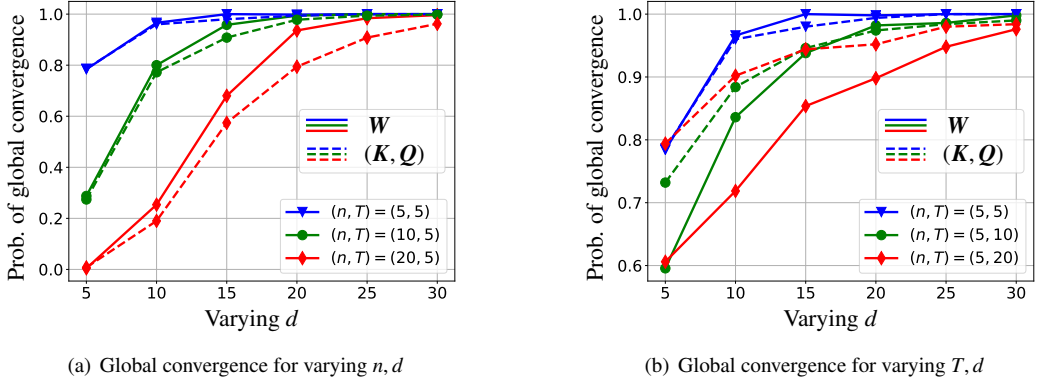
\begin{figure}[t]
    \centering
    \subfigure[Global convergence for varying $n,d$]{
        \begin{tikzpicture}
        \node at (0,0) {\includegraphics[height=.25\columnwidth, trim={1.3cm 1.3cm 0 0}, clip]{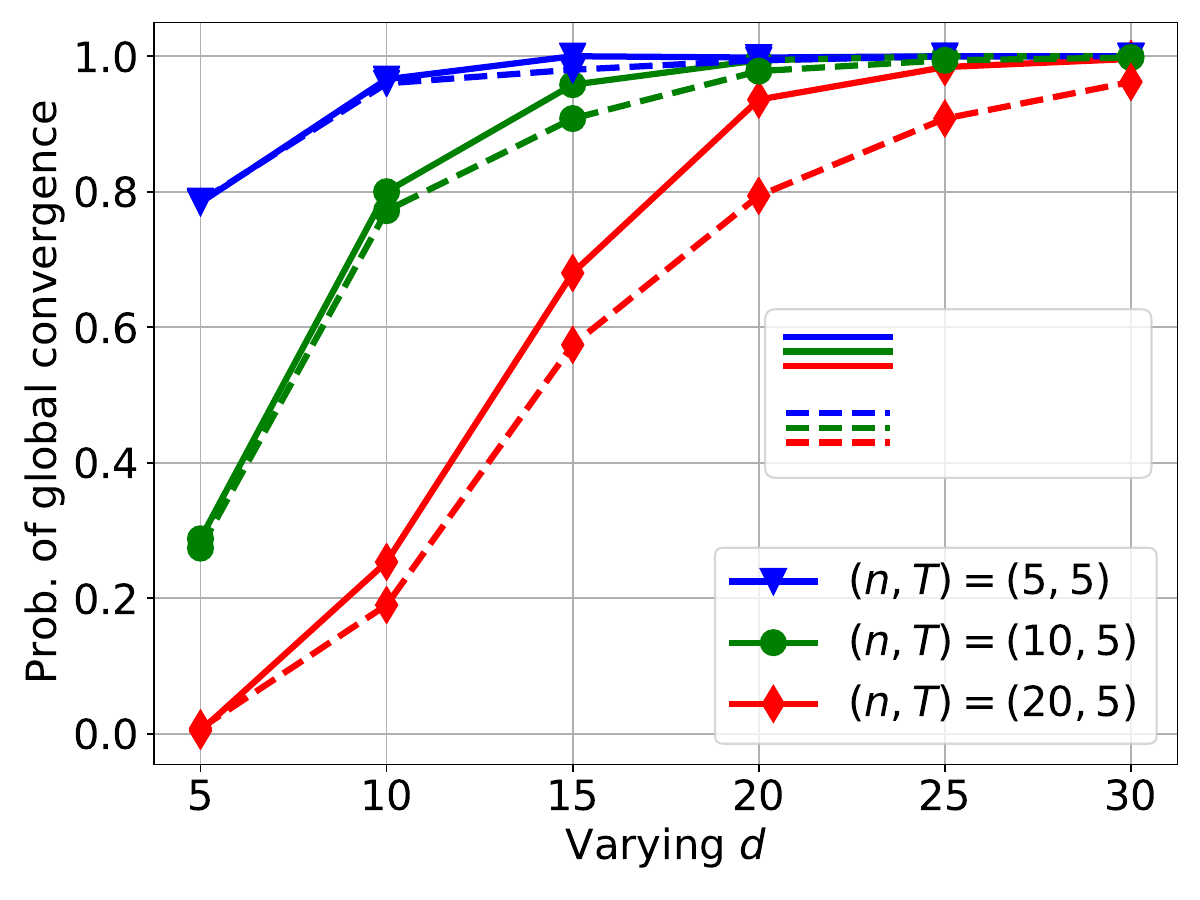}};
        \node at (0,-2.3) {\small{Varying $d$}};
        \node[rotate=90] at (-3.1,0) {\small{Prob. of global convergence}};
        \node at (1.7,0.3){\small{{$\W$}}};
        \node at (1.9,-.1){\small{{$(\Kb,\Qb)$}}};
        \end{tikzpicture}
        \label{fig overparam diff n}
    }
    \hspace{30pt}
    \subfigure[Global convergence for varying $T,d$]{
        \begin{tikzpicture}
        \node at (0,0) {\includegraphics[height=.25\columnwidth, trim={1.3cm 1.3cm 0 0}, clip]{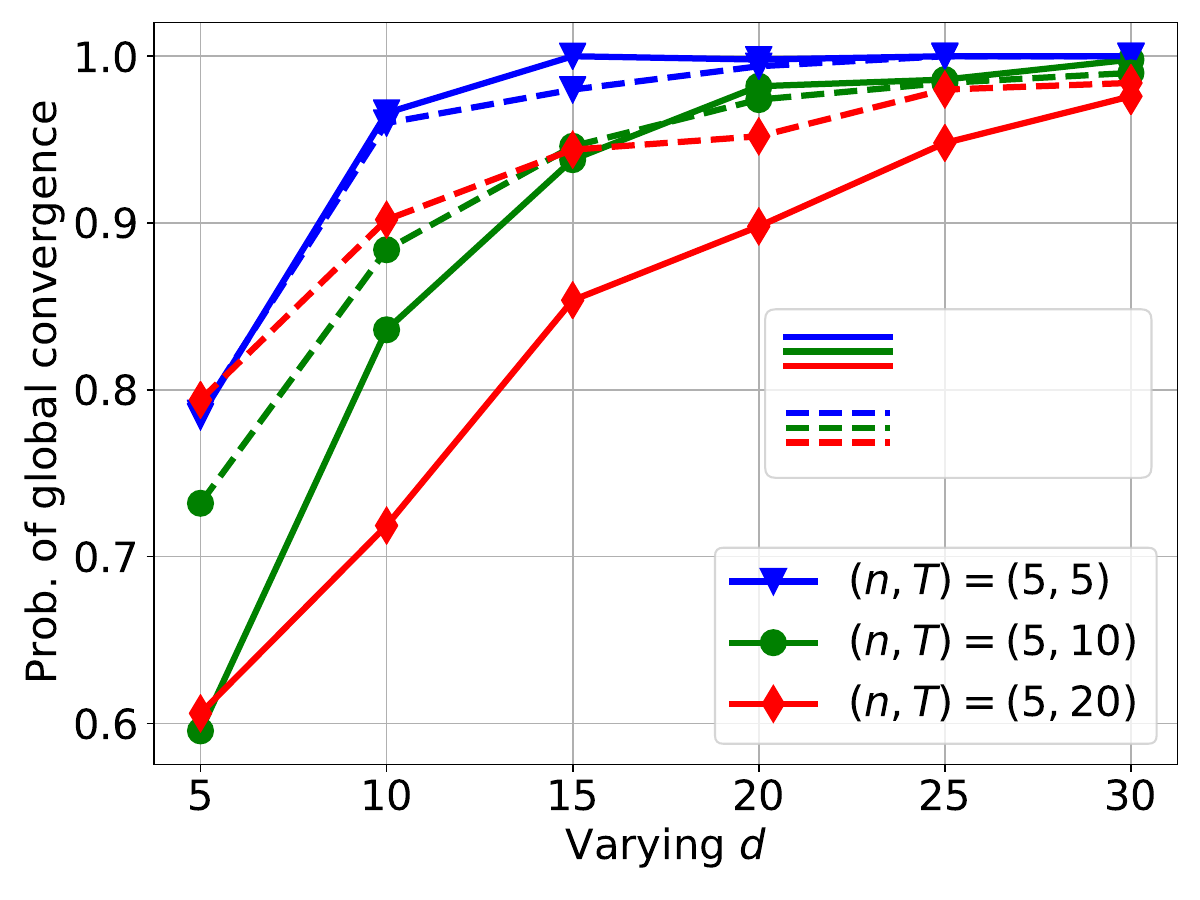}};
        \node at (0,-2.3) {\small{Varying $d$}};
        \node[rotate=90] at (-3.1,0) {\small{Prob. of global convergence}};
        \node at (1.7,0.3){\small{{$\W$}}};
        \node at (1.9,-.1){\small{{$(\Kb,\Qb)$}}};
        \end{tikzpicture}
        \label{fig overparam diff T}
    }
    \caption{Global convergence behavior of GD when training cross-attention weights $\W$ (solid) or $(\Kb,\Qb)$ (dashed) with random data. The blue, green, and red curves represent the probabilities of global convergence for \textbf{(a)}: fixing $T=5$ and varying $n \in \{5,10,20\}$ and \textbf{(b)}: fixing $n=5$ and varying $T \in\{5,10,20\}$. Results demonstrate that for both attention models, as $d$ increases (due to over-parameterization), attention weights tend to select optimal tokens $(\opt_i)_{i=1}^n$.
}
   \label{fig overparam}
\end{figure}

\paragraph{(II) Global convergence via overparameterization.} 
In the context of standard neural networks, overparameterization has been recognized as a pivotal factor for the global convergence of GD \cite{du2018gradient,allen2019convergence,li2018learning,oymak2019overparameterized}. However, conventional approaches like the neural tangent kernel \cite{jacot2018neural} do not seamlessly apply to our scenario, given our assumption on fixed $h$ and the avoidance of achieving zero loss by trivially fitting $h$. Furthermore, even when we train $h$ to achieve zero loss, it doesn't provide substantial insights into the implicit bias of the attention weights $\W$. Conversely, Theorem \ref{diverg:norm:w} illustrates the benefits of over-parameterization in terms of convergence. Considering that Assumption \ref{assum:token:supp} is anticipated to hold as the dimension $d$ increases, the norm of the GD solution is bound to diverge to infinity. This satisfies a prerequisite for converging towards the globally-optimal SVM direction $\Ws$.

The trend depicted in Figure \ref{fig overparam bar}, where the percentage of global convergence (red bars) approaches $100\%$ and Assumption \ref{assum:token:supp} holds with higher probability (green bars) as $d$ grows, reinforces this insight. Specifically, Fig.~\ref{fig overparam W bar all} is similar to Figure~\ref{fig overparam W bar} but with additional green bars representing the percentage of the scenarios where almost all tokens act as support vectors (Assumption~\ref{assum:token:supp}), and Fig.~\ref{fig overparam KQ bar all} displays the same evaluation over $(\Kb,\Qb)$-parameterization setting. In both experiments, and for each chosen $d$ value, a total of $500$ random instances are conducted under the conditions of $n=T=5$. The outcomes are reported in terms of the percentages of Not Local, Local, and Global convergence, represented by the teal, blue, and red bars, respectively. We validate Assumption~\ref{assum:token:supp} as follows: Given a problem instance, we compute the average margin over all non-optimal tokens of all inputs and declare that problem satisfies Assumption~\ref{assum:token:supp}, if the average margin is below 1.1 (where 1 is the minimum). 

%
%

Furthermore, the observations in Figure \ref{fig overparam} regarding the percentages of achieving global convergence reaching 100 with larger $d$ reaffirm that overparameterization leads the attention weights to converge directionally towards the optimal max-margin direction outlined by \eqref{eqn:sattnsvm} and \eqref{eqn:sattnsvmst}. 

In the upcoming section, we will introduce locally-optimal directions, to which GD can be proven to converge when appropriately initialized. We will then establish a condition that ensures the \emph{globally-optimal direction is the sole viable locally-optimal direction}. This culmination will result in a formal conjecture detailed in Section \ref{sec overparam}.

\section{Understanding Local Convergence of 1-Layer Transformer}\label{sec local}

So far, we have primarily focused on the convergence to the global direction dictated by \eqref{eqn:sattnsvm}. In this section, we investigate and establish the local directional convergence of GD as well as RP.
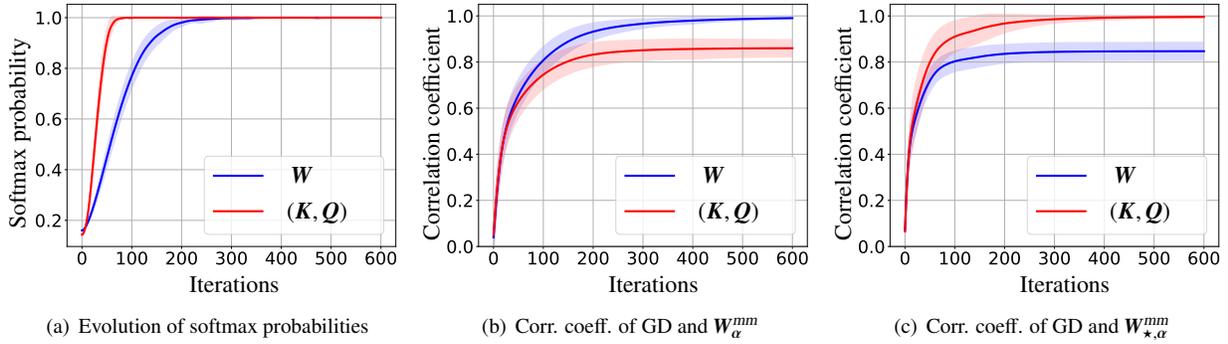
\begin{figure}[t]
    \centering
    \hspace{-10pt}
    \subfigure[Evolution of softmax probabilities]{
        \begin{tikzpicture}
        \node at (0,0) {\includegraphics[height=.22\columnwidth, trim={1.3cm 1.3cm 0 0}, clip]{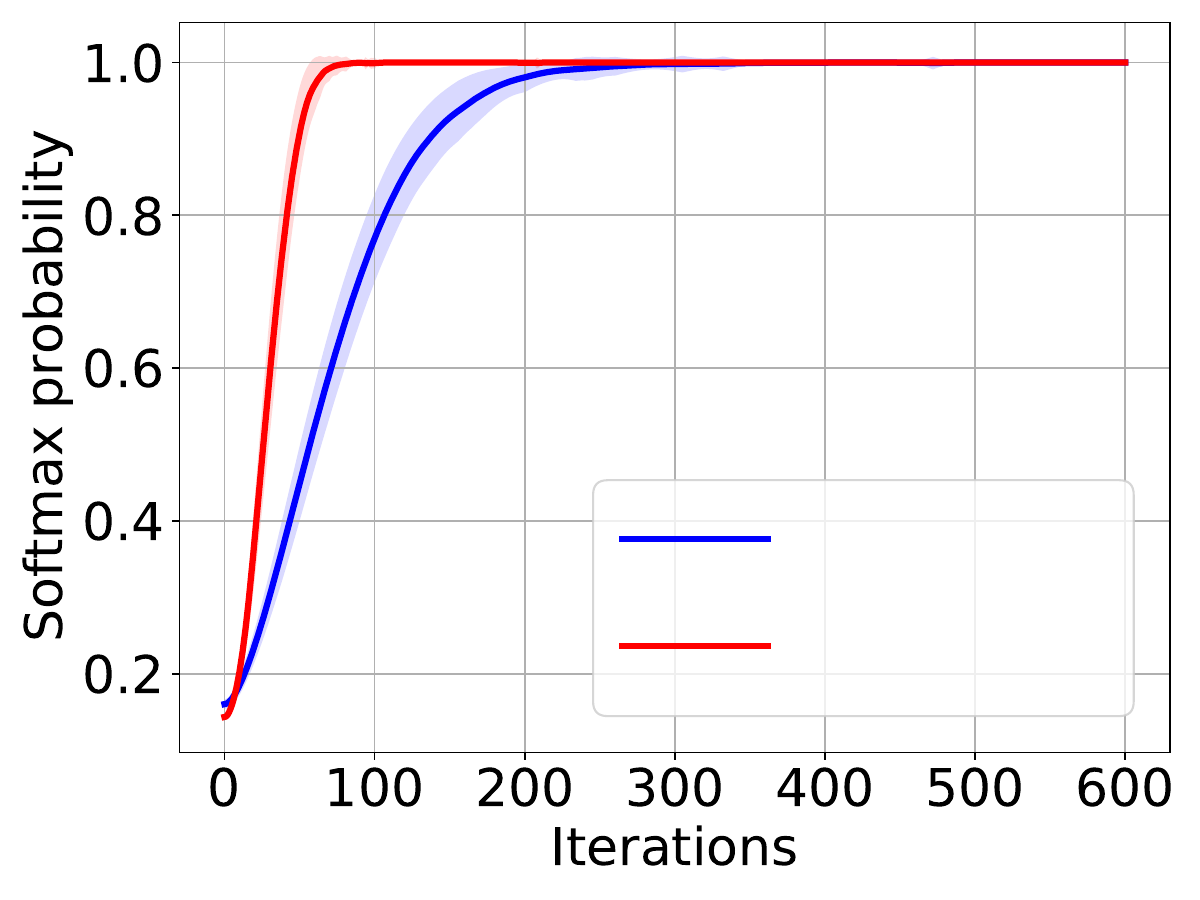}};
        \node at (1.1,-0.55) {\small{$\W$}};
        \node at (1.3,-1.05) {\small{$(\Kb,\Qb)$}};
        \node at (0.2,-2.) {\small{Iterations}};
        \node[rotate=90] at (-2.65,0) {\small{Softmax probability}};
        \end{tikzpicture}
        \label{fig general sfx prob}
    }
    \hspace{-10pt}
    \subfigure[ Corr. coeff. of GD and $\Ws_\bal$]{
        \begin{tikzpicture}
        \node at (0,0) {\includegraphics[height=.22\columnwidth, trim={1.3cm 1.3cm 0 0}, clip]{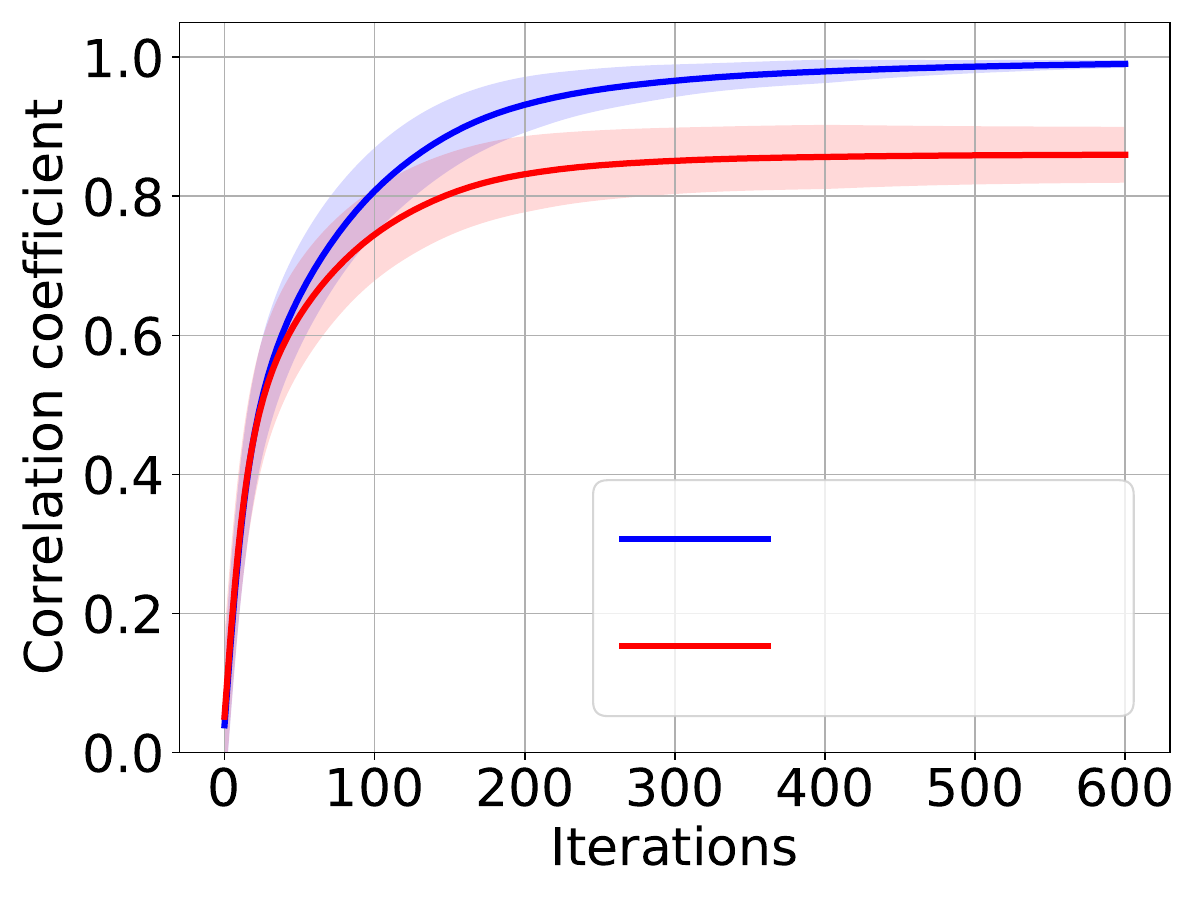}};
        \node at (1.1,-0.55) {\small{$\W$}};
        \node at (1.3,-1.05) {\small{$(\Kb,\Qb)$}};
        \node at (0.3,-2.) {\small{Iterations}};
        \node[rotate=90] at (-2.65,0) {\small{Correlation coefficient}};
        \end{tikzpicture}
        \label{fig general fro corr}
    }
    \hspace{-10pt}
    \subfigure[ Corr. coeff. of GD and $\Ws_{\star,\bal}$]{
        \begin{tikzpicture}
        \node at (0,0) {\includegraphics[height=.22\columnwidth, trim={1.3cm 1.3cm 0 0}, clip]{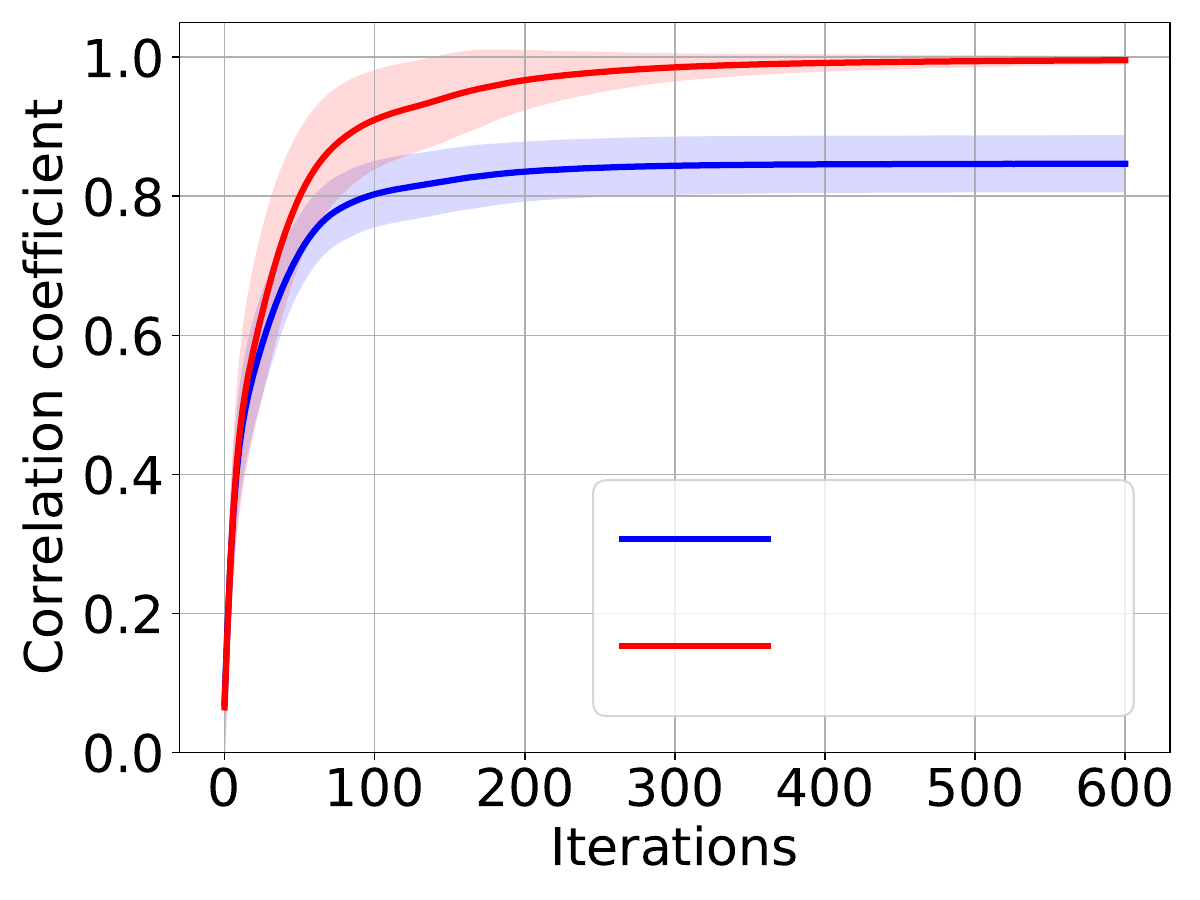}};
        \node at (1.1,-0.55) {\small{$\W$}};
        \node at (1.3,-1.05) {\small{$(\Kb,\Qb)$}};
        \node at (0.3,-2.) {\small{Iterations}};
        \node[rotate=90] at (-2.65,0) {\small{Correlation coefficient}};
        \end{tikzpicture}
        \label{fig general nuc corr}
    }
   \caption{Local convergence behaviour of GD when training cross-attention weights $\W$ (blue) or $(\Kb,\Qb)$ (red) with random data: \textbf{  (a)} displays the largest entry of the softmax outputs averaged over the dataset;  \textbf{(b\&c)} display the Pearson correlation coefficients of GD trajectories and the SVM solutions \textbf{(b)} with the Frobenius norm objective $\Ws_\bal$ (solution of \eqref{eqn:sattnsvm}) and \textbf{(c)} with the nuclear norm objective $\Ws_{\st,\bal}$ 
 (solution of \eqref{eqn:sattnsvmst}). These demonstrate the Frobenius norm bias of $\W(k)$ and the nuclear norm bias of $\Kb(k)\Qb(k)^\top$.}
     \label{fig general}
\end{figure}
\subsection{Local convergence of gradient descent}\label{local GD sec}

To proceed, we introduce locally-optimal directions by adapting Definition 2 of \cite{tarzanagh2023margin}. 

\begin{definition}[\NEIS and Locally-Optimal Direction]\label{def loc opt} 
Fix token indices $\bal=(\alpha_i)_{i=1}^n$. Solve \eqref{eqn:sattnsvm} with $ (\opt_i)_{i=1}^n$ replaced with $\boldsymbol{\alpha} = (\alpha_i)_{i=1}^n$ to obtain $\Wma$. Consider the set $\Tc_i\subset[T]$ such that $(\x_{i\alpha_i}-\x_{it})^\top \Wma \z_i=1$ for all $t\in\Tc_i$. We refer to $(\Tc_i)_{i=1}^n$ as the \neis of $\bal$. Additionally, if for all $i\in[n]$ and $t\in\Tc_i$ scores per Definition~\ref{score def} obey $\bgam_{i\alpha_i}>\bgam_{it}$, indices $\bal=(\alpha_i)_{i=1}^n$ are called \emph{locally-optimal} and $\Wma$ is called a \emph{locally-optimal direction}.
\end{definition}

In words,  the concept of local optimality requires that the selected tokens denoted as $\bal$ should have scores that are higher than the scores of their neighboring tokens referred to as \neis. It is important to observe that the tokens defined as $\op=(\op_i)_{i=1}^n$, which we term as the optimal tokens, inherently satisfy the condition of local optimality. Moving forward, we will provide Theorem $\ref{thm:local:gd}$ which establishes that when the process of GD is initiated along a direction that is locally optimal, it gradually converges in that particular direction, eventually aligning itself with $\Wma$. This theorem immediately underscores the fact that if there exists a direction of local optimality (apart from the globally optimal direction $\Wm$), then when GD commences from any arbitrary starting point, it does not achieve global convergence towards $\Wm$. 


To provide a basis for discussing local convergence of GD, we establish a cone centered around $\Wma$ 
using the following construction. For parameters $\mu \in (0,1)$ and $R>0$, we define $\Cc_{\mu,R}(\Wma)$ as the set of matrices $\W \in\R^{d\times d}$ such that $\tf{\W}\geq R$ and  the correlation coefficient between $\W$ and $\Wma$ is at least $1-\mu$:
%
%
\begin{align}\label{eqn:coneofw:r:main}
\Cc_{\mu,R}({\Wma})=\left\{ \tf{\W}\geq R~\Big|~  \left\langle \frac{\W}{\tf{\W}},\frac{\Wma}{\tf{{\Wma}}} \right\rangle \geq 1-\mu \right\}.
\end{align}
%
%

\begin{theorem}
\label{thm:local:gd} 
Suppose Assumption~\ref{assum:loss:prope} on the loss $\ell$ holds, and let $\bal=(\alpha_i)_{i=1}^n$ be locally optimal tokens according to Definition \ref{def loc opt}. Let $ \Wma$ denote the SVM solution obtained via \eqref{eqn:sattnsvm} by  replacing $(\opt_i)_{i=1}^n$ with $\boldsymbol{\alpha} = (\alpha_i)_{i=1}^n$. 
\begin{itemize}
    \item \label{lem:cond:t1}  There exist parameters $\mu=\mu(\bal) \in (0,1)$ and  $R>0$ such  that   $ \Cc_{\mu,R} (\Wma)$ does not contain any  stationary points.
    \item  Algorithm~\ref{GD-W} with $\eta \leq 1 /L_{\W}$ and any $\W(0) \in \Cc_{\mu,R}(\Wma)$ satisfies $\lim_{k\rightarrow\infty} \tf{\W(k)} = \infty$  and $\lim_{k\rightarrow\infty} \frac{\W(k)}{\tf{\W(k)}} = \frac{\Wma}{\tf{\Wma}}$.
\end{itemize}
\end{theorem}    
This theorem establishes the existence of positive parameters $\mu=\mu(\bal)>0$ and $R>0$ such that there are no stationary points within  $\Cc_{\mu,R}(\Wma)$. Furthermore, if GD is initiated within $\Cc_{\mu,R}(\Wma)$, it will converge in the direction of $\Wma/\tf{\Wma}$. It is worth mentioning that stronger Theorem \ref{diverg:norm:w} (e.g. global absence of stationary points) is applicable whenever all tokens are support i.e.~$\Tcb_i=\emptyset$ for all $i\in[n]$. 

In Figure~\ref{fig general}, we consider setting where $n=6$, $T=8$, and $d=10$. The displayed results are averaged from $100$ random trials. We train cross-attention models with $\x_{it},\z_{i},\vb\in\R^d$ randomly sampled from unit sphere, and apply the normalized GD approach with fixed step size $\eta=0.1$.  In Figure~\ref{fig general sfx prob} we calculate the softmax probability via $\frac{1}{n}\sum_{i=1}^n\max_{t\in[T]}\sft{\X_i\tilde\W(k)\z_i}_t$ for either $\tilde\W=\W$ or $\Kb\Qb^\top$ at each iteration. Both scenarios result in probability $1$, which indicates that attention weights succeed in selecting one token per input. Then following Definition~\ref{def loc opt} let $\bal=(\alpha_i)_{i=1}^n$ be the token indices selected by GD and denote $\Ws_{\star,\bal}$ as the corresponding SVM solution of \eqref{eqn:sattnsvmst}. Define the correlation coefficient of two matrices as 
$\texttt{corr\_coef}(\W_1,\W_2):=\li\W_1,\W_2\ri/\|\W_1\|_F\|\W_2\|_F$. 
Figures~\ref{fig general fro corr}
 and \ref{fig general nuc corr} illustrate the correlation coefficients of attention weights ($\W(k)$ and $\Kb(k)\Qb(k)^\top$) with respect to $\Ws_\bal$ and $\Ws_{\star,\bal}$. The results demonstrate that $\W$ ($\Kb\Qb^\top$) ultimately reaches a $1$ correlation with $\Wsf_\bal$ ($\Ws_{\star,\bal}$), which suggests that $\W$ ($\Kb\Qb^\top$) converges in the direction of $\Wsf_\bal$ ($\Ws_{\star,\bal}$). This further validates Theorem~\ref{thm:local:gd}. 

\subsection{Overparameterization conjecture: When local-optimal directions disappear}\label{sec overparam}

In Section \ref{provable global} we demonstrated that larger $d$ serves as a catalyst for global convergence to select the optimal indices $\op=(\op_i)_{i=1}^n$. However, Section \ref{local GD sec} shows that the convergence can be towards locally-optimal directions rather than global ones. How do we reconcile these? Under what precise conditions, can we expect global convergence?

The aim of this section is gathering these intuitions and stating a concrete conjecture on the global convergence of the attention layer under geometric assumptions related to overparameterization. To recap, Theorem \ref{thm:separation} characterizes when \eqref{eqn:sattnsvm} is feasible and Theorem \ref{diverg:norm:w} characterizes when the parameter norm provably diverges to infinity, i.e. whenever all tokens are support vectors of \eqref{eqn:sattnsvm} (Assumption \ref{assum:token:supp} holds). On the other hand, this is not sufficient for global convergence, as GD can converge in direction to locally-optimal directions per Section \ref{local GD sec}. Thus, to guarantee global convergence, we need to ensure that \textbf{globally-optimal direction is the only viable one}. Our next assumption is a fully-geometric condition that precisely accomplishes this.

\begin{assumption}[There is always an optimal \nei]\label{assume:all_opt_supp} For any choice of $\bal=(\alpha_i)_{i=1}^n$ with $\bal\neq \op$ when solving \eqref{eqn:sattnsvm} with $\op\gets\bal$, there exists $i\in[n]$ such that $\alpha_i\neq\op_i$ and $\op_i\in \Tc_i$.
\end{assumption}


This guarantees that no $\bal\neq \opt$ can be locally-optimal because it has a \nei with higher score at the input $i$ with $\alpha_i\neq \op_i$. Thus, this ensures that global direction $\Wm$ is the unique locally-optimal direction obeying Def.~\ref{def loc opt}. Finally, note  that local-optimality in Def.~\ref{def loc opt} is one-sided: GD can provably converge to locally-optimal directions, while we do not provably exclude the existence of other directions. Yet, Theorem 4 of \cite{tarzanagh2023margin} shows that local RPs (see Section \ref{sec:local reg path} for details) can only converge to locally-optimal directions for almost all datasets\footnote{To be precise, they prove this for their version of Def.~\ref{def loc opt}, which is stated for an attention model $f(\pb)=\vb^\top\X^\top\sft{\X\W\pb}$ admitting an analogous SVM formulation. }. This and Figure \ref{fig overparam bar} provide strong evidence that Def.~\ref{def loc opt} captures all possible convergence directions of GD, and as a consequence, that Assumption \ref{assume:all_opt_supp} guarantees that $\Wm$ is the only viable direction to converge.\smallskip

\noindent\ding{225} \textbf{Integrating the results and global convergence conjecture.} Combining Assumptions \ref{assum:token:supp} and \ref{assume:all_opt_supp}, we have concluded that gradient norm diverges and $\Wm$ is the only viable direction to converge. Thus, we conclude this section with the following \textbf{conjecture}: Suppose $\op=(\op_i)_{i=1}^n$ are the unique optimal indices with strictly highest score per sequence and Assumptions \ref{assum:token:supp} and \ref{assume:all_opt_supp} hold. Then, for almost all datasets (e.g.~add small IID gaussian noise to input features), GD with a proper constant learning rate converges to $\Wm$ of \eqref{eqn:sattnsvm} in direction.


To gain some intuition, consider Figure \ref{fig overparam bar}: Here, red bars denote the frequency of global convergence whereas green bars denote the frequency of Assumption \ref{assum:token:supp} holding over random problem instances. In short, this suggests that Assumption \ref{assum:token:supp} occurs less frequently than global convergence, which is consistent with our conjecture. 
On the other hand, verifying Assumption \ref{assume:all_opt_supp} is more challenging due to its combinatorial nature. A stronger condition that implies Assumption \ref{assume:all_opt_supp} is when \emph{all optimal indices $(\op_i)_{i=1}^n$ are support vectors of the SVM. That is, either $\op_i=\alpha_i$ or $\op_i\in\Tc_i$, $\forall~i\in[n]$.} When the data follows a statistical model, this stronger condition could be verified through probabilistic tools building on our earlier ``all training points are support vectors'' discussion \cite{muthukumar2021classification}. More generally, we believe a thorough statistical investigation of \eqref{eqn:sattnsvm} is a promising direction for future work.

\begin{figure}[t]
    \centering
    \hspace{-10pt}
    \begin{tikzpicture}
        \node at (-2.8,0) {\includegraphics[height=0.45\textwidth, trim={0 0 745 0}, clip]{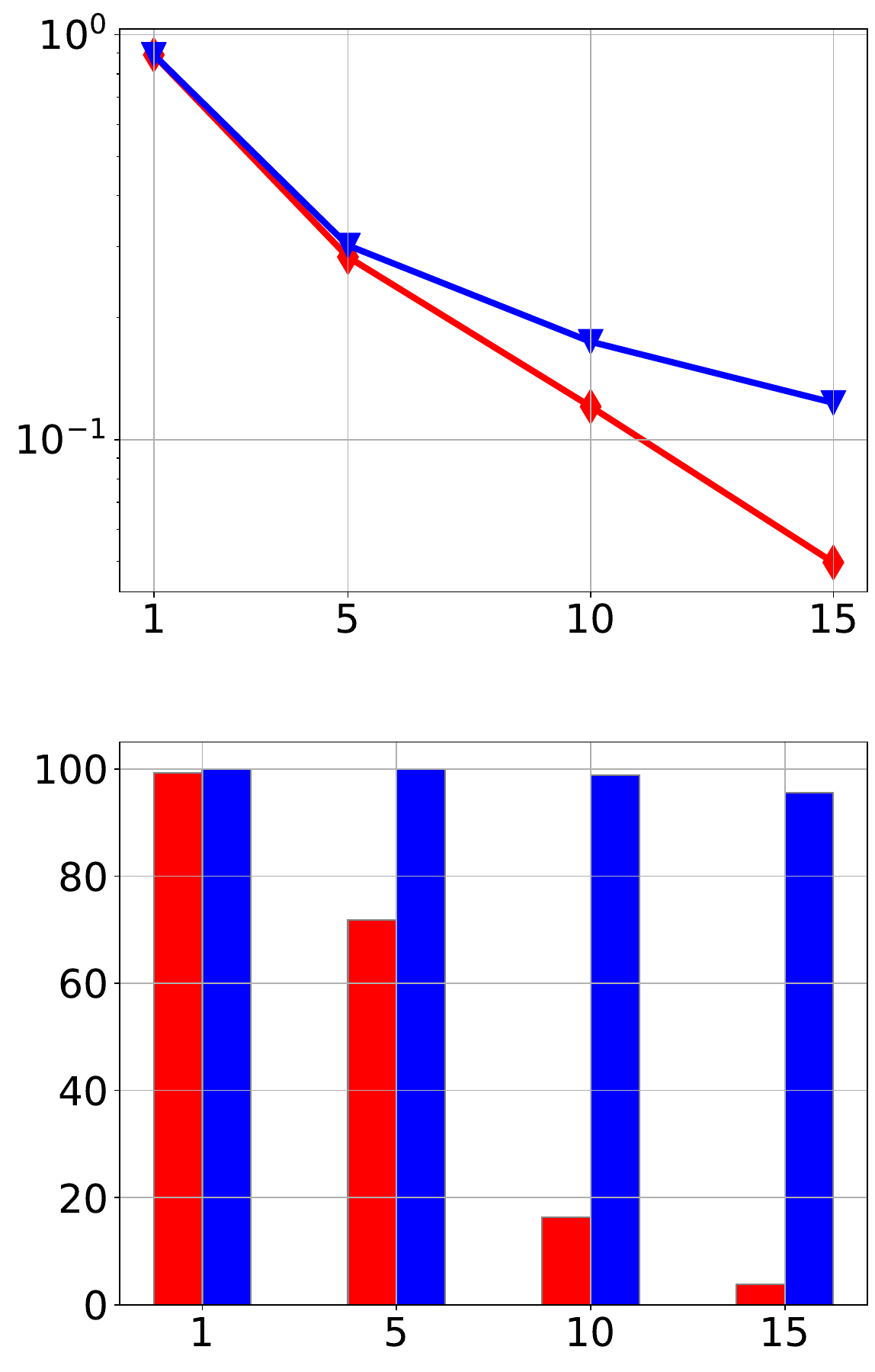}};
        \node at (3.3,0) {\includegraphics[height=0.45\textwidth, trim={548 0 0 0}, clip]{figs/svm_margin.pdf}};
        \node at (-2.5,0.) {\small{Varying $n$}};
        \node at (2.2,0.) {\small{Varying $d$}};
        \node[rotate=90] at (-5.25, 2) {\small{SVM margin}};
        \node[rotate=90] at (-5.25, -1.8) {\small{Percentage}};

        \node[right] at (5.2,3.08) {\small{Global}};
        \node[right] at (5.2,2.6) {\small{Local}};

        \node at (-2.5,-3.8) {\small{Varying $n$}};
        \node at (2.2,-3.8) {\small{Varying $d$}};
        \node[right] at (5.2,-.77) {\small{Global}};
        \node[right] at (5.2,-1.25) {\small{Local}};
    \end{tikzpicture}
    \caption{Performance of GD convergence and corresponding SVM margin. \textbf{Upper:} The SVM margins correspond to globally-optimal (red) and locally-optimal (blue) token indices, denoted as $1/\tf{\Wm}$ and $1/\tf{\Wma}$, respectively. \textbf{Lower:} Percentages of global convergence (when $\bal=\opt$, red) and local convergence (when $\bal\neq\opt$, blue).
    }
    \label{fig svm margin}
\end{figure}
\subsection{Investigation on SVM objectives and GD convergence}
Until now, we have discussed the global and local convergence performances of gradient descent (GD). Theorem~\ref{thm:local:gd} suggests that, without specific restrictions on tokens, when training with GD, the attention weight $\W$ converges towards $\Wma$. Here, the selected token indices $\bal = (\alpha_{i})_{i=1}^n$ may not necessarily be identical to $\opt = (\op_i)_{i=1}^n$. Experiments presented in Figures~\ref{fig overparam W bar}, \ref{fig overparam bar}, and \ref{fig overparam} also support this observation. In this section, we focus on scenarios where $\bal \neq \opt$ (e.g., when $\Wm$ is not feasible) and investigate the question: Towards which local direction is GD most likely to converge?

To this goal, in Figure~\ref{fig svm margin}, we consider SVM margin, which is defined by $1/\tf{\Wma}$, and investigate its connection to the convergence performance of GD.  On the left, we set $T=d=5$ and vary $n$ among $1, 5, 10, 15$; on the right, we fix $T=n=5$ and change $d$ to $2, 5, 10, 15$. All tokens are randomly generated from the unit sphere. The SVM margins corresponding to the selected tokens $\bal$ are depicted as blue curves in the upper subfigures, while the SVM margins corresponding to the globally--optimal token indices ($\bal=\opt$) are shown as red curves. The red and blue bars in the lower subfigures represent the percentages of global and local convergence, respectively. Combining all these findings empirically demonstrates that when the global SVM objective yields a solution with a small margin (i.e., $1/\tf{\Wm}$ is small, and 0 when global SVM is not feasible), GD tends to converge towards a local direction with a comparatively larger margin.
\subsection{Guarantees on local regularization path}\label{sec:local reg path}

In this section, we provide a \emph{localized} regularization path analysis for general objective functions. As we shall see, this will also allow us to predict the local solutions of gradient descent described in Section \ref{local GD sec}. Let $\dm$ denote a general norm objective. Given indices $\bal=(\alpha_i)_{i=1}^{n}$, consider the formulation
\begin{equation}\tag{$\dm$-SVM}
 \Wm_{\dm, \alpha}=\underset{\texttt{rank}(W)\leq m}{\arg\min}\|\W\|_{\diamond}
\quad \text{subj. to} \quad (\x_{i\alpha_{i}}-\x_\itt)^\top\W\z_i\geq 1\quad   \text{for all} \quad t \neq \alpha_{i}, \quad 
 i\in[n]\label{dmattnsvm}.
\end{equation}

In this section, since $\dm$ is clear from the context, we will use the shorthand $\Wm_{\alpha} := \Wm_{\dm, \alpha}$ and denote the optimal solution set of \eqref{dmattnsvm} as $\Wcs := \Wcs_{\dm, \alpha}$. It is important to note that if the $\dm$-norm is not strongly convex, $\Wcs$ may not be a singleton. Additionally, when $m=d$, the rank constraint becomes vacuous, and the problem becomes convex. The following result is a slight generalization of Theorem \ref{thm:separation} and demonstrates that choosing a large $d$ ensures the feasibility of \eqref{dmattnsvm} uniformly over all choices of $\bal$. The proof is similar to that of Theorem~\ref{thm:separation}, as provided in Appendix~\ref{app sep}.
\begin{theorem}\label{separation thm} Suppose $d\geq \max(T-1,n)$ and $m=d$. Then, almost all datasets\footnote{Here, \emph{``almost all datasets''} means that adding i.i.d.~gaussian noise, with arbitrary nonzero variance, to the input features will almost surely result in SVM's feasibility.} $(Y_i,\X_i,\z_i)_{i=1}^n$ -- including the self-attention setting with $\z_i\gets\x_{i1}$ -- obey the following: For any choice of indices $\bal=(\alpha_i)_{i=1}^n\subset[T]$, \eqref{dmattnsvm} is feasible,  i.e.~the attention layer can separate and select indices $\bal$.
\end{theorem}

To proceed, we define the \emph{local regularization path}, which is obtained by solving the $\dm$-norm-constrained problem over a $\bal$-dependent cone denoted as $\con{\bal}$. This cone has a simple interpretation: it prioritizes tokens with a lower score than $\bal$ over tokens with a higher score than $\bal$. This interpretation sheds light on the convergence towards locally optimal directions: lower-score tokens create a barrier for $\bal$ and prevent optimization from moving towards higher-score tokens.
\begin{definition}[Low\&High Score Tokens and Separating Cone]\label{HL cone def main} Given $\al\in[T]$, input sequence $\X$ with label $Y$, $h(\cdot):\R^d\rightarrow\R$, and score $\bgam_t=Y\cdot h(\x_t)$ for all $t\in[T]$, define the low and high score tokens as
\[
\low:=\left\{t\in[T]\bgl \bgam_t<\bgam_\al\right\},\quad \high:=\left\{t\in[T]-\{\alpha\}\bgl \bgam_t\geq \bgam_\al\right\}.
\]
For input $\X_i$ and index $\alpha_i$, we use the shorthand notations $\texttt{low}^\alpha_i$ and $\texttt{high}^\alpha_i$. Finally define $\con{\bal}$ as
\begin{align}
\con{\bal}:=\left\{\texttt{rank}(W)\leq m\bgl \min_{i\in[n]}\max_{t\in\texttt{low}^\alpha_i}\min_{\tau\in\texttt{high}^\alpha_i} (\x_\itt-\x_\ittt)^\top\W\z_i\geq \eps\tf{\W}\right\}.\label{cone alpha eq1}
\end{align}
\end{definition}

Our next lemma relates this cone definition to locally-optimal directions of Definition \ref{def loc opt}.
\begin{lemma} \label{lemma cone main}Suppose \eqref{dmattnsvm} is feasible. If indices $\bal$ are locally-optimal, $\Wma\in \con{\bal}$ for all sufficiently small $\eps>0$. Otherwise, $\Wma\not\in \con{\bal}$ for all $\eps>0$. Additionally, suppose optimal indices $\op_i\in\arg\max_{t\in[T]}\bgam_\itt$ are unique and set $\bal\gets\op$. Then, $\con{\opt}$ is the set of all rank-$\leq$$m$ matrices (i.e.~global set).
\end{lemma}

Lemma~\ref{lemma cone main} can be understood as follows: Among the SVM solutions $\Wma$, only those that are locally optimal demonstrate a barrier of low-score tokens, effectively acting as a protective shield against higher-score tokens. Moreover, in the case of globally optimal tokens (with the highest scores), the global set $\con{\opt}$ can be chosen, as they inherently do not require protective measures. The subsequent result introduces our principal theorem, which pertains to the regularization path converging towards the locally-optimal direction over $\con{\bal}$ whenever $\bal$ is locally optimal.

\begin{theorem} [Convergence of Local Regularization Path]\label{local RP thm1} Suppose Assumption \ref{assum:loss:prope} holds. Fix locally-optimal token indices $\bal=(\al_i)_{i=1}^n$ and $R_0,\eps>0$. Consider the norm-constrained variation of \eqref{cone alpha eq1} defined as 
\[
\Ccd:=\con{\bal}\bigcap \left\{\W\bgl \td{\W}\geq R_0\right\}.
\]
Define local RP as $\Wb_R=\min_{\Ccd,\td{\W}\leq R}\Lc(\W)$ where $\Lc(\W)$ is given by \eqref{eqn:erm:w}. Let $\Wcs$ be the set of minima for \eqref{dmattnsvm} and $\xdm>0$ be the associated margin i.e.~$\xdm=1/\td{\Wma}$. For any sufficiently small $\eps>0$ and sufficiently large $R_0= \order{1/\eps}>0$, $\lim_{R\rightarrow\infty} \dist{\frac{\Wb_R}{R\xdm},\Wcs}=0$. Additionally, suppose optimal indices $\op=(\op_i)_{i=1}^n$ are unique and set $\bal\gets\op$. Then, the same convergence guarantee on regularization path holds by setting $\Ccd$ as the set of rank-$\leq$$m$ matrices.
\end{theorem}

Note that when setting $m=d$, the rank constraint is eliminated. Consequently, specializing this theorem to the Frobenius norm aligns it with Theorem \ref{thm:local:gd}. On the other hand, by assigning $\dm$ as the nuclear norm and $\bal\gets\op$, the global inductive bias of the nuclear norm is recovered, as stated in Theorem \ref{thm global reg path}.

We would like to emphasize that both this theorem and Theorem \ref{thm global reg path} are specific instances of Theorem \ref{local RP thm} found in Appendix \ref{sec multioutput}. It is worth noting that, within this appendix, we establish all regularization path results for sequence-to-sequence classification, along with a general class of \emph{monotonicity-preserving} prediction heads outlined in Assumption \ref{ass cvx seq}. The latter significantly generalizes linear heads, highlighting the versatility of our theory. The following section presents our discoveries concerning general nonlinear heads.

\section{Toward A More General SVM Equivalence for Nonlinear Prediction Heads}\label{sec:multi}
So far, our theory has focused on the setting where the attention layer selects a single optimal token within each sequence. As we have discussed, this is theoretically well-justified under linear head assumption and certain nonlinear generalizations. On the other hand, for arbitrary nonconvex $h(\cdot)$ or multilayer transformer architectures, it is expected that attention will select multiple tokens per sequence. This motivates us to ask:
\begin{quote}
    \textbf{Q:}~What is the implicit bias and the form of $\W(k)$ when the GD solution is composed by multiple tokens?
\end{quote}

In this section, our goal is to derive and verify the generalized behavior of GD. Let $\xat_i=\X_i^\top \s^{\W}_i$ denote the composed token generated by the attention layer where $\s^{\W}_i=\sft{\X_i\W\z_i}$ are the softmax probabilities corresponding to $\W$. Suppose GD trajectory converges to achieve the risk $\Lc_\star=\min_{\W}\Lc(\W)$, and the eventual token composition achieving $\Lc_\star$ is given by 
\[
\xast_i=\X_i^\top \s^\st_i,
\]
where $\s^\st_i$ are the eventual softmax probability vectors that dictate the token composition. Since attention maps are sparse in practice, we are interested in the scenario where $\s^\st_i$ is sparse i.e.~it contains some zero entries. This can only be accomplished by letting $\tf{\W}\rightarrow\infty$. However, unlike the earlier sections, we wish to allow for arbitrary $\s^\st_i$ rather than a one-hot vector which selects a single token. 

To proceed, we aim to understand the form of GD solution $\W(k)$ responsible for composing $\xast_i$ via the softmax map $\s^\st_i$ as $\tf{\W}\rightarrow\infty$. Intuitively, $\W(k)$ should be decomposed into two components via
\begin{align}
\W(k)\approx \Wf+\tf{\W(k)}\cdot \Wsb,\label{multi-token soln}
\end{align}
where $\Wf$ is the {finite component} and $\Wsb$ is the {directional component} with $\tf{\Wsb}=1$. Define the {selected set} $\Rc_i\subseteq[T]$ to be the indices $\s^\st_\itt\neq 0$ and the {masked (i.e.~suppressed) set} as $\Rcb_i=[T]-\Rc_i$ where softmax entries are zero. In the context of earlier sections, we could also call these the \emph{optimal set} and the \emph{non-optimal set}, respectively.
\begin{itemize}[label=$\bullet$, wide, labelwidth=!,itemindent=!, labelindent=5pt]
\item \textbf{Finite component:} The job of $\Wf$ is to assign nonzero softmax probabilities within each $\s^\st_i$. This is accomplished by ensuring that, $\Wf$ induces the probabilities of $\s^\st_i$ over $\Rc_i$ by satisfying the softmax equations
\[
\frac{e^{\x_\itt^\top \Wf\z_i}}{e^{\x_\ittt^\top \Wf\z_i}}=e^{(\x_\itt-\x_\ittt)^\top \Wf\z_i}=\s^\st_\itt/\s^\st_\ittt,
\]
for $t,\tau\in\Rc_i$. Consequently, this $\Wf$ should satisfy the following linear constraints
\begin{equation}
(\x_\itt-\x_\ittt)^\top \Wf\z_i=\log(\s^\st_\itt/\s^\st_\ittt)\quad\text{for all}\quad t,\tau\in\Rc_i,~i\in[n].\label{smax eqn}
\end{equation}
\item \textbf{Directional component:} While $\Wf$ creates the composition by allocating the nonzero softmax probabilities, it does not explain sparsity of attention map. This is the role of $\Wsb$, which is responsible for selecting the selected tokens $\Rc_i$ and suppressing the masked ones $\Rcb_i$ by assigning zero softmax probability to them. To predict direction component, we build on the theory developed in earlier sections. Concretely, there are two constraints $\Wsb$ should satisfy
\begin{enumerate}
\item \textbf{Equal similarity over selected tokens:} For all $t,\tau\in\Rc_i$, we have that $(\x_\itt-\x_\ittt)^\top \W\z_i=0$. This way, softmax scores assigned by $\Wf$ are not disturbed by the directional component and $\Wf+R\cdot\Wsb$ will still satisfy the softmax equations \eqref{smax eqn}.
\item \textbf{Max-margin against masked tokens:} For all $t\in\Rc_i,\tau\in\Rcb_i$, enforce the margin constraint $(\x_\itt-\x_\ittt)^\top \W\z_i\geq 1$ subject to minimum norm $\tf{\W}$. 
\end{enumerate}
Combining these yields the following convex generalized SVM formulation
 \begin{tcolorbox}[colback=white!5!white,colframe=black!5!black,colback=green!1!white]
 \vspace{-7pt}
\begin{align}\tag{Gen-SVM}
\Wm=\arg\min_{\W}\tf{\W}
\quad \text{subj. to} \quad\begin{cases} \forall~t\in\Rc_i,\tau\in\Rcb_i:~(\x_\itt-\x_\ittt)^\top\W\z_i\geq 1,\\
\forall~t,\tau\in\Rc_i:~\quad\quad(\x_\itt-\x_\ittt)^\top\W\z_i=0,\end{cases}\quad  \forall  1\leq i\leq n.
\label{eqn:mattnsvm}
\end{align}
\end{tcolorbox}
\noindent and set the normalized direction in \eqref{multi-token soln} to $\Wsb=\Ws/\tf{\Ws}$.
\end{itemize}
It is important to note that  \eqref{eqn:mattnsvm} offers a substantial generalization beyond the scope of the previous sections, where the focus was on selecting a single token from each sequence, as described in the main formulation \eqref{eqn:sattnsvm}. This broader solution class introduces a more flexible approach to the problem.
\begin{figure}
    \centering
    \hspace{-10pt}
    \begin{tikzpicture}
        \node at (0,0) {\includegraphics[width=0.985\textwidth, trim={1.3cm 1.5cm 0 0}, clip]{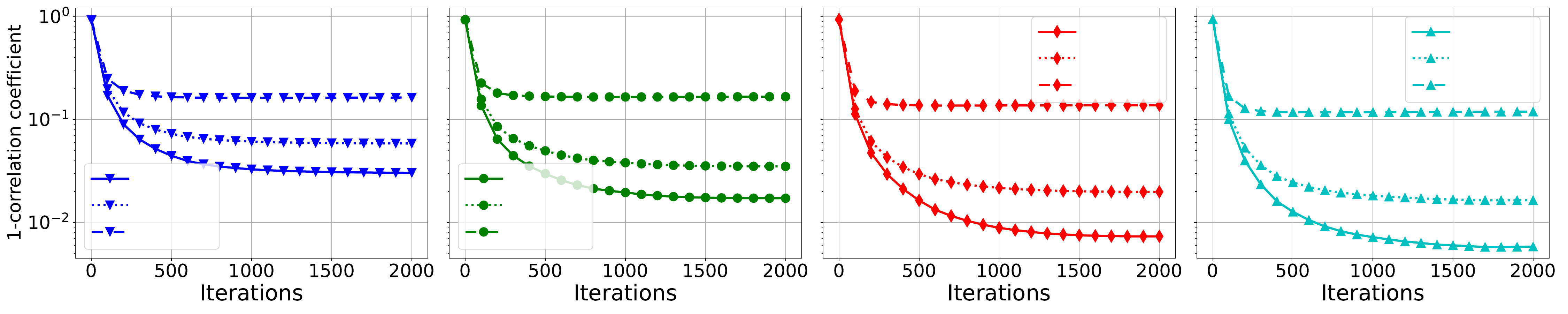}};
        \node at (-5.75,-1.7) {\small{Iterations}};
        \node at (-5.75,1.7) {\small{$d=4$}};
        \node at (-1.8,-1.7) {\small{Iterations}};
        \node at (-1.8,1.7) {\small{$d=6$}};
        \node at (2.15,-1.7) {\small{Iterations}};
        \node at (2.15,1.7) {\small{$d=8$}};
        \node at (6.15,-1.7) {\small{Iterations}};
        \node at (6.15,1.7) {\small{$d=10$}};
        \node[rotate=90] at (-8.35, 0) {\small{$1-$correlation coefficient}};

        \node at (-6.63,-0.4) {\scriptsize{$\W^\rfn$}};
        \node at (-6.74,-0.7) {\scriptsize{$\Ws$}};
        \node at (-6.6,-.95) {\scriptsize{$\W^\ont$}};

        \node at (-2.67,-0.4) {\scriptsize{$\W^\rfn$}};
        \node at (-2.78,-0.7) {\scriptsize{$\Ws$}};
        \node at (-2.64,-.95) {\scriptsize{$\W^\ont$}};


        \node at (3.39,1.15) {\scriptsize{$\W^\rfn$}};
        \node at (3.28,0.85) {\scriptsize{$\Ws$}};
        \node at (3.42,0.6) {\scriptsize{$\W^\ont$}};

        \node at (7.33,1.15) {\scriptsize{$\W^\rfn$}};
        \node at (7.22,0.85) {\scriptsize{$\Ws$}};
        \node at (7.36,0.6) {\scriptsize{$\W^\ont$}};
    \end{tikzpicture}
    \begin{tikzpicture}
        \node at (0,0) {\includegraphics[width=0.985\textwidth, trim={1.3cm 1.5cm 0 0}, clip]{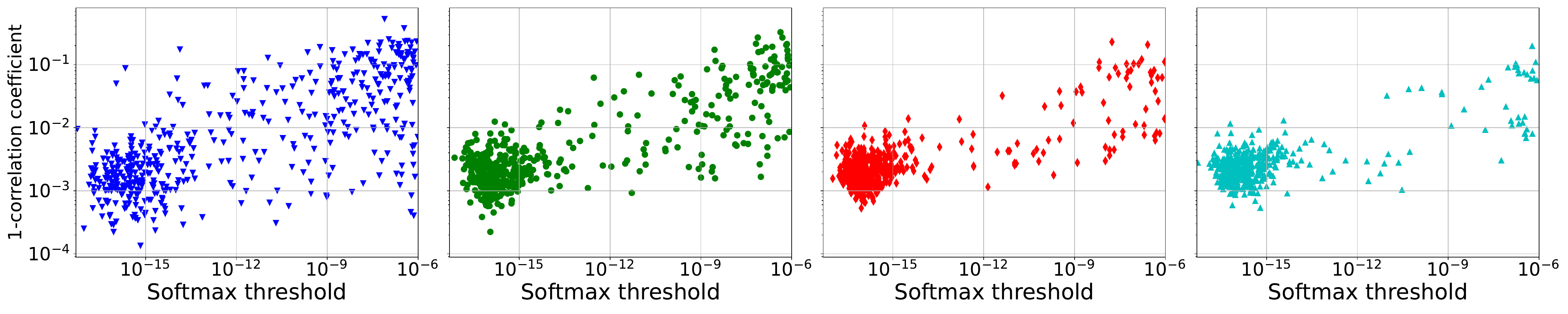}};
        \node at (-5.75,-1.7) {\small{$\max_{i,\tau}s_{i\tau},~\tau\in\Rcb_i$}};
        \node at (-1.8,-1.7) {\small{$\max_{i,\tau}s_{i\tau},~\tau\in\Rcb_i$}};
        \node at (2.15,-1.7) {\small{$\max_{i,\tau}s_{i\tau},~\tau\in\Rcb_i$}};
        \node at (6.15,-1.7) {\small{$\max_{i,\tau}s_{i\tau},~\tau\in\Rcb_i$}};
        \node[rotate=90] at (-8.35, 0) {\small{$1-$correlation coefficient}};
    \end{tikzpicture}
    \caption{ Behavior of GD  with nonlinear nonconvex prediction head and multi-token compositions. \textbf{Upper:} The correlation between GD solution and three distinct baselines: ({\color{black}{\textbf{$\cdots$}}}) $\Ws$ obtained from \eqref{eqn:mattnsvm};  ({\color{black}{\textbf{---}}}) $\W^\rfn$ obtained by calculating $\Wf$ and determining the best linear combination $\Wf+\gamma \Wsb$ that maximizes correlation with the GD solution; and  ({\color{black}{\textbf{-~-}}}) $\W^\ont$ obtained by solving \eqref{eqn:sattnsvm} and selecting the highest probability token from the GD solution.  \textbf{Lower:} Scatterplot of the largest softmax probability over masked tokens (per our $s_{i\tau}\leq 10^{-6}$ criteria) vs correlation coefficient.}
    \label{fig nn diff d}
\end{figure}

We present experiments showcasing the predictive power of the \eqref{eqn:mattnsvm} equivalence in nonlinear scenarios. We conducted these experiments on random instances using an MLP denoted as $h(\cdot)$, which takes the form of $\onebb^\top\texttt{ReLU}(\x)$. We begin by detailing the preprocessing step and our setup. For the attention SVM equivalence analytical prediction, clear definitions of the selected and masked sets are crucial. These sets include token indices with nonzero and zero softmax outputs, respectively. However, practically, reaching a precisely zero output is not feasible. Hence, we define the selected set as tokens with softmax outputs exceeding $10^{-3}$, and the masked set as tokens with softmax outputs below $10^{-6}$. We also excluded instances with softmax outputs falling between $10^{-6}$ and $10^{-3}$ to distinctly separate the concepts of \emph{selected} and \emph{masked} sets, thereby enhancing the predictive accuracy of the attention SVM equivalence. In addition to the filtering process, we focus on scenarios where the label $Y=-1$ exists to enforce \emph{non-convexity} of prediction head $Y_i\cdot h(\cdot)$. It is worth mentioning that when all labels are $1$, due to the convexity of $Y_i\cdot h(\cdot)$, GD tends to select one token per input, and Equations \eqref{eqn:mattnsvm} and \eqref{eqn:sattnsvm} yield the same solutions.  The results are displayed in Figure~\ref{fig nn diff d}, where $n=3$, $T=4$, and $d$ varies within ${4, 6, 8, 10}$. We conduct 500 random trials for different choices of $d$, each involving ${\x}_{it}$, ${\z}_i$, and ${\vb}$ randomly sampled from the unit sphere. We apply normalized GD with a step size $\eta=0.1$ and run $2000$ iterations for each trial.

\begin{enumerate}[label=$\bullet$, wide, labelwidth=!,itemindent=!, labelindent=5pt]
\item Figure \ref{fig nn diff d} (upper) illustrates the correlation evolution between the GD solution and three distinctive baselines: ({\color{black}{\textbf{$\cdots$}}}) $\Ws$ obtained from \eqref{eqn:mattnsvm};  ({\color{black}{\textbf{---}}}) $\W^\rfn$ obtained by calculating $\Wf$ and determining the best linear combination $\Wf+\gamma \Wsb$ that maximizes correlation with the GD solution; and  ({\color{black}{\textbf{-~-}}}) $\W^\ont$ obtained by solving \eqref{eqn:sattnsvm} and selecting the highest probability token from the GD solution.  For clearer visualization, the logarithmic scale of correlation misalignment is presented in Figure~\ref{fig nn diff d}. In essence, our findings show that $\W^\ont$ yields unsatisfactory outcomes, whereas $\Ws$ attains a significant correlation coefficient in alignment with our expectations. Ultimately, our comprehensive SVM-equivalence $\W^\rfn$ further enhances correlation, lending support to our analytical formulas. It's noteworthy that SVM-equivalence displays higher predictability in a larger $d$ regime (with an average correlation exceeding $0.99$). This phenomenon might be attributed to more frequent directional convergence in higher dimensions, with overparameterization contributing to a smoother loss landscape, thereby expediting optimization. 
\item Figure \ref{fig nn diff d} (lower) offers a scatterplot overview of the $500$ random problem instances that were solved. The $x$-axis represents the largest softmax probability over the masked set, denoted as $\max_{i,\tau}s_{i\tau}$ where $\tau\in\Rcb_i$. Meanwhile, the $y$-axis indicates the predictivity of the SVM-equivalence, quantified as $1-\texttt{corr\_coef}(\W,\W^\rfn)$. From this analysis, two significant observations arise. Primarily, there exists an inverse correlation between softmax probability and SVM-predictivity. This correlation is intuitive, as higher softmax probabilities signify a stronger divergence from our desired \emph{masked set} state (ideally set to $0$). Secondly, as dimensionality ($d$) increases, softmax probabilities over the masked set tend to converge towards the range of $10^{-15}$ (effectively zero). Simultaneously, attention SVM-predictivity improves, creating a noteworthy correlation.
\end{enumerate}

\subsection{When does attention select multiple tokens?}\label{sec when}
In this section, we provide a concrete example where the optimal solution indeed requires combining multiple tokens in a nontrivial fashion. Here, by nontrivial we mean that, we select more than 1 tokens from an input sequence but we don't select all of its tokens. Recall that, for linear prediction head, attention will ideally select the single token with largest score for almost all datasets. Perhaps not surprisingly, this behavior will not persist for nonlinear prediction heads. For instance in Figure~\ref{fig nn diff d}, the GD output $\W$ aligned better in direction with $\Ws$ than $\W^\ont$. Specifically, here we prove that if we make the function $h_Y(\x):=Y\cdot h(\x)$ concave, then optimal softmax map can select multiple tokens in a controllable fashion. $h_Y(\x)$ can be viewed as generalization of the linear score function $Y\cdot \vb^\top\x$. In the example below, we induce concavity by incorporating a small $-\la\tn{\x}^2$ term within a linear prediction head and setting $h(\x)=\vb^\top\x-\la\tn{\x}^2$ with $Y=1$.

\begin{lemma}\label{example dataset} Given $\vb\in\R^d$, recall the score vector $\bgam=\X\vb$. Without losing generality, assume $\bgam$ is non-increasing. Define the vector of score gaps $\bbg\in\R^{T-1}$ with entries $\bbg_t=\bgam_{t}-\bgam_{t+1}$. Suppose all tokens within the input sequence are orthonormal and for some $\tau\geq 2$, we have that 
\begin{align}
\tau\bbg_\tau/2>\bbg_1.\label{tau description}
\end{align}
Set $h(\x)=\vb^\top\x-\la\tn{\x}^2$ where $\tau\bbg_\tau/2>\la>\bbg_1$, $\ell(x)=-x$, and $Y=1$. Let $\Bal_T$ denote the $T$-dimensional simplex. Define the unconstrained softmax optimization associated to the objective $h$ where we make $\s:=\sft{\X\W\z}$ a free variable, namely,
\begin{align} 
\min_{\s\in\Bal_T}\ell(h(\X\s))=\min_{\s\in\Bal_T}\la \tn{\X^\top \s}^2-\vb^\top\X^\top \s.\label{direct opt}
\end{align}
Then, the optimal solution $\s^\st$ contains at least $2$ and at most $\tau$ nonzero entries.
\end{lemma}
\begin{figure}[t]
    \centering
    \hspace{-10pt}
    \subfigure[$\lambda$ vs \# selected tokens]{
        \begin{tikzpicture}
        \node at (0,0) {\includegraphics[height=.22\columnwidth, trim={1.3cm 1.3cm 0 0}, clip]{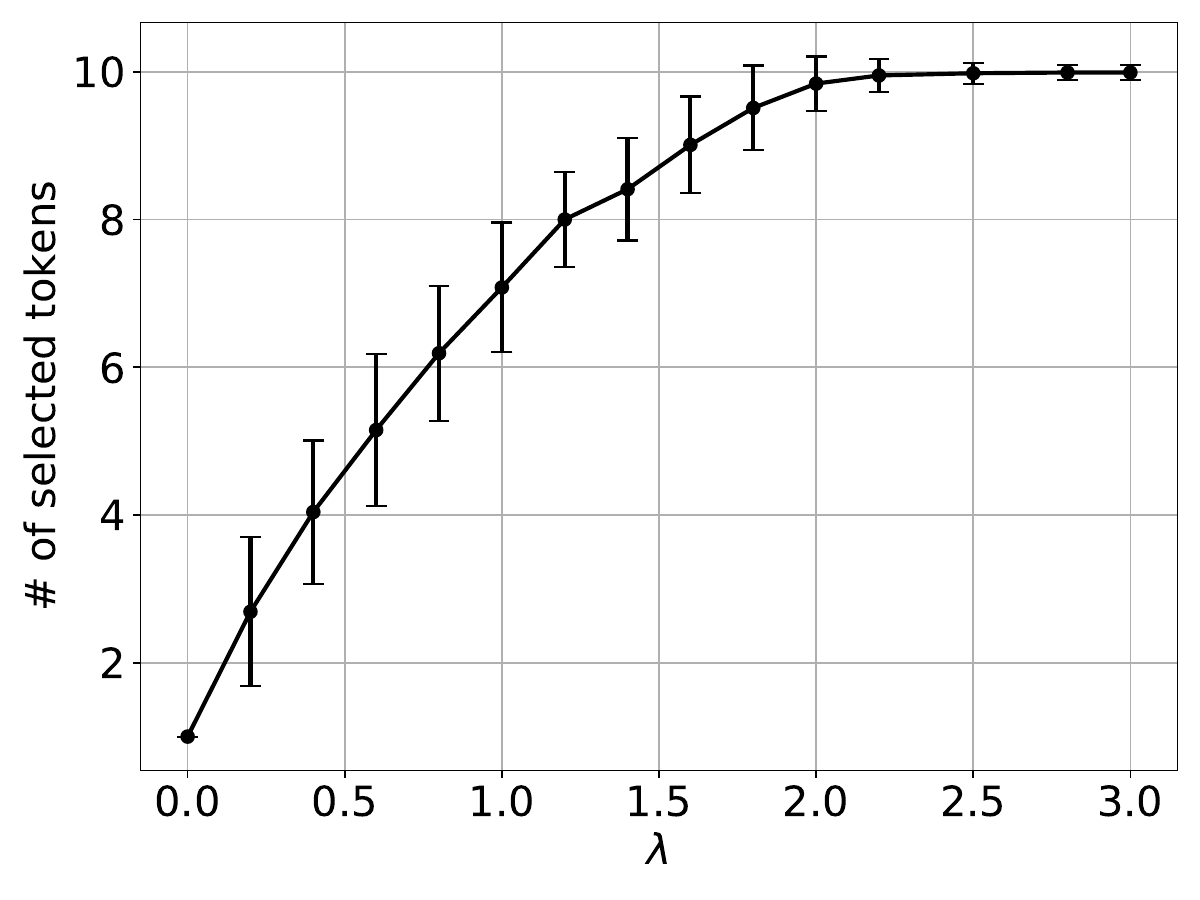}};
        \node at (0,-2.) {\small{$\lambda$}};
        \node[rotate=90] at (-2.65,0) {\small{\# of selected tokens}};
        \end{tikzpicture}
        \label{fig multi ns diff lambda}
    }
    \hspace{-10pt}
    \subfigure[$\lambda$ vs correlation coefficient]{
        \begin{tikzpicture}
        \node at (0,0) {\includegraphics[height=.22\columnwidth, trim={1.3cm 1.3cm 0 0}, clip]{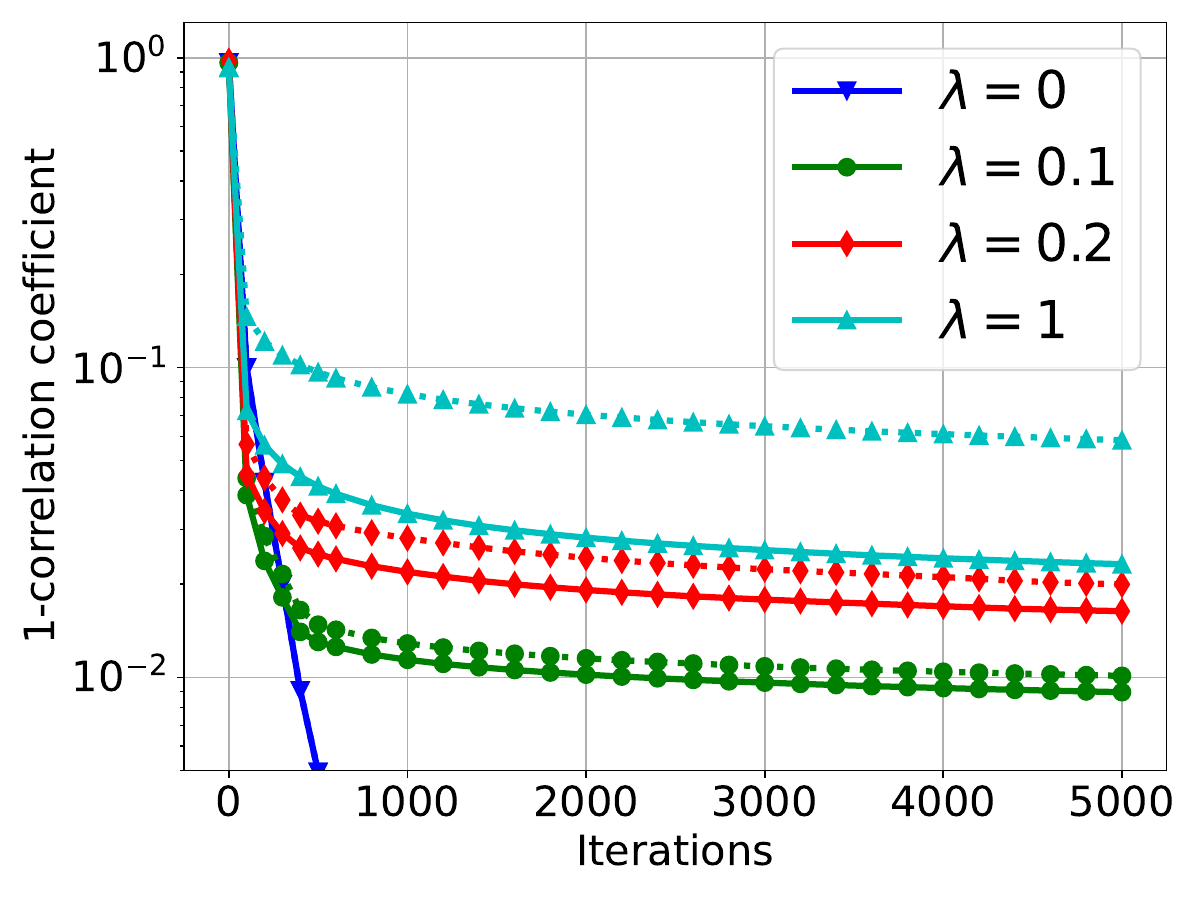}};
        \node[rotate=90] at (-2.7,0) {\small{$1-$correlation coefficient}};
        \node at (0,-2.){\small{Iterations}};
        \end{tikzpicture}
        \label{fig multi corr diff lambda}
    }
    \hspace{-10pt}
    \subfigure[\# selected tokens vs correlation coefficient]{
        \begin{tikzpicture}
        \node at (0,0) {\includegraphics[height=.22\columnwidth, trim={1.3cm 1.3cm 0 0}, clip]{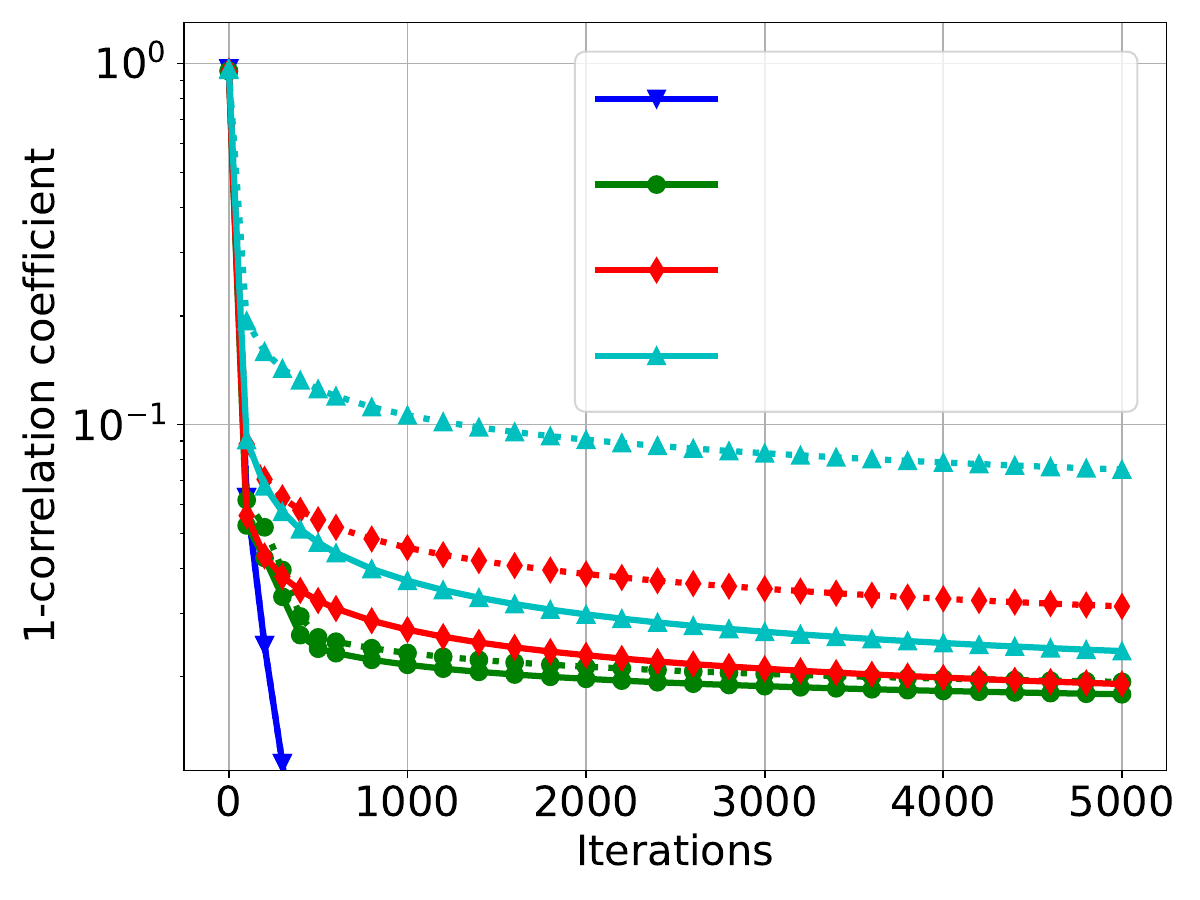}};
        \node at (0,-2.) {\small{Iterations}};
        \node[rotate=90] at (-2.65,0) {\small{$1-$correlation coefficient}};
        \node at (1.25,1.4) {\footnotesize{$\text{\# selected}=1$}};
        \node at (1.25,1.02) {\footnotesize{$\text{\# selected}=2$}};
        \node at (1.25,0.64) {\footnotesize{$\text{\# selected}=5$}};
        \node at (1.25,0.26) {\footnotesize{$\text{\# selected}=9$}};
        \end{tikzpicture}
        \label{fig multi corr diff ns}
    }
\caption{ Behavior of GD when selecting multiple tokens. \textbf{(a)} The number of selected tokens increases with $\lambda$. \textbf{(b)} Predictivity of attention SVM solutions for varying $\lambda$; Dotted curves depict the correlation corresponding to $\Ws$ calculated via \eqref{eqn:mattnsvm} and solid curves represent the correlation to $\W^\rfn$, which incorporates the $\Wf$ correction. \textbf{(c)} Similar to (b), but evaluating correlations over different numbers of selected tokens.}
    \label{fig multi corrs}
\end{figure}
Figure \ref{fig multi corrs} presents experimental findings concerning Lemma \ref{example dataset} across random problem instances. For this experiment, we set $n=1$, $T=10$, and $d=10$. The results are averaged over $100$ random trials, with each trial involving the generation of randomly orthonormal vectors $\x_{1t}$ and the random sampling of vector $\vb$ from the unit sphere. Similar to the processing step in Figure~\ref{fig nn diff d}, and following Figure~\ref{fig nn diff d} (lower) which illustrates that smaller softmax outputs over masked sets correspond to higher correlation coefficients, we define the selected and masked token sets. Specifically, tokens with softmax outputs $>10^{-3}$ are considered selected, while tokens with softmax outputs $<10^{-8}$ are masked. Instances with softmax outputs between $10^{-8}$ and $10^{-3}$ are filtered out.

Figure \ref{fig multi ns diff lambda} shows that the number of selected tokens grows alongside $\lambda$, a prediction consistent with Lemma \ref{example dataset}. When $\lambda=0$, the head $h(\x)=\vb^\top\x$ is linear, resulting in the selection of only one token per input. Conversely, as $\lambda$ exceeds a certain threshold (e.g., $\lambda>2.0$ based on our criteria), the optimization consistently selects all tokens. Figure \ref{fig multi corr diff lambda} and \ref{fig multi corr diff ns} delve into the predictivity of attention SVM solutions for varying $\lambda$ and different numbers of selected tokens. The dotted curves in both figures represent $1-\corr{\W,\Ws}$, while solid curves indicate $1-\corr{\W,\W^{\rfn}}$, where $\W$ denotes the GD solution. Overall, the SVM-equivalence demonstrates a strong correlation with the GD solution (consistently above $0.95$). However, selecting more tokens (aligned with larger $\lambda$ values) leads to reduced predictivity.

To sum up, we have showcased the predictive capacity of the generalized SVM equivalence regarding the inductive bias of 1-layer transformers with nonlinear heads. Nevertheless, it's important to acknowledge that this section represents an initial approach to a complex problem, with certain caveats requiring further investigation (e.g., the use of filtering in Figures \ref{fig nn diff d} and \ref{fig multi corrs}, and the presence of imperfect correlations). We aspire to conduct a more comprehensive investigation, both theoretically and empirically, in forthcoming work.

\section{Extending the Theory to Sequential and Causal Predictions}\label{sec seq extend}


While our formulations (\ref{eqn:erm:w} \& \ref{eqn:erm:kq}) regress a single label $Y_i$ per $(\X_i,\z_i)$, we extend in Appendix \ref{sec multioutput}  our findings to the sequence-to-sequence classification setting, where we output and classify all $T$ tokens per input $\X_i,\Z_i\in\R^{T\times d}$. In this scenario, we prove that all of our RP guarantees remain intact after introducing a slight generalization of \eqref{eqn:sattnsvm}. Concretely, consider the following ERM problems for sequential and causal settings:
\begin{align}
\Lc^{\texttt{seq}}(\W)=\frac{1}{n}\sum_{i=1}^n\sum_{k=1}^T \ell(Y_\ik\cdot h( \X_i^\top \sft{\X_i\W\z_\ik})),\quad\text{and}\quad
\Lc^{\texttt{csl}}(\W)=\frac{1}{n}\sum_{i=1}^n\sum_{k=1}^T \ell(Y_\ik\cdot h( \X_i^\top \sftk{\X_i\W\z_\ik})).
\end{align}
Both equations train $T$ tokens per input $(\X_i,\Z_i)$ and, as usual, we recover self-attention via $\Z_i\gets \X_i$. For the causal setting, we use the masked attention $\sftk{\cdot}$ which calculates the softmax probabilities over the first $k$ entries of its input and sets the remaining $T-k$ entries to zero. This way, the $k$'th prediction of the transformer only utilizes tokens from $1$ to $k$ and not the future tokens.

Let $\bal=(\alpha_\ik)\ikix$ be tokens to be selected by attention (e.g.~locally-optimal indices, see Def.~\ref{seq loc opt}). Then, the sequential generalization of \eqref{eqn:sattnsvm} corresponding to $\Lc^{\texttt{seq}}(\W)$ is given by
\begin{equation}
 \min_{\W}\tf{\W}
\quad \text{subj. to} \quad  (\x_{i\aik}-\x_\itt)^\top\W\z_\ik\geq 1\quad\textnormal{for all}\quad {t\neq \aik},~~k\in[T],~~ i\in[n]~~\label{seqattnsvm1}.
\end{equation}
We refer the reader to Appendix \ref{sec multioutput} which rigorously establishes all RP results for this sequential classification setting. On the other hand, for the causal inference setting SVM should reflect the fact that the model is not allowed to make use of future tokens. Note that the SVM constraints directly arise from softmax calculations. Thus, since attention is masked over the indices $t\leq k$ and $k\in[T]$, the SVM constraints should apply over the same mask. Thus, we can consider the straightforward generalization of global-optimality where $\op_\ik$ is the token index with highest score over indices $t\leq k$ and introduce an analogous definition for local-optimality. This leads to the following variation of \eqref{seqattnsvm1}, which aims to select indices $\alpha_\ik\in [k]$ from the first $k$ tokens
\begin{equation*}
 \min_{\W}\tf{\W}
\quad \text{subj. to} \quad   (\x_{i\aik}-\x_\itt)^\top\W\z_\ik\geq 1\quad \textnormal{for all} \quad  t\neq \aik,~~t\leq k, ~~k \in[T],~~i\in[n]\label{cslattnsvm}.
\end{equation*}

Causal attention is a special case of a general attention mask which can restrict softmax to arbitrary subset of the tokens. Such general masking can be handled similar to \eqref{cslattnsvm} by enforcing SVM constraints over the nonzero support of the mask. Finally, the discussion so far extends our main theoretical results and focuses on selecting single token per sequence. It can further be enriched by the generalized SVM equivalence developed in Section \ref{sec:multi} to select and compose multiple tokens by generalizing \eqref{eqn:mattnsvm}.
\section{Related work}\label{sec:related}
\subsection{Implicit regularization, matrix factorization, and sparsity}
Extensive research has delved into gradient descent's implicit bias in separable classification tasks, often using logistic or exponentially-tailed losses for margin maximization \cite{soudry2018implicit,gunasekar2018characterizing,nacson2019convergence,ji2021characterizing,kini2021label,moroshko2020implicit,ji2020directional}. The findings have also been extended to non-separable data using gradient-based techniques \cite{ji2018risk,ji2019implicit,ji2020gradient}. Implicit bias in regression problems and losses has been investigated, utilizing methods like mirror descent \cite{woodworth2020kernel, gunasekar2018characterizing,
yun2020unifying, vaskevicius2019implicit, amid2020winnowing, amid2020reparameterizing,azizan2021stochastic,sun2022mirror}. Stochastic gradient descent has also been a subject of interest regarding its implicit bias \cite{li2019towards,blanc2020implicit,liang2020just, haochen2020shape, li2022what, damian2021label, zou2021benefits}. This extends to the implicit bias of adaptive and momentum-based methods \cite{qian2019implicit, wang2021momentum, wang2021implicit, ji2021fast}.

In linear classification, GD iterations on logistic loss and separable datasets converge to the hard margin SVM solution \cite{soudry2018implicit,rosset2003margin,zhang2005boosting}. The attention layer's softmax nonlinearity behaves similarly, potentially favoring margin-maximizing solutions. Yet, the layer operates on tokens in input sequences, not for direct classification. Its bias leans toward an \eqref{eqn:sattnsvm}, selecting relevant tokens while suppressing others. However, formalizing this intuition presents significant challenges: Firstly, our problem is nonconvex (even in terms of the $\W$-parameterization), introducing new challenges and complexities. Secondly, it requires the introduction of novel concepts such as locally-optimal tokens, demanding a tailored analysis focused on the cones surrounding them. Our findings on the implicit bias of $(\Kb,\Qb)$-parameterization share conceptual similarities with \cite{srebro2004maximum}, which proposes and analyzes a max-margin matrix factorization problem. Similar problems have also been studied more recently   in the context of neural-collapse phenomena \cite{papyan2020prevalence} through an analysis of the implicit bias and regularization path of the unconstrained features model with cross-entropy  loss \cite{thrampoulidis2022imbalance}. However, a fundamental distinction from these works lies in the fact that attention solves a different max-margin problem that separate tokens. Moreover, our results on $(\Kb,\Qb)$-parameterization are inherently connected to the rich literature on low-rank factorization \cite{gunasekar2017implicit,arora2019implicit,timor2023implicit,tu2016low,stoger2021small}, stimulating further research. \cite{tarzanagh2023margin} is the first work to establish the connection between attention and SVM, which is closest to our work. Here, we augment their framework, initially developed for a simpler attention model, to transformers by providing the first guarantees for self/cross-attention layers, nonlinear prediction heads, and realistic global convergence guarantees. While our Assumption \ref{assum:opt:token} and local-convergence analysis align with \cite{tarzanagh2023margin}, our contributions in global convergence analysis, benefits of overparameterization, and the generalized SVM-equivalence in Section \ref{sec:multi} are unique to this work.

It is well-known that attention map (i.e.~softmax outputs) act as a feature selection mechanism and reveal the tokens that are relevant to classification. On the other hand, sparsity and lasso regression (i.e.~$\ell_1$ penalization) \cite{donoho2006compressed,tibshirani1996regression,tropp2007signal,chen2001atomic,candes2006robust} have been pivotal tools in the statistics literature for feature selection. Softmax and lasso regression exhibit interesting parallels: The Softmax output $\s=\sft{\X\W\z}$ obeys $\|\s\|_{\ell_1}=1$ by design. Softmax is also highly receptive to being sparse because decreasing the temperature (i.e.~scaling up the weights $\W$) eventually leads to a one-hot vector unless all logits are equal. We (also, \cite{tarzanagh2023margin}) have used these intuitions to formalize attention as a \emph{token selection mechanism}. This aspect is clearly visible in our primary SVM formulation \eqref{eqn:sattnsvm} which selects precisely one token from each input sequence (i.e.~hard attention). Section \ref{sec:multi} has also demonstrated how  \eqref{eqn:mattnsvm} can explain more general sparsity patterns by precisely selecting desired tokens and suppressing others. We hope that this SVM-based token-selection viewpoint will motivate future work and deeper connections to the broader feature-selection and compressed sensing literature.
\subsection{Attention mechanism and transformers}

Transformers, as highlighted by \cite{vaswani2017attention}, revolutionized the domains of NLP and machine translation. Prior work on self-attention \cite{cheng2016long,parikh2016decomposable,paulus2017deep,lin2017structured} laid the foundation for this transformative paradigm. In contrast to conventional models like MLPs and CNNs, self-attention models employ global interactions to capture feature representations, resulting in exceptional empirical performance.

Despite their achievements, the mechanisms and learning processes of attention layers remain enigmatic. Recent investigations \cite{edelman2022inductive,sahiner2022unraveling,ergen2022convexifying,baldi2022quarks,dong2021attention} have concentrated on specific aspects such as sparse function representation, convex relaxations, and expressive power. Expressivity discussions concerning hard-attention \cite{hahn2020theoretical} or attention-only architectures \cite{dong2021attention} are connected to our findings when $h(\cdot)$ is linear. In fact, our work reveals how linear $h$ results in attention's optimization dynamics to collapse on a single token whereas nonlinear $h$ provably requires attention to select and compose multiple tokens. This supports the benefits of the MLP layer for expressivity of transformers. There is also a growing body of research aimed at a theoretical comprehension of in-context learning and the role played by the attention mechanism \cite{akyurek2022learning,li2023transformers,ahn2023transformers,zhang2023trained,bai2023transformers,giannou2023looped}. \cite{sahiner2022unraveling} investigate self-attention with linear activation instead of softmax, while \cite{ergen2022convexifying} approximate softmax using a linear operation with unit simplex constraints. Their primary goal is to derive convex reformulations for training problems grounded in empirical risk minimization (ERM). In contrast, our methodologies, detailed in equations \eqref{eqn:erm:w} and \eqref{eqn:erm:kq}, delve into the nonconvex domain.

\cite{merrill2020effects,boix2023transformers} offer insights into the implicit bias of optimizing transformers. Specifically, \cite{merrill2020effects} provide empirical evidence that an increase in attention weights results in a sparser softmax, which aligns with our theoretical framework. \cite{boix2023transformers} study incremental learning and furnish both theory and numerical evidence that increments of the softmax attention weights ($\Kb\Qb^\top$) are low-rank. Our theory aligns with this concept, as the SVM formulation \eqref{eqn:qk:svm} of $(\Kb,\Qb)$ parameterization inherently exhibits low-rank properties through the nuclear norm objective, rank-$m$ constraint, and implicit constraint induced by Lemma \ref{lem:rank}.

Several recent works \cite{jelassi2022vision,li2023theoretical,tian2023scan,noci2023shaped,oymak2023role,nguyen2023primal,fu2023can} aim to delineate the optimization and generalization dynamics of transformers. However, their findings usually apply under strict statistical assumptions about the data, while our study offers a comprehensive optimization-theoretic analysis of the attention model, establishing a formal linkage to max-margin problems and SVM geometry. This allows our findings to encompass the problem geometry and apply to diverse datasets. Overall, the max-margin equivalence  provides a fundamental comprehension of the optimization geometry of transformers, offering a framework for prospective research endeavors, as outlined in the subsequent section.
\section{Discussion, Future Directions, and Open Problems}\label{sec:conc}

Our optimization-theoretic characterization of the self-attention model provides a comprehensive understanding of its underlying principles. The developed framework, along with the research presented in \cite{tarzanagh2023margin}, introduces new avenues for studying transformers and language models. The key findings include:
\begin{enumerate}[label=$\checkmark$, wide, labelwidth=!,itemindent=!, labelindent=1pt]
\item The optimization geometry of self-attention exhibits a fascinating connection to  hard-margin SVM problems. By leveraging linear constraints formed through outer products of token pairs, optimal input tokens can be effectively separated from non-optimal ones.
\item When gradient descent is employed without early-stopping, implicit regularization and convergence of self-attention naturally occur. This convergence leads to the maximum margin solution when minimizing specific requirements using logistic loss, exp-loss, or other smooth decreasing loss functions. Moreover, this implicit bias is unaffected by the step size, as long as it is sufficiently small for convergence, and  remains independent of the initialization process.
\end{enumerate}
The fact that gradient descent leads to a maximum margin solution may not be surprising to those who are familiar with the relationship between regularization path and gradient descent in linear and nonlinear neural networks  \cite{soudry2018implicit,
gunasekar2018characterizing, nacson2019convergence, ji2021characterizing,moroshko2020implicit,ji2020directional}. However, there is a lack of prior research or discussion regarding this connection to the  attention mechanism. Moreover, there has been no rigorous analysis or investigation into the exactness and independence of this bias with respect to the initialization and step size. Thus, we believe our findings and insights deepen our understanding of transformers and language models, paving the way for further research in this domain. Below, we discuss some notable directions and highlight open problems that are not resolved by the existing theory.
\begin{itemize}
\item \textbf{Convergence Rates}: The current paper establishes asymptotic convergence of gradient descent; nonetheless, there is room for further exploration to characterize non-asymptotic convergence rates. Indeed, such an exploration can also provide valuable insights into the choice of learning rate, initialization, and the optimization method. 

\item \textbf{Gradient descent on $(\Kb,\Qb)$ parameterization:} We find it remarkable that regularization path analysis was able to predict the implicit bias of gradient descent. Complete analysis of gradient descent is inherently connected to the fundamental question of low-rank factorization \cite{gunasekar2017implicit,li2018algorithmic}. We believe formalizing the implicit bias of gradient descent under margin constraints presents an exciting open research direction for further research.

\item \textbf{Generalization analysis:} An important direction is the generalization guarantees for gradient-based algorithms. The established connection to hard-margin SVM can facilitate this because the SVM problem is amenable to statistical analysis. This would be akin to how kernel/NTK analysis for deep nets enabled a rich literature on generalization analysis for traditional deep learning.
\item \textbf{Global convergence of gradient descent:} We lack a complete characterization of the directional convergence of gradient descent. We ask: \emph{Where does gradient descent directionally-converge from arbitrary initialization for 1-layer self-attention?} The role of over-parameterization as conjectured in Section \ref{sec overparam} and the notion of locally-optimal directions discussed in Section \ref{sec local} constitute important pieces of this puzzle (also see the discussion in \cite{tarzanagh2023margin}).
\item \textbf{Realistic architectures:} Naturally, we wish to explore whether max-margin equivalence can be extended to more realistic settings: Can the theory be expanded to handle multi-head attention, multi-layer architectures, and MLP nonlinearities? We believe the results in Section \ref{sec:multi} take an important step towards this direction by including analytical formulae for the implicit bias of the attention layer under nonlinear prediction heads. 

\item \textbf{Jointly optimizing attention and prediction head:} It would be interesting to study the joint optimization dynamics of attention weights and prediction head $h(\cdot)$. This problem can be viewed as a novel low-rank factorization type problem where $h(\cdot)$ and $\W$ are factors of the optimization problem, only, here, $\W$ passes through the softmax nonlinearity. To this aim, \cite{tarzanagh2023margin} provides a preliminary geometric characterization of the implicit bias for a simpler attention model using regularization path analysis. Such findings can potentially be generalized to the analysis of gradient methods and full transformer block.
\end{itemize}
\section*{Acknowledgements}

This work was supported by the NSF grants CCF-2046816 and CCF-2212426, Google Research Scholar award, and Army Research Office grant W911NF2110312. The authors thank Xuechen Zhang, Ankit Singh Rawat, Mahdi Soltanolkotabi, Jason Lee, Arkadas Ozakin, Ramya Korlakai Vinayak, and Babak Hassibi for helpful suggestions and discussion.
\bibliographystyle{alpha}
\bibliography{refs}

\newcommand{\etalchar}[1]{$^{#1}$}
\begin{thebibliography}{BALA{\etalchar{+}}23}

\bibitem[ACDS23]{ahn2023transformers}
Kwangjun Ahn, Xiang Cheng, Hadi Daneshmand, and Suvrit Sra.
\newblock Transformers learn to implement preconditioned gradient descent for
  in-context learning.
\newblock {\em arXiv preprint arXiv:2306.00297}, 2023.

\bibitem[ACHL19]{arora2019implicit}
Sanjeev Arora, Nadav Cohen, Wei Hu, and Yuping Luo.
\newblock Implicit regularization in deep matrix factorization.
\newblock {\em Advances in Neural Information Processing Systems}, 32, 2019.

\bibitem[ALH21]{azizan2021stochastic}
Navid Azizan, Sahin Lale, and Babak Hassibi.
\newblock Stochastic mirror descent on overparameterized nonlinear models.
\newblock {\em IEEE Transactions on Neural Networks and Learning Systems},
  33(12):7717--7727, 2021.

\bibitem[ASA{\etalchar{+}}22]{akyurek2022learning}
Ekin Aky{\"u}rek, Dale Schuurmans, Jacob Andreas, Tengyu Ma, and Denny Zhou.
\newblock What learning algorithm is in-context learning? investigations with
  linear models.
\newblock {\em arXiv:2211.15661}, 2022.

\bibitem[AW20a]{amid2020winnowing}
Ehsan Amid and Manfred~K Warmuth.
\newblock Winnowing with gradient descent.
\newblock In {\em Conference on Learning Theory}, pages 163--182. PMLR, 2020.

\bibitem[AW20b]{amid2020reparameterizing}
Ehsan Amid and Manfred~KK Warmuth.
\newblock Reparameterizing mirror descent as gradient descent.
\newblock {\em Advances in Neural Information Processing Systems},
  33:8430--8439, 2020.

\bibitem[AZLS19]{allen2019convergence}
Zeyuan Allen-Zhu, Yuanzhi Li, and Zhao Song.
\newblock A convergence theory for deep learning via over-parameterization.
\newblock In {\em International Conference on Machine Learning}, pages
  242--252. PMLR, 2019.

\bibitem[BALA{\etalchar{+}}23]{boix2023transformers}
Enric Boix-Adsera, Etai Littwin, Emmanuel Abbe, Samy Bengio, and Joshua
  Susskind.
\newblock Transformers learn through gradual rank increase.
\newblock {\em arXiv preprint arXiv:2306.07042}, 2023.

\bibitem[BCW{\etalchar{+}}23]{bai2023transformers}
Yu~Bai, Fan Chen, Huan Wang, Caiming Xiong, and Song Mei.
\newblock Transformers as statisticians: Provable in-context learning with
  in-context algorithm selection.
\newblock {\em arXiv preprint arXiv:2306.04637}, 2023.

\bibitem[BGVV20]{blanc2020implicit}
Guy Blanc, Neha Gupta, Gregory Valiant, and Paul Valiant.
\newblock Implicit regularization for deep neural networks driven by an
  ornstein-uhlenbeck like process.
\newblock In {\em Conference on learning theory}, pages 483--513. PMLR, 2020.

\bibitem[BLLT20]{bartlett2020benign}
Peter~L Bartlett, Philip~M Long, G{\'a}bor Lugosi, and Alexander Tsigler.
\newblock Benign overfitting in linear regression.
\newblock {\em Proceedings of the National Academy of Sciences},
  117(48):30063--30070, 2020.

\bibitem[BMR{\etalchar{+}}20]{brown2020language}
Tom Brown, Benjamin Mann, Nick Ryder, Melanie Subbiah, Jared~D Kaplan, Prafulla
  Dhariwal, Arvind Neelakantan, Pranav Shyam, Girish Sastry, Amanda Askell, and
  et~al.
\newblock Language models are few-shot learners.
\newblock In {\em Advances in neural information processing systems},
  volume~33, pages 1877--1901, 2020.

\bibitem[BV22]{baldi2022quarks}
Pierre Baldi and Roman Vershynin.
\newblock The quarks of attention.
\newblock {\em arXiv preprint arXiv:2202.08371}, 2022.

\bibitem[Car21]{carlsson2021neumann}
Marcus Carlsson.
\newblock von neumann’s trace inequality for hilbert--schmidt operators.
\newblock {\em Expositiones Mathematicae}, 39(1):149--157, 2021.

\bibitem[CDL16]{cheng2016long}
Jianpeng Cheng, Li~Dong, and Mirella Lapata.
\newblock Long short-term memory-networks for machine reading.
\newblock In {\em Proceedings of the 2016 Conference on Empirical Methods in
  Natural Language Processing}, pages 551--561, Austin, Texas, November 2016.
  Association for Computational Linguistics.

\bibitem[CDS01]{chen2001atomic}
Scott~Shaobing Chen, David~L Donoho, and Michael~A Saunders.
\newblock Atomic decomposition by basis pursuit.
\newblock {\em SIAM review}, 43(1):129--159, 2001.

\bibitem[CLR{\etalchar{+}}21]{chen2021decision}
Lili Chen, Kevin Lu, Aravind Rajeswaran, Kimin Lee, Aditya Grover, Misha
  Laskin, Pieter Abbeel, Aravind Srinivas, and Igor Mordatch.
\newblock Decision transformer: Reinforcement learning via sequence modeling.
\newblock In {\em Advances in Neural Information Processing Systems},
  volume~34, pages 15084--15097, 2021.

\bibitem[CRT06]{candes2006robust}
Emmanuel~J Cand{\`e}s, Justin Romberg, and Terence Tao.
\newblock Robust uncertainty principles: Exact signal reconstruction from
  highly incomplete frequency information.
\newblock {\em IEEE Transactions on information theory}, 52(2):489--509, 2006.

\bibitem[CSL{\etalchar{+}}23]{chen2023jigsaw}
Yingyi Chen, Xi~Shen, Yahui Liu, Qinghua Tao, and Johan~AK Suykens.
\newblock Jigsaw-vit: Learning jigsaw puzzles in vision transformer.
\newblock {\em Pattern Recognition Letters}, 166:53--60, 2023.

\bibitem[DCL21]{dong2021attention}
Yihe Dong, Jean-Baptiste Cordonnier, and Andreas Loukas.
\newblock Attention is not all you need: Pure attention loses rank doubly
  exponentially with depth.
\newblock In {\em International Conference on Machine Learning}, pages
  2793--2803. PMLR, 2021.

\bibitem[DML21]{damian2021label}
Alex Damian, Tengyu Ma, and Jason Lee.
\newblock Label noise sgd provably prefers flat global minimizers.
\newblock {\em arXiv preprint arXiv:2106.06530}, 2021.

\bibitem[Don06]{donoho2006compressed}
David~L Donoho.
\newblock Compressed sensing.
\newblock {\em IEEE Transactions on information theory}, 52(4):1289--1306,
  2006.

\bibitem[DZPS18]{du2018gradient}
Simon~S Du, Xiyu Zhai, Barnabas Poczos, and Aarti Singh.
\newblock Gradient descent provably optimizes over-parameterized neural
  networks.
\newblock {\em arXiv preprint arXiv:1810.02054}, 2018.

\bibitem[EGKZ22]{edelman2022inductive}
Benjamin~L Edelman, Surbhi Goel, Sham Kakade, and Cyril Zhang.
\newblock Inductive biases and variable creation in self-attention mechanisms.
\newblock In {\em International Conference on Machine Learning}, pages
  5793--5831. PMLR, 2022.

\bibitem[ENM22]{ergen2022convexifying}
Tolga Ergen, Behnam Neyshabur, and Harsh Mehta.
\newblock Convexifying transformers: Improving optimization and understanding
  of transformer networks.
\newblock {\em arXiv:2211.11052}, 2022.

\bibitem[Faz02]{fazel2002matrix}
Maryam Fazel.
\newblock {\em Matrix rank minimization with applications}.
\newblock PhD thesis, PhD thesis, Stanford University, 2002.

\bibitem[FGBM23]{fu2023can}
Hengyu Fu, Tianyu Guo, Yu~Bai, and Song Mei.
\newblock What can a single attention layer learn? a study through the random
  features lens.
\newblock {\em arXiv preprint arXiv:2307.11353}, 2023.

\bibitem[FXM{\etalchar{+}}21]{fan2021multiscale}
Haoqi Fan, Bo~Xiong, Karttikeya Mangalam, Yanghao Li, Zhicheng Yan, Jitendra
  Malik, and Christoph Feichtenhofer.
\newblock Multiscale vision transformers.
\newblock In {\em Proceedings of the IEEE/CVF International Conference on
  Computer Vision}, pages 6824--6835, 2021.

\bibitem[GLSS18]{gunasekar2018characterizing}
Suriya Gunasekar, Jason Lee, Daniel Soudry, and Nathan Srebro.
\newblock Characterizing implicit bias in terms of optimization geometry.
\newblock In {\em International Conference on Machine Learning}, pages
  1832--1841. PMLR, 2018.

\bibitem[Goo23]{bard}
Google.
\newblock Try bard, an ai experiment by google.
\newblock {\em https://bard.google.com}, 2023.

\bibitem[GRS{\etalchar{+}}23]{giannou2023looped}
Angeliki Giannou, Shashank Rajput, Jy-yong Sohn, Kangwook Lee, Jason~D Lee, and
  Dimitris Papailiopoulos.
\newblock Looped transformers as programmable computers.
\newblock {\em arXiv:2301.13196}, 2023.

\bibitem[GWB{\etalchar{+}}17]{gunasekar2017implicit}
Suriya Gunasekar, Blake~E Woodworth, Srinadh Bhojanapalli, Behnam Neyshabur,
  and Nati Srebro.
\newblock Implicit regularization in matrix factorization.
\newblock {\em Advances in neural information processing systems}, 30, 2017.

\bibitem[Hah20]{hahn2020theoretical}
Michael Hahn.
\newblock Theoretical limitations of self-attention in neural sequence models.
\newblock {\em Transactions of the Association for Computational Linguistics},
  8:156--171, 2020.

\bibitem[HMX21]{hsu2021proliferation}
Daniel Hsu, Vidya Muthukumar, and Ji~Xu.
\newblock On the proliferation of support vectors in high dimensions.
\newblock In {\em International Conference on Artificial Intelligence and
  Statistics}, pages 91--99. PMLR, 2021.

\bibitem[HWLM20]{haochen2020shape}
Jeff~Z HaoChen, Colin Wei, Jason~D Lee, and Tengyu Ma.
\newblock Shape matters: Understanding the implicit bias of the noise
  covariance.
\newblock {\em arXiv preprint arXiv:2006.08680}, 2020.

\bibitem[JDST20]{ji2020gradient}
Ziwei Ji, Miroslav Dud{\'\i}k, Robert~E Schapire, and Matus Telgarsky.
\newblock Gradient descent follows the regularization path for general losses.
\newblock In {\em Conference on Learning Theory}, pages 2109--2136. PMLR, 2020.

\bibitem[JGH18]{jacot2018neural}
Arthur Jacot, Franck Gabriel, and Cl{\'e}ment Hongler.
\newblock Neural tangent kernel: Convergence and generalization in neural
  networks.
\newblock {\em arXiv preprint arXiv:1806.07572}, 2018.

\bibitem[JLL21]{janner2021offline}
Michael Janner, Qiyang Li, and Sergey Levine.
\newblock Offline reinforcement learning as one big sequence modeling problem.
\newblock {\em Advances in neural information processing systems},
  34:1273--1286, 2021.

\bibitem[JSL22]{jelassi2022vision}
Samy Jelassi, Michael~Eli Sander, and Yuanzhi Li.
\newblock Vision transformers provably learn spatial structure.
\newblock In Alice~H. Oh, Alekh Agarwal, Danielle Belgrave, and Kyunghyun Cho,
  editors, {\em Advances in Neural Information Processing Systems}, 2022.

\bibitem[JST21]{ji2021fast}
Ziwei Ji, Nathan Srebro, and Matus Telgarsky.
\newblock Fast margin maximization via dual acceleration.
\newblock In {\em International Conference on Machine Learning}, pages
  4860--4869. PMLR, 2021.

\bibitem[JT18]{ji2018risk}
Ziwei Ji and Matus Telgarsky.
\newblock Risk and parameter convergence of logistic regression.
\newblock {\em arXiv preprint arXiv:1803.07300}, 2018.

\bibitem[JT19]{ji2019implicit}
Ziwei Ji and Matus Telgarsky.
\newblock The implicit bias of gradient descent on nonseparable data.
\newblock In {\em Conference on Learning Theory}, pages 1772--1798. PMLR, 2019.

\bibitem[JT20]{ji2020directional}
Ziwei Ji and Matus Telgarsky.
\newblock Directional convergence and alignment in deep learning.
\newblock In H.~Larochelle, M.~Ranzato, R.~Hadsell, M.~F. Balcan, and H.~Lin,
  editors, {\em Advances in Neural Information Processing Systems}, volume~33,
  pages 17176--17186. Curran Associates, Inc., 2020.

\bibitem[JT21]{ji2021characterizing}
Ziwei Ji and Matus Telgarsky.
\newblock Characterizing the implicit bias via a primal-dual analysis.
\newblock In {\em Algorithmic Learning Theory}, pages 772--804. PMLR, 2021.

\bibitem[KPOT21]{kini2021label}
Ganesh~Ramachandra Kini, Orestis Paraskevas, Samet Oymak, and Christos
  Thrampoulidis.
\newblock Label-imbalanced and group-sensitive classification under
  overparameterization.
\newblock {\em Advances in Neural Information Processing Systems},
  34:18970--18983, 2021.

\bibitem[KT19]{kenton2019bert}
Jacob Devlin Ming-Wei~Chang Kenton and Lee~Kristina Toutanova.
\newblock Bert: Pre-training of deep bidirectional transformers for language
  understanding.
\newblock In {\em Proceedings of NAACL-HLT}, pages 4171--4186, 2019.

\bibitem[LFS{\etalchar{+}}17]{lin2017structured}
Zhouhan Lin, Minwei Feng, Cicero Nogueira~dos Santos, Mo~Yu, Bing Xiang, Bowen
  Zhou, and Yoshua Bengio.
\newblock A structured self-attentive sentence embedding.
\newblock In {\em International Conference on Learning Representations}, 2017.

\bibitem[LIPO23]{li2023transformers}
Yingcong Li, M~Emrullah Ildiz, Dimitris Papailiopoulos, and Samet Oymak.
\newblock Transformers as algorithms: Generalization and stability in
  in-context learning.
\newblock In {\em International Conference on Machine Learning}, 2023.

\bibitem[LL18]{li2018learning}
Yuanzhi Li and Yingyu Liang.
\newblock Learning overparameterized neural networks via stochastic gradient
  descent on structured data.
\newblock {\em Advances in neural information processing systems}, 31, 2018.

\bibitem[LL19]{lyu2019gradient}
Kaifeng Lyu and Jian Li.
\newblock Gradient descent maximizes the margin of homogeneous neural networks.
\newblock {\em arXiv preprint arXiv:1906.05890}, 2019.

\bibitem[LLC{\etalchar{+}}21]{liu2021swin}
Ze~Liu, Yutong Lin, Yue Cao, Han Hu, Yixuan Wei, Zheng Zhang, Stephen Lin, and
  Baining Guo.
\newblock Swin transformer: Hierarchical vision transformer using shifted
  windows.
\newblock In {\em Proceedings of the IEEE/CVF International Conference on
  Computer Vision}, pages 10012--10022, 2021.

\bibitem[LMZ18]{li2018algorithmic}
Yuanzhi Li, Tengyu Ma, and Hongyang Zhang.
\newblock Algorithmic regularization in over-parameterized matrix sensing and
  neural networks with quadratic activations.
\newblock In {\em Conference On Learning Theory}, pages 2--47. PMLR, 2018.

\bibitem[LR20]{liang2020just}
TENGYUAN LIANG and ALEXANDER RAKHLIN.
\newblock Just interpolate: Kernel “ridgeless” regression can generalize.
\newblock {\em The Annals of Statistics}, 48(3):1329--1347, 2020.

\bibitem[LWA22]{li2022what}
Zhiyuan Li, Tianhao Wang, and Sanjeev Arora.
\newblock What happens after {SGD} reaches zero loss? --a mathematical
  framework.
\newblock In {\em International Conference on Learning Representations}, 2022.

\bibitem[LWLC23]{li2023theoretical}
Hongkang Li, Meng Wang, Sijia Liu, and Pin-Yu Chen.
\newblock A theoretical understanding of shallow vision transformers: Learning,
  generalization, and sample complexity.
\newblock {\em arXiv preprint arXiv:2302.06015}, 2023.

\bibitem[LWM19]{li2019towards}
Yuanzhi Li, Colin Wei, and Tengyu Ma.
\newblock Towards explaining the regularization effect of initial large
  learning rate in training neural networks.
\newblock {\em arXiv preprint arXiv:1907.04595}, 2019.

\bibitem[MNS{\etalchar{+}}21]{muthukumar2021classification}
Vidya Muthukumar, Adhyyan Narang, Vignesh Subramanian, Mikhail Belkin, Daniel
  Hsu, and Anant Sahai.
\newblock Classification vs regression in overparameterized regimes: Does the
  loss function matter?
\newblock {\em The Journal of Machine Learning Research}, 22(1):10104--10172,
  2021.

\bibitem[MRG{\etalchar{+}}20]{merrill2020effects}
William Merrill, Vivek Ramanujan, Yoav Goldberg, Roy Schwartz, and Noah Smith.
\newblock Effects of parameter norm growth during transformer training:
  Inductive bias from gradient descent.
\newblock {\em arXiv preprint arXiv:2010.09697}, 2020.

\bibitem[MWG{\etalchar{+}}20]{moroshko2020implicit}
Edward Moroshko, Blake~E Woodworth, Suriya Gunasekar, Jason~D Lee, Nati Srebro,
  and Daniel Soudry.
\newblock Implicit bias in deep linear classification: Initialization scale vs
  training accuracy.
\newblock {\em Advances in neural information processing systems},
  33:22182--22193, 2020.

\bibitem[NLG{\etalchar{+}}19]{nacson2019convergence}
Mor~Shpigel Nacson, Jason Lee, Suriya Gunasekar, Pedro Henrique~Pamplona
  Savarese, Nathan Srebro, and Daniel Soudry.
\newblock Convergence of gradient descent on separable data.
\newblock In {\em The 22nd International Conference on Artificial Intelligence
  and Statistics}, pages 3420--3428. PMLR, 2019.

\bibitem[NLL{\etalchar{+}}23]{noci2023shaped}
Lorenzo Noci, Chuning Li, Mufan~Bill Li, Bobby He, Thomas Hofmann, Chris
  Maddison, and Daniel~M Roy.
\newblock The shaped transformer: Attention models in the infinite
  depth-and-width limit.
\newblock {\em arXiv preprint arXiv:2306.17759}, 2023.

\bibitem[NNH{\etalchar{+}}23]{nguyen2023primal}
Tan~Minh Nguyen, Tam~Minh Nguyen, Nhat Ho, Andrea~L Bertozzi, Richard Baraniuk,
  and Stanley Osher.
\newblock A primal-dual framework for transformers and neural networks.
\newblock In {\em The Eleventh International Conference on Learning
  Representations}, 2023.

\bibitem[OH10]{oymak2010new}
Samet Oymak and Babak Hassibi.
\newblock New null space results and recovery thresholds for matrix rank
  minimization.
\newblock {\em arXiv preprint arXiv:1011.6326}, 2010.

\bibitem[OMFH11]{oymak2011simplified}
Samet Oymak, Karthik Mohan, Maryam Fazel, and Babak Hassibi.
\newblock A simplified approach to recovery conditions for low rank matrices.
\newblock In {\em 2011 IEEE International Symposium on Information Theory
  Proceedings}, pages 2318--2322. IEEE, 2011.

\bibitem[{Ope}22]{openai_chatgpt}
{OpenAI}.
\newblock {OpenAI: Introducing ChatGPT}.
\newblock \url{https://openai.com/blog/chatgpt}, 2022.

\bibitem[Ope23]{gpt4}
OpenAI.
\newblock Gpt-4 technical report.
\newblock {\em arXiv preprint arXiv:2303.08774}, 2023.

\bibitem[ORST23]{oymak2023role}
Samet Oymak, Ankit~Singh Rawat, Mahdi Soltanolkotabi, and Christos
  Thrampoulidis.
\newblock On the role of attention in prompt-tuning.
\newblock In {\em International Conference on Machine Learning}, 2023.

\bibitem[OS19]{oymak2019overparameterized}
Samet Oymak and Mahdi Soltanolkotabi.
\newblock Overparameterized nonlinear learning: Gradient descent takes the
  shortest path?
\newblock In {\em International Conference on Machine Learning}, pages
  4951--4960. PMLR, 2019.

\bibitem[PHD20]{papyan2020prevalence}
Vardan Papyan, XY~Han, and David~L Donoho.
\newblock Prevalence of neural collapse during the terminal phase of deep
  learning training.
\newblock {\em Proceedings of the National Academy of Sciences},
  117(40):24652--24663, 2020.

\bibitem[PTDU16]{parikh2016decomposable}
Ankur Parikh, Oscar T{\"a}ckstr{\"o}m, Dipanjan Das, and Jakob Uszkoreit.
\newblock A decomposable attention model for natural language inference.
\newblock In {\em Proceedings of the 2016 Conference on Empirical Methods in
  Natural Language Processing}, pages 2249--2255, Austin, Texas, November 2016.
  Association for Computational Linguistics.

\bibitem[PXS18]{paulus2017deep}
Romain Paulus, Caiming Xiong, and Richard Socher.
\newblock A deep reinforced model for abstractive summarization.
\newblock In {\em International Conference on Learning Representations}, 2018.

\bibitem[QQ19]{qian2019implicit}
Qian Qian and Xiaoyuan Qian.
\newblock The implicit bias of adagrad on separable data.
\newblock {\em Advances in Neural Information Processing Systems}, 32, 2019.

\bibitem[RFP10]{recht2010guaranteed}
Benjamin Recht, Maryam Fazel, and Pablo~A Parrilo.
\newblock Guaranteed minimum-rank solutions of linear matrix equations via
  nuclear norm minimization.
\newblock {\em SIAM review}, 52(3):471--501, 2010.

\bibitem[RSR{\etalchar{+}}20]{raffel2020exploring}
Colin Raffel, Noam Shazeer, Adam Roberts, Katherine Lee, Sharan Narang, Michael
  Matena, Yanqi Zhou, Wei Li, and Peter~J Liu.
\newblock Exploring the limits of transfer learning with a unified text-to-text
  transformer.
\newblock {\em Journal of Machine Learning Research}, 21(1):5485--5551, 2020.

\bibitem[RXH11]{recht2011null}
Benjamin Recht, Weiyu Xu, and Babak Hassibi.
\newblock Null space conditions and thresholds for rank minimization.
\newblock {\em Mathematical programming}, 127:175--202, 2011.

\bibitem[RZH03]{rosset2003margin}
Saharon Rosset, Ji~Zhu, and Trevor Hastie.
\newblock Margin maximizing loss functions.
\newblock {\em Advances in neural information processing systems}, 16, 2003.

\bibitem[SATA22]{sun2022mirror}
Haoyuan Sun, Kwangjun Ahn, Christos Thrampoulidis, and Navid Azizan.
\newblock Mirror descent maximizes generalized margin and can be implemented
  efficiently.
\newblock {\em Advances in Neural Information Processing Systems},
  35:31089--31101, 2022.

\bibitem[SEO{\etalchar{+}}22]{sahiner2022unraveling}
Arda Sahiner, Tolga Ergen, Batu Ozturkler, John Pauly, Morteza Mardani, and
  Mert Pilanci.
\newblock Unraveling attention via convex duality: Analysis and interpretations
  of vision transformers.
\newblock In {\em International Conference on Machine Learning}, pages
  19050--19088. PMLR, 2022.

\bibitem[SHN{\etalchar{+}}18]{soudry2018implicit}
Daniel Soudry, Elad Hoffer, Mor~Shpigel Nacson, Suriya Gunasekar, and Nathan
  Srebro.
\newblock The implicit bias of gradient descent on separable data.
\newblock {\em The Journal of Machine Learning Research}, 19(1):2822--2878,
  2018.

\bibitem[SPR18]{suggala2018connecting}
Arun Suggala, Adarsh Prasad, and Pradeep~K Ravikumar.
\newblock Connecting optimization and regularization paths.
\newblock {\em Advances in Neural Information Processing Systems}, 31, 2018.

\bibitem[SRJ04]{srebro2004maximum}
Nathan Srebro, Jason Rennie, and Tommi Jaakkola.
\newblock Maximum-margin matrix factorization.
\newblock {\em Advances in neural information processing systems}, 17, 2004.

\bibitem[SS21]{stoger2021small}
Dominik St{\"o}ger and Mahdi Soltanolkotabi.
\newblock Small random initialization is akin to spectral learning:
  Optimization and generalization guarantees for overparameterized low-rank
  matrix reconstruction.
\newblock {\em Advances in Neural Information Processing Systems},
  34:23831--23843, 2021.

\bibitem[TBS{\etalchar{+}}16]{tu2016low}
Stephen Tu, Ross Boczar, Max Simchowitz, Mahdi Soltanolkotabi, and Ben Recht.
\newblock Low-rank solutions of linear matrix equations via procrustes flow.
\newblock In {\em International Conference on Machine Learning}, pages
  964--973. PMLR, 2016.

\bibitem[TCD{\etalchar{+}}21]{touvron2021training}
Hugo Touvron, Matthieu Cord, Matthijs Douze, Francisco Massa, Alexandre
  Sablayrolles, and Herv{\'e} J{\'e}gou.
\newblock Training data-efficient image transformers \& distillation through
  attention.
\newblock In {\em International Conference on Machine Learning}, pages
  10347--10357. PMLR, 2021.

\bibitem[TG07]{tropp2007signal}
Joel~A Tropp and Anna~C Gilbert.
\newblock Signal recovery from random measurements via orthogonal matching
  pursuit.
\newblock {\em IEEE Transactions on information theory}, 53(12):4655--4666,
  2007.

\bibitem[Tib96]{tibshirani1996regression}
Robert Tibshirani.
\newblock Regression shrinkage and selection via the lasso.
\newblock {\em Journal of the Royal Statistical Society Series B: Statistical
  Methodology}, 58(1):267--288, 1996.

\bibitem[TKVB22]{thrampoulidis2022imbalance}
Christos Thrampoulidis, Ganesh~Ramachandra Kini, Vala Vakilian, and Tina
  Behnia.
\newblock Imbalance trouble: Revisiting neural-collapse geometry.
\newblock {\em Advances in Neural Information Processing Systems},
  35:27225--27238, 2022.

\bibitem[TLI{\etalchar{+}}23]{touvron2023llama}
Hugo Touvron, Thibaut Lavril, Gautier Izacard, Xavier Martinet, Marie-Anne
  Lachaux, Timoth{\'e}e Lacroix, Baptiste Rozi{\`e}re, Naman Goyal, Eric
  Hambro, Faisal Azhar, et~al.
\newblock Llama: Open and efficient foundation language models.
\newblock {\em arXiv preprint arXiv:2302.13971}, 2023.

\bibitem[TLZO23]{tarzanagh2023margin}
Davoud~Ataee Tarzanagh, Yingcong Li, Xuechen Zhang, and Samet Oymak.
\newblock Margin maximization in attention mechanism.
\newblock {\em arXiv preprint arXiv:2306.13596}, 2023.

\bibitem[TVS23]{timor2023implicit}
Nadav Timor, Gal Vardi, and Ohad Shamir.
\newblock Implicit regularization towards rank minimization in relu networks.
\newblock In {\em International Conference on Algorithmic Learning Theory},
  pages 1429--1459. PMLR, 2023.

\bibitem[TWCD23]{tian2023scan}
Yuandong Tian, Yiping Wang, Beidi Chen, and Simon Du.
\newblock Scan and snap: Understanding training dynamics and token composition
  in 1-layer transformer.
\newblock {\em arXiv:2305.16380}, 2023.

\bibitem[VKR19]{vaskevicius2019implicit}
Tomas Vaskevicius, Varun Kanade, and Patrick Rebeschini.
\newblock Implicit regularization for optimal sparse recovery.
\newblock {\em Advances in Neural Information Processing Systems},
  32:2972--2983, 2019.

\bibitem[VSP{\etalchar{+}}17]{vaswani2017attention}
Ashish Vaswani, Noam Shazeer, Niki Parmar, Jakob Uszkoreit, Llion Jones,
  Aidan~N Gomez, {\L}ukasz Kaiser, and Illia Polosukhin.
\newblock Attention is all you need.
\newblock {\em Advances in neural information processing systems}, 30, 2017.

\bibitem[WGL{\etalchar{+}}20]{woodworth2020kernel}
Blake Woodworth, Suriya Gunasekar, Jason~D Lee, Edward Moroshko, Pedro
  Savarese, Itay Golan, Daniel Soudry, and Nathan Srebro.
\newblock Kernel and rich regimes in overparametrized models.
\newblock In {\em Conference on Learning Theory}, pages 3635--3673. PMLR, 2020.

\bibitem[WMCL21]{wang2021implicit}
Bohan Wang, Qi~Meng, Wei Chen, and Tie-Yan Liu.
\newblock The implicit bias for adaptive optimization algorithms on homogeneous
  neural networks.
\newblock In {\em International Conference on Machine Learning}, pages
  10849--10858. PMLR, 2021.

\bibitem[WMZ{\etalchar{+}}21]{wang2021momentum}
Bohan Wang, Qi~Meng, Huishuai Zhang, Ruoyu Sun, Wei Chen, and Zhi-Ming Ma.
\newblock Momentum doesn't change the implicit bias.
\newblock {\em arXiv preprint arXiv:2110.03891}, 2021.

\bibitem[WT22]{wang2022binary}
Ke~Wang and Christos Thrampoulidis.
\newblock Binary classification of gaussian mixtures: Abundance of support
  vectors, benign overfitting, and regularization.
\newblock {\em SIAM Journal on Mathematics of Data Science}, 4(1):260--284,
  2022.

\bibitem[WWX{\etalchar{+}}22]{wu2022flowformer}
Haixu Wu, Jialong Wu, Jiehui Xu, Jianmin Wang, and Mingsheng Long.
\newblock Flowformer: Linearizing transformers with conservation flows.
\newblock In {\em International Conference on Machine Learning}, pages
  24226--24242, 2022.

\bibitem[YKM20]{yun2020unifying}
Chulhee Yun, Shankar Krishnan, and Hossein Mobahi.
\newblock A unifying view on implicit bias in training linear neural networks.
\newblock {\em arXiv preprint arXiv:2010.02501}, 2020.

\bibitem[ZFB23]{zhang2023trained}
Ruiqi Zhang, Spencer Frei, and Peter~L Bartlett.
\newblock Trained transformers learn linear models in-context.
\newblock {\em arXiv preprint arXiv:2306.09927}, 2023.

\bibitem[ZWB{\etalchar{+}}21]{zou2021benefits}
Difan Zou, Jingfeng Wu, Vladimir Braverman, Quanquan Gu, Dean~P Foster, and
  Sham Kakade.
\newblock The benefits of implicit regularization from sgd in least squares
  problems.
\newblock {\em Advances in Neural Information Processing Systems},
  34:5456--5468, 2021.

\bibitem[ZY05]{zhang2005boosting}
Tong Zhang and Bin Yu.
\newblock Boosting with early stopping: Convergence and consistency.
\newblock {\em Annals of Statistics}, page 1538, 2005.

\end{thebibliography}
\newpage
\appendix

\paragraph{Roadmap.} The appendix is organized as follows:

\begin{itemize}
\item Appendix \ref{app sep} provides the proof of Theorem \ref{thm:separation}. 
\item Appendix~\ref{app:sec:aux} provides auxiliary lemmas about the training risk. 
\item Appendix~\ref{app:sec:gd:global} presents the proofs for the global convergence of gradient descent (Section \ref{provable global}). 
\item Appendix \ref{app local proofs} presents the proofs for the local convergence of gradient descent (Section \ref{sec local}). 
\item Appendix~\ref{sec multioutput} provides a general regularization path analysis. This analysis addresses the inductive bias of the attention layer for general norm objectives and beyond-linear prediction heads under a sequence-to-sequence classification model. The seq2seq aspect also goes beyond our results in the main body where we predict using single output token (Sections \ref{sec:bias} and \ref{sec:local reg path}).
\item Appendix \ref{app supp exp} provides additional experiments and their discussion.

\end{itemize}

\addtocontents{toc}{\protect\setcounter{tocdepth}{3}}
\tableofcontents
\section{Proof of Theorem \ref{thm:separation}: Separability Under Mild Over-Parameterization}\label{app sep}


We denote Kronecker product of two matrices via $\kron$. Additionally, given $\W\in\R^{d\times d}$, let us denote its vectorization $\w=\text{vec}(\W)\in\R^{d^2}$. We first note that separation is implied by the linear independence of the constraints. Specifically, we are interested in guaranteeing
\[
\li\w,\fb_\itt\ri\geq 1\quad \text{for all}\quad i\in[n],~~t\neq\bal_i, ~~ \textnormal{where} ~~\fb_\itt:=(\x_{i\bal_i}-\x_\itt)\kron \z_i.
\]
Note that, the inequality constraints above are feasible as soon as $\fb_\itt$'s are linearly independent. Thus, we will instead prove linear independence of the vectors $(\fb_\itt)_{i\in[n],t\neq\bal_i}$. Also note that, since there are finitely many $\bal$ choices, if we show almost sure separation for a fixed but arbitrary $\bal$ choice, through union bound, we recover the result for all $\bal$. Thus, we prove the result for a fixed $\bal$ choice.

We will prove this result inductively.  Let $\M_{n-1}\in\R^{(n-1)(T-1)\times d^2}$ denote the matrix whose rows are given by the features $(\fb_\itt)_{i\in[n-1],t\neq\bal_i}$. Suppose the result is correct for $n-1$, thus, $\M_{n-1}$ is full row-rank almost surely (post random Gaussian perturbation). Now, fix $\M_{n-1}$ and, conditioned on $\M_{n-1}$ being full row-rank, let us show that $\M_n$ is also full row-rank almost surely. To prove this, consider the $n$'th example $(\X_n,\z_n)$. Let $(\g_t)_{t=1}^T,\hb\in\R^d$ be random vectors with i.i.d.~$\Nn(0,\sigma^2)$ entries. Consider the perturbed input $\X'_n \in\R^{T\times d}$ with tokens $\x'_{nt}=\x_{nt}+\g_t$ and $\z'_n=\z_n+\hb$. Note that for self-attention, we set $\z_n=\x_{n1}$ and $\hb=\g_1$. From these, create the matrix $\tilde{\M}_n \in \R^{(T-1)\times d^2}$ with rows $(\fb'_{nt})_{t\neq\bal_n}$ where $\fb'_{nt}=(\x'_{n\bal_n}-\x'_{nt})\kron \z'_n$. Observe that $\M_n=\begin{bmatrix}\tilde{\M}_n\\\M_{n-1}\end{bmatrix}$. To conclude with the result, we will apply Lemma \ref{lem add up}. To apply this lemma, we have two claims.

\noindent\textbf{Claim 1:} Let $\bar{\z}_n$ be the projection of $\z'_n$ on the orthogonal complement of $(\z_i)_{i=1}^{n-1}$. Consider the matrix $\bar{\M}_n$ with rows $\bar{\fb}_{nt}=(\x'_{n\bal_n}-\x'_{nt})\kron \bar{\z}_n$ for $t\neq \bal_n$. $\bar{\M}_n$ is rank $T-1$ almost surely whenever $d\geq \max(T-1,n)$.

To see this claim, first denote the orthogonal complement of the span of the vectors $(\z_i)_{i=1}^{n-1}$ by $Z_{n-1}$. The span of the vectors $(\z_i)_{i=1}^{n-1}$ is at most $n-1$ dimensional and, since $d\geq n$, $\text{dim}(Z_{n-1})\geq 1$. Consequently, $\bar{\z}_n\neq 0$ almost surely because the Gaussian variable $\z_n+\hb$ will have nonzero projection on $Z_{n-1}$ almost surely. Secondly, let $\bar{\X}\in\R^{(T-1)\times d}$ be the matrix whose rows are equal to $\x'_{n\bal_n}-\x'_{nt}$ for $t\neq \bal_n$. $\bar{\X}$ is full row-rank almost surely, this is because conditioned on $\g_{\bal_n}$, the matrix $\bar{\X}$ is written as $\bar{\X}=\tilde{\X}+\Gb$ where $\tilde{\X}$ is deterministic and $\Gb$ is i.i.d. Gaussian. The latter perturbation ensures full row-rank almost surely whenever $T-1\leq d$. Finally, note that $\bar{\M}_n=\bar{\X}\kron \bar{\z}_n$. Since the rank of the Kronecker product is multiplicative, we conclude with the claim.

\noindent\textbf{Claim 2:} Let $S_{n-1}\subset\R^{d^2}$ be the null space of $\M_{n-1}$. There exists a subspace $P\subseteq S_{n-1}$ such that rows of $\bar{\M}_n$ are projections of the rows of $\M_n$ on $P$, that is, $\bar{\fb}_{nt}=\Pi_{P}(\fb'_{nt})$ where $\Pi$ denotes set projection.

To show this claim, let us consider the matrix forms of the vectorized features i.e.~let us work with $\R^{d\times d}$ rather than $\R^{d^2}$. Denote the notation change as $\Fb_\itt=(\x_{i\bal_i}-\x_\itt) \z_i^\top\leftrightarrow \fb_\itt=(\x_{i\bal_i}-\x_\itt)\kron \z_i$. Recall that $Z_{n-1}$ denotes the orthogonal complement of $(\z_i)_{i=1}^{n-1}$. Define $Q$ to be the set of matrices in $\R^{d\times d}$ whose column space lies in $Z_{n-1}$ and $P$ to be the vectorization of $Q$. We first show that $P$ is a subset of the null space of $S_{n-1}$. To see this, fix any matrix $\A\in P$ and a row $\fb_\itt$ from $\M_{n-1}$. Matricized $\fb_\itt$ can be written as $\Fb_\itt=\ab\z_i^\top$ for $\z_i\in Z_{n-1}^\perp$. Since $\A\in Q$, this implies $\li\Fb_\itt,\A\ri=\ab^\top\A\z_i=0$ as $\A\z_i=0$. This holds for all $\Fb_\itt$, thus, $\texttt{vectorized}(\A)\in\texttt{null}(S_{n-1})$.

Next, we need to show that $\bar{\fb}_{nt}$ is the projection of $\fb'_{nt}$ on $P$. To see this, we will show that $\bar{\fb}_{nt}\in P$ whereas $\fb'_{nt}-\bar{\fb}_{nt}\in P^\perp$ for all $t$. Write $\Fb'_{nt}=\texttt{matricized}(\fb'_{nt})=\ab{\z'}_n^\top$. We have that $\bar{\Fb}_{nt}=\texttt{matricized}(\bar{\fb}_{nt})=\ab\bar{\z}_n^\top$ where $\bar{\z}_n=\Pi_{Z_{n-1}}(\z'_n)$. This implies $\bar{\Fb}_{nt}\in Q$ and $\bar{\fb}_{nt}\in P$. Similarly, since $\z'_n-\bar{\z}_n\in Z_{n-1}^\perp$ which implies $\Fb'_{nt}-\bar{\Fb}_{nt}\in Q^\perp$.  

To conclude with the proof, observe that, through Claims 1 and 2, $\M_n=\begin{bmatrix}\tilde{\M}_n\\\M_{n-1}\end{bmatrix}$ satisfies the requirements of Lemma \ref{lem add up} almost surely, namely, projection of $\tilde{\M}_n$ onto a subset of the null space of $\M_{n-1}$ being full rank. Thus, $\M_n$ is full rank almost surely.
 $\qed$ 
\begin{lemma}\label{lem add up}Let $\A\in\R^{n\times p},\B\in\R^{m\times p}$. Suppose $n+m\leq p$ and $\A$ is full row-rank. Denote the null space of $\A$ by $S_\A^\perp$. Let $P$ be a subspace that is its subset i.e.~$P\subseteq S_\A^\perp$. Let $\B'$ be the matrix obtained by projecting each of row of $\B$ on $P$ and suppose $\B'$ is full rank. Then, the concatenation $\Cb= [\A; \B ]$
is full row-rank.
\end{lemma}
\begin{proof} Let $(\ab_i)_{i=1}^n$, $(\bb_i)_{i=1}^m$, $(\bb'_i)_{i=1}^m$ be the rows of $\A,\B,\B'$, respectively. Suppose the set of rows of $\A$ and $\B$ are linearly dependent. Then, for some $(c_i)_{i=1}^{n},(c'_i)_{i=1}^{m}$ (which are not all-zeros), we have that
\begin{align}
\sum_{i=1}^n c_i\ab_i+\sum_{i=1}^m c'_i\bb_i=0.\label{cs are nonzero}
\end{align}
We now rewrite this as follows to decouple $P$ and $P^\perp$:
\[
\sum_{i=1}^n c_i\ab_i+\sum_{i=1}^m c'_i\bb'_i+\sum_{i=1}^m c'_i(\bb'_i-\bb_i)=0.
\]
Projecting above inequality to $P$, we find that $\sum_{i=1}^m c'_i\bb'_i=0$. Since $(\bb'_i)_{i=1}^m$ are linearly independent, we find $c'_i=0$ for all $i\in[m]$. This implies $\sum_{i=1}^n c_i\ab_i=0$. Since $(\ab_i)_{i=1}^n$ are linearly independent, this implies $c_i=0$ for all $i\in[n]$. Thus, \eqref{cs are nonzero} can only hold if all coefficients are zero which is a contradiction.
\end{proof}

\section{Auxiliary Lemmas}\label{app:sec:aux}

\subsection{Proof of Lemma \ref{lem:rank}}\label{app low rank proof}
Let $\W^\svm_\dm$ denote either solution of \eqref{eqn:sattnsvm} or \eqref{eqn:sattnsvmst}. We claim that $\W^\svm_\dm$ is at most rank $n$.
Suppose the claim is wrong and row space of $\W^\svm_\dm$ does not lie within $\Sc=\texttt{span}(\{\z_i\}_{i=1}^n)$. Let $\W=\Pi_{\Sc}(\W^\svm_\dm)$ denote the matrix obtained by projecting the rows of $\W^\svm_\dm$ on $\Sc$. Observe that $\W$ satisfies all SVM constraints since $\W\z_i=\W^\svm_\dm\z_i$ for all $i\in[n]$. For Frobenius norm, using $\W^\svm_\dm\neq \W$, we obtain a contradiction via $\tf{\W^\svm_\dm}^2=\tf{\W}^2+\tf{\W^\svm_\dm-\W}^2>\tf{\W}^2$. For nuclear norm, we can write $\W=\Ub\bSi\Vb^\top$ with $\bSi\in\R^{r\times r}$ where $r$ is dimension of $\Sc$ and $\texttt{column\_span}(\Vb)= \Sc$. 
\\
To proceed, we split the problem into two scenarios.
\\
\noindent\textbf{Scenario 1:} Let $\Ub_\perp,\Vb_\perp$ be orthogonal complements of $\Ub,\Vb$ -- viewing matrices with orthonormal columns as subspaces. Suppose $\Ub_\perp^\top \W^\svm_\dm\Vb_\perp\neq 0$. Then, singular value inequalities (which were also used in earlier works on nuclear norm analysis \cite{recht2011null,oymak2010new,oymak2011simplified}) guarantee that $\tnuc{\W^\svm_\dm}\geq \tnuc{\Ub^\top \W^\svm_\dm\Vb}+\tnuc{\Ub_\perp^\top \W^\svm_\dm\Vb_\perp}>\tnuc{\W}$.
\\
\noindent\textbf{Scenario 2:} Now suppose $\Ub_\perp^\top \W^\svm_\dm\Vb_\perp= 0$. Since $\W^\svm_\dm\Vb_\perp\neq 0$, this implies $\Ub^\top \W^\svm_\dm\Vb_\perp\neq 0$. Let $\W'=\Ub\Ub^\top\W^\svm_\dm$ which is a rank-$r$ matrix. Since $\W'$ is a subspace projection, we have $\tnuc{\W'}\leq \tnuc{\W^\svm_\dm}$. Next, observe that $\tnuc{\W}=\texttt{trace}(\Ub^\top \W\Vb)=\texttt{trace}(\Ub^\top \W'\Vb)$. On the other hand, $\texttt{trace}(\Ub^\top \W'\Vb)<\tnuc{\W'}$ because the equality in \emph{von Neumann's trace inequality} happens if and only if the two matrices we are inner-producting, namely $(\W',\Ub\Vb^\top)$, share a joint set of singular vectors \cite{carlsson2021neumann}. However, this is not true as the row space of $\W^\svm_\dm$ does not lie within $\Sc$. Thus, we obtain $\tnuc{\W}<\tnuc{\W'}\leq \tnuc{\W^\svm_\dm}$ concluding the proof via contradiction. $\qed$ 

\subsection{Proof of Lemma \ref{lem min risk}}
We first show that  $\Lc(\W)>\Lc_\st=\frac{1}{n}\sum_{i=1}^n \ell(\bgam_{i\op_i})$. The token at the output of the attention layer is given by $\ab_i=\X_i^\top \s_i$, where $\sft{\X_i\W\z_i}=\s_i$. Here, $\ab_i$ can be written as $\ab_i=\sum_{t\in[T]}\s_\itt\x_\itt$, where $\s_\itt\geq 0$ and $\sum_{t\in[T]}\s_\itt=1$. 
To proceed, using the linearity of $h(\x)=\vb^\top\x$, we find that
\begin{align}\label{eqn:w:low}
\nonumber
  \Lc(\W)= \frac{1}{n} \sum_{i=1}^n  \ell(Y_i\cdot  h(\ab_i)) &= \frac{1}{n} \sum_{i=1}^n  \ell( Y_i\cdot  \sum_{t\in[T]}\s_\itt h(\x_\itt))\\
  &\geq  \frac{1}{n} \sum_{i=1}^n  \ell( Y_i\cdot  h(\x_{i\op_i})) =\frac{1}{n} \sum_{i=1}^n  \ell(\bgam_{i\op_i})=\Lc_\st. 
\end{align}
Here, the inequality follows since  $\bgam_{it} = Y_i \cdot h(\x_\itt)=  Y_i 
 \cdot \vb^\top\x_{it} \leq \bgam_{i\op_i}$ by Definition \ref{score def} and strictly-decreasing nature of the loss $\ell$ due to Assumption~\ref{assum:loss:prope}.  

On the other hand, since not all tokens are optimal, there exists a token index $(i,t)$ for which $Y_i\cdot h(\x_\itt)<Y_i\cdot h( \x_{i\op_i})$. Since all softmax entries obey $\s_\itt>0$ for finite $\W$, this implies the strict inequality $\ell(Y_i\cdot h(\ab_i)) >\ell(Y_i\cdot h(\x_{i\op_i}))$. This leads to the desired conclusion $\Lc(\W)>\Lc_\st$.

Next, we show that if \eqref{eqn:sattnsvm} is feasible i.e.~there exists a $\W$ separating some optimal indices $(\op_i)_{i=1}^n$ from the other tokens, then  $\lim_{R\rightarrow\infty}\Lc(R\cdot\W)=\Lc_\st$. Note that, this assumption does not exclude the existence of other optimal indices. This implies that, letting $\lim_{R\rightarrow\infty} \sft{\X_i(R\cdot\W)\z_i}$ saturates the softmax and will be equal to the indicator function at $\op_i$ for all inputs $i\in[n]$. Thus, $\s_\itt\rightarrow 0$ for $t\neq\op_i$ and $\s_\itt\rightarrow 1$ for $t=\op_i$. Using $M_1$-Lipschitzness of $\ell$, we can write 
\[
\left|\ell(Y_i\cdot h(\x_{i\op_i}))-\ell( Y_i\cdot h(\ab_i))\right|\leq M_1 \left|h(\ab_i)-h(\x_{i\op_i})\right|.
\]
Since $h$ is linear, it is $\tn{\vb}$-Lipschitz implying 
\begin{align*}
 \left|\ell(Y_i\cdot h(\x_{i\op_i}))-\ell(Y_i\cdot h(\ab_i))\right|\leq M_1\tn{\vb}\cdot\tn{\ab_i-\x_{i\op_i}}.   
\end{align*}
Since $\ab_i\rightarrow\x_{i\op_i}$ as $R\rightarrow\infty$, \eqref{eqn:w:low} gives $\lim_{R\rightarrow\infty}\Lc(R\cdot\W)=\Lc_\st$.
$\qed$

\subsection{Proof of Lemma~\ref{lem:lip}}
Let
\begin{equation*}
\bgam_i=Y_i\cdot \X_i\vb, \quad 
 \hb_i=\X_i\W \z_{i}.
\end{equation*}
From Assumption~\ref{assum:loss:prope}, $\ell:\R\rightarrow\R$ is differentiable. Hence,  the gradient evaluated at $\W$ is given by
\begin{equation}\label{grad def}
\nabla\Lc(\W)=\frac{1}{n}\sum_{i=1}^n \ell' \left(\bgam_i^\top \sft{\hb_i}\right)\cdot \X_i^\top  \sfp{\hb_i}  \bgam_i  \z_{i}^\top,
\end{equation}
where
\begin{equation}\label{eqn:der:soft}
\sfp{\hb} = \text{diag}\left(\sft{\hb}\right) - \sft{\hb} \sft{\hb}^\top \in \R^{T\times T}.    
\end{equation}
Note that 
\begin{equation}\label{eqn:sprime:bnorm}
 \| \sfp{\hb} \| \leq \| \sfp{\hb} \|_F \leq  1.  
\end{equation}
Hence,  for any $\W,\dot{\W}\in \R^{d\times d}$, $i\in[n]$,  we have
\begin{subequations}
\begin{align}\label{eqn:soft:lipcons1}
\left\|\sft{\hb_i}-\sft{\dot{\hb}_i}\right\| \leq \left\|\hb_i-\dot{\hb}_i\right\| \leq \|\X_i\|~\|\z_i\|~\left\|\W-\dot{\W}\right\|_F,
\end{align}
where $\dot{\hb}_i=\X_i\dot{\W} \z_{i}$.

Similarly, 
\begin{align}\label{eqn:soft:lipcons2}
\nonumber
\left\|\sfp{\hb_i}-\sfp{\dot{\hb}_i}\right\|_F & \leq \left\|\sft{\hb_i} - \sft{\dot{\hb_i}}\right\| +   \left\|\sft{\hb_i} \sft{\hb_i}^\top- \sft{\dot{\hb_i}} \sft{ \dot{\hb_i}}^\top\right\|_F
\\
 & \leq 3 \|\X_i\|~\|\z_i\|~\left\|\W-\dot{\W}\right\|_F.
\end{align}
\end{subequations}
Next, for any $\W,\dot{\W}\in\R^{d\times d}$, we get
\begin{align}\label{eqn:obj:lipcons}
  \nonumber
 \left\|\nabla \mc{L}(\W)-\nabla \mc{L}(\dot{\W})\right\|_F
  &\leq \frac{1}{n}  \sum_{i=1}^n\left\| \ell' \left(\bgam_i^\top \sft{\hb_i}\right) \cdot \z_{i}  \bgam_i^\top \sfp{\hb_i} \X_i - \ell' \left(\bgam_i^\top \sft{\dot{\hb}_i}\right) \cdot \z_{i}  \bgam_i^\top \sfp{\dot{\hb}_i} \X_i \right\|_F\\ 
    \nonumber
       & \le \frac{1}{n}\sum_{i=1}^{n}  \left\|\z_{i}  \bgam_i^\top \sfp{\dot{\hb}_i} \X_i \right\|_F~\left| \ell' \left(\bgam_i^\top \sft{\hb_i}\right) - \ell' \left(\bgam_i^\top \sft{\dot{\hb}_i}\right)  \right| \\
         \nonumber
       &+ \frac{1}{n}\sum_{i=1}^{n} \left|  \ell' \left(\bgam_i^\top \sft{\hb_i}\right)\right|~\left\| \z_{i}  \bgam_i^\top \sfp{\hb_i} \X_i - \z_{i}  \bgam_i^\top \sfp{\dot{\hb}_i} \X_i \right\|_F \\
         \nonumber
       & \le \frac{1}{n}\sum_{i=1}^{n}  M_0 ~\|\bgam_i\|^2~\|\z_i\| ~\|\X_i\|~\left\|\sft{\hb_i}-\sft{\dot{\hb}_i}\right\| \\
       & +   \frac{1}{n}\sum_{i=1}^{n}  M_1  ~\|\bgam_i\|~\|\z_i\|~\|\X_i \|~\left\|\sfp{\hb_i}-\sfp{\dot{\hb}_i}\right\|_F,
\end{align}
where the second inequality follows from the fact that $|ab - cd| \leq |d||a-c|+ |a||b-d|$ and the third inequality uses Assumption~\ref{assum:loss:prope} and \eqref{eqn:sprime:bnorm}.

Substituting \eqref{eqn:soft:lipcons1}  and \eqref{eqn:soft:lipcons2} into \eqref{eqn:obj:lipcons}, we get
\begin{align*}
\left\|\nabla \mc{L}(\W)-\nabla \mc{L}(\dot{\W})\right\|_F &\leq  \frac{1}{n}\sum_{i=1}^{n} \left( M_0  ~\|\bgam_i\|^2\|\z_i\|^2\|\X_i\|^2+ 3  M_1 \|\bgam_i\|~\|\z_i\|^2~\|\X_i \|^2\right)  \|\W-\dot{\W}\|_F\\
 &\leq  \frac{1}{n}\sum_{i=1}^{n} \left( M_0  ~\|\vb\|^2\|\z_i\|^2\|\X_i\|^4+ 3  M_1 \|\vb\|~\|\z_i\|^2~\|\X_i \|^3\right)  \|\W-\dot{\W}\|_F\\
&\leq  L_{\W}~\|\W-\dot{\W}\|_F,
\end{align*}
where $L_{\W}$ is defined in \eqref{eqn:lip:cons:erm}.

Let $\g_{i}=\X_i\Kb \Qb^\top\z_{i}$. We have
\begin{subequations}\label{grad def KQ}
\begin{align}
\nabla_{\Kb} \Lc(\Kb,\Qb)=\frac{1}{n}\sum_{i=1}^n \ell' \left(\bgam_i^\top \sft{\g_i}\right) \cdot  \z_{i}  \bgam_i^\top \sfp{\g_i} \X_i \Qb,\\
\nabla_{\Qb} \Lc(\Kb,\Qb)=\frac{1}{n}\sum_{i=1}^n \ell' \left(\bgam_i^\top \sft{\g_i}\right) \cdot \X_i^\top  \sfp{\g_i}  \bgam_i  \z_{i}^\top \Kb.
\end{align}
\end{subequations}
By the similar argument as in \eqref{eqn:obj:lipcons}, for any $\Qb$ and $\dot\Qb\in\R^{d\times m}$, we have
\begin{align}\label{eqn:objqk:lipcons}
  \nonumber
 \left\|\nabla_{\Qb} \mc{L}(\Kb,\Qb)-\nabla_{\Qb} \mc{L}(\Kb,\dot{\Qb})\right\|_F
  &\leq \frac{\|\Kb\|}{n}  \sum_{i=1}^n\left\| \ell' \left(\bgam_i^\top \sft{\hb_i}\right) \cdot \z_{i}  \bgam_i^\top \sfp{\hb_i} \X_i - \ell' \left(\bgam_i^\top \sft{\dot{\hb}_i}\right) \cdot \z_{i}  \bgam_i^\top \sfp{\dot{\hb}_i} \X_i \right\|_F\\ 
       & \leq L_{\W} \|\Kb\| ~\|\Qb-\dot{\Qb}\|_F.
\end{align}
Similarly, for any $\Kb,\dot\Kb\in\R^{d\times m}$, we get   
$$
\left\|\nabla_{\Kb} \mc{L}(\Kb,\Qb)-\nabla_{\Kb} \mc{L}(\dot{\Kb},\Qb)\right\|_F \leq  L_{\W} \|\Qb\| ~\|\Kb-\dot{\Kb}\|_F.
$$
$\qed$
%
\subsection{A useful lemma for gradient descent analysis}
\begin{lemma}\label{lem:q_reduce} For any $\X \in\R^{T\times d}$, $\W,\V \in \R^{d\times d}$ and $\z, \vb \in \R^{d}$, let $\ab= \X\V \z$, $\s=\sft{\X\W\z}$, and $\bgam=\X\vb$. Set
\begin{equation*}
\Gamma=\sup_{t,\tau\in[T]}|\bgam_t-\bgam_\tau|~~~\textnormal{and}~~~A=\sup_{t\in[T]}\tn{\ab_t}.
\end{equation*}
We have that
  \[
    \left|\ab^\top\textnormal{diag}(\s) \bgam-\ab^\top\s\s^\top\bgam-\sum_{t\geq 2}^T (\ab_1-\ab_t)\s_t(\bgam_1-\bgam_t)\right|\leq 2\Gamma A(1-\s_1)^2.
  \]
\end{lemma}

 \begin{proof}
The proof is similar to \cite[Lemma~4]{tarzanagh2023margin}, but for the sake of completeness, we provide it here.  Set $\gamb=\sum_{t=1}^T \bgam_t\s_t$.  We have 
\begin{align*}
\bgam_1-\gamb=\sum_{t\geq 2}^T (\bgam_1-\bgam_t)\s_t,~~\textnormal{and}~~|\bgam_1-\gamb|\leq \Gamma (1-\s_1).
\end{align*}    
Then,
  \begin{align} 
  \nonumber 
    \ab^\top\diag{\s}\bgam-\ab^\top\s\s^\top\bgam&=\sum_{t=1}^T \ab_t\bgam_t\s_t-\sum_{t=1}^T \ab_t\s_t\sum_{t=1}^T \bgam_t\s_t\\
    &=\ab_1\s_1(\bgam_1-\gamb)-\sum_{t\geq 2}^T\ab_t\s_t(\gamb-\bgam_t). \label{grad def step3}
  \end{align}
Since 
$$
\left|\sum_{t\geq 2}^T\ab_t\s_t(\gamb-\bgam_t)-\sum_{t\geq 2}^T\ab_t\s_t(\bgam_1-\bgam_t)\right|\leq A\Gamma (1-\s_1)^2,
$$
we obtain\footnote{For simplicity, we use $\pm$ on the right hand side to denote the upper and lower bounds.}
  \begin{align*}  
    \ab^\top\diag{\s}\bgam-\ab^\top\s\s^\top\bgam&=\ab_1\s_1(\bgam_1-\gamb)-\sum_{t\geq 2}^T\ab_t\s_t(\bgam_1-\bgam_t)\pm A\Gamma (1-\s_1)^2\\
    &=\ab_1\s_1\sum_{t\geq 2}^T (\bgam_1-\bgam_t)\s_t-\sum_{t\geq 2}^T\ab_t\s_t(\bgam_1-\bgam_t)\pm A\Gamma (1-\s_1)^2\\
    &=\sum_{t\geq 2}^T (\ab_1\s_1-\ab_t)\s_t(\bgam_1-\bgam_t)\pm A\Gamma (1-\s_1)^2\\
    &=\sum_{t\geq 2}^T (\ab_1-\ab_t)\s_t(\bgam_1-\bgam_t)\pm 2A\Gamma (1-\s_1)^2.
  \end{align*}
Here,  $\pm$ on the right handside uses the fact that
  \[
  \left|\sum_{t\geq 2}^T (\ab_1\s_1-\ab_1)\s_t(\bgam_1-\bgam_t)\right|\leq (1-\s_1)A\Gamma\sum_{t\geq 2}^T\s_t=(1-\s_1)^2A\Gamma.
  \]
\end{proof}

\section{Global Convergence of Gradient Descent}\label{app:sec:gd:global}

\subsection{Divergence of norm of the iterates}

The next lemma establishes the descent property of gradient descent for $\mathcal{L}(\W)$ under Assumption \ref{assum:loss:prope}. 

\begin{lemma}[Descent Lemma]\label{lem:grad:descent}
Under Assumption \ref{assum:loss:prope}, if $\eta \leq 1/L_{\W}$, then for any initialization $\W(0)$, Algorithm~\ref{GD-W} satisfies:
\begin{align}\label{eq:descent:obj new}
\mathcal{L}(\W(k+1))-\mathcal{L}(\W(k))\leq-\frac{\eta}{2} \tf{\nabla \mathcal{L}(\W(k))}^2,
\end{align}
for all $k\ge0$. Additionally, it holds that $\sum_{k=0}^{\infty} \tf{\nabla\mathcal{L}\left(\W(k)\right)}^{2}<\infty$, and $\lim_{k\rightarrow \infty}
\tf{\nabla\mathcal{L}\left(\W\left(k\right)\right)}^{2}=0$.
\end{lemma}
\begin{proof}
The proof is similar to \cite[Lemma~6]{tarzanagh2023margin}.
\end{proof}

The lemma below reveals that the correlation between the training loss's gradient at any arbitrary matrix $\W$ and the attention SVM solution  $ \Wm$ is negative. Consequently, for any finite $\W$, $\li\nabla\Lc(\W), \Wm\ri$ cannot be equal to zero.

\begin{lemma}\label{global des lem} 
Let $ \Wm$ be the SVM solution of \eqref{eqn:sattnsvm}. Suppose Assumptions \ref{assum:loss:prope} and \ref{assum:token} hold.  Then,  for all $\W\in\R^{d\times d}$, the training loss \eqref{eqn:erm:w} obeys $\li\nabla\Lc(\W),\Wm\ri<0$. 
\end{lemma}
\begin{proof}
Let
\begin{equation}
\hbm_i=  \X_{i} \Wm \z_i, ~~~\bgam_i=Y_i\cdot \X_i\vb,~~~\textnormal{and}~~~
 \hb_i=\X_i\W \z_{i}.    
\end{equation}
Let us recall the gradient evaluated at $\W$ which is given by 
\begin{align}\label{grad def new}
\nabla\Lc(\W)=\frac{1}{n}\sum_{i=1}^n\ell' \left(\bgam_i^\top \sft{\hb_i}\right) \cdot \X_i^\top  \sfp{\hb_i}  \bgam_i  \z_{i}^\top,
\end{align}
 which implies that 
\begin{equation}\label{eqn:grad:prod:p}
    \begin{split}
\li\nabla\Lc(\W),\Wm\ri&= \frac{1}{n}\sum_{i=1}^n \ell' \left(\bgam_i^\top \sft{\hb_i}\right)  \cdot \iprod{ \X_i^\top  \sfp{\hb_i}  \bgam_i  \z_{i}^\top}{\Wm}\\
&= \frac{1}{n}\sum_{i=1}^n
\ell'_i  \cdot \tr\left(  (\Wm)^\top  \X_{i}^\top \sfp{\hb_i} \bgam_i  \z_{i}^\top\right)\\
&= \frac{1}{n}\sum_{i=1}^n
\ell'_i \cdot \hbm_{i}^\top \sfp{\hb_i} \bgam_i \\
&=  \frac{1}{n}\sum_{i=1}^n
\ell'_i \cdot  \left(\hbm^\top_i\diag{\s_i}\bgam_i-\hbm^\top_i\s_i\s^\top_i\bgam_i\right).        
    \end{split}
\end{equation}
Here, let $\ell'_i:=\ell'(\bgam_i^\top\sft{\hb_i})$, $\s_i=\sft{\hb_i}$ and the third equality uses $\tr\left(\bb\ab^\top\right) = \ab^\top \bb$.

In order to move forward, we will establish the following result, with a focus on the equal score condition (Assumption~\ref{assum:opt:token}): Let $\gamma=\bgam_{t\geq 2}$ be a constant, and let $\bgam_1$ and $\bar{\hb}_1$ represent the largest indices of vectors $\bgam$ and $\hbm$ respectively. For any vector $\s$ that satisfies $\sum_{t\in[T]}\s_t=1$ and $\s_t> 0$, we aim to prove that $\hbm^\top\diag{\s}\bgam-\hbm^\top\s\s^\top\bgam>0$. To demonstrate this, we proceed by writing the following:
\begin{equation}\label{grad def2}
\begin{split}
\hbm^\top\diag{\s}\bgam-\hbm^\top\s\s^\top\bgam&=\sum_{t=1}^T \bar{\hb}_t\bgam_t \s_t-\sum_{t=1}^T  \bar{\hb}_t \s_t\sum_{t=1}^T \bgam_t \s_t\\
&=\left(\bar{\hb}_1\bgam_1\s_1+\gamma\sum_{t\geq 2}^T\bar{\hb}_t\s_t\right)-\Big(\bgam_1\s_1+\gamma(1-\s_1)\Big)\left(\bar{\hb}_1\s_1+\sum_{t\geq 2}^T \bar{\hb}_t\s_t\right)\\
&=\bar{\hb}_1(\bgam_1-\gamma) \s_1(1-\s_1)-(\bgam_1-\gamma)\s_1\sum_{t\geq 2}^T \bar{\hb}_t \s_t\\
&=(\bgam_1-\gamma)(1- \s_1) \s_1\left[\bar{\hb}_1-\frac{ \sum_{t\geq 2}^T \bar{\hb}_t \s_t}{\sum_{t\geq 2}^T\s_t}\right]\\
&\geq(\bgam_1-\gamma)(1- \s_1) \s_1 (\bar{\hb}_1-\max_{t\geq 2}\bar{\hb}_t).
\end{split}
\end{equation}
To proceed, define
\begin{equation*}
\bgag^i=\bgam_{i\opt_i}-\max_{t\neq\opt_i}\bgam_{it}~~\textnormal{and}~~\bhbg^i=\bar \hb_{i\opt_i}-\max_{t\neq\opt_i}\bar \hb_{it}.
\end{equation*}
With these, we obtain 
\begin{equation}\label{eqn:al:lem}
\hbm^\top_i\diag{\s_i}\bgam_i-\hbm^\top_i\s_i\s^\top_i\bgam_i\geq\bgag^i\bhbg^i(1-\s_{i\opt_i})\s_{i\opt_i}.
\end{equation}
Note that 
\begin{equation*}
\begin{split}
& \bhbg^i=\min_{t\neq \opt_i}~(\x_{i \opt_i}-\x_{it})^\top\Wm \z_i \geq1,  \\
&\bgag^i=\min_{t\neq \opt_i}~\bgam_{i\opt_i}-\bgam_{it} >0,\\
&\s_{i\opt_i}(1-\s_{i\opt_i}) > 0.    
\end{split}
\end{equation*}
Hence,
\begin{equation}\label{eqn:lower}
\min_{i \in [n]}\left\{ \left(\min_{t\neq \opt_i}~(\x_{i \opt_i}-\x_{it})^\top\Wm \z_i\right) \cdot \left(\min_{t\neq \opt_i}~\bgam_{i\opt_i}-\bgam_{it}\right) \cdot \s_{i\opt_i}(1-\s_{i\opt_i}) \right\}>0.
\end{equation}
It follows from  \eqref{eqn:al:lem} and \eqref{eqn:lower} that 
\begin{equation}\label{eqn:al:lem:2}
\min_{i \in [n]}\left\{\hbm^\top_i\diag{\s_i}\bgam_i-\hbm^\top_i\s_i\s^\top_i\bgam_i \right\}>0.
\end{equation}
Further, by Assumption \ref{assum:loss:prope}, $\ell'_i<0$, $\ell'$ is continuous and the domain is bounded, the maximum is attained and negative, and thus  
\begin{equation}\label{eqn:bound:lprim}
\max_{x} \ell'(x)<0.    
\end{equation}
Hence, using \eqref{eqn:al:lem:2} and \eqref{eqn:bound:lprim} in  \eqref{eqn:grad:prod:p}, we obtain 
\begin{equation}\label{eqn:grad:prod:p:fin}
    \begin{split}
\li\nabla\Lc(\W),\Wm\ri <0.
    \end{split}
\end{equation}


In the scenario that Assumption~\ref{assum:token:supp} holds (all tokens are support), $\hbm_t=\x_{it}^\top\Wm \z_i $ is constant for all $t\geq 2$. Hence, following similar steps as in \eqref{grad def2} completes the proof. 

\end{proof}

\subsubsection{Proof of Theorem ~\ref{diverg:norm:w}}
It follows from Lemma~\ref{lem:grad:descent} that under Assumption \ref{assum:loss:prope}, $\eta \leq 1/L_{\W}$, and for any initialization $\W(0)$, the gradient descent sequence $\W(k+1)=\W(k)-\eta\nabla \mathcal{L}(\W(k))$ satisfies $\lim_{k\rightarrow \infty}
\tf{\nabla\mathcal{L}\left(\W\left(k\right)\right)}^{2}=0$.  

Further,  it follows from Lemma~\ref{global des lem} that $\li\nabla\Lc(\W), \Wm\ri <0$  for all $\W\in\R^{d\times d}$. Hence, for any finite $\W$, $\li\nabla\Lc(\W), \Wm\ri$ cannot be equal to zero.  Therefore, there are no finite critical points $\W$, for which $\nabla \mc{L} (\W)=0$ which contradicts Lemma~\ref{lem:grad:descent}. This
implies that $\left\Vert \W\left(k\right)\right\Vert \rightarrow\infty$. 
\vspace{.1cm}
$\qed$

\subsection{Global convergence under good initial gradient}\label{app B4} To ensure global convergence, we identify an assumption that prevents GD from getting trapped at suboptimal tokens that offer no scoring advantage compared to other choices. To establish a foundation for providing the convergence of GD to the globally optimal solution $\Ws$, we present the following definitions.  For parameters $\mu \in (0,1)$ and $R>0$, consider the following subset of the sphere and its associated cone:
\begin{subequations}\label{eqn:con:nabla0}
\begin{align}
&\Scc_{\mu} (\Ws):=\left\{\W \in \mathbb{R}^{d \times d}~\Big|~  \li(\x_{i\op_i}-\x_\itt)\z_i^\top, \frac{\W}{\tf{\W}}\ri\geq \frac{\mu}{\tf{\Ws}}\quad \textnormal{for all}\quad t\neq \op_i, \quad  i\in[n]\right\},\\
&\conb_{\mu,R}(\Ws):=\left\{  \W\in\Scc_\mu (\Ws) ~\Big|~   \tf{\W}\geq R\right\}.
\end{align}
\end{subequations}
Note that the $\conb_{\mu,R}(\Ws)$ definition is equivalent to the $\conb_{\mu,R}$ definition in \eqref{eqn:con:nabla0:main} with a change of variable $\mu\gets\tf{\Wm}\cdot \mu$.

\begin{lemma}
\label{glocal cond} 
Suppose Assumption~\ref{assum:loss:prope} holds and let $\op=(\op_i)_{i=1}^n$ be the unique globally-optimal indices with $\Wm$ denoting the \ref{eqn:sattnsvm} solution. Define the margin $\Theta=1/\tf{\Ws}$. Let $\s_{i}=\sft{\X_i\W\z_i}$. For any $\mu>0$, there exists a sufficiently large $\RR_\mu=\order{1/\mu}$ (see \eqref{R bound2}) such that:
\begin{enumerate}[label={\textnormal{\textbf{L\arabic*.}}}, wide, labelwidth=!,itemindent=!, labelindent=5pt]
\item \label{lem:gcond:l1} There is no stationary point within  $ \conb_{\mu,\RR_\mu}(\Ws)$, where $\conb_{\mu,\RR_\mu} (\Ws)$ is defined in \eqref{eqn:con:nabla0}. 
\item\label{lem:gcond:l2} For all $\V\in \Scc_{\mu} (\Ws)$ with $\tf{\V}=\tf{\Wm}$  and $\W\in\conb_{\mu,\RR_\mu}(\Wm)$, there exist dataset dependent constants $C,c>0$ such that 
\begin{subequations}\label{zero:g:lbound}
\begin{align}
&C\cdot \frac{1}{n}\sum_{i=1}^n \left(1-\s_{i\op_i}\right) \geq -\Big\langle\nabla\Lc(\W),\V \Big\rangle\geq c\cdot \mu\cdot  \frac{1}{n} \sum_{i=1}^n  \left(1-\s_{i\op_i}\right)>0, \label{zero1:g:bound} \\
& -\li\frac{\V}{\tf{\V}},\frac{\nabla\Lc(\W)}{\tf{\nabla\Lc(\W)}}\ri \geq  \frac{c}{C} \cdot \frac{\Theta}{\bar{A}}>0, \label{zero2:g:bound}\\
&\tf{\nabla\Lc(\W)}\leq \bar{A}C \cdot \frac{1}{n} \sum_{i=1}^n  \left(1-\s_{i\op_i}\right). \label{zero3:g:bound}
\end{align}
\end{subequations}
Here,  $\s_{i\opt_i}=(\sft{\X_i\W \z_{i}})_{\opt_i}$, $\bar{A}=\max_{i\in[n],t,\tau\in[T]}\tn{\x_{it}- \x_{i\tau}}~\tn{\z_i}$, and $\Theta=1/\tf{\Ws}$.
 
\end{enumerate}
\end{lemma}
\begin{proof} For simplicity let $R=\RR_\mu$, $\W\in\conb_{\mu,R}(\Ws)$ and 
\begin{equation}\label{mu choice2}
\begin{split}
&A=\max_{i\in[n],t,\tau\in[T]} \frac{(\tn{\x_{it}}\vee\tn{\x_{it}-\x_{i\tau}})\cdot\tn{\z_i}}{\Theta}.
\end{split}
\end{equation}
The following inequalities hold for all $\V\in \Scc_{\mu},~\tf{\V}=\tf{\Wm}$ and all $i\in[n], t\neq \op_i$:
\begin{equation}\label{cone-A-eq}
\begin{split}
A\geq(\x_{i\op_i}-\x_{it})^\top \V \z_i &\geq \mu.
\end{split}
\end{equation}
To proceed, we write the gradient correlation following \eqref{grad def} and \eqref{eqn:grad:prod:p}
\begin{align}\label{grad def32}
\li\nabla\Lc(\W),\V\ri&=\frac{1}{n}\sum_{i=1}^n\ell'_i\cdot\li\hb_i,\sfp{\hp_i}\bgam_i\ri,
\end{align}
where we denoted $\ell'_i=\ell'(Y_i\cdot \vb^\top \X_i^\top\sft{\hp_i})$, $\hb_i=\X_i\V \z_{i}$, $\hp_i= \X_i\W \z_{i}$, $\s_i=\sft{\hp_i}$. 

It follows from \eqref{mu choice2} that $A\geq \max_{i\in[n],t\in[T]}\tn{\hb_{it}}$. Using \eqref{cone-A-eq}, we can bound the softmax probabilities $\s_i=\sft{\hp_i}$ as follows, for all $i\in[n]$:
\begin{align}\label{soft prob bound2}
&S_i:= \sum_{\tau\neq \op_i}\s_{i\tau}\leq T e^{-R\mu\Theta}\s_{i\op_i}\leq T e^{-R\mu\Theta}.
\end{align}
Recall scores $\bgam_{it}=Y_i\cdot\vb^\top \x_{it}$. Define the score gaps:
\begin{equation*}
 \bgg_i=\bgam_{i\op_i}-\max_{t\neq\op_i}\bgam_{it},~~~ \bgm_i=\bgam_{i\op_i}-\min_{t\neq\op_i}\bgam_{it},~~~\textnormal{and}~~~\Gamma=\sup_{i\in[n],t,\tau\in[T]}|\bgam_{it}-\bgam_{i\tau}|. 
\end{equation*}
Let us focus on a fixed datapoint $i\in[n]$, assume (without losing generality) $\op_i=1$, and drop subscripts $i$.
Directly applying Lemma \ref{lem:q_reduce}, we obtain
\[
  \big|\hb^\top\diag{\s}\bgam-\hb^\top\s\s^\top\bgam-\sum_{t\geq 2}^T (\hb_1-\hb_t)\s_t(\bgam_1-\bgam_t)\big|\leq 2\Gamma A(1-\s_1)^2.
\]
\noindent To proceed, let us upper/lower bound the gradient correlation. Since $A\geq \hb_1-\hb_t\geq \mu>0$ from \eqref{cone-A-eq}, setting $S:=\sum_{t\neq\op_i}\s_t=1-\s_1$, we find
\begin{equation}
 A\cdot S\cdot \bgm  \geq\sum_{t\neq\op} (\hb_1-\hb_t)\s_t(\bgam_1-\bgam_t)\geq \mu\cdot S\cdot \bgg.\label{aggregate2}
\end{equation}
Next we show that $S=1-\s_1$ dominates $(1-\s_1)^2=S^2$ for large $R$. Specifically, we wish for 
\begin{align}\label{wishfor2}
\mu S \bgg/2\geq 2\Gamma A(1-\s_1)^2\iff S\geq \frac{4}{\mu}\frac{\Gamma A}{\bgg}S^2\iff S\leq \frac{\mu\bgg}{4\Gamma A}.
\end{align}
Using \eqref{soft prob bound2}, what we wish is ensured for all $i\in[n]$, by guaranteeing $Te^{-R\mu\Theta}\leq \frac{\mu\bgg}{4\Gamma A}$. That is, by choosing
\begin{align}\label{R bound2}
R\geq \frac{1}{\mu\Theta}\log\left(\frac{4T\Gamma A}{\mu\bggm}\right),
\end{align}
where $\bggm=\min_{i\in[n]}\bgg_i$ is the global scalar corresponding to the worst case score gap over all inputs. 

With the above choice of $R$, we guaranteed
\[
  2 A(1-\s_1)\cdot \bgm\geq 2A\cdot S\cdot \bgm \geq \sum_{t\neq\op} (\hb_1-\hb_t)\s_t(\bgam_1-\bgam_t)\geq\frac{\mu\cdot S\cdot \bgg}{2}\geq\frac{\mu(1-\s_1) \bgg}{2},
\]
via \eqref{wishfor2} and \eqref{aggregate2}.

Since this holds over all inputs, going back to the gradient correlation \eqref{grad def32} and averaging above over all inputs $i\in[n]$ and plugging back the indices $i$, we obtain the advertised bound 
\begin{align}\label{pbb corr2}
  \frac{2A}{n}\sum_{i\in [n]} -\ell'_i\cdot S_i\cdot \bgm_i\geq -\li\nabla\Lc(\W),\V\ri\geq \frac{\mu}{2n}\sum_{i\in [n]} -\ell'_i\cdot S_i\cdot \bgg_i.
\end{align}
Let $-\ell'_{\min/\max}$ be the min/max values negative loss derivative admits over the ball $[-A,A]$ and note that $\max_{i\in[n]}\bgm_i>0$ and $\min_{i\in[n]}\bgg_i>0$ are dataset dependent constants. Then, we declare the constants $C=-2A\ell'_{\max}\cdot \max_{i\in[n]}\bgm_i>0,c=-(1/2)\ell'_{\min}\cdot \min_{i\in[n]}\bgg_i>0$ to obtain the bound \eqref{zero1:g:bound}. 

The proof of  \eqref{zero2:g:bound} and \eqref{zero3:g:bound} follows similarly as the proof of Lemma \ref{local cond}. 


\end{proof}

The following lemma shows that as $\pi$ approaches zero, the negative gradient of the loss function at $ \W\in \conb_{\mu,R}(\Wm)$  becomes more correlated with the max-margin solution ($\Ws$) than with $\W$ itself.

\begin{lemma}
\label{lem:glocal:corr} 
Suppose Assumption~\ref{assum:loss:prope} holds and let $\op=(\op_i)_{i=1}^n$ be the unique optimal tokens with $ \Wm$ denoting the SVM solution. Fix any $\mu>0$ (per Lemma \ref{glocal cond}). For any choice of $\pi>0$, there exists $R:=R_{\pi} \geq \bar{R}_\mu$ such that, for any $ \W\in \conb_{\mu,R}(\Wm)$, we have
\[
 \li \nabla\Lc(\W), \frac{\W}{\tf{\W}} \ri\geq (1+\pi)\li \nabla\Lc(\W), \frac{\Wm}{\tf{\Wm}}\ri.
\]
Here, $\conb_{\mu,R}(\Wm)$ is the cone defined at  \eqref{eqn:con:nabla0}.
\end{lemma}
\begin{proof}
Let  $\Wb= \tf{\Wm} \W/\tf{\W}$, $\hb_i=\X_i\Wb \z_{i}$, $\hbm_i= \X_i\Ws \z_{i}$, and $\s_i=\sft{\X_i\W \z_{i}}$. To establish the result, we will prove that, for sufficiently large $R$ and for any $\W\in \conb_{\mu,R}(\Wm)$:
\begin{align}\label{main local cond2}
\nonumber 
\li -\nabla\Lc(\W),\frac{\W}{\tf{\W}}\ri&= -\frac{1}{n}\sum_{i=1}^n\ell'_i \cdot  \li \hb_i, \sfp{\X_i\W \z_{i}}\bgam_i\ri\\
&\leq - \frac{1+\pi}{n}\sum_{i=1}^n\ell'_i \cdot  \li\hbm_i, \sfp{\X_i\W \z_{i}}\bgam_i\ri=(1+\pi)\li-\nabla\Lc(\W), \frac{\ps}{\tf{\Ws}}\ri.
\end{align}
Directly applying Lemma \ref{lem:q_reduce}, for all $\V\in \Scc_\mu$ with $\tf{\V}=\tf{\Wm}$ and $\hp_i=\X_i\V \z_i$, we have found
\begin{align}
  \big|\hp^\top_i\diag{\s_i}\bgam_i-\hp^\top_i\s_i\s^\top_i\bgam_i-\sum_{t\neq\op_i} (\hp_{i1}-\hp_{it})\s_{it}(\bgam_{i1}-\bgam_{it})\big|\leq 2\Gamma A(1-\s_{i1})^2.
\end{align}
Recalling $\hbm_{i1}-\hbm_{it}\geq 1$, we note that $\sum_{t\neq\op_i}\s_{it}(\bgam_{i1}-\bgam_{it})\leq \sum_{t\neq\op_i} (\hbm_{i1}-\hbm_{it})\s_{it}(\bgam_{i1}-\bgam_{it})$. Now plugging in $\hb,\hbm$ in the bound above and assuming $\pi\leq 1$ (w.l.o.g.), \eqref{main local cond2} is implied by the following stronger inequality
\begin{align*}
-\frac{1}{n}&\sum_{i=1}^n\ell'_i \cdot \left(6\Gamma A(1-\s_{i1})^2+ \sum_{t\neq \op_i} (\hb_{i1}-\hb_{it})\s_{it}(\bgam_{i1}-\bgam_{it}) \right)\\
&\leq -\frac{1+\pi}{n}\sum_{i=1}^n\ell'_i  \cdot \sum_{t\neq \op_i} (\hbm_{i1}-\hbm_{it})\s_{it}(\bgam_{i1}-\bgam_{it})\\
&\leq-\frac{1+\pi}{n}\sum_{i=1}^n\ell'_i \cdot \sum_{t\neq \op_i}\s_{it}(\bgam_{i1}-\bgam_{it}).
\end{align*}
First, we claim that $0.5\pi\sum_{t\in \op_i}\s_{it}(\bgam_{i1}-\bgam_{it})\geq 6\Gamma A(1-\s_{i1})^2$ for all $i \in [n]$.  
The proof of this claim directly follows the argument in Lemma~\ref{glocal cond}, 
(namely following \eqref{soft prob bound2}, \eqref{wishfor2}, \eqref{R bound2}) 
we have that $1-\s_{i1}\leq Te^{-R\mu\Theta}$ and $\bgam_{i1}-\bgam_{it}\geq \bggm$ for all $i \in [n]$. This leads to the choice (for $D_0\geq 12$)
\begin{align}
  R\geq R_\pi =\frac{1}{\mu\Theta}\log\left(\frac{D_{0}\cdot T\Gamma A}{\pi\bggm}\right).\label{Rpi choice2}
\end{align}
We shall choose $D_0$ sufficiently large such that $R_{\pi}\geq \bar{R}_{\mu}$, where $\bar{R}_{\mu}$ is defined in Lemma \ref{glocal cond}.

Following this control over the perturbation term $6\Gamma A(1-\s_{i1})^2$, to conclude with the result, what remains is proving the comparison
\begin{align}\label{desired comp2}
-\frac{1}{n} \sum_{i=1}^n\ell'_i \cdot \sum_{t\neq\op_i} (\hb_{i1}-\hb_{it})\s_{it}(\bgam_{i1}-\bgam_{it})\leq - \frac{1+0.5\pi}{n}\sum_{i=1}^n\ell'_i \cdot \sum_{t\neq\op_i}\s_{it}(\bgam_{i1}-\bgam_{it}).
\end{align}
\noindent\textbf{Scenario 1:} $\tf{\Wb-\Wm}\leq \eps=\frac{\pi}{4A\Theta}$ for some $\eps>0$.  In this scenario, for any $ t\neq \op_i$ and $i\in [n ]$, we have
\[
|\hb_{it}-\hbm_{it}|=|\x_{it}^\top (\Wb-\Wm)  \z_{i}|\leq A\Theta\eps=\frac{\pi}{4}.
\]
Consequently, we obtain 
\[
\hb_{i1}-\hb_{it}\leq \hbm_{i1}-\hbm_{it}+2A\Theta\eps= 1+0.5\pi.
\] 
Similarly, $\hb_{i1}-\hb_{it}\geq 1-0.5\pi\geq 0.5$. Since all terms $\hb_{i1}-\hb_{it},\s_{it},\bgam_{i1}-\bgam_{it}$ in \eqref{desired comp2} are nonnegative, we  obtain \eqref{desired comp2}.

\noindent\textbf{Scenario 2:} $\tf{\Wb-\Wm}\geq \eps=\frac{\pi}{4A\Theta}$.  Since $\Wb$ is not max-margin solution, in this scenario, for some $i \in  [n]$, $\nu=\nu(\eps)>0$, and $\tau\neq\op_i$, we have that
\begin{align*}
\hb_{i1}-\hb_{i\tau}\leq 1-2\nu.
\end{align*}
Here $\tau=\arg\max_{\tau\neq\op_i} \x_{i\tau}\Wb \z_i$ denotes the nearest point to $\hb_{i1}$ (along the $\Wb$ direction). Recall that $\s=\sft{  \RR\hb}$, where  $\RR=R\Theta=\tf{\W}/\tf{\Wm}$. To proceed, let $ \underline{\hb}_i:=\min_{t \neq\op_i}\hb_{i1}-\hb_{it}$,
\begin{align*}
\mc{I}:=\left\{ i\in[n]: \underline{\hb}_i \leq 1-2\nu \right\}, \qquad [n]-\mc{I}:=\left\{ i\in[n]:  1-2\nu  <  \underline{\hb}_i\right\}.
\end{align*}
For all $ i \in [n]-\mc{I}$,
\begin{equation}\label{eqn:grad:difff0}
\begin{split}
      \sum_{t\neq \op_i} (\hb_{i1}-\hb_{it})\s_{it}(\bgam_{i1}-\bgam_{it}) &- (1+0.5\pi) \sum_{t\neq \op_i}\s_{it}(\bgam_{i1}-\bgam_{it})\\
      & \leq  \left(2A - (1+0.5\pi)\right)\Gamma\sum_{t\neq \op_i,~\hb_{i1}-\hb_{it} \geq 1+\frac{\pi}{2} } \s_{it} \\
      & \leq  \left(2A - (1+0.5\pi)\right)\Gamma Te^{-\RR(1+\frac{\pi}{2})} \\
      &\leq   2A\Gamma  T e^{-\RR(1+\frac{\pi}{2})}.
\end{split}
\end{equation}


For all $ i \in \mc{I}$, split the tokens into two groups: Let $\Nc_i$ be the group of tokens obeying $ \hb_{i1}-\hb_{it} \leq 1-\nu$ and $\bar\Nc_i:=[T]-\{\op_i\}-\Nc_i$ be the rest of the neighbors. Observe that
\[
\frac{\sum_{t\in\bar{\Nc}_i}\s_{it}}{\sum_{t\neq\op_i}\s_{it}}\leq  T\frac{e^{\nu \RR}}{e^{2\nu\RR}}=Te^{-\RR\nu}.
\]
Using $|\hb_{i1}-\hb_{it}|\leq 2A$ and  $\bggm=\min_{i\in[n]}\bgg_i =\min_{i\in[n]} (\bgam_{i1}-\max_{t\neq\op_i}\bgam_{it})$, observe that 
\[
\sum_{t\in\bar\Nc_i} (\hb_{i1}-\hb_{it})\s_{it}(\bgam_{i1}-\bgam_{it})\leq \frac{2\Gamma A Te^{-\RR\nu}}{\bggm} \sum_{t\neq\opt_i} \s_{it}(\bgam_{i1}-\bgam_{it}).
\]
Thus, 
\begin{align*}
  \sum_{t\neq \op_i} (\hb_{i1}-\hb_{it})\s_{it}(\bgam_{i1}-\bgam_{it})&= \sum_{t\in \Nc_i} (\hb_{i1}-\hb_{it})\s_{it}(\bgam_{i1}-\bgam_{it})+\sum_{t\in\bar\Nc_i} (\hb_{i1}-\hb_{it})\s_{it}(\bgam_{i1}-\bgam_{it})\nonumber\\
  &\leq \sum_{t\in \Nc_i} (1-\nu)\s_{it}(\bgam_{i1}-\bgam_{it})+\frac{2\Gamma A Te^{-\RR\nu}}{\bggm} \sum_{t\neq \op_i} \s_{it}(\bgam_{i1}-\bgam_{it})\\
  &\leq \left(1-\nu+\frac{2\Gamma A Te^{-\RR\nu}}{\bggm}\right)\sum_{t\neq \op_i}\s_{it}(\bgam_{i1}-\bgam_{it})\\
 &\leq \left(1+\frac{2\Gamma A Te^{-\RR\nu}}{\bggm}\right)\sum_{t\neq \op_i}\s_{it}(\bgam_{i1}-\bgam_{it}).
\end{align*}
Hence, choosing 
\begin{align}
R\geq\frac{1}{\nu\Theta}\log\left(\frac{8\Gamma AT}{\bggm\pi}\right)\label{R bound pi}
\end{align}
results in that
\begin{equation}\label{eqn:grad:difff1}
    \begin{split}
     &\sum_{t\neq\op_i} (\hb_{i1}-\hb_{it})\s_{it}(\bgam_{i1}-\bgam_{it})  - \left(1+\frac{\pi}{2}\right) \sum_{t\neq\op_i}\s_{it}(\bgam_{i1}-\bgam_{it}) \\
   &\leq\left(\frac{2\Gamma A Te^{-\RR\nu}}{\bggm}-\frac{\pi}{2}\right)\sum_{t\neq\op_i}\s_{it}(\bgam_{i1}-\bgam_{it})\\
   &\leq -\frac{\pi}{4}\sum_{t\neq\op_i}\s_{it}(\bgam_{i1}-\bgam_{it})\\
   &\leq-\frac{\pi}{4T}\bggm  e^{-\bar{R} (1-2\nu)}.      
    \end{split}
\end{equation}
Here, the last inequality follows from the fact that $\sum_{t\neq\op_i}\s_{it}\geq \max_{t\neq\op_i}\s_{it}\geq\frac{e^{-\bar{R}(1-2\nu)}}{\sum_{t=1}^Te^{-\bar{R}(\hb_{i1}-\hb_{it})}}\geq e^{-\bar{R}(1-2\nu)}/T$.

From Assumption~\ref{assum:loss:prope}, we have $c_{\min}\leq-\ell'\leq c_{\max}$ for some positive constants $c_{\min}$ and $c_{\max}$. It follows from  \eqref{eqn:grad:difff0} and \eqref{eqn:grad:difff1} that 
\begin{align*}
-\frac{1}{n}\sum_{i}^n \ell_i' \cdot&\left(
      \sum_{t\neq\op_i} (\hb_{i1}-\hb_{it})\s_{it}(\bgam_{i1}-\bgam_{it})- \sum_{t\neq\op_i} (1+0.5\pi)\s_{it}(\bgam_{i1}-\bgam_{it})\right)\\
      & \leq    c_{\max}2A\Gamma  T \Gamma e^{-\RR(1 +\frac{\pi}{2})}-\frac{c_{\min}}{nT}\cdot\frac{\pi\bggm}{4}e^{-\bar{R} (1-2\nu)}\\
      & \leq 0.
\end{align*}
Combing with \eqref{R bound pi}, this is guaranteed by 
choosing 
\[
  R\geq \max\left\{\frac{1}{\nu\Theta}\log\left(\frac{8\Gamma AT}{\bggm\pi}\right),\frac{1}{(2\nu+\pi/2)\Theta}\log\left(\frac{8n\Gamma AT^2 c_{\max}}{c_{\min}\bggm\pi}\right)\right\},
\]
where $\nu=\nu(\frac{\pi}{4A\Theta})$ depends only on $\pi$ and global problem variables. 

Combining this with the prior $R$ choice \eqref{Rpi choice2} (by taking maximum), we conclude with the statement.
\end{proof}


\subsubsection{Proof of Theorem~\ref{conv:gd:w:global:nabla0}}\label{app:glob:nab0}
The following theorem is a restatement of Theorem \ref{conv:gd:w:global:nabla0}. Two minor differences are: (1) We state with the change of variable $\mu\gets\tf{\Wm}\cdot \mu$; see discussions below \eqref{eqn:con:nabla0}. (2) We also include \eqref{lem:zglob:l3} in the statement of the theorem.
\begin{theorem}\label{conv:gd:w:global:nabla0:app}
Suppose Assumption~\ref{assum:loss:prope} on the loss function $\ell$ and Assumption \ref{assum:nabla0} on the initial gradient hold.  
\begin{enumerate}[label={\textnormal{\textbf{L\arabic*.}}}, wide, labelwidth=!,itemindent=!, labelindent=5pt]
  \item   \label{lem:zglob:l1}For any $\mu>0$, there exists  $R>0$ such  that   $\conb_{\mu,R}(\Ws)$ defined in \eqref{eqn:con:nabla0} does not contain any  stationary points. 
\item \label{lem:zglob:l2} Fix any  $\mu \in  (0,\min (1,\iota \tf{\Ws}/\tf{\nabla \Lc(0)})$. Consider GD iterations with $\W(0)=0$, $\W(1)=-R\nabla\Lc(0)/\tf{\nabla\Lc(0)}$, and $\W(k+1)=\W(k)-\eta\nabla\Lc(\W(k))$ for $\eta\le 1/L_{\W}$, $k\ge 1$, and $R$ sufficiently large. If all iterates remain within $\conb_{\mu,R}$, then $\lim_{k\rightarrow\infty} \tf{\W(k)}=\infty$ and $\lim_{k\rightarrow\infty}\frac{\W(k)}{\tf{\W(k)}}=\frac{\Wm}{\tf{\Wm}}$.
\item  \label{lem:zglob:l3} Assume $\eta\leq 1/L_{\W}$ and for all $\W \in \conb_{\mu,R}(\Ws)$ with sufficiently large $R$?
\begin{align}\label{assum:extra}
   \min_{i \in [n]}\li(\x_{i\op_i}-\x_\itt)\z_i^\top, \W-\eta\nabla\Lc(\W) \ri \geq     \min_{i \in [n]}\li(\x_{i\op_i}-\x_\itt)\z_i^\top, \W\ri - \frac{2\eta\mu}{\tf{\Wm}^2}\iprod{\nabla\mc{L}(\W)}{\Wm},
\end{align}
then all GD iterations remain within $\conb_{\mu,R}(\Wm)$.
\end{enumerate}
\end{theorem}

\begin{proof}
Note that \ref{lem:zglob:l1} is a direct corollary of Lemma~\ref{glocal cond}. We proceed with the proof of \ref{lem:zglob:l2} and \ref{lem:zglob:l3}.  
We provide the proof in four steps:\\
\textbf{Step~1: $\conb_{\mu,R^0_\mu}(\Wm)$ construction}.  Let us denote the initialization lower bound as $R^0_\mu:=R$, where $R$ is given in the Theorem~\ref{conv:gd:w:global:nabla0:app}'s statement. Consider an arbitrary value of  $\epsilon \in (0, \mu/2)$ and let $1/(1+\pi)=1-\epsilon$. We additionally denote $R_\eps\gets R_\pi\vee 1/2$ where $R_\pi$ was defined in Lemma~\ref{lem:glocal:corr}. At initialization $\W(0)$, we set $\eps=\mu/2$ to obtain $R^0_\mu= R_{\mu/2}$. 

We proceed to show  $\mu \in  (0,\min (1,\iota \tf{\Ws}/\tf{\nabla \Lc(0)})$.  It follows from  Assumption~\ref{assum:nabla0}  and under zero initialization for GD ($\W(0)=0$) that
$$
\left\langle (\x_{i\op_i}-\x_\itt) \z_i^\top,   - \nabla \mc{L}(\W(0)) \right\rangle =\left\langle (\x_{i\op_i}-\x_\itt) \z_i^\top,   - \nabla \mc{L}(0) \right\rangle\ge \iota  > 0, 
$$
for some positive constant $\iota $. Hence, for any initial step size $\eta (0)>0$ and $\W(1)=-\eta(0) \nabla\Lc(0)$, 
\begin{equation}\label{eqn:decpath:zinit}
\begin{split}
    \li(\x_{i\op_i}-\x_\itt)\z_i^\top, \frac{\W (1)}{\tf{\W(1)}}\ri &=  \frac{\eta (0) }{\tf{\W(1)}}  \left\langle (\x_{i\op_i}-\x_\itt) \z_i^\top,   - \nabla \mc{L}(0) \right\rangle \\
     & \geq  \frac{\iota  \eta (0) }{\tf{\W(1)}}=  \frac{\iota }{ \tf{\nabla \mc{L}(0)}}\\
     &\geq  \frac{\mu}{\tf{\Wm}}.
\end{split}
\end{equation}
Here, the last inequality follows from our choice of $\mu$ in the theorem statement, i.e.
\begin{align}\label{eqn:mu:zero}
 \mu \in \left(0, \min\left(1, \frac{\iota \tf{\Wm}}{\tf{\nabla \mc{L}(0)}}\right)\right).
\end{align}
This $\mu$ choice induces the conic set $\conb_{\mu,R^0_\mu}(\Wm)$ with  $R^0_\mu= R_{\mu/2}$, where $R_{\mu/2}$ was defined in Lemma~\ref{lem:glocal:corr}. 
Now, given the parameter  $ \mu $ satisfying \eqref{eqn:mu:zero}, we can choose $\eta (0)$ such that $\tf{\W (1)} \geq R^0_\mu$ and $\W(1)\in\conb_{\mu, R^0_\mu}(\Wm)$. To achieve this, since $\W(0)=0$, we obtain  
\begin{equation}\label{eqn:stepeta0}
    \eta (0) =\frac{R^0_\mu}{\tf{\nabla \mc{L}(0) }}. 
\end{equation}
Since by our definition, $R^0_\mu \leftarrow R$,  \eqref{eqn:stepeta0} gives $\W(1)$ in the theorem's statement.

\noindent \textbf{Step~2: There are no stationary points within $\conb_{\mu,R_\mu^0}(\Wm)$.} 
This step follows from \ref{lem:zglob:l1}. Specifically, 
we can apply Lemma~\ref{glocal cond} to find that: For all $\V,\W\in  \bar{\Sc}_{\mu}(\Ws)$ with $\tf{\W} \neq 0$ and $\tf{\W} \geq R^0_\mu$,  we have that $-\li\V, \nabla \Lc(\W)\ri$ is strictly positive.
\\
\emph{Gradient correlation holds for large parameter norm.}  
It follows from  Lemma~\ref{lem:glocal:corr} that, there exists $ R_\epsilon\geq \bar{R}_\mu\vee 1/2$ such that all  $ \W \in \conb_{\mu,R_\epsilon}(\Wm)$ satisfy
\begin{align}\label{eqn:neg:corr:0}
\iprod{-\nabla\mc{L}(\W)}      {\frac{\Wm}{\tf{\Wm}}} \geq (1-\epsilon)    \iprod{-\nabla \mc{L}(\W)}{\frac{\W}{\tf{\W}}}.
\end{align}
The following argument applies to a general $\eps\in(0,\mu/2)$. However, at initialization $\W(0)=0$, we have set $\eps=\mu/2$ and defined the initialization radius as $R^0_\mu= R_{\mu/2}$. To proceed, we will prove the main statements \eqref{lem:zglob:l2} and \eqref{lem:zglob:l3} as follows.
\begin{itemize}
\item Proving \ref{lem:zglob:l3}: In \textbf{Step 3}, we will assume Condition \eqref{assum:extra} to prove that gradient iterates remain within $\conb_{\mu,R_\eps}(\Ws)$. Concretely, for any $\epsilon \in (0, \mu/2)$, we will show that after gradient descent enters the conic set $\conb_{\mu,R_\eps}(\Ws)$ for the first time, it will never leave the set under Condition \eqref{assum:extra} of the theorem statement and \eqref{eqn:neg:corr:0}. In what follows, let us denote $k_\eps$ to be the first time gradient descent enters $\conb_{\mu,R_\eps}(\Ws)$. Note that for $\eps\gets\mu/2$, $k_\eps=0$ i.e.~the point of initialization.

\item Proving \ref{lem:zglob:l2}: In \textbf{Step 4}, assuming iterates within $\conb_{\mu,R_\eps}(\Ws)$, we will prove that the norm diverges (as a result such $k_\eps$ is guaranteed to exist) and, additionally, the gradient updates asymptotically aligns with $\Ws$. 
\end{itemize}
%
%

\textbf{Step~3 (Proof of \ref{lem:zglob:l3}): Updates remain inside the cone $\conb_{\mu,R_\eps}(\Ws)$.}   Note that if $\W(k) \in \conb_{\mu,R_\eps}(\Ws)$ for all $k \geq 1$, the required condition in \ref{lem:zglob:l2} holds, and we proceed to \textbf{Step 4}. In this step, we show \ref{lem:zglob:l3}. Specifically, we show that under Condition \eqref{assum:extra} and using  \eqref{eqn:neg:corr:0}, all iterates $\W(k) \in \conb_{\mu,R_\eps}(\Ws)$ remain within $\conb_{\mu,R_\eps}(\Ws)$.

To proceed, by leveraging the results from \textbf{Step 1} and \textbf{Step 2}, we demonstrate that the gradient iterates, with an appropriate constant step size, starting from $\W(k_\eps) \in \conb_{\mu,R_\eps}(\Ws)$, remain within this set. We proceed by induction.  Suppose that the claim holds up to iteration $k \geq k_\eps$. This implies that $ \W(k) \in \conb_{\mu,R_\eps}(\Ws)$. Hence, recalling $\conb_{\mu,R_\eps}(\Ws)$ defined in \eqref{eqn:con:nabla0}, there exists scalar $\mu=\mu(\bal) \in (0,1)$  and $R_\eps$ such that  $\tf{\W(k)}\geq R_\eps$, and
\begin{equation*}
\begin{split}
\left\langle (\x_{i\op_i}-\x_{it})\z_i^\top,\frac{\W(k)}{\tf{\W(k)}} \right\rangle  \geq \mu\Theta,
\end{split}
\end{equation*}
where $\Theta=1/\tf{\Wm}$. 

Let 
\begin{subequations}\label{eqn:rho:def:nabla0}
\begin{align}
\frac{1}{1-\epsilon} \iprod{ {\frac{\Wm}{\tf{\Wm}}}}{-\nabla\mc{L}(\W(k))} =:\rho(k)>0.
\end{align}
\end{subequations}
Using \eqref{assum:extra}, we have 
\begin{equation}\label{eqn:localgd:1:nabla0}
    \begin{split}
   \left\langle (\x_{i\op_i}-\x_\itt) \z_i^\top,    \frac{\W(k+1)}{\tf{\W(k)}} \right\rangle &=   \left\langle (\x_{i\op_i}-\x_\itt) \z_i^\top,    \frac{\W(k)}{\tf{\W(k)}} -\frac{\eta}{\tf{\W(k)}}\nabla \mc{L}(\W(k)) \right\rangle\\   
      & \geq \mu \Theta +\frac{ 2\eta  (1-\epsilon)\mu \Theta \rho(k) }{\tf{\W(k)}}.
    \end{split}
\end{equation}
From Lemma~\ref{glocal cond},   we have $\left\langle \nabla \Lc(\W(k)),\W(k)\right\rangle<0$~ which implies that $\tf{\W(k+1)} \geq \tf{\W(k)}$.  This together with  $R_\eps$ definition and $\tf{\W(k)}\geq 1/2$ implies that  
\begin{align*}
\tf{\W(k+1)}&\leq\frac{1}{{2\tf{\W(k)}}} \left(\tf{\W(k+1)}^2+\tf{\W(k)}^2\right)\\
& = \frac{1}{2\tf{\W(k)}} \left(2\tf{\W(k)}^2-2\eta\left\langle \nabla \Lc(\W(k)),\W(k)\right\rangle+\eta^2\tf{\nabla \Lc(\W(k))}^2\right)\\
       &\leq  \tf{\W(k)}- \frac{\eta}{\tf{\W(k)}}\left\langle \nabla \Lc(\W(k)),\W(k)\right\rangle + \eta^2 \|\nabla \Lc(\W(k))\|_F^2.
\end{align*}
Thus,
\begin{equation}\label{eqn:localgd:2:nabla0}
\begin{split}
  \frac{\tf{\W(k+1)}}{\tf{\W(k)}}& \leq  1- \frac{\eta}{\tf{\W(k)}}
       \left\langle \nabla \Lc(\W(k)),\frac{\W(k)}{\tf{\W(k)}} \right\rangle + \eta^2 \frac{\|\nabla \mc{L}(\W(k))\|_F^2}{\tf{\W(k)}}\\
& \leq 1- \frac{\eta}{(1-\epsilon)\tf{\W(k)}}  \iprod{\nabla\mc{L}(\W(k))}
     {\frac{\Wm}{\tf{\Wm}}}+ \eta^2 \frac{\|\nabla \mc{L}(\W(k))\|_F^2}{\tf{\W(k)}}\\
      & \leq  1 + \frac{\eta \rho(k)}{\tf{\W(k)}} + \frac{\eta^2\|\nabla \mc{L}(\W(k))\|_F^2}{\tf{\W(k)}}=:C_1(\rho(k),\eta).
\end{split}
\end{equation}
Here, the second inequality uses \eqref{eqn:neg:corr:0}. 

Now, it follows from \eqref{eqn:localgd:1:nabla0} and \eqref{eqn:localgd:2:nabla0} that 
\begin{equation}\label{eqn:localgd:3:nabla0}
\begin{split}
\min_{t\neq \op_i,~i\in[n]} ~~  \left\langle  (\x_{i\op_i}-\x_\itt) \z_i^\top, \frac{\W(k+1)}{\tf{\W(k+1)}}\right\rangle   &\geq \frac{1}{C_1({\rho}(k),\eta)} \left(\mu \Theta+\frac{2\eta (1-\epsilon)\mu \Theta  {\rho}(k)}{\tf{\W(k)}}\right)\\
& = \mu \Theta+\frac{\eta\mu \Theta}{C_1({\rho}(k),\eta)} \left(\frac{ \big(2(1-\epsilon) -1 \big)  \rho(k)}{\tf{\W(k)}}
-   \eta \frac{\tf{\nabla \mc{L}(\W(k))}^2 }{\tf{\W(k)}}\right)\\
& = \mu \Theta+\frac{\eta\mu \Theta}{C_1({\rho}(k),\eta)} \left(\frac{ (1-2\epsilon)   \rho(k)}{\tf{\W(k)}}
-   \eta  \frac{\tf{\nabla \mc{L}(\W(k))}^2 }{\tf{\W(k)}}\right)\\
& \geq \mu \Theta,
\end{split}
\end{equation}
 where the last inequality uses our choice of stepsize $\eta\leq 1/L_W$ in Theorem~\ref{conv:gd:w:global:nabla0}'s statement. Specifically, we need $\eta$ to be small to ensure the last inequality. We will guarantee this by choosing a proper $R_\eps$ in Lemma \ref{lem:glocal:corr}. Specifically, Lemma \ref{lem:glocal:corr} leaves the choice of $D_0$ in $R_\eps$ lower bound of \eqref{Rpi choice2} open (it can always be chosen larger). Here, by choosing $D_0\gtrsim 1/L_{\W}$ will ensure $\eta\leq 1/L_{\W}$ works well.
\begin{equation}\label{eqn:zeta:mu:0}
\begin{split}
    \eta &\leq    \big( 1-\mu\big)\mu
 \frac{c}{C}  \frac{\Theta}{\bar{A}}   \frac{1}{\bar{A}C T}    e^{R_\mu^0\Theta/2}\\
 &\leq      \frac{1-2\epsilon  }{1-\epsilon}   \frac{c \mu}{C}   \frac{\Theta}{\bar{A}}    \frac{1}{\bar{A}C T}  e^{R_\mu^0\Theta/2} \\
 &\leq   \big(1-2\epsilon \big)   \frac{\rho(k) } { \|\nabla \mc{L}(\W(k))\|^2_F}.
\end{split}    
\end{equation}
Here, the first inequality follows since $\epsilon \in (0, \mu/2)$ (as seen in \textbf{Step 2}). Also,  $\mu < 1$ implies that $1-\mu > 0$, we obtain $\eta > 0$. The last inequality is obtained from Lemma~\ref{glocal cond}:
\begin{align*}
    \frac{\rho(k) } { \tf{\nabla \mc{L}(\W(k))}} &= - \frac{1}{1-\epsilon} \iprod{ \frac{\nabla\mc{L}(\W(k))}{\tf{\nabla \mc{L}(\W(k))}}}
     {\frac{\Wm}{\tf{\Wm}}}  \geq \frac{1}{1-\epsilon} \cdot \frac{c \mu}{C} \cdot \frac{\Theta}{\bar{A}},\\
         \frac{1} { \tf{\nabla \mc{L}(\W(k))}} &{\geq \frac{1}{\bar{A}C \cdot \frac{1}{n} \sum_{i=1}^n  \left(1-\s_{i\op_i}\right)} \geq     \frac{1}{ \bar{A} C T e^{-R_\mu^{0}\Theta/2}} }
\end{align*}
for some data dependent constrants $c$, $C$, $\bar{A}=\max_{i\in[n],t,\tau\in[T]}\tn{(\x_{it}- \x_{i\tau})}~\tn{\z_i}$, and $\Theta=1/\tf{\Ws}$.

The remainder of the proof of this step is identical to \eqref{eqn:pitoC0}--\eqref{eqn:pitoC02}, with the replacement of $C_0$ by $D_0$ and the tracking of changes. Specifically, Lemma \ref{lem:glocal:corr} leaves the choice of $D_0$ in $R_\eps$ lower bound of \eqref{Rpi choice2} open (it can always be chosen larger). Hence,  for sufficiently large $D_0$, we have
\begin{align}
\eta \leq \frac{1}{L_{\W}}\leq   \big( 1-\mu\big)\mu
 \frac{c}{C}  \frac{\Theta}{\bar{A}}   \frac{1}{\bar{A}C T}    e^{R_\mu^0\Theta/2}.
\end{align}
This implies \eqref{eqn:localgd:3:nabla0} and  $\W(k+1) \in\conb_{\mu,R_\eps}(\Ws)$. 

\noindent\textbf{Step 4 (Proof of \ref{lem:zglob:l2}): $\W(k)$ and $\Wm$ perfectly align over time.} 
By theorem statement (alternatively via \textbf{Step 3}), we have that all iterates remain within the initial conic set i.e.~$\W(k)\in\conb_{\mu,R^0_\mu}(\Ws)$ for all $k\geq 0$. Note that it follows from Lemma~\ref{glocal cond}  that  $\li\nabla\Lc(\W), \Ws/\tf{\Ws}\ri<0$, for any finite $\W \in \conb_{\mu,R^0_\mu}(\Ws)$. Hence, there are no finite critical points $\W \in \conb_{\mu,R^0_\mu}(\Ws)$, for which $\nabla \mc{L} (\W)=0$. Now, based on Lemma~\ref{lem:grad:descent}, which guarantees that $\nabla\Lc(\W(k))\rightarrow 0$, this
implies that $\left\Vert \W\left(k\right)\right\Vert \rightarrow\infty$. Consequently, for any choice of $\eps\in (0,\mu/2)$ there is an iteration $k_\eps$ such that, for all $k\geq k_\eps$, $\W(k)\in\conb_{\mu,R_\eps}(\Ws)$. Once within $\conb_{\mu,R_\eps}(\Ws)$,  multiplying both sides of \eqref{eqn:neg:corr:0} by the stepsize $\eta$ and using the gradient descent update, we get
\begin{equation*}
\begin{split}
     \left\langle \W(k+1)-\W(k),\frac{ \Wm}{\tf{\Wm}} \right\rangle &\geq  (1-\epsilon) \left\langle \W(k+1)-\W(k), \frac{\W(k)}{\tf{\W(k)}}\right\rangle\\
     &= \frac{(1-\epsilon)}{2\tf{\W(k)}}\left(\tf{\W(k+1)}^2- \tf{\W(k)}^2-\tf{\W(k+1)-\W(k)}^2\right) \\
     & \geq (1-\epsilon)\left( \frac{1}{2\tf{\W(k)}} \left(\tf{\W(k+1)}^2- \tf{\W(k)}^2\right)-\tf{\W(k+1)-\W(k)}^2\right) \\
     & \geq (1-\epsilon)\left(\tf{\W(k+1)}- \tf{\W(k)}-\tf{\W(k+1)-\W(k)}^2\right) \\
          & \geq (1-\epsilon)\Big(\tf{\W(k+1)}- \tf{\W(k)}- 2\eta  \left(\mc{L}(\W(k))-\mc{L}(\W(k+1))\right) \Big).
\end{split}
\end{equation*}
Here, the second inequality is obtained from  $\tf{\W(k)}\geq 1/2$; the third inequality follows since  for any $a, b >0$, we have $  (a^2-b^2)/(2b) -  (a-b) \geq 0$; and the last inequality  uses Lemma~\ref{lem:grad:descent}.


Summing the above inequality over $k\geq k_\eps$ gives 


\begin{align*}
      \left\langle\frac{\W(k)}{\tf{\W(k)}}, \frac{\Wm}{\tf{\Wm}} \right\rangle \ge1-\epsilon+ \frac{C(\epsilon,\eta)}{\tf{\W(k)}}, \qquad \W(k)\in\conb_{\mu,R_\eps}(\Ws),   
\end{align*}
where $\mathcal{L}_{\star}\leq\mathcal{L}\left(\W\left(k\right)\right)$ for all $k\geq k_\eps$, and 
\begin{equation*}
C(\epsilon,\eta)= \left\langle \W(k_\eps), \frac{ \Wm}{\tf{\Wm}}\right\rangle-(1-\epsilon)\tf{\W(k_\eps)} -2\eta (1-\epsilon) (\mc{L}(\W(k_\eps))-\mathcal{L}_{\star}).
\end{equation*}
Consequently,
    \begin{align*}
      \liminf_{k\to\infty}\iprod{\frac{\W(k)}{\tf{\W(k)}}}{\frac{\Wm}{\tf{\Wm}}}\ge1-\epsilon, \qquad \W(k)\in\conb_{\mu,R_\eps}(\Ws).  
    \end{align*}
Since  $\epsilon \in (0, \mu/2)$  is arbitrary, this implies $\W(k)/\tf{\W(k)}\to  \Wm/\tf{\Wm}$.
\end{proof}

\section{Local Convergence of Gradient Descent}\label{app local proofs}
To provide a basis for discussing local convergence of GD, we establish a cone centered around $\Wma$  using the following construction. For parameters $\mu \in (0,1)$ and $R>0$, we define $\Cc_{\mu,R}(\Wma)$ as the set of matrices $\W \in\R^{d\times d}$ such that $\tf{\W}\geq R$ and  the correlation coefficient between $\W$ and $\Wma$ is at least $1-\mu$:
\begin{subequations}
\begin{align}\label{eqn:coneofw:r}
\Sc_{\mu}(\Wma)&:= \left\{\W\in\R^{d\times d}~:~\left\langle\frac{\W}{\tf{\W}},\frac{{\Wma}}{\tf{{\Wma}}} \right\rangle \geq 1-\mu\right\}, \\
\Cc_{\mu,R}({\Wma})&:= \Sc_{\mu}(\Wma) \cap  \left\{\W\in\R^{d\times d}~:~\tf{\W}\geq R\right\}.
\end{align}
\end{subequations}
\begin{lemma}
\label{local cond} 
Suppose Assumption~\ref{assum:loss:prope} on the loss function $\ell$ holds, and let $\bal=(\alpha_i)_{i=1}^n$ be locally optimal tokens according to Definition \ref{def loc opt}. Let $ \Wm= \Wm_\bal$ denote the SVM solution obtained via \eqref{eqn:sattnsvm} by applying the Frobenius norm and replacing $(\opt_i)_{i=1}^n$ with $\boldsymbol{\alpha} = (\alpha_i)_{i=1}^n$. 
%
%
There exists a scalar $\mu=\mu(\bal)>0$ such that for sufficiently large $\RR_\mu$:
\begin{enumerate}[label={\textnormal{\textbf{L\arabic*.}}}, wide, labelwidth=!,itemindent=!, labelindent=5pt]
\item \label{lem:cond:l1} There is no stationary point within  $ \Cc_{\mu,\RR_\mu} (\Wm)$.
\item\label{lem:cond:l2} For all $\V\in \Sc_{\mu}(\Wm)$ with $\tf{\V}=\tf{\Wm}$  and $\W\in\Cc_{\mu,\RR_\mu}(\Wm)$, there exist dataset dependent constants $C,c>0$ such that 
\begin{subequations}\label{local:g:lbound}
\begin{align}
&C\cdot \frac{1}{n}\sum_{i=1}^n \left(1-\s_{i\alpha_i}\right) \geq -\Big\langle\nabla\Lc(\W),\V \Big\rangle\geq c\cdot  \frac{1}{n} \sum_{i=1}^n  \left(1-\s_{i\alpha_i}\right)>0, \label{local1:g:bound} \\
&\tf{\nabla\Lc(\W)}\leq \bar{A}C \cdot \frac{1}{n} \sum_{i=1}^n  \left(1-\s_{i\alpha_i}\right), \label{local2:g:bound}\\
& -\li\frac{\V}{\tf{\V}},\frac{\nabla\Lc(\W)}{\tf{\nabla\Lc(\W)}}\ri \geq  \frac{c}{C} \cdot \frac{\Theta}{\bar{A}}>0. \label{local3:g:bound}
\end{align}
\end{subequations}
Here, $\s_{i\alpha_i}= (\sft{\X_i\W \z_{i}})_{\alpha_i}$, $\bar{A}=\max_{i\in[n],t,\tau\in[T]}\tn{(\x_{it}- \x_{i\tau})}~\tn{\z_i}$, and $\Theta=1/\tf{\Ws}$.
\end{enumerate}
\end{lemma}

\begin{proof}
Let $R=\RR_\mu$, $(\Tc_i)_{i=1}^n$ be the set of all \neis per Definition \ref{def loc opt}. Let $\Tcb_i=[T]-\Tc_i-\{\alpha_i\}$ be the non-\neis. Let
\begin{equation}\label{mu choice}
\begin{split}
&\Theta=1/\tf{\Wm},\\
&\delta= \frac{1}{2}\min_{i\in[n]}\min_{t\in\Tc_i,\tau\in\Tcb_i}(\x_{it}-\x_{i\tau})^\top \Wm \z_{i},\\
&A=\max_{i\in[n],t\in[T]} \frac{\tf{\x_{it} \z_i^\top}}{\Theta},\\
& \mu\leq \mu(\delta)=\frac{1}{8}\left(\frac{\min(0.5,\delta)}{A}\right)^2.
\end{split}
\end{equation}
%
Since $\Wm$ is the max-margin model ensuring $(\x_{i\alpha_i}-\x_{it})^\top\Wm \z_i\geq 1$, the following inequalities hold for all $\W\in \cone_\mu(\Wm),~\tf{\W}=\tf{\Wm}$ and all $i\in[n], t\in\Tc_i,\tau\in\Tcb_i$:
\begin{equation}\label{cone-non-nei}
\begin{split}
(\x_{it}-\x_{i\tau})^\top \W \z_i&\geq \delta>0,\\
(\x_{i\alpha_i}-\x_{i\tau})^\top \W \z_i&\geq 1+\delta,\\
\frac{3}{2}\geq(\x_{i\alpha_i}-\x_{it})^\top \W \z_i &\geq \frac{1}{2}.
\end{split}
\end{equation}
Here, we used $\tf{\W-\Wm}^2/\tf{\Wm}^2\leq 2\mu$ which implies $\tf{\W-\Wm}\leq \sqrt{2\mu}/\Theta$.

To proceed, we write the gradient correlation following \eqref{grad def} and \eqref{grad def2}
\begin{align}\label{grad def3}
\li\nabla\Lc(\W),\V\ri&=\frac{1}{n}\sum_{i=1}^n\ell'_i\cdot\hb_i^\top\sfp{\hp_i}\bgam_i,
\end{align}
where we denoted $\ell'_i=\ell'(Y_i\cdot \vb^\top \X_i^\top\sft{\hp_i})$, $\hb_i=\X_i\V \z_{i}$, $\hp_i= \X_i\W \z_{i}$, and $\s_i=\sft{\hp_i}$.  

Using \eqref{cone-non-nei}, for all $t\in\Tc_i,\tau\in \Tcb_i$, for all $\W\in \Cc_{\mu,R}(\Wm)$, we have that
\begin{align*}
&\hp_{it}-\hp_{i\tau}\geq R\Theta\delta,\\
&\hp_{i\alpha_i}-\hp_{i\tau}\geq R\Theta(1+\delta),\\
&\hp_{i\alpha_i}-\hp_{it}\geq R\Theta/2.    
\end{align*}
Consequently, we can bound the softmax probabilities $\s_i=\sft{\hp_i}$ over non-\neis as follows: For all $i\in[n]$ and any $t_i\in \Tc_i$
\begin{subequations}
\begin{align}\label{soft prob bound}
&S_i:=\sum_{\tau\in\Tc_i}\s_{i\tau} 
\leq T e^{-R\Theta/2}\s_{i\alpha_i}\leq T e^{-R\Theta/2},\\
&Q_i:=\sum_{\tau\in\Tcb_i}\s_{i\tau} \leq T e^{-R\Theta\delta}\s_{it_i}\leq T e^{-R\Theta\delta}S_i.
\end{align}
\end{subequations}
Recall scores $\bgam_{it}=Y_i\cdot\vb^\top \x_{it}$. Define the score gaps over \neis:
\begin{equation*}
 \bgg_i=\bgam_{i\alpha_i}-\max_{t\in\Tc_i}\bgam_{it}~~~ \textnormal{and}~~~ \bgm_i=\bgam_{i\alpha_i}-\min_{t\in\Tc_i}\bgam_{it}. 
\end{equation*}
It follows from \eqref{mu choice} that 
\begin{align*}
&A=\max_{i\in[n],t\in[T]} \frac{\tf{\x_{it} \z_i^\top}}{\Theta}\geq \max_{i\in[n],t\in[T]}\tn{\hb_{it}}.
\end{align*}
Define the $\bal$-dependent global scalar $\Gamma=\sup_{i\in[n],t,\tau\in[T]}|\bgam_{it}-\bgam_{i\tau}|$.

Let us focus on a fixed datapoint $i\in[n]$, assume (without losing generality) $\alpha_i=1$, and drop subscripts $i$.
Directly applying Lemma \ref{lem:q_reduce}, we obtain
\[
  \big|\hb^\top\diag{\s}\bgam-\hb^\top\s\s^\top\bgam-\sum_{t\geq 2}^T (\hb_1-\hb_t)\s_t(\bgam_1-\bgam_t)\big|\leq 2\Gamma A(1-\s_1)^2.
\]
To proceed, let us decouple the non-\neis within $\sum_{t\geq 2}^T (\hb_1-\hb_t)\s_t(\bgam_1-\bgam_t)$ via
\[
\big|\sum_{t\in\Tcb} (\hb_1-\hb_t)\s_t(\bgam_1-\bgam_t)\big|\leq 2Q\Gamma A.
\]
Aggregating these, we found
\begin{align}
  \big|\hb^\top\diag{\s}\bgam-\hb^\top\s\s^\top\bgam-\sum_{t\in \Tc} (\hb_1-\hb_t)\s_t(\bgam_1-\bgam_t)\big|\leq 2\Gamma A((1-\s_1)^2+Q).\label{aggregate}
\end{align}
To proceed, let us upper/lower bound the gradient correlation.  We use two bounds depending on $\V\in\Sc_{\mu}(\Ws)$ (\textbf{Case 1}) or general $\V\in\R^{d\times d}$ (\textbf{Case 2}).

\noindent$\bullet$ \textbf{Case 1:  $\V\in\Sc_{\mu}(\Ws)$.} Since $1.5\geq \hb_1-\hb_t\geq 0.5$ following \eqref{cone-non-nei}, we find
\[
 1.5\cdot S\cdot \bgm  \geq\sum_{t\in \Tc} (\hb_1-\hb_t)\s_t(\bgam_1-\bgam_t)\geq 0.5\cdot S\cdot \bgg,
\]
where recall the definition of $S$ (having dropped subscripts) in \eqref{soft prob bound}. 

\noindent$\bullet$ \textbf{Case 2: $\Vb\in\R^{d\times d}$ and $\tf{\V}=\tf{\Wm}$.}  Define $\bar{A}=\max_{i\in[n],t,\tau\in[T]}\tn{\x_{it}-\x_{i\tau}}~\tn{\z_i}$. For any $\tf{\V}=\tn{\Ws}$, we use the fact that
$$\tn{\hb_1-\hb_t}\leq \tf{(\x_{it}-\x_{i\tau}) \z_i^\top}\cdot\tf{\V}\leq \frac{\bar{A}}{\Theta}.$$
Note that by definition $ \frac{\bar{A}}{\Theta} \geq 1$. To proceed, we can upper bound
\begin{align}
\frac{\bar{A}}{\Theta}\cdot S\cdot \bgm  \geq\sum_{t\in \Tc} (\hb_1-\hb_t)\s_t(\bgam_1-\bgam_t).\label{wishwish2}
\end{align}

Next we claim that for both cases, $S$ dominates $((1-\s_1)^2+Q)$ for large $R$. Specifically, we wish for 
\begin{align}\label{wishfor}
\frac{S\cdot \bgg}{4}\geq 4\Gamma A\max((1-\s_1)^2,Q)\iff S\geq 16\frac{\Gamma A}{\bgg}\max((1-\s_1)^2,Q).
\end{align}
Now choose $R\geq \delta^{-1}\log(T)/\Theta$  to ensure $Q\leq S$ since $Q\leq Te^{-R\Theta\delta}S$ from \eqref{soft prob bound}. Consequently
\[
(1-\s_1)^2=(Q+S)^2\leq 4S^2\leq 4STe^{-R\Theta/2}.
\]
Combining these, what we wish is ensured by guaranteeing
\begin{align}\label{s bound}
  S\geq 16\frac{\Gamma A}{\bgg}\max(4STe^{-R\Theta/2},Te^{-R\Theta\delta}S).
\end{align}
This in turn is ensured for all inputs $i\in[n]$ by choosing 
\begin{align}\label{R bound}
R\geq \frac{\max(2,\delta^{-1})}{\Theta}\log\left(\frac{64T\Gamma A}{\bggm}\right),
\end{align}
where $\bggm=\min_{i\in[n]}\bgg_i$ is the global scalar which is the worst case score gap over all inputs. 
\\
$\bullet$ \textbf{Case 1: $\V\in\Sc_{\mu}(\Ws)$}. With the above choice of $R$, we guaranteed
\[
  2 (1-\s_1)\cdot \bgm\geq 2\cdot S\cdot \bgm \geq \hb^\top\diag{\s}\bgam-\hb^\top\s\s^\top\bgam\geq\frac{S\cdot \bgg}{4}\geq\frac{(1-\s_1) \bgg}{8}.
\]
via \eqref{wishfor} and \eqref{aggregate}. 

Since this holds over all inputs, going back to the gradient correlation \eqref{grad def3} and averaging above over all inputs $i\in[n]$ and plugging back the indices $i$, we obtain the advertised bound 
\begin{align}\label{pbb corr}
\frac{2}{n}\sum_{i\in [n]} -\ell'_i\cdot S_i\cdot \bgm_i\geq -\li\nabla\Lc(\W),\V\ri\geq \frac{1}{8n}\sum_{i\in [n]} -\ell'_i\cdot S_i\cdot \bgg_i.
\end{align}
Let $-\ell'_{\min/\max}$ be the min/max values negative loss derivative admits over the ball $[-A,A]$ and note that $\max_{i\in[n]}\bgm_i>0$ and $\min_{i\in[n]}\bgg_i>0$ are dataset dependent constants. Then, we declare the constants $C=-2\ell'_{\max}\cdot \max_{i\in[n]}\bgm_i>0,c=-(1/8)\ell'_{\min}\cdot \min_{i\in[n]}\bgg_i>0$ to obtain the bound \eqref{local1:g:bound}. 
\vspace{.2cm}
\\
\noindent$\bullet$ \textbf{Case 2: $\Vb\in\R^{d\times d}$ and $\tf{\V}=\tf{\Wm}$.} Next, we show \eqref{local2:g:bound} and \eqref{local3:g:bound}. For any $\V \in \mathbb{R}^{d \times d}$ satisfying $\tf{\V}=\tf{\Ws}$, using \eqref{wishwish2} and the  choice of $R$ in \eqref{R bound} similarly guarantees 
$$
\frac{2\bar{A}}{\Theta }(1-\s_1) \bgm\geq \hb^\top\diag{\s}\bgam-\hb^\top\s\s^\top\bgam,
$$
for fixed input. Going back to the gradient correlation \eqref{grad def3} and averaging above over all inputs $i\in[n]$, with the same definition of $C>0$, we obtain
\begin{align}
\frac{ \bar{A} C}{  \Theta n}\sum_{i\in [n]} (1-\s_{i\alpha_i})\geq -\li\nabla\Lc(\W),\V\ri.\label{local lamma general upper}
\end{align}
To proceed, since \eqref{local lamma general upper} holds for any $\V\in\R^{d\times d}$, we observe that when setting $\V=\frac{\tf{\Ws}}{\tf{\nabla\Lc(\W)}}\cdot \nabla\Lc(\W)$, this implies that
\[ 
\li\nabla\Lc(\W),\V\ri = \tf{\nabla\Lc(\W)}\cdot \tf{\Ws}\leq \frac{\bar{A} C}{\Theta 
 n}\sum_{i\in [n]} (1-\s_{i\alpha_i}).
\]
Simplifying $\Theta=1/\tf{\Ws}$ on both sides gives \eqref{local2:g:bound}. 
\\
Combining the above inequality with \eqref{pbb corr}, we obtain that for all $\V,\W\in\Sc_{\mu}(\Ws)$
\[ 
-\li\frac{\V}{\tf{\V}},\frac{\nabla\Lc(\W)}{\tf{\nabla\Lc(\W)}}\ri\geq \frac{c \Theta }{C\bar{A}},
\]
which gives \eqref{local3:g:bound}.

\end{proof}

\begin{lemma}
\label{lem:local:corr} 
Suppose Assumption~\ref{assum:loss:prope} on the loss function $\ell$ holds, and let $\bal=(\alpha_i)_{i=1}^n$ be locally optimal tokens according to Definition \ref{def loc opt}. Let $ \Wm= \Wm_\bal$ denote the SVM solution obtained via \eqref{eqn:sattnsvm} by replacing $(\opt_i)_{i=1}^n$ with $\boldsymbol{\alpha} = (\alpha_i)_{i=1}^n$. Let $\mu=\mu(\bal)>0$ and $\bar{R}_{\mu}$ be defined as in Lemma~\ref{local cond}. For any choice of $\pi>0$, there exists $R_\pi \geq \bar{R}_{\mu}$ such that, for any $ \W\in \Cc_{\mu,R_\pi}(\Wm)$, we have
\[
 \li \nabla\Lc(\W), \frac{\W}{\tf{\W}} \ri\geq (1+\pi)\li \nabla\Lc(\W), \frac{\Wm}{\tf{\Wm}}\ri.
\]
\end{lemma}
\begin{proof}
Let  $R=R_{\pi}$, $\Wb=\tf{\Wm} \W/\tf{\W} $, $\hb_i=\X_i\Wb \z_{i}$, and $\hbm_i= \X_i \Wm \z_{i}$.   To establish the result, we will prove that, for sufficiently large $R$ and for any $\W\in \Cc_{\mu,R}(\Wm)$:
\begin{align}\label{main local cond}
\nonumber 
\li -\nabla\Lc(\W),\frac{\W}{\tf{\W}}\ri&= -\frac{1}{n}\sum_{i=1}^n\ell'_i \cdot  \li \hb_i, \sfp{\X_i\W \z_{i}}\bgam_i\ri\\
&\leq - \frac{1+\pi}{n}\sum_{i=1}^n\ell'_i \cdot  \li\hbm_i, \sfp{\X_i\W \z_{i}}\bgam_i\ri=(1+\pi)\li-\nabla\Lc(\W), \frac{\ps}{\tf{\Ws}}\ri.
\end{align}

Following \eqref{aggregate}, for all $\W\in \Sc_{\mu}(\Wm)$ with $\tf{\W}=\tf{\Wm}$, $\hp=\X\W \z$, and $\s=\sft{\hp}$, we have found
\begin{align}
  \big|\hp^\top_i\diag{\s_i}\bgam_i-\hp^\top_i\s_i\s^\top_i\bgam_i-\sum_{t\in \Tc_i} (\hp_{i1}-\hp_{it})\s_{it}(\bgam_{i1}-\bgam_{it})\big|\leq 2\Gamma A((1-\s_{i1})^2+Q_i), 
\end{align}
where $\Tc_i$ is the set of support indices.

Plugging in $\hb,\hbm$ in the bound above and assuming $\pi\leq 1$ (w.l.o.g.), \eqref{main local cond} is implied by the following stronger inequality
\begin{align*}
-\frac{1}{n}&\sum_{i=1}^n\ell'_i \cdot \left(6\Gamma A((1-\s_{i1})^2+Q_i)+ \sum_{t\in \Tc_i} (\hb_{i1}-\hb_{it})\s_{it}(\bgam_{i1}-\bgam_{it}) \right)\\
&\leq -\frac{1+\pi}{n}\sum_{i=1}^n\ell'_i  \cdot \sum_{t\in \Tc_i} (\hbm_{i1}-\hbm_{it})\s_{it}(\bgam_{i1}-\bgam_{it})\\
&=-\frac{1+\pi}{n}\sum_{i=1}^n\ell'_i \cdot \sum_{t\in \Tc_i}\s_{it}(\bgam_{i1}-\bgam_{it}).
\end{align*}
First, we claim that $0.5\pi\sum_{t\in \Tc_i}\s_{it}(\bgam_{i1}-\bgam_{it})\geq 6\Gamma A((1-\s_{i1})^2+Q_i)$ for all $i \in [n]$.  The proof of this claim directly follows the earlier argument, namely, following \eqref{wishfor}, \eqref{s bound}, and \eqref{R bound}  which leads to the choice 
\begin{equation}\label{R boundC0}
R \ge\frac{\max(2,\delta^{-1})}{\Theta}\log\left(\frac{C_0\cdot T\Gamma A}{\pi\bggm}\right),    
\end{equation}
for some constant $C_0>0$. Using \eqref{R bound}, we choose $C_0 \geq 64 \pi$ to guarantee $R=R_\pi \geq \bar{R}_{\mu}$.

Following this control over the perturbation term $6\Gamma A((1-\s_{i1})^2+Q_i)$, to conclude with the result, what remains is proving the comparison
\begin{align}\label{desired comp}
-\frac{1}{n} \sum_{i=1}^n\ell'_i \cdot \sum_{t\in \Tc_i} (\hb_{i1}-\hb_{it})\s_{it}(\bgam_{i1}-\bgam_{it})\leq - \frac{1+0.5\pi}{n}\sum_{i=1}^n\ell'_i \cdot \sum_{t\in \Tc_i}\s_{it}(\bgam_{i1}-\bgam_{it}).
\end{align}
To proceed, we split the problem into two scenarios. 

\noindent\textbf{Scenario 1:} $\tf{\Wb-\Wm}\leq \eps=\frac{\pi}{4A\Theta}$ for some $\eps>0$.  In this scenario, for any $ t\in \Tc_i$ and $i\in [n ]$, we have
\[
|\hb_{it}-\hbm_{it}|=|\x_{it}^\top (\Wb-\Wm)  \z_{it}|\leq A\Theta\eps=\frac{\pi}{4}.
\]
Consequently, we obtain 
\[
\hb_{i1}-\hb_{it}\leq \hbm_{i1}-\hbm_{it}+2A\Theta\eps= 1+0.5\pi.
\] 
Similarly, $\hb_{i1}-\hb_{it}\geq 1-0.5\pi\geq 0.5$. Since all terms $\hb_{i1}-\hb_{it},\s_{it},\bgam_{i1}-\bgam_{it}$ in \eqref{desired comp} are nonnegative and $(\hb_{i1}-\hb_{it})\s_{it}(\bgam_{i1}-\bgam_{it})\leq (1+0.5\pi)\s_{it}(\bgam_{i1}-\bgam_{it})$, above implies the desired result in \eqref{desired comp}.

\vspace{3pt}
\noindent\textbf{Scenario 2:} $\tf{\Wb-\Wm}\geq \eps=\frac{\pi}{4A\Theta}$.  Since $\Wb$ is not (locally) max-margin, in this scenario, for some $i \in  [n]$, $\nu=\nu(\eps)>0$, and $\tau\in\Tc_i$, we have that
\begin{align*}
\hb_{i1}-\hb_{i\tau}\leq 1-2\nu.
\end{align*}
Here $\tau=\arg\max_{\tau\in\Tc_i} \x_{i\tau}\Wb \z_i$ denotes the nearest point to $\hb_{i1}$ (along the $\Wb$ direction). Note that a non-neighbor $t\in\Tcb_i$ cannot be nearest because $\Wb\in \cone_{\mu}(\ps)$ and \eqref{cone-non-nei} holds. Recall that $\s_i=\sft{\RR\hb_i}$ where $\RR=\tf{\W}\Theta \geq R\Theta$. To proceed, let $ \underline{\hb}_i:=\min_{t \in\mc{T}_i}\hb_{i1}-\hb_{it}$,
\begin{align*}
\mc{I}:=\left\{ i\in[n]: \underline{\hb}_i \leq 1-2\nu \right\}, \qquad [n]-\mc{I}:=\left\{ i\in[n]:  1-2\nu  <  \underline{\hb}_i\right\}.
\end{align*}
For all $ i \in [n]-\mc{I}$,
\begin{equation}\label{eqn:grad:difff2}
\begin{split}
      \sum_{t\in \Tc_i} (\hb_{i1}-\hb_{it})\s_{it}(\bgam_{i1}-\bgam_{it}) &- (1+0.5\pi) \sum_{t\in \Tc_i}\s_{it}(\bgam_{i1}-\bgam_{it})\\
      & \leq  \left(2A - (1+0.5\pi)\right)\Gamma\sum_{t\in \Tc_i,~\hb_{i1}-\hb_{it} \geq 1+\frac{\pi}{2} } \s_{it} \\
      & \leq  \left(2A - (1+0.5\pi)\right)\Gamma Te^{-\RR(1+\frac{\pi}{2})} \\
      &\leq   2A\Gamma  T e^{-\RR(1+\frac{\pi}{2})}.
\end{split}
\end{equation}


For all $ i \in \mc{I}$, split the tokens into two groups: Let $\Nc_i$ be the group of tokens obeying $ \hb_{i1}-\hb_{it} \leq 1-\nu$ and $\Tc_i-\Nc_i$ be the rest of the neighbors. Observe that
\[
\frac{\sum_{t\in \Tc_i-\Nc_i}\s_{it}}{\sum_{t\in\Tc_i}\s_{it}}\leq  T\frac{e^{\nu \RR}}{e^{2\nu\RR}}=Te^{-\RR\nu}.
\]
Using $|\hb_{i1}-\hb_{it}|\leq 2A=2 \max_{i\in[n],t\in[T]}\tn{\kb_{it}}/\Theta$ and  $\bggm=\min_{i\in[n]}\bgg_i =\min_{i\in[n]} (\bgam_{i1}-\max_{t\in\Tc_i}\bgam_{it})$, observe that 
\[
\sum_{t\in\Tc_i-\Nc_i} (\hb_{i1}-\hb_{it})\s_{it}(\bgam_{i1}-\bgam_{it})\leq \frac{2\Gamma A Te^{-\RR\nu}}{\bggm} \sum_{t\in \Tc_i} \s_{it}(\bgam_{i1}-\bgam_{it}).
\]
Thus, 
\begin{align*}
  \sum_{t\in \Tc_i} (\hb_{i1}-\hb_{it})\s_{it}(\bgam_{i1}-\bgam_{it})&= \sum_{t\in \Nc_i} (\hb_{i1}-\hb_{it})\s_{it}(\bgam_{i1}-\bgam_{it})+\sum_{t\in\Tc_i-\Nc_i} (\hb_{i1}-\hb_{it})\s_{it}(\bgam_{i1}-\bgam_{it})\nonumber\\
  &\leq \sum_{t\in \Nc_i} (1-\nu)\s_{it}(\bgam_{i1}-\bgam_{it})+\frac{2\Gamma A Te^{-\RR\nu}}{\bggm} \sum_{t\in \Tc_i} \s_{it}(\bgam_{i1}-\bgam_{it})\\
  &\leq \left(1-\nu+\frac{2\Gamma A Te^{-\RR\nu}}{\bggm}\right)\sum_{t\in \Tc_i}\s_{it}(\bgam_{i1}-\bgam_{it})\\
 &\leq \left(1+\frac{2\Gamma A Te^{-\RR\nu}}{\bggm}\right)\sum_{t\in \Tc_i}\s_{it}(\bgam_{i1}-\bgam_{it}).
\end{align*}
Hence, choosing 
\begin{align}
R\geq\frac{1}{\nu\Theta}\log\left(\frac{8\Gamma AT}{\bggm\pi}\right)\label{R bound pi 1}
\end{align}
results in that
\begin{equation}\label{eqn:grad:difff3}
    \begin{split}
     &\sum_{t\in \Tc_i} (\hb_{i1}-\hb_{it})\s_{it}(\bgam_{i1}-\bgam_{it})  - (1+\frac{\pi}{2}) \sum_{t\in \Tc_i}\s_{it}(\bgam_{i1}-\bgam_{it}) \\
   &\leq\left(\frac{2\Gamma A Te^{-\RR\nu}}{\bggm}-\frac{\pi}{2}\right)\sum_{t\in \Tc_i}\s_{it}(\bgam_{i1}-\bgam_{it})\\
   &\leq -\frac{\pi}{4}\sum_{t\in \Tc_i}\s_{it}(\bgam_{i1}-\bgam_{it})\\
   &\leq-\frac{\pi}{4T}\bggm  e^{-\bar{R} (1-2\nu)}.      
    \end{split}
\end{equation}
Here, the last inequality follows from the fact that $\sum_{t\in \Tc_i}\s_{it}\geq \max_{t\in\Tc_i}s_{it}\geq\frac{e^{-\bar{R}(1-2\nu)}}{\sum_{t=1}^Te^{-\bar{R}(\hb_{i1}-\hb_{it})}}\geq e^{-\bar{R}(1-2\nu)}/T$.

From Assumption~\ref{assum:loss:prope}, we have $c_{\min}\leq-\ell'\leq c_{\max}$ for some positive constants $c_{\min}$ and $c_{\max}$. It follows from  \eqref{eqn:grad:difff2} and \eqref{eqn:grad:difff3} that 
\begin{align*}
-\frac{1}{n}\sum_{i}^n \ell_i' \cdot&\left(
      \sum_{t\in \Tc_i} (\hb_{i1}-_{it})\s_{it}(\bgam_{i1}-\bgam_{it})- \sum_{t\in \Tc_i} (1+0.5\pi)\s_{it}(\bgam_{i1}-\bgam_{it})\right)\\
      & \leq    c_{\max}2A\Gamma  T \Gamma e^{-\RR(1 +\frac{\pi}{2})}-\frac{c_{\min}}{nT}\cdot\frac{\pi\bggm}{4}e^{-\bar{R} (1-2\nu)}\\
      & \leq 0.
\end{align*}
Combing with \eqref{R bound pi 1}, this is guaranteed by 
choosing 
\[
  R\geq \max\left\{\frac{1}{\nu\Theta}\log\left(\frac{8\Gamma AT}{\bggm\pi}\right),\frac{1}{(2\nu+\pi/2)\Theta}\log\left(\frac{8n\Gamma AT^2 c_{\max}}{c_{\min}\bggm\pi}\right)\right\},
\]
where $\nu=\nu(\frac{\pi}{4A\Theta})$ depends only on $\pi$ and global problem variables. 

Combining this with the prior $R$ choice \eqref{R boundC0} (by taking maximum), we conclude with the statement.

\end{proof}

\subsection{Proof of Theorem~\ref{thm:local:gd}}
The proof of this theorem follows the proof of \cite[Theorem 3]{tarzanagh2023margin}. Let us denote the initialization lower bound as $R^0_\mu:=R$, where $R$ is given in the Theorem~\ref{thm:local:gd}'s statement. Consider an arbitrary value of  $\epsilon \in (0, \mu/2)$ and let $1/(1+\pi)=1-\epsilon$. We additionally denote $R_\eps\gets R_\pi\vee 1/2$ where $R_\pi$ was defined in Lemma~\ref{lem:local:corr}.  At initialization $\W(0)$, we set $\eps=\mu/2$ to obtain $R^0_\mu= R_{\mu/2}$, and provide the proof in four steps:
\\
\textbf{Step~1: There are no stationary points within $\Cc_{\mu,R^0_\mu}(\Ws)$.} We begin by proving that there are no stationary points within  $\Cc_{\mu,R^0_\mu}(\Ws)$. Let $(\Tc_i)_{i=1}^n$ denote  the sets of \neis as defined in Definition \ref{def loc opt}. We define $\Tcb_i=[T]-\Tc_i-\{\alpha_i\}$ as the tokens that are non-\neis. Additionally, let $\mu$ be defined as in \eqref{mu choice}. Then, since $R^0_\mu\geq \RR_\mu$ per Lemma \ref{lem:local:corr}, we can apply Lemma~\ref{local cond} to find that: For all $\V,\W\in  \Sc_{\mu}(\Ws)$ with $\tf{\W} \neq 0$ and $\tf{\W} \geq R^0_\mu$,  we have that $- \li\V, \nabla \Lc(\W)\ri$ is strictly positive.
\\
\textbf{Step~2:}  It follows from Lemma~\ref{lem:local:corr} that, there exists $ R_\epsilon\geq \bar{R}_\mu\vee 1/2$ such that all  $ \W \in \Cc_{\mu,R_\epsilon}(\Wm)$ satisfy
\begin{align}\label{eqn:neg:corr:local}
\iprod{-\nabla\mc{L}(\W)}      {\frac{\Wm}{\tf{\Wm}}} \geq (1-\epsilon)    \iprod{-\nabla \mc{L}(\W)}{\frac{\W}{\tf{\W}}}.
\end{align}
The argument below applies to a general $\eps\in(0,\mu/2)$. However, at initialization $\W(0)$, we set $\eps=\mu/2$ and, recalling above, initialization lower bound was defined as $R^0_\mu:= R_{\mu/2}$. To proceed, for any $\epsilon \in (0, \mu/2)$, we will show that after gradient descent enters the conic set $\Cc_{\mu,R_\eps}(\Ws)$ for the first time, it will never leave the set. Let $t_\eps$ be the first time gradient descent enters $\Cc_{\mu,R_\eps}(\Ws)$. In \textbf{Step 4}, we will prove that such $t_\eps$ is guaranteed to exist. Additionally, for $\eps\gets\mu/2$, note that $t_\eps=0$ i.e.~the point of initialization.
\\
\textbf{Step~3: Updates remain inside the cone $\Cc_{\mu,R_\eps}(\Ws)$.} 
By leveraging the results from \textbf{Step 1} and \textbf{Step 2}, we demonstrate that the gradient iterates, with an appropriate constant step size, starting from $\W(k_\eps) \in \Cc_{\mu,R_\eps}(\Ws)$, remain within this cone. 

We proceed by induction. Suppose that the claim holds up to iteration $k \geq k_\eps$. This implies that $ \W(k) \in \Cc_{\mu,R_\eps}(\Ws)$. Hence, recalling cone definition, there exists scalar $\mu=\mu(\bal) \in (0,1)$  and $R $ such that  $\tf{\W(k)}\geq R$, and
\begin{equation*}
\begin{split}
\left\langle \frac{\W(k)}{\tf{\W(k)}},\frac{\Wm}{\tf{\Wm}} \right\rangle  \geq 1-\mu. 
\end{split}
\end{equation*}
For all $k\geq 1$, let
\begin{align}\label{eqn:rho:def}
\rho (k) := - \frac{1}{1-\epsilon} \iprod{\nabla\mc{L}(\W(k))}
     {\frac{\Wm}{\tf{\Wm}}}.
\end{align}
Note that $\rho (k) >0$ due to \textbf{Step 1}. This together with the gradient descent update rule gives 
\begin{subequations}
\begin{equation}\label{eqn:localgd:1}
\begin{split}
     \left\langle \frac{\W(k+1)}{\tf{\W(k)}},\frac{\Wm}{\tf{\Wm}} \right\rangle  &=    \left\langle \frac{\W(k)}{\tf{\W(k)}} -\frac{\eta}{\tf{\W(k)}}\nabla \mc{L}(\W(k)), \frac{\Wm}{\tf{\Wm}} \right\rangle\\
      &\ge 1-\mu- \frac{\eta}{\tf{\W(k)}}\iprod{\nabla \mc{L}(\W(k))} {\frac{\Wm}{\tf{\Wm}}} \\
      & \geq 1-\mu +\frac{\eta \rho (k)  (1-\epsilon)}{\tf{\W(k)}}.
\end{split}
\end{equation}
Note that from Lemma~\ref{local cond}, we have $\left\langle \nabla \Lc(\W(k)),\W(k)\right\rangle<0$~ which implies that $\tf{\W(k+1)} \geq \tf{\W(k)}$.  This together with  $R_\eps$ definition and $\tf{\W(k)}\geq 1/2$ implies that
\begin{align*}
\tf{\W(k+1)}&\leq\frac{1}{{2\tf{\W(k)}}} \left(\tf{\W(k+1)}^2+\tf{\W(k)}^2\right)\\
& = \frac{1}{2\tf{\W(k)}} \left(2\tf{\W(k)}^2-2\eta\left\langle \nabla \Lc(\W(k)),\W(k)\right\rangle+\eta^2\tf{\nabla \Lc(\W(k))}^2\right)\\
       &\leq  \tf{\W(k)}- \frac{\eta}{\tf{\W(k)}}\left\langle \nabla \Lc(\W(k)),\W(k)\right\rangle + \eta^2 \tf{\nabla \Lc(\W(k))}^2, 
\end{align*}
which gives
\begin{equation}\label{eqn:localgd:2}
\begin{split}
  \frac{\tf{\W(k+1)}}{\tf{\W(k)}}& \leq  1- \frac{\eta}{\tf{\W(k)}}
       \left\langle \nabla \Lc(\W(k)),\frac{\W(k)}{\tf{\W(k)}} \right\rangle + \eta^2 \frac{\|\nabla \mc{L}(\W(k))\|^2}{\tf{\W(k)}}\\
& \leq 1- \frac{\eta}{(1-\epsilon)\tf{\W(k)}}  \iprod{\nabla\mc{L}(\W(k))}
     {\frac{\Wm}{\tf{\Wm}}}+ \eta^2 \frac{\|\nabla \mc{L}(\W(k))\|^2}{\tf{\W(k)}}\\
      & \leq  1 + \frac{\eta \rho (k) }{\tf{\W(k)}} + \frac{\eta^2\|\nabla \mc{L}(\W(k))\|^2}{\tf{\W(k)}}=:C_1(\rho (k) ,\eta).
\end{split}
\end{equation}
\end{subequations}
Here, the second inequality follows from \eqref{eqn:neg:corr:local} and \eqref{eqn:rho:def}.

Now, it follows from \eqref{eqn:localgd:1} and \eqref{eqn:localgd:2} that   
\begin{equation}\label{eqn:wt+1:cone}
\begin{split}
\left\langle \frac{\W(k+1)}{\|\W(k+1)\|},\frac{\Ws}{\|\Ws\|} \right\rangle   &\geq \frac{1}{C_1(\rho(k),\eta)} \left(1-\mu +\frac{\eta \rho(k)(1-\epsilon)}{\tf{\W(k)}}\right)\\
& = 1-\mu+ \frac{1}{C_1(\rho(k),\eta)} \left((1-\mu)(1-C_1(\rho(k),\eta)) +\frac{\eta \rho(k)(1-\epsilon)}{\tf{\W(k)}}\right)\\
& = 1-\mu+ \frac{\eta}{C_1(\rho(k),\eta)} \left((\mu-1)(\frac{ \rho(k)}{\tf{\W(k)}} + \frac{\eta\|\nabla \mc{L}(\W(k))\|^2}{\tf{\W(k)}}) +\frac{ \rho(k)(1-\epsilon)}{\tf{\W(k)}}\right)\\
& = 1-\mu+\frac{\eta}{C_1(\rho(k),\eta)} \left(\frac{\rho(k)(\mu -\epsilon)}{\tf{\W(k)}}
-   \eta (1-\mu) \frac{\|\nabla \mc{L}(\W(k))\|^2 }{\tf{\W(k)}}\right)\\
& \geq 1-\mu, 
\end{split}
\end{equation}
where the last inequality uses our choice of stepsize $\eta\leq 1/L_W$ in Theorem~\ref{thm:local:gd}'s statement. Specifically, we need $\eta$ to be small to ensure the last inequality. We will guarantee this by choosing a proper $R_\eps$ in Lemma \ref{lem:local:corr}. Specifically, Lemma \ref{lem:local:corr} leaves the choice of $C_0$ in $R_\eps$ lower bound of \eqref{R boundC0} open (it can always be chosen larger). Here, by choosing $C_0\gtrsim 1/L_{\W}$ will ensure $\eta\leq 1/L_W$ works well.
\begin{align}\label{eqn:zeta:mu}
\nonumber
\eta &\leq \frac{\mu}{2(1-\mu)(1-\frac{\mu}{2})}
 \frac{c}{C}  \frac{\Theta}{\bar{A}}   \frac{1}{\bar{A}C T}    e^{R_\mu^0\Theta/2} \\
& \leq \frac{\mu-\epsilon}{1-\mu} \cdot  \frac{1}{1-\epsilon} \cdot \frac{c}{C} \cdot \frac{\Theta}{\bar{A}} \cdot  \frac{1}{\bar{A}C T}  e^{R_\mu^0\Theta/2} \leq \frac{(\mu-\epsilon)}{1-\mu}  \frac{\rho(k) } { \|\nabla \mc{L}(\W(k))\|^2_F}.
\end{align}
Here, the first inequality uses our choice of $\epsilon \in (0, \mu/2)$ (see \textbf{Step 2}), and the last inequality is obtained from Lemma~\ref{local cond} since
\begin{align*}
    \frac{\rho(k) } { \tf{\nabla \mc{L}(\W(k))}} &= - \frac{1}{1-\epsilon} \iprod{ \frac{\nabla\mc{L}(\W(k))}{\tf{\nabla \mc{L}(\W(k))}}}
     {\frac{\Wm}{\tf{\Wm}}}  \geq \frac{1}{1-\epsilon} \cdot \frac{c}{C} \cdot \frac{\Theta}{\bar{A}},\\
         \frac{1} { \tf{\nabla \mc{L}(\W(k))}} &\geq \frac{1}{\bar{A}C \cdot \frac{1}{n} \sum_{i=1}^n  \left(1-\s_{i\alpha_i}\right)}  \geq    \frac{1}{ \bar{A} C T e^{-R_\mu^{0}\Theta/2}} 
\end{align*}
for some data dependent constrants $c$ and $C$, $\bar{A}=\max_{i\in[n],t,\tau\in[T]}\tn{(\x_{it}- \x_{i\tau})}~\tn{\z_i}$, and $\Theta=1/\tf{\Ws}$.

Next, we will demonstrate that the choice of $\eta$ in \eqref{eqn:zeta:mu} does indeed meet our step size condition as stated in the theorem, i.e., $\eta \leq 1/L_{\W}$. Recall that $1/(1+\pi)=1-\epsilon$, which implies that $\pi =\epsilon/(1-\epsilon)$. Combining this with \eqref{R boundC0}, we obtain:
\begin{align}\label{eqn:pitoC0}
R_\pi &\geq\frac{\max(2,\delta^{-1})}{\Theta}\log\left(\frac{C_0 T\Gamma A}{\pi\bggm}\right), \quad \textnormal{where} \quad C_0 \geq 64 \pi.\\
& \Rightarrow R_\epsilon \geq\frac{\max(2,\delta^{-1})}{\Theta}\log\left(\frac{ (1-\epsilon)C_0 T\Gamma A}{\epsilon\bggm}\right),\quad  \textnormal{where}   \quad C_0 \geq 64  \frac{\epsilon}{1-\epsilon}.
\end{align}
On the other hand, at the initialization, we have  $\epsilon=\mu/2$ which implies that 
\begin{align}\label{eqn:rmu:c0}
  R_{\mu}^0 \geq \frac{\max(2,\delta^{-1})}{\Theta}\log\left(\frac{ (2-\mu)C_0 T\Gamma A}{\mu\bggm}\right),  \quad  \textnormal{where} \quad  C_0 \geq 64  \frac{\mu}{2(1-\frac{\mu}{2})}.
\end{align}
In the following, we will determine  a lower bound on  $C_0$ such that our step size condition in Theorem~\ref{thm:local:gd}'s statement, i.e., $\eta \leq 1/L_{\W}$, is satisfied. Note that for the choice of $\eta$ in \eqref{eqn:zeta:mu} to meet the condition $\eta \leq 1/L_{\W}$, the following condition must hold:
\begin{equation}
 \frac{1}{L_{\W}}\leq 
\frac{\mu}{ (2-\mu)} \frac{1}{C_2T} e^{R_\mu^0\Theta/2}  \Rightarrow R_\mu^0 \geq  \frac{2}{\Theta} \log  \left(\frac{1}{L_{\W}}   \frac{2-\mu}{\mu} C_2 T\right).
\end{equation}
where
$C_2 = (1-\mu)  \frac{  \bar{A}^2 C^2 }{ \Theta c}$.

This together with \eqref{eqn:rmu:c0} implies that 
\begin{align}\label{eqn:pitoC02}
\frac{C_0 \Gamma A}{\bggm}  & \geq (1-\mu) \frac{C_2}{L_{\W}}   \Rightarrow  C_0 \geq  \max \left( \frac{(1-\mu)C_2}{L_{\W}}    \frac{\bggm }{\Gamma A},  \frac{64\mu}{2-\mu} \right). 
\end{align}
Therefore, with this lower bound on $C_0$, the step size bound in \eqref{eqn:zeta:mu} is sufficiently large to ensure that $\eta \leq 1/L_{\W}$ guarantees \eqref{eqn:wt+1:cone}. 

Hence,  it follows from \eqref{eqn:wt+1:cone} that $\W(k+1) \in \Cc_{\mu,R_\eps}(\Ws)$.
\\
\textbf{Step 4: The correlation of $\W(k)$ and $\Wm$ increases over $k$.} 
From Step 3, we have that all iterates remain within the initial conic set i.e.~$\W(k)\in\Cc_{\mu,R^0_\mu}(\Ws)$ for all $k\geq 0$. Note that it follows from Lemma~\ref{local cond} that  $\li\nabla\Lc(\W), \Ws/\tf{\Ws}\ri<0$, for any finite $\W \in \Cc_{\mu,R^0_\mu}(\Ws)$. Hence, there are no finite critical points $\W \in \Cc_{\mu,R^0_\mu}(\Ws)$, for which $\nabla \mc{L} (\W)=0$. Now, based on Lemma~\ref{lem:grad:descent}, which guarantees that $\nabla\Lc(\W(k))\rightarrow 0$, this
implies that $\left\Vert \W\left(t\right)\right\Vert_F\rightarrow\infty$. Consequently, for any choice of $\eps\in (0,\mu/2)$ there is an iteration $k_\eps$ such that, for all $k\geq k_\eps$, $\W(k)\in\Cc_{\mu,R_\eps}(\Ws)$. Once within $\Cc_{\mu,R_\eps}(\Ws)$, multiplying both sides \eqref{eqn:neg:corr:local} by the stepsize $\eta$ and using the gradient descent update, we get
\begin{equation*}
\begin{split}
     \left\langle \W(k+1)-\W(k),\frac{ \Wm}{\tf{\Wm}} \right\rangle &\geq  (1-\epsilon) \left\langle \W(k+1)-\W(k), \frac{\W(k)}{\tf{\W(k)}}\right\rangle\\
     &= \frac{(1-\epsilon)}{2\tf{\W(k)}}\left(\tf{\W(k+1)}^2- \tf{\W(k)}^2-\tf{\W(k+1)-\W(k)}^2\right) \\
     & \geq (1-\epsilon)\left( \frac{1}{2\tf{\W(k)}} \left(\tf{\W(k+1)}^2- \tf{\W(k)}^2\right)-\tf{\W(k+1)-\W(k)}^2\right) \\
     & \geq (1-\epsilon)\left(\tf{\W(k+1)}- \tf{\W(k)}-\tf{\W(k+1)-\W(k)}^2\right) \\
          & \geq (1-\epsilon)\Big(\tf{\W(k+1)}- \tf{\W(k)}- 2\eta  \left(\mc{L}(\W(k))-\mc{L}(\W(k+1))\right) \Big).
\end{split}
\end{equation*}
Here, the second inequality is obtained from  $\tf{\W(k)}\geq 1/2$; the third inequality follows since  for any $a, b >0$, we have $  (a^2-b^2)/(2b) -  (a-b) \geq 0$; and the last inequality  uses Lemma~\ref{lem:grad:descent}.


Summing the above inequality over $k\geq k_\eps$ gives 
\begin{align*}
      \left\langle\frac{\W(k)}{\tf{\W(k)}}, \frac{\Wm}{\tf{\Wm}} \right\rangle \ge1-\epsilon+ \frac{C(\epsilon,\eta)}{\tf{\W(k)}},    \qquad \W(k)\in\Cc_{\mu,R_\eps}(\Ws),   
\end{align*}
where $\mathcal{L}_{\star}\leq\mathcal{L}\left(\W\left(k\right)\right)$ for all $k\geq 0$, and 
\begin{equation*}
C(\epsilon,\eta)= \left\langle \W(k_\eps), \frac{ \Wm}{\tf{\Wm}}\right\rangle-(1-\epsilon)\tf{\W(k_\eps)} -2\eta (1-\epsilon) (\mc{L}(\W(k_\eps))-\mathcal{L}_{\star}).
\end{equation*}
Consequently, as $k\rightarrow\infty$
    \begin{align*}
      \liminf_{k\to\infty}\iprod{\frac{\W(k)}{\tf{\W(k)}}}{\frac{\Wm}{\tf{\Wm}}}\ge1-\epsilon, \qquad \W(k)\in\Cc_{\mu,R_\eps}(\Ws).
    \end{align*}
Since  $\epsilon \in (0, \mu/2)$  is arbitrary, we get $\W(k)/\tf{\W(k)}\to  \Wm/\tf{\Wm}$.
 $\qed$

\section{Convergence of Regularization Path for Sequence-to-Sequence Setting}\label{sec multioutput}



In this section, we provide proofs for the regularization path analysis. We first provide a more general formulation of the optimization problem that allows for regressing multiple token outputs. To distinguish from \eqref{eqn:erm:w}\&\eqref{eqn:erm:kq}, let us call this more general version Sequence Empirical Risk Minimization (SERM).


\noindent\textbf{Problem definition:} 
Rather than a single input sequence $\X$, let us allow for two input sequences $\X\in\R^{T\times d}$ and $\Z\in\R^{K\times d}$ with $(\x_t)_{t=1}^T$ and $(\z_k)_{k=1}^K$. The cross-attention admits $\X,\Z$ and outputs $K$ tokens. We will also allow for $K$ separate prediction heads $(h_k)_{k=1}^K$ for individual cross attention outputs which strictly generalizes the setting where we used single prediction head $h(\cdot)$. Denote the training labels associated to each token as $\Y=(Y_\ik)_{i=1}^K$. Given $n$ samples $(\Y_i,\X_i,\Z_i)_{i=1}^n$, for a decreasing loss function $\ell(\cdot)$, minimize the empirical risk by the prediction of first attention output either univariate ($\W\in\R^{d\times d}$) or bivariate ($\Kb,\Qb\in\R^{d\times m}$) fashion:
\begin{align}
\Lc(\W)&=\frac{1}{n}\sum_{i=1}^n\sum_{k=1}^K \ell(Y_\ik\cdot h_k( \X_i^\top \sft{\X_i\W\z_\ik})),\label{serm-w}\tag{SERM-W}\\
\Lc(\Kb,\Qb)&=\frac{1}{n}\sum_{i=1}^n \ell(Y_\ik\cdot h_k( \X_i^\top \sft{\X_i\Kb\Qb^\top\z_\ik})).\label{serm-kq}\tag{SERM-KQ}
\end{align}
In order to recover the single-output self-attention model, we can simply set $K=1$ and $\z_{i1}\gets\x_{i1}$ and $h_k\gets h$. To proceed, we introduce the more general version of the \eqref{eqn:sattnsvm} problem, which we refer to as \emph{Sequential Cross-Attention SVM} to preserve consistent phrasing. Suppose $\Kb,\Qb\in\R^{d\times m}$ with $m\leq d$ and let $\Rcm$ denote the set of rank-$m$ matrices in $\R^{d\times d}$. Given indices $\bal=(\aik)_{\ik=(1,1)}^{(n,K)}$, consider the SVM with $\dm$-norm constraint
\begin{equation}\tag{S\name}
 \Wm_{\dm,\bal}\in\arg\min_{\W\in\Rcm}\|\W\|_{\diamond}
 \quad \text{s.t.}   \quad \min_{t\neq \aik}(\x_{i\aik}-\x_\itt)^\top\W\z_\ik\geq 1\quad\forall\quad i\in[n],k\in[K]\label{seqattnsvm}.
\end{equation}
When solution is non-unique, we denote the solution set by $\Wcs(\bal)$. In what follows, we denote $\F_\ikt:=\x_\itt\z_{\ik}^\top$ and given $\bal$, we denote $\xa_\ik=\x_{i\alpha_\ik}$ and $\Fa_\ik:=\x_{i\alpha_\ik}\z_{\ik}^\top$. With this notation, we can equivalently write
\begin{equation}\tag{S\name'}
 \Wm_{\dm,\bal}\in\arg\min_{\W\in\Rcm}\|\W\|_{\diamond}
 \quad \text{s.t.} \quad \min_{t\neq \aik}\inn{\Fa_\ik-\F_\ikt,\W}\geq 1\quad\forall~ i\in[n],k\in[K]\label{seqattnsvm2}.
\end{equation}

\begin{definition}[\Neis and Locally-Optimal Indices]\label{seq loc opt} Fix token indices $\bal=(\alpha_\ik)\ikix$ for which \eqref{seqattnsvm} is feasible to obtain $\Wma:= \Wm_{\dm,\bal}$. Define token scores as
\[
\bgam_\ikt=Y_\ik\cdot h_k(\x_\itt),\quad \bga_\ik:=\bgam_{ik\alpha_\ik}=Y_\ik\cdot h_k(\xa_\ik).
\]
Consider tokens $\Tc_\ik\subset[T]$ such that $\inn{\Fa_\ik-\F_\ikt,\Wma}=1$ for all $t\in\Tc_\ik$. $\Tc_\ik$ is allowed to be an empty set. We refer to $\Tc_\ik$ as \neis of $\Fa_\ik=\xa_\ik\z_\ik^\top$ and define its complement $\Tcb_\ik=[T]-\Tc_\ik-\{\al_\ik\}$.  Additionally, token indices $\bal=(\alpha_\ik)\ikix$ are called locally-optimal if for all $i\in[n],k\in[K]$ and  $t\in\Tc_\ik$, token scores obey $\bga_\ik>\bgam_\ikt$. Associated $\Wma$ is called a locally-optimal direction. Finally, let $\op_\ik\in\arg\max_{t\in[T]}\bgam_\ikt$ be the optimal indices and define the associated $\Ws(\op)$ to be a globally-optimal direction.
\end{definition}

\begin{lemma}[Mapping regularization path of $(\Kb,\Qb)$ to $\W$]\label{kqw mapping} Let $\Kb,\Qb\in\R^{d\times m}$ and consider regularization path solutions of \eqref{serm-w} and \eqref{serm-kq}
\begin{align}
&\Wb_R\in\underset{\W\in\Rcm:\tnuc{\W}\leq R}{\arg\min}\Lc(\W)\label{Wpath}\\
&\Kbb_R,\Qbb_R\in\underset{\tf{\Kb}^2+\tf{\Qb}^2\leq 2R}{\arg\min}\Lc(\Kb,\Qb).\label{KQpath}
\end{align}
For all $R\geq 0$, there is a one-to-one map between the set of solutions $\Wb_R$ of \eqref{Wpath} and $\Kbb_R\Qbb_R^\top$ of \eqref{KQpath}.
\end{lemma}
\begin{proof} To prove the mapping, first fix a $\Wb_R$ solution with rank $m$, set $\Lc_F=\Lc(\Wb_R)$ and show the existence of $\Kb,\Qb$ with $\Kb\Qb^\top=\Wb_R$ feasible for \eqref{KQpath} and $\Lc(\Kb,\Qb)\leq \Lc_F$. Use the singular value decomposition $\Wb_R=\Ub\bSi\Vb^\top$ with $\bSi\in\R^{m\times m}$ being diagonal matrix of singular values. Set $\Kb=\Ub\sqrt{\bSi}$ and $\Qb=\Vb\sqrt{\bSi}$. Observe that $\Kb\Qb^\top=\W$ and
\[
\tf{\Kb}^2=\tf{\Qb}^2=\sum_{i=1}^m\sqrt{\bSi_{ii}}^2=\tnuc{\Wb_R}\leq R.
\]
Thus, $\Kb,\Qb$ achieves $\Lc(\Kb,\Qb)=\Lc_F$. Conversely, given $\Kbb_R,\Qbb_R$ with $\Lc_\st=\Lc(\Kbb_R,\Qbb_R)$, $\W=\Kbb_R\Qbb_R^\top$ obeys $\Lc(\W)=\Lc_\st$ and, using the standard nuclear norm inequality, we have
\[
\tnuc{\W}=\tnuc{\Kbb_R\Qbb_R^\top}\leq \frac{1}{2}(\tf{\Kbb_R}^2+\tf{\Qbb_R}^2)=R.
\]
This shows $\W$ is feasible for \eqref{Wpath}. Combining the two findings above, we find that optimal costs are equal ($\Lc_\st=\Lc_F$) and for any $(\Kbb_R,\Qbb_R)$ solution there exists a $\Wb_R$ solution and vice versa.
\end{proof}


To proceed with our analysis, let us define the set of optimal solutions $\Wm_{\bal}:=\Ws_{\dm,\bal}$ to \eqref{seqattnsvm}. Let us denote this set by $\Wcs_{\bal}:=\Wcs_\dm(\bal)$. Note that, if the $\dm$-norm is not strongly-convex, $\Wcs_{\bal}$ may not be singleton. 


\subsection{Local regularization path and proof of Theorem \ref{local RP thm1}}\label{app:local:RP:thm1}

We first recall \emph{local regularization path} which solves the $\dm$-norm-constrained problem over a conic slice $\Cc$, namely $\Wb(R)=\min_{\td{\W}\leq R,\W\in\Cc}\Lc(\W)$. We will show that proper local RP directionally converge to locally-optimal directions. Setting the cone to be the rank-$m$ manifold $\Rcm$ (and more specifically the set of all matrices $\R^{d\times d}$), this will also establish the convergence of global RP to the globally-optimal direction.

Our cone definition $\con{\bal}$ is induced by a token selection $\bal=(\al_\ik)\ikix$ and has a simple interpretation: It prioritizes tokens with lower score than $\bal$ over tokens with high-score than $\bal$. This way lower score tokens create a barrier for $\bal$ and prevents optimization to move towards higher score tokens.
\begin{definition}[Low\&High Score Tokens and Separating Cone]\label{HL cone def} Given $\al\in[T]$, input sequence $\X$ with label $Y$, $h(\cdot):\R^d\rightarrow\R$, and score $\bgam_t=Y\cdot h(\x_t)$ for all $t\in[T]$, define the low and high score tokens as
\[
\low=\left\{t\in[T]\bgl \bgam_t<\bgam_\al\},\quad \high=\{t\in[T]-\{\alpha\}\bgl \bgam_t\geq \bgam_\al \right\}.
\]
For input $\X_\ik$ and index $\alpha_\ik$, we use the shorthand notations $\lowi,\higi$. Finally define $\con{\bal}$ as
\begin{align}
\con{\bal}=\left\{\W\in\Rcm\bgl \min_{i\in[n]}\max_{t\in\lowi}\min_{\tau\in\higi} \inn{\F_\ikt-\F_\iktt,\W}\geq \eps\tf{\W} \right\}.\label{cone alpha eq}
\end{align}
\end{definition}

\begin{lemma}\label{lemma cone} Consider the cone definition of \eqref{cone alpha eq} and suppose an SVM solution $\Wma$ exists. If indices $\bal$ are locally-optimal, $\Wma\in \con{\bal}$ for all sufficiently small $\eps>0$. Otherwise, $\Wma\not\in \con{\bal}$ for all $\eps>0$. Additionally, suppose optimal indices $\op_\ik\in\arg\max_{t\in[T]}\bgam_\ikt$ are unique. Then, $\con{\opt}=\Rcm$.
\end{lemma}
\begin{proof} Suppose $\bal$ is locally optimal. Observe that, thanks to local optimality, $\Wma$ obeys
\[
\min_{t\in\Tc_\ik}\inn{\F_\ikt,\Wma}>\max_{\tau\not\in\Tc_\ik\cup\{\al_\ik\}}\inn{\F_\iktt,\Wma},
\]
for all $i\in[n]$. Next, observe that $\Tc_\ik\subseteq \lowi$ and $\higi\subseteq\Tcb_\ik=[T]-\Tc_\ik-\{\al_\ik\}$. Thus, the inequality \eqref{cone alpha eq} holds for small enough $\eps>0$.

Conversely, suppose $\bal$ is not locally-optimal. Fix \nei $t\in\Tc_\ik$ with $t\in \higi$. Since $t\in\Tc_\ik$, observe that
\[
\inn{\F_\ikt,\Wma}\geq \max_{\tau\neq \al_\ik}\inn{\F_\iktt,\Wma}.
\]
In other words, for this $i\in[n]$, we found
\[
\max_{\tau\in\lowi} \inn{\F_\iktt-\F_\ikt,\Wma}\leq 0,
\]
violating \eqref{cone alpha eq} definition for any $\eps>0$. To show the final claim, observe that, setting $\bal:=\op$, we have that $\higi=\emptyset$ for all $i\in[n]$ as $\op_\ik$ are unique optimal indices. Thus, there is no constraint enforced on the cone definition in \eqref{cone alpha eq} making it equal to the rank-$m$ manifold $\Rcm$.
\end{proof}

Our main assumption regarding prediction head is a monotonicity condition which is a strict generalization of linearity: We ask for $h$ to preserve the order of token scores under convex combinations.
\begin{assumption}[$h$ preserves the top score]\label{ass cvx seq} The functions $(h_k)_{k=1}^K$ are $L_h$-Lipschitz in Euclidean distance. Given $\bal=(\alpha_\ik)\ikix$, there exists a scalar $c:=c_\bal>0$ such that for all $i\in[n],k\in[K]$ the following holds: Consider any convex combination 
\[ \x(\s)=\sum_{t\in \lowi\cup\{\alpha_\ik\}}s_t\cdot\x_\itt\quad\text{where}\quad \sum_{t\in \lowi\cup\{\alpha_\ik\}}s_t=1,~s_t\geq 0.
\] 
We have that $Y_\ik\cdot h_k(\x(\s))\leq \bga_\ik-c(1-s_{\alpha_\ik})$ where $\bga_\ik=Y_\ik\cdot h_k(\xa_{\ik})$ is the score of $\alpha_\ik$.
\end{assumption}
This condition states that \textbf{convex combinations of tokens with scores lower than $\alpha_\ik$ cannot achieve a score higher than $\alpha_\ik$}. Here, $1-s_{\alpha_\ik}$ term denotes the total share of non-optimal tokens. We require this condition to hold over the training dataset rather than the full domain $\R^d$. Crucially, it is a strict generalization of the linearity assumption: Any linear $h_k$ satisfies Assumption \ref{ass cvx seq} by setting $c_\bal>0$ to be the difference between the score of $\alpha_\ik$ and the largest score within $\lowi$ i.e.
\begin{align}
c_\bal:=\min_{i\in[n],k\in[K]}\{\bga_\ik-\max_{t\in\lowi} \bgam_\ikt\}>0.\label{c choice}
\end{align}
This can be seen by writing $h_k(\x(\s))=\sum_{t\in \lowi\cup\{\alpha_\ik\}}s_t\cdot \bgam_\ikt=\bga_\ik+\sum_{t\in \lowi}s_t\cdot (\bgam_\ikt-\bga_\ik)\leq \bga_\ik-c_\bal(1-s_{\alpha_\ik})$.
To provide a nonlinear example, consider the setting all labels are $Y_\ik=1$ and $h$ is an arbitrary convex function. Thanks to convexity, we can write $h(\x(\s))\leq \sum_{t=1}^Ts_t\cdot h(\x_t)=\sum_{t\in \lowi\cup\{\alpha_\ik\}}s_t\cdot \bgam_\ikt$. Thus, we can use the same choice of $c_\bal$ in \eqref{c choice}. 

We remark that, in Section \ref{sec:multi}, we derive formulae describing general inductive bias of attention without enforcing any assumption on $h$. These formulae allow for arbitrary output tokens generated by the transformer model trained by gradient descent. This includes the setting where gradient descent selects and composes multiple tokens from each sequence rather than a single token $\alpha_\ik$.

The following result is our main theorem regarding the convergence of regularization path to locally-optimal directions when restricted over the cone $\con{\bal}$.
\begin{theorem} [Convergence of Local Regularization Path]\label{local RP thm} Suppose \eqref{seqattnsvm} is feasible and $\bal=(\al_\ik)\ikix$ are locally-optimal token indices. Suppose Assumptions \ref{assum:loss:prope}\&\ref{ass cvx seq} hold. Recall $\con{\bal}$ of \eqref{cone alpha eq} and consider the norm-constrained cone \[
\Ccd:=\con{\bal}\bigcap\{\W\bgl \td{\W}\geq R_0\}.
\]
Define the conic regularization path $\wrb{R}=\min_{\Ccd,\td{\W}\leq R}\Lc(\W)$. Let $\Wcs_{\bal}$ be its set of minima and $\xdm>0$ be the associated margin i.e.~$\xdm=1/\td{\Wcs_{\bal}}$. For any sufficiently small $\eps>0$ and sufficiently large $R_0= \order{1/\eps}>0$, $\lim_{R\rightarrow\infty} \dist{\frac{\wrb{R}}{R\xdm},\Wcs_{\bal}}=0$. Additionally, suppose optimal indices $\op=(\op_\ik)\ikix$ are unique and set $\bal\gets\op$. Then, the same RP convergence guarantee holds with $\Ccd=\Rcm$.
\end{theorem}
\begin{proof} We will prove that $\wrb{R}$ is the optimal direction and also $\td{\wrb{R}}\rightarrow \infty$. Define the absolute constant 
\[
\cdm=\min_{\td{\W}=1}\tf{\W}.
\]
This guarantees that for any $\W$ we have $\tf{\W}\geq \cdm\td{\W}$. Also denote $\epsd=\cdm\eps$. Let us first determine the $\eps$ parameter: Fix $\Wma\in\Wcs_{\bal}$. For general $\bal$, we can choose any $\eps>0$ that is sufficiently small to guarantee $\Wma\in\con{\bal}$ based on Lemma \ref{lemma cone}. For $\bal=\op$, our analysis will entirely avoid using $\eps$, specifically, observe that $\con{\bal}=\Rcm$ based on Lemma \ref{lemma cone}.
    
\noindent\textbf{Step 1:} Let us first prove that $\wrb{R}$ achieves the optimal risk as $R\rightarrow\infty$ -- rather than problem having finite optima. Define norm-normalized $\Wsb=\xdm\Wma$. Note that $\Wma$ separates tokens $\bal$ from rest of the tokens for each $i,k\in[n]\times [K]$. Thus, we have that
\begin{align}
\lim_{R\rightarrow\infty}\Lc(\wrb{R})\leq\lim_{R\rightarrow\infty}\Lc(R\cdot\Wsb):=\Lc_\star= \frac{1}{n}\sum_{i=1}^n\sum_{k=1}^K\ell(\bga_\ik).\label{asymp loss}
\end{align}
On the other hand, for any choice of $\W\in \con{\bal}$, set $\x^{\W}_\ik=\sum_{t=1}^T \sft{\X_i\W\z_\ik}_t\x_t$. Set softmax probabilities $\sik=\sft{\X_i\W\z_\ik}$. Recalling $\lowi,\higi$ definitions, we can decompose the attention features as
\begin{align}
\x^{\W}_\ik=\sik_{\al_\ik}\xa_\ik+\sum_{t\in\lowi}\sik_t\x_\itt+\sum_{\tau\in\higi}\sik_\tau\x_{\ittt}.
\end{align}
When $\bal=\op$, note that we simply have $\higi=\emptyset$. This will be important for setting $R_0=0$ and $\Ccd=\Rcm$ in the proof for $\op$ indices.

Set $\bgg_\ikt=\bgam_\ikt-\bga_\ik=Y_\ik\cdot (h_k(\x_\itt)-h_k(\xa_\ik))$. Building on $L_h$-Lipschitzness of the prediction head $h_k(\cdot)$, we define
\begin{align}
&B:=\max_{i\in[n],k\in[K]}\max_{t,\tau\in[T]}L_h\cdot \tn{\x_\itt-\x_\ittt}\geq |\bgg_\ikt|.\label{BB eq}
\end{align}
Define $P^\ik:=\sum_{t\in\lowi}\sik_t$, $Q^\ik:=\sum_{t\in\higi}\sik_t$, and $\bgam^{\W}_\ik=Y_\ik\cdot h_k(\x^{\W}_\ik)$. Also set temporary variables $\x'=(\sik_{\al_\ik}+Q^\ik)\xa_\ik+\sum_{t\in\lowi}\sik_t\x_\itt$ and $\bgam'=Y_\ik\cdot h_k(\x^{\W}_\ik)$. 
Using Assumption \ref{ass cvx seq} on $\x'$ and noticing $P^\ik=1-\sik_{\al_\ik}-Q^\ik$, observe that
\[ 
|\bgam^{\W}_\ik-\bgam'|\leq BQ^\ik\quad\text{and}\quad \bgam'\leq \bga_\ik-c_\bal P^\ik. 
\] 
Recall from \eqref{c choice} that, when $h_k$ are linear functions, $c_\bal$ can be chosen as 
\begin{align*}
c_\bal:=\min_{i\in[n],k\in[K]}\min_{t\in\lowi}-\bgg_\ikt>0.
\end{align*}

To summarize, applying Assumption \ref{ass cvx seq}, we obtain the following score inequalities
\begin{align}\label{score decomp}
&\bgam^{\W}_\ik\leq \bga_\ik-c_\bal P^\ik+BQ^\ik,\\
&|\bgam^{\W}_\ik-\bga_\ik|\leq L_h\tn{\x^{\W}_\ik-\xa_\ik}\leq L_h \sum_{t\neq \al_\ik}\sik_t\tn{\x_\ikt-\xa_\ik}\leq B(1-\sik_{\al_\ik}).\label{lip score gap}
\end{align}
We will use the $\bgam^{\W}_\ik-\bga_\ik$ term in \eqref{score decomp} to evaluate $\W$ against the reference loss \eqref{asymp loss}. Let $\abik=\X_i\W\z_\ik$. Now since $\W\in \con{\bal}$, there exists $t\in \lowi$ obeying $\abik_t-\max_{\tau\in\higi} \abik_\tau\geq \eps \tf{\W}\geq \epsd\td{\W}$. Denote $D^\ik:=(\sum_{t\in [T]}e^{\abik_t})^{-1}$ to be the softmax denominator i.e.~sum of exponentials. We find that,
\begin{align}
Q^\ik=\sum_{\tau\in\higi}\sik_\tau=D^\ik\sum_{\tau\in\higi}e^{\abik_\tau}\leq D^\ik Te^{\abik_t-\eps\tf{\W}}\leq Te^{-\epsd\td{\W}}P^\ik.\label{qikeq}
\end{align}
Consequently, the score difference obeys
\[
\bgam^{\W}_\ik-\bga_\ik\leq BQ^\ik-c_\bal P^\ik\leq (BTe^{-\epsd\td{\W}}-c_\bal)P^\ik.
\]
Above, the right hand side is strictly negative as soon as $\td{\W}\geq R_0:=\frac{1}{\epsd}\log\frac{BT}{c_\bal}$. Note that, this condition applies to all $(i,k)\in[n]\times [K]$ pairs uniformly for the same $R_0$. Consequently, for any $\td{\W}\geq R_0$, for all $i,k$ and $\W\in \con{\bal}$, we have that $\bgam^{\W}_\ik<\bga_\ik$. Additionally, when $\bal=\op$, note that $Q^\ik=0$ since $\higi=\emptyset$. Thus, $R_0=0$ suffices to ensure $\bgam^{\W}_\ik<\bga_\ik$. Using the strictly-decreasing nature of $\ell$, we conclude with the fact that for all (finite) $\W\in \con{\bal}$, 
\[
\Lc(\W)=\frac{1}{n}\sum_{i=1}^n\sum_{k=1}^K\ell(\bgam^{\W}_\ik)> \Lc_\st=\frac{1}{n}\sum_{i=1}^n\sum_{k=1}^K\ell(\bga_\ik),
\]
which implies $\td{\wrb{R}}\rightarrow\infty$.

\noindent\textbf{Step 2:} To proceed, we show that $\wrb{R}$ converges in direction to $\Wcs_{\bal}$. Suppose this is not the case i.e.~convergence fails. We will obtain a contradiction by showing that $\Wsb_R=R\cdot\Wsb$ achieves a strictly superior loss compared to $\wrb{R}$. Also define the normalized parameter $\wrt{R}=\frac{\wrb{R}}{R\xdm}$ and $\W'=\frac{\wrb{R}}{\td{\wrb{R}}\xdm}$. Note that $\wrt{R}$ is obtained by scaling down $\W'$ since $\td{\wrb{R}}\leq R$ and $\W'$ obeys $\td{\W'}=\td{\Wma}$.

Since $\wrt{R}$ fails to converge to $\Wcs_{\bal}$, for some $\delta>0$, there exists arbitrarily large $R>0$ such that $\dist{\wrt{R},\Wcs_{\bal}}\geq \delta$. This translates to the suboptimality in terms of margin constraints as follows: First, distance with respect to the $\dm$-norm obeys $\distd{\wrt{R},\Wcs_{\bal}}\geq \delta$ for some updated $\delta\gets \cdm\delta$. Secondly, using triangle inequality,
\[ 
\text{This implies that either~~~}\td{\wrt{R}}\leq \td{\Wma}-\delta/2\text{~~~or~~~}\distd{\W',\Wcs_{\bal}}\geq \delta/2.
\]
In either scenario, $\wrt{R}$ strictly violates one of the margin constraints of \eqref{seqattnsvm}: If $\td{\wrt{R}}\leq \td{\Wma}-\delta/2$, then, since the optimal SVM objective is $\td{\Wma}$, there exists a constraint $(i,k)$ for which $\inn{\Fa_{\ik}-\F_{\ikt},\wrt{R}}\leq 1-\frac{\delta}{2\td{\Wma}}$. If $\distd{\W',\Wcs_{\bal}}\geq \delta/2$, then, $\W'$ has same SVM objective but it is strictly bounded away from the solution set. Thus, for some $\eps:=\eps(\delta)>0$, $\W'$ and its scaled down version $\wrt{R}$ strictly violate an SVM constraint achieving margin $\leq 1-\eps$. Without losing generality, suppose $\wrt{R}$ violates the first constraint. Thus, for a properly updated $\delta>0$ (that is function of the initial $\delta>0$) and for $(i,k)=(1,1)$ and some \nei $\tau\in \Tc_\oo$,
\begin{align}
\inn{\Fa_{\oo}-\F_{\oo t},\wrt{R}}\leq 1-\delta.\label{margin violate}
\end{align}
Now, we will argue that this will lead to a contradiction by proving $\Lc(\Wsb_R)<\Lc(\wrb{R})$ for sufficiently large $R$.

To obtain the result, we establish a refined softmax probability control as in Step 1 by studying distance to $\Lc_\star$. Following \eqref{score decomp}, denote the score function at $\wrb{R}$ via $\bgam^R_\ik=\bgam_\ik^{\wrb{R}}$ as shorthand notation. Similarly, let $\sir_\ik=\sft{\abr_\ik}$ with $\abr_\ik=\X_i\wrb{R}\z_\ik$. Set the corresponding notation for the reference parameter $\Wsb_R$ as $\bgam^\st_\ik,\s^\st_\ik,\ab^\st_\ik$. 

Critically, recall the above inequalities \eqref{qikeq} that applies to both $\W\in\{\wrb{R},\Wsb_R\}\subset\con{\bal}$ for an index $(i,k)$ and \nei $t\in\Tc_\ik$
\begin{align}
\nonumber 
Q^\ik&=\sum_{\tau\in\higi}\s_\iktt=D^\ik\sum_{\tau\in\higi}e^{\ab_{\iktt}} \\
&\leq D^\ik Te^{\ab_\ikt-\epsd\td{\W}}\leq Te^{-\epsd\td{\W}}P^\ik\leq Te^{-\epsd\td{\W}}(1-\s_{\ik\al_\ik}), \label{qik bound}
\end{align}
where $P^\ik=\sum_{\tau\in\lowi}\s_{\iktt}$ and $P^\ik+Q^\ik= 1-\s_{\ik\al_\ik}$.

Note that, setting $R_0\geq \order{1/\epsd}=\order{1/\eps}$, we guarantee that, for any $(i,k)\in[n]\times [K]$
\begin{align}
P^\ik\geq Q^\ik\implies P^\ik \geq 0.5(1-\s_{\ik\al_\ik}). \label{pik bound}
\end{align}
Additionally, when $\bal=\op$, note that $Q^\ik=0$ since $\higi=\emptyset$. Thus, $R_0=0$ suffices to ensure \eqref{pik bound}.

To proceed, recall that $R\geq \td{\wrb{R}}\geq R_0$ by definition since $\wrb{R}\in \Ccd$ and recall $\xdm:=1/\td{\Wma}$. Equipped with these, we note the following softmax inequalities on the selected tokens $\al_\ik$
\begin{align}
&\s^\st_{\ik\al_\ik}\geq \frac{1}{1+Te^{-R\xdm}}\geq 1-Te^{-R\xdm}\quad \text{for all}\quad (i,k)\in[n]\times [K], \label{salpha bounds}\\
&s^R_{\ik\al_\ik}\leq \frac{1}{1+e^{-(1-\delta)\td{\wrb{R}}\xdm}}\leq \frac{1}{1+e^{-(1-\delta)R\xdm}}\quad\text{for}\quad (i,k)=(1,1).\nn
\end{align}
The former inequality is thanks to $\Wma$ achieving $\geq 1$ margins on all tokens $[T]-\al_\ik$ and the latter arises from the $\delta$-margin violation of $\wrb{R}$ at $(i,k)=(1,1)$ i.e.~Eq.~\eqref{margin violate}. Since $\ell$ is strictly decreasing with Lipschitz gradient and the scores are upper/lower bounded by an absolute constant (as tokens are bounded, $(h_k)_{k=1}^K$ are Lipschitz, and both are fixed), we know that $\cop\geq -\ell'(\bgam_\ik^{\W})\geq \cdn$ for some constants $\cop>\cdn>0$. Thus, following Eq.~\eqref{BB eq} and the score decomposition \eqref{score decomp}, and using \eqref{qik bound},\eqref{pik bound},\eqref{salpha bounds} we can write
\begin{align}
\nonumber
\Lc(\wrb{R})-\Lc_\star&\geq \frac{1}{n}[\ell(\bgam_\oo^{\wrb{R}})-\ell(\bga_\oo)]\geq \frac{\cdn}{n}(\bga_{\oo}-\bgam_\oo^{\wrb{R}})\\
&\geq \frac{\cdn}{n}(c_\bal P^\oo_{\wrb{R}}-BQ^\oo_{\wrb{R}})\label{q11 eq}\\
&\geq \frac{\cdn}{n}(1-\s_{\oo\al_\oo}^R)(0.5c_\bal -BTe^{-\epsd \td{\wrb{R}}}) \nonumber\\
\nonumber
&\geq \frac{\cdn}{n}\frac{1}{1+e^{(1-\delta)R\xdm}}(0.5c_\bal-BTe^{-\epsd R_0}).
\end{align}
Above, recalling the choice $R_0\geq \order{1/\epsd}=\order{1/\eps}$, $R\geq R_0$ implies $BTe^{-\epsd R_0}\leq c_\bal/4$ to obtain
\begin{align}
\Lc(\wrb{R})-\Lc_\star\geq \frac{\cdn\cdot c_\bal}{4n}\frac{1}{1+e^{(1-\delta)R\xdm}}.\label{ineq prl}
\end{align}
Additionally when $\bal=\op$, since $Q^\oo_{\wrb{R}}=0$ in \eqref{q11 eq}, the bound above holds with $R_0=0$ by directly using \eqref{q11 eq}.

Conversely, we upper bound the difference between $\Lc(\Wsb_R)$ and $\Lc_\star$ as follows. Define the worst-case loss difference for $\wrb{R}$ as $(i',k')=\arg\max_{i\in[n],k\in[K]}[\ell(\bgam_\ik^\st)-\ell(\bga_\ik)]$. Using \eqref{lip score gap}\&\eqref{salpha bounds}, we write
\begin{equation}
\begin{aligned}
\Lc(\Wsb_R)-\Lc_\star&\leq \max_{i\in[n],k\in[K]}[\ell(\bgam_\ik^\st)-\ell(\bga_\ik)]\leq \cop\cdot(\bga_{i'k'}-\bgam^\st_{i'k'})\\
&\leq \cop\cdot(1-\s_{i'k'\al_{i'k'}}^\st)B\\
&\leq \cop\cdot Te^{-R\xdm}B.\label{desired Wmm bound}
\end{aligned}
\end{equation}
Combining the last inequality and \eqref{ineq prl}, we conclude that $\Lc(\Wsb_R)<\Lc(\wrb{R})$ whenever
\[
\cop T\cdot e^{-R\xdm}B<\frac{\cdn\cdot c_\bal }{4n}\frac{1}{1+e^{(1-\delta)R\xdm}}\iff \frac{e^{R\xdm}}{1+e^{(1-\delta)R\xdm}}> \frac{4\cop Tn B}{\cdn c_\bal }.
\]
The left hand-side inequality holds for all sufficiently large $R$: Specifically, as soon as $R$ obeys $R>\frac{1}{\delta\xdm}\log(\frac{8\cop Tn B}{\cdn c_\bal})$. This completes the proof of the theorem via contradiction as we obtained $\Lc(\wrb{R})>\Lc(\Wsb_R)$.
\end{proof}

\subsection{Global regularization path}\label{app:global:RP:thm}

The following result is a direct corollary of Theorem \ref{local RP thm}. Namely, we simply restate the final line of this theorem that applies to optimal tokens.
\begin{corollary}[Global Convergence of Regularization Path]\label{cor gm} Suppose Assumptions \ref{assum:loss:prope}\&\ref{ass cvx seq} hold and the optimal indices $\op_\ik=\arg\max_{t\in[T]}\bgam_\ikt$ are unique. Consider the global regularization path $\Wb_{\dm,R}=\min_{\W\in\Rcm,\td{\W}\leq R}\Lc(\W)$. Let $\Wcb^\svm_\dm$ be the non-empty solution set of \eqref{seqattnsvm} with $\bal\gets\op$ normalized to have unit $\dm$-norm. Then
\[
\lim_{R\rightarrow\infty}\dist{\frac{\Wb_{\dm,R}}{R},\Wcb^\svm_\dm}
\]
\end{corollary}

The next corollary directly targets application to \eqref{serm-w} and \eqref{serm-kq}. This corollary is also a strict generalization of Theorem \ref{thm global reg path}. Specifically, we immediately recover Theorem \ref{thm global reg path} by specializing this to the single-output setting $K\gets 1$ and full-dimensional parameterization $m\gets d$.
\begin{corollary}\label{cor global reg path} Suppose Assumptions \ref{assum:loss:prope}\&\ref{ass cvx seq} hold and the optimal indices $\op_\ik=\arg\max_{t\in[T]}\bgam_\ikt$ are unique. Consider the regularization paths associated to \eqref{serm-w} and \eqref{serm-kq}:
\begin{align}
&\Wb_R=\underset{\W\in\Rcm,\tf{\W}\leq R}{\arg\min}\Lc(\W)\quad\text{and}\quad \Kbb_R,\Qbb_R=\underset{\tf{\Kb}^2+\tf{\Qb}^2\leq 2R}{\arg\min}\Lc(\Kb,\Qb)
\end{align}
Suppose \eqref{seqattnsvm} is feasible for $\bal\gets\op$. Let $\Ws$ be the unique solution of \eqref{seqattnsvm} with Frobenius norm and $\Wc^{\svm}_\star$ be the solution set of \eqref{seqattnsvm} with nuclear norm and cost function $\tnuc{\Wc^{\svm}_\star}$. We have that
\[
\lim_{R\rightarrow\infty} \frac{\Wb_R}{R}=\frac{\Ws}{\tf{\Ws}},\quad\lim_{R\rightarrow\infty} \dist{\frac{\Qbb_R\Kbb_R^\top}{R},\frac{\Wcs_\star}{\tnuc{\Wc^{\svm}_\star}}}=0.
\]
\end{corollary}
\begin{proof} We directly apply Corollary \ref{cor gm} with $\dm=F$ and $\dm=\star$ respectively. To obtain the result on $\Wb_R$, we note that $\Ws$ is unique because Frobenius norm-squared is strongly convex. To obtain the result on $(\Qbb_R,\Kbb_R)$, we use Lemma \ref{kqw mapping} and observe that
\[
\Wb_{\st,R}:=\Qbb_R\Kbb_R^\top\in\underset{\W\in\Rcm,\tnuc{\W}\leq R}{\arg\min}\Lc(\W).
\]
We then apply Corollary \ref{cor gm} with $\dm=\star$ to conclude with the convergence of the path $ \Wb_{\st,R}$.
\end{proof}

\subsubsection{Proof of Theorem \ref{thm global reg path}}
Corollary \ref{cor global reg path} already proves Theorem \ref{thm global reg path} through the more general Theorem \ref{local RP thm}. Below, we provide a self contained proof of Theorem \ref{thm global reg path} for clarity.

\begin{proof} Throughout $\dm$ denotes either Frobenius norm or nuclear norm. We will prove that $\wrb{R}$ asymptotically aligns with the set of globally-optimal directions and also $\td{\wrb{R}}\rightarrow \infty$. $\Rcm\subseteq\R^{d\times d}$ denote the manifold of rank $\leq$$m$ matrices.
\\
\noindent\textbf{Step 1:} Let us first prove that $\wrb{R}$ achieves the optimal risk as $R\rightarrow\infty$ -- rather than problem having finite optima. Define $\xdm=1/\td{\Wm}$ and norm-normalized $\Wsb=\xdm\Wm$. Note that $\Wm$ separates tokens $\op$ from rest of the tokens for each $i \in[n]$. Thus, we have that
\begin{align}
\lim_{R\rightarrow\infty}\Lc(\wrb{R})\leq\lim_{R\rightarrow\infty}\Lc(R\cdot\Wsb):=\Lc_\star= \frac{1}{n}\sum_{i=1}^n \ell(\bgam^{\op}_i).\label{glob:asymp loss}
\end{align}
On the other hand, for any $\W\in \Rcm$, define the softmax probabilities $\s^{(i)}=\sft{\X_i\W\z_i}$ and attention features $\x^{\W}_i=\sum_{t=1}^T \s^{(i)}_t\x_t$. Decompose $\x^{\W}_i$ as 
$
\x^{\W}_i=\s^{(i)}_{\op_i}\x_{i\opt_i}+\sum_{t\neq \op_i}\s^{(i)}_t\x_\itt.
$ Set $\bgg_\itt=\bgam^{\op}_i-\bgam_\itt=Y_i\cdot \vb^\top(\x_{i\opt_i}-\x_\itt)>0$, and define
\begin{align}
&B:=\max_{i\in[n]}\max_{t,\tau\in[T]}\tn{\vb}\cdot \tn{\x_\itt-\x_\ittt}\geq \bgg_\itt.\label{glob:BB eq}
\end{align}
Define $c_\op=\min_{i\in[n],t\neq\op_i}\bgg_\itt>0$ and $\bgam^{\W}_i=Y_i\cdot \vb^\top\x^{\W}_i$. We obtain the following score inequalities
\begin{align}\label{glob:score decomp}
&\bgam^{\W}_i\leq \bgam^{\op}_i-c_\op (1-\s^{(i)}_{\op_i})<\bgam^{\op}_i,\\
&|\bgam^{\W}_i-\bgam^{\op}_i|\leq \tn{\vb}\cdot\tn{\x^{\W}_i-\xa_i}\leq \tn{\vb} \sum_{t\neq \op_i}\s^{(i)}_t\tn{\x_\itt-\xa_i}\leq B (1-\s^{(i)}_{\op_i}).\nonumber
\end{align}
We will use the $\bgam^{\W}_i-\bgam^{\op}_i$ term in \eqref{glob:score decomp} to evaluate $\W$ against the reference loss $\Lc_\star$ of \eqref{glob:asymp loss}. 
Using the strictly-decreasing nature of $\ell$, we conclude with the fact that for all (finite) $\W\in \Rcm$, 
\[
\Lc(\W)=\frac{1}{n}\sum_{i=1}^n \ell(\bgam^{\W}_i)> \Lc_\st=\frac{1}{n}\sum_{i=1}^n \ell(\bgam^{\op}_i),
\]
which implies $\td{\wrb{R}}\rightarrow\infty$ together with \eqref{glob:asymp loss}.

\noindent\textbf{Step 2:} To proceed, we show that $\wrb{R}$ converges in direction to $\Wcs$, which denotes the set of SVM minima. Suppose this is not the case and~convergence fails. We will obtain a contradiction by showing that $\Wsb_R=R\cdot\Wsb$ achieves a strictly superior loss compared to $\wrb{R}$. Let us introduce the normalized parameters $\wrt{R}=\frac{\wrb{R}}{R\xdm}$ and $\W'=\frac{\wrb{R}}{\td{\wrb{R}}\xdm}$. Note that $\wrt{R}$ is obtained by scaling down $\W'$ since $\td{\wrb{R}}\leq R$ and $\W'$ obeys $\td{\W'}=\td{\Wm}$.
Since $\wrt{R}$ fails to converge to $\Wcs$, for some $\delta>0$, there exists arbitrarily large $R>0$ such that $\dist{\wrt{R},\Wcs}\geq \delta$. This translates to the suboptimality in terms of the margin constraints as follows: First, since nuclear norm dominates Frobenius, distance with respect to the $\dm$-norm obeys $\distd{\wrt{R},\Wcs}\geq \delta$. Secondly, using triangle inequality,
\[ 
\text{this implies that either~~~}\td{\wrt{R}}\leq \td{\Wm}-\delta/2\text{~~~or~~~}\distd{\W',\Wcs}\geq \delta/2.
\]
In either scenario, $\wrt{R}$ strictly violates one of the margin constraints of \eqref{seqattnsvm}: If $\td{\wrt{R}}\leq \td{\Wm}-\delta/2$, then, since the optimal SVM objective is $\td{\Wm}$, there exists a constraint $i,t\neq\op_i$ for which $\inn{(\x^\op_i-\x_\itt)\z_i^\top,\wrt{R}}\leq 1-\frac{\delta}{2\td{\Wm}}$. If $\distd{\W',\Wcs}\geq \delta/2$, then, $\W'$ has the same SVM objective but it is strictly bounded away from the solution set. Thus, for some $\eps:=\eps(\delta)>0$, $\W'$ and its scaled down version $\wrt{R}$ strictly violate an SVM constraint achieving margin $\leq 1-\eps$. Without losing generality, suppose $\wrt{R}$ violates the first constraint $i=1$. Thus, for a properly updated $\delta>0$ (that is function of the initial $\delta>0$) and for $i=1$ and some \nei $\tau\in \Tc_1$,
\begin{align}
\inn{ (\x^\op_1-\x_{1t})\z_1^\top,\wrt{R}}\leq 1-\delta.\label{margin violate:glob}
\end{align}
Now, we will argue that this leads to a contradiction by proving $\Lc(\Wsb_R)<\Lc(\wrb{R})$ for sufficiently large $R$.

To obtain the result, we establish a refined softmax probability control as in Step 1 by studying distance to $\Lc_\star$. Following \eqref{glob:score decomp}, denote the score function at $\wrb{R}$ via $\bgam^R_i:=\bgam_i^{\wrb{R}}$. Similarly, let $\sir_i=\sft{\abr_i}$ with $\abr_i=\X_i\wrb{R}\z_i$. Set the corresponding notation for the reference parameter $\Wsb_R$ as $\bgam^\st_i,\s^\st_i,\ab^\st_i$.  Recall that $R\geq \td{\wrb{R}}$ and $\xdm:=1/\td{\Wm}$. We note the following softmax inequalities 
\begin{align}
&\s^\st_{i\op_i}\geq \frac{1}{1+Te^{-R\xdm}}\geq 1-Te^{-R\xdm}\quad \text{for all}\quad i\in[n], \label{glob:salpha bounds}\\
&s^R_{i\op_i}\leq \frac{1}{1+e^{-(1-\delta)\td{\wrb{R}}\xdm}}\leq \frac{1}{1+e^{-(1-\delta)R\xdm}}\quad\text{for}\quad i=1.\nn
\end{align}
The former inequality is thanks to $\Wm$ achieving $\geq$$1$ margins on all tokens $[T]-\op_i$ and the latter arises from the $\delta$-margin violation of $\wrb{R}$ at $i=1$ i.e.~Eq.~\eqref{margin violate:glob}. Since $\ell$ is strictly decreasing with Lipschitz derivative and the scores are upper/lower bounded by an absolute constant (as tokens are bounded and fixed), we have that $\cop\geq -\ell'(\bgam_i^{\W})\geq \cdn$ for some constants $\cop>\cdn>0$. Thus, following Eq.~\eqref{glob:BB eq}, the score decomposition \eqref{glob:score decomp}, and \eqref{glob:salpha bounds} we can write
\begin{align}\label{glob:ineq prl}
\Lc(\wrb{R})-\Lc_\star&\geq \frac{1}{n}[\ell(\bgam_1^{\wrb{R}})-\ell(\bgam^{\op}_1)]\geq \frac{\cdn}{n}(\bgam^{\op}_{1}-\bgam_1^{\wrb{R}})\\
\nonumber
&\geq \frac{\cdn}{n}c_\op (1-\s^R_{1\op_1}).
\\
\nonumber 
&\geq \frac{\cdn c_\op}{n}\frac{1}{1+e^{(1-\delta)R\xdm}}.
\end{align}
Conversely, we upper bound the difference between $\Lc(\Wsb_R)$ and $\Lc_\star$ as follows. Define the worst-case loss difference for $\wrb{R}$ as $j=\arg\max_{i\in[n]}[\ell(\bgam_i^\st)-\ell(\bgam^{\op}_i)]$. Using \eqref{glob:score decomp}\&\eqref{glob:salpha bounds}, we write
\begin{equation*}
\begin{aligned}
\Lc(\Wsb_R)-\Lc_\star&\leq \max_{i\in[n]}[\ell(\bgam_i^\st)-\ell(\bgam^{\op}_i)]\leq \cop\cdot(\bgam^{\op}_{j}-\bgam^\st_{j})\\
&\leq \cop\cdot(1-\s^\st_{j\op_j})B\\
&\leq \cop\cdot Te^{-R\xdm}B.
\end{aligned}
\end{equation*}
Combining the last inequality and \eqref{glob:ineq prl}, we conclude that $\Lc(\Wsb_R)<\Lc(\wrb{R})$ whenever
\[
\cop T\cdot e^{-R\xdm}B<\frac{\cdn\cdot c_\op }{n}\frac{1}{1+e^{(1-\delta)R\xdm}}\iff \frac{e^{R\xdm}}{1+e^{(1-\delta)R\xdm}}> \frac{\cop Tn B}{\cdn c_\op }.
\]
The left hand-side inequality holds for all sufficiently large $R$: Specifically, as soon as $R$ obeys $R>\frac{1}{\delta\xdm}\log(\frac{2\cop Tn B}{\cdn c_\op})$. This completes the proof of the theorem by contradiction since we obtained $\Lc(\wrb{R})>\Lc(\Wsb_R)$.
\end{proof}

\section{Supporting Experiments}\label{app supp exp}
\begin{figure}[t]
    \centering
    \begin{minipage}{.37\textwidth}
        \begin{tikzpicture}
        \node at (0,0) {\includegraphics[height=.6\columnwidth, trim={1.2cm 1.4cm 0 -1cm}, clip]{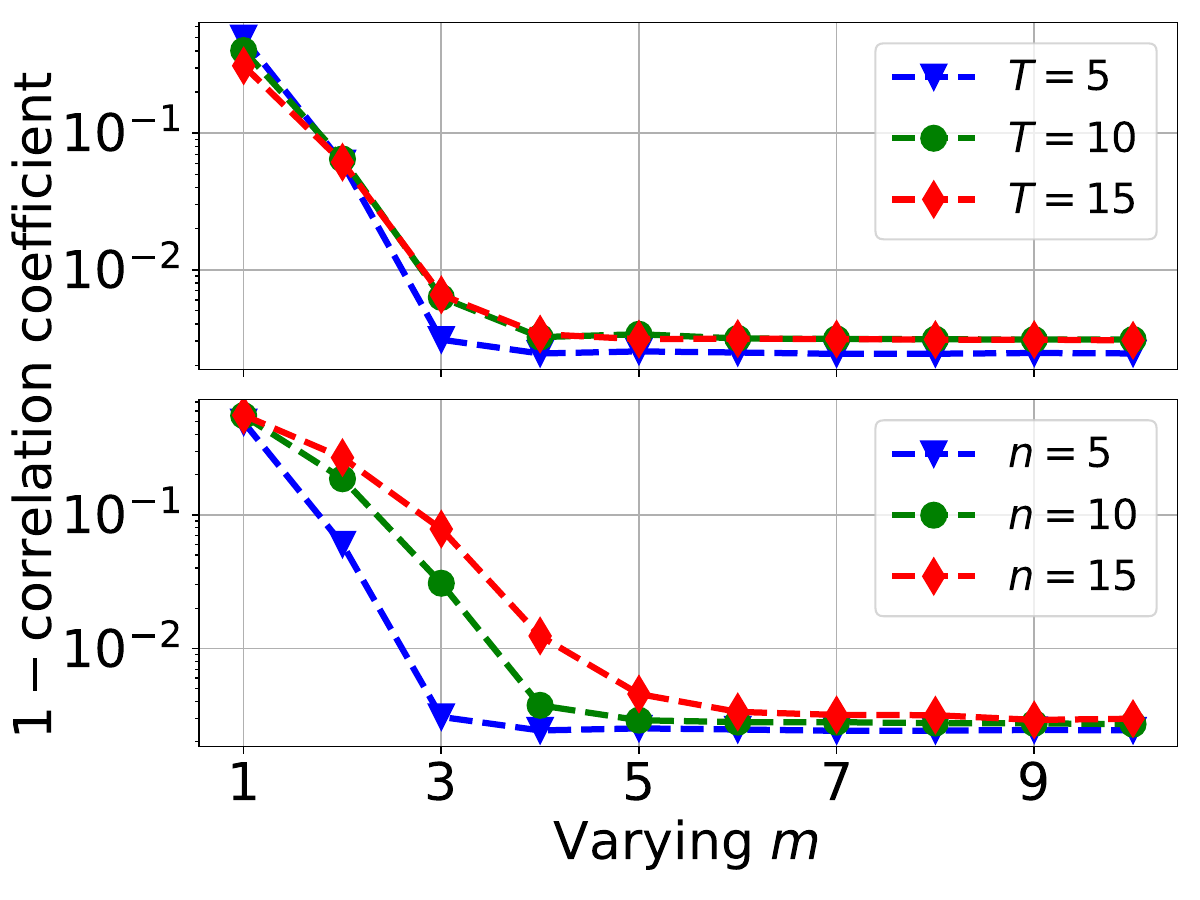}};
        \node at (0,-2.) {\small{{Varying $m$}}};
        \node[rotate=90] at (-2.7,0) {\small{$1-$correlation coefficient}};
        \end{tikzpicture}
        \vspace{10pt}
        \caption{Convergence behavior of GD when training attention weights $(\Kb,\Qb)\in\R^{d\times m}$ with random data and varying $m$. The misalignment between attention SVM and GD, $1-\corr{\Ws_{\star,\bal},\Kb\Qb^\top}$, is studied. $\Ws_{\star,\bal}$ is from \eqref{eqn:sattnsvmst} with GD tokens $\bal$ and $m=d$. Subfigures with fixed $n=5$ and $T=5$ show that as $m$ approaches or exceeds $n$, $\Kb\Qb^\top$ aligns more with $\Ws_{\star,\bal}$.  }
        \label{fig rank m}
    \end{minipage}
    \hspace{5pt}
    \begin{minipage}{.6\textwidth}
    \subfigure[Evolution of correlation under varying $d$]{
        \begin{tikzpicture}
        \node at (0,0) {\includegraphics[height=.33\columnwidth, trim={1.3cm 1.3cm 0 0}, clip]{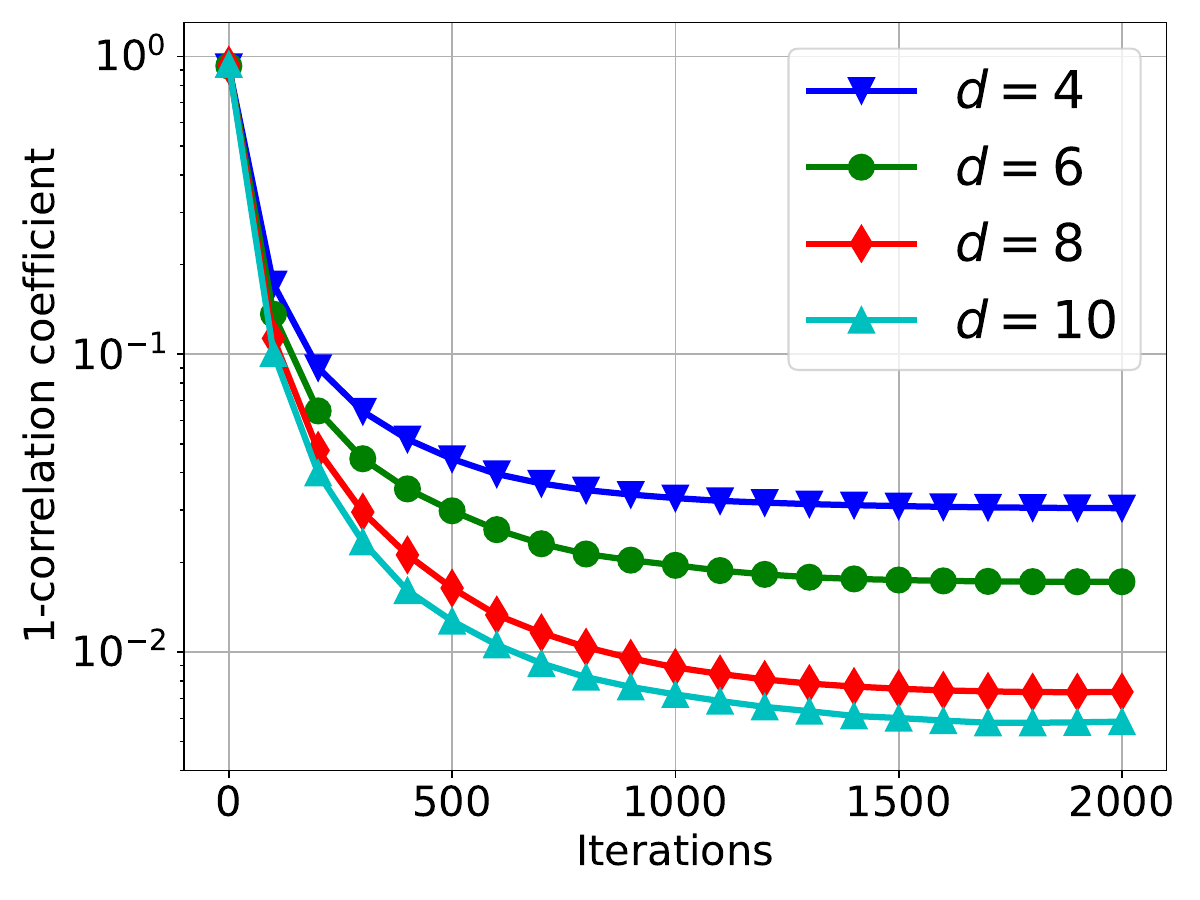}};
        \node at (0.2,-1.9) {\small{Iterations}};
        \node[rotate=90] at (-2.55,0) {\small{$1-$correlation coefficient}};
        \end{tikzpicture}
        \label{fig nn itr}
    }
    \hspace{-10pt}
    \subfigure[$\Gamma$ vs correlation coefficient]{
        \begin{tikzpicture}
        \node at (0,0) {\includegraphics[height=.33\columnwidth, trim={1.3cm 1.3cm 0 0}, clip]{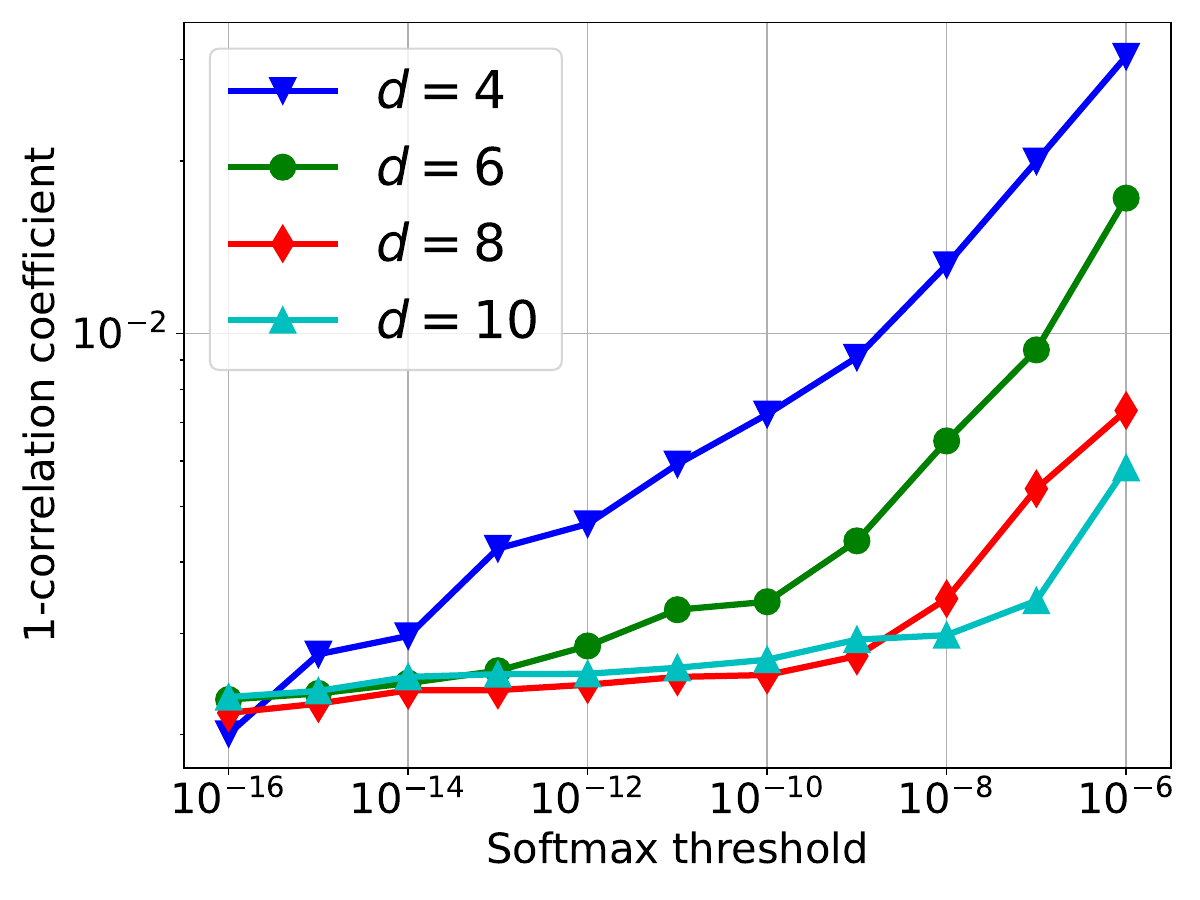}};
        \node at (0.2,-1.9){\small{Masked token threshold ($\Gamma$)}};
        \end{tikzpicture}
        \label{fig nn sfx zero}
    }
    \vspace{-10pt}
    \caption{Behavior of GD  with nonlinear nonconvex prediction head and multi-token compositions. \textbf{(a)}: Blue, green, red and teal curves represent the evolution of $1-\corr{\W,\W^{\rfn}}$ for $d=4,6,8$ and $10$ respectively, which have been displayed in Figure~\ref{fig nn diff d}(upper). \textbf{(b)}: Over the $500$ random instances as discussed in Figure~\ref{fig nn diff d}, we filter different instances by constructing masked set with tokens whose softmax output $<\Gamma$ and vary $\Gamma$ from $10^{-16}$ to $10^{-6}$. The corresponding results of $1-\corr{\W,\W^\rfn}$ are displayed in blue, green, red and teal curves.} 
    \label{fig nn main}
    \end{minipage}
\end{figure}

In this section, we introduce implementation details and additional experiments. Code is available at
\begin{center}
\url{https://github.com/umich-sota/TF-as-SVM}
\end{center}
We create a 1-layer self-attention using \texttt{PyTorch}, training it with the SGD optimizer and a learning rate of $\eta=0.1$. We apply normalized gradient descent to ensure divergence of attention weights. The attention weight $\W$ is then updated through
\[
\W(k+1)=\W(k)-\eta\frac{\nabla\Lc(\W(k))}{\|\nabla\Lc(\W(k))\|_F}.
\]
In the setting of $(\Kb,\Qb)$-parameterization, we noted that with extended training iterations, the norm of the combined parameter $\Kb\Qb^\top$ consistently rises, despite the gradient being treated as zero due to computational limitations. To tackle this issue, we introduce a minor regularization penalty to the loss function, ensuring that the norms of $\Kb$ and $\Qb$ remain within reasonable bounds. This adjustment involves
\[
\widetilde\Lc(\Kb,\Qb)=\Lc(\Kb,\Qb)+\lambda(\|\Kb\|^2_F+\|\Qb\|^2_F).
\]
Here, we set $\lambda$ to be the the smallest representable number, e.g. computed as $1+\lambda\neq1$ in \texttt{Python}, which is around $2.22\times10^{-16}$. Therefore, $\Kb,\Qb$ parameters are updated as follows.
\[
\Kb(k+1)=\Kb(k)-\eta\frac{\nabla\widetilde\Lc_\Kb(\Kb(k),\Qb(k))}{\|\nabla\widetilde\Lc_\Kb(\Kb(k),\Qb(k))\|_F}, \qquad \Qb(k+1)=\Qb(k)-\eta\frac{\nabla\widetilde\Lc_\Qb(\Kb(k),\Qb(k))}{\|\nabla\widetilde\Lc_\Qb(\Kb(k),\Qb(k))\|_F}.
\]
$\bullet$ As observed in previous work \cite{tarzanagh2023margin}, and due to the exponential expression of softmax nonlinearity and computation limitation, \texttt{PyTorch} has no guarantee to select optimal tokens when the score gap is too small. Therefore in Figures~\ref{fig overparam W bar}, \ref{fig overparam bar} and \ref{fig overparam}, we generate random tokens making sure that $\min_{i\in[n],t\neq\op_i}\bgam_{i\op_i}-\bgam_{it}\geq\underline\gamma$ and we choose $\underline\gamma=0.1$ in our experiments.
\paragraph{Rank sensitivity of $(\Kb,\Qb)$-parameterization (Figure~\ref{fig rank m}).} In Figure \ref{fig rank} and Lemma~\ref{lem:rank}, we have both theoretically and empirically established that the rank of the SVM solution, denoted as $\Ws$ in \eqref{eqn:sattnsvm} or $\Ws_\st$ in \eqref{eqn:sattnsvmst}, is at most rank $\max(n,d)$. Now, moving to Figure~\ref{fig rank m}, we delve into GD performance across various dimensions of $\Kb,\Qb\in\R^{d\times m}$ while keeping $d=20$ fixed and varying $m$ from $1$ to $10$. In the upper subfigure, we maintain a constant $n=5$ and vary $T$ within $\{5,10,15\}$, while in the lower subfigure, $T$ is fixed at $5$ and $n$ changes within $\{5,10,15\}$. Results are depicted using blue, green, and red dashed curves, with both $y$-axes representing $1-\corr{\W,\Ws_{\st,\bal}}$, where $\W$ represents the GD solution and $\Ws_{\st,\bal}$ is obtained from \eqref{eqn:sattnsvmst} by employing token indices $\bal$ selected via GD and setting the rank limit to $m=d$. Observing both subfigures, we note that a larger $n$ necessitates a larger $m$ for attention weights $\Kb\Qb^\top$ to accurately converge to the SVM solution (Figure~\ref{fig rank m}(lower)). Meanwhile, performances remain consistent across varying $T$ values (Figure \ref{fig rank m}(upper)). This observation further validates Lemma \ref{lem:rank}. Furthermore, the results demonstrate that $\W$ converges directionally towards $\Ws_{\st,\bal}$ as long as $m\gtrsim n$, thereby confirming the assertion in our Theorem~\ref{thm:local:gd}.  
%
%
\paragraph{Behavior of GD  with nonlinear nonconvex prediction head and multi-token compositions (Figure~\ref{fig nn main}).} To better investigate how correlation changes with data dimension $d$, we collect the solid curves in Figure~\ref{fig nn diff d}(upper) and construct as Figure~\ref{fig nn itr}. Moreover, Figure \ref{fig nn sfx zero} displays the average correlation of instances (refer to scatters in Figure~\ref{fig nn diff d} (lower)), considering masked tokens with softmax probability $<\Gamma$. Both findings highlight that higher $d$ enhances alignment. For $d\geq8$ or $\Gamma\leq10^{-9}$, the GD solution $\W$ achieves a correlation of $>0.99$ with the SVM-equivalence $\W^{\rfn}$, defined in Section~\ref{sec:multi}.
\begin{figure}[t]
    \centering
    \hspace{-10pt}
    \subfigure[$\tau$ and $\lambda$ parameters relationship]{
        \begin{tikzpicture}
        \node at (0,0) {\includegraphics[height=.22\columnwidth, trim={1cm 1.3cm 0 0}, clip]{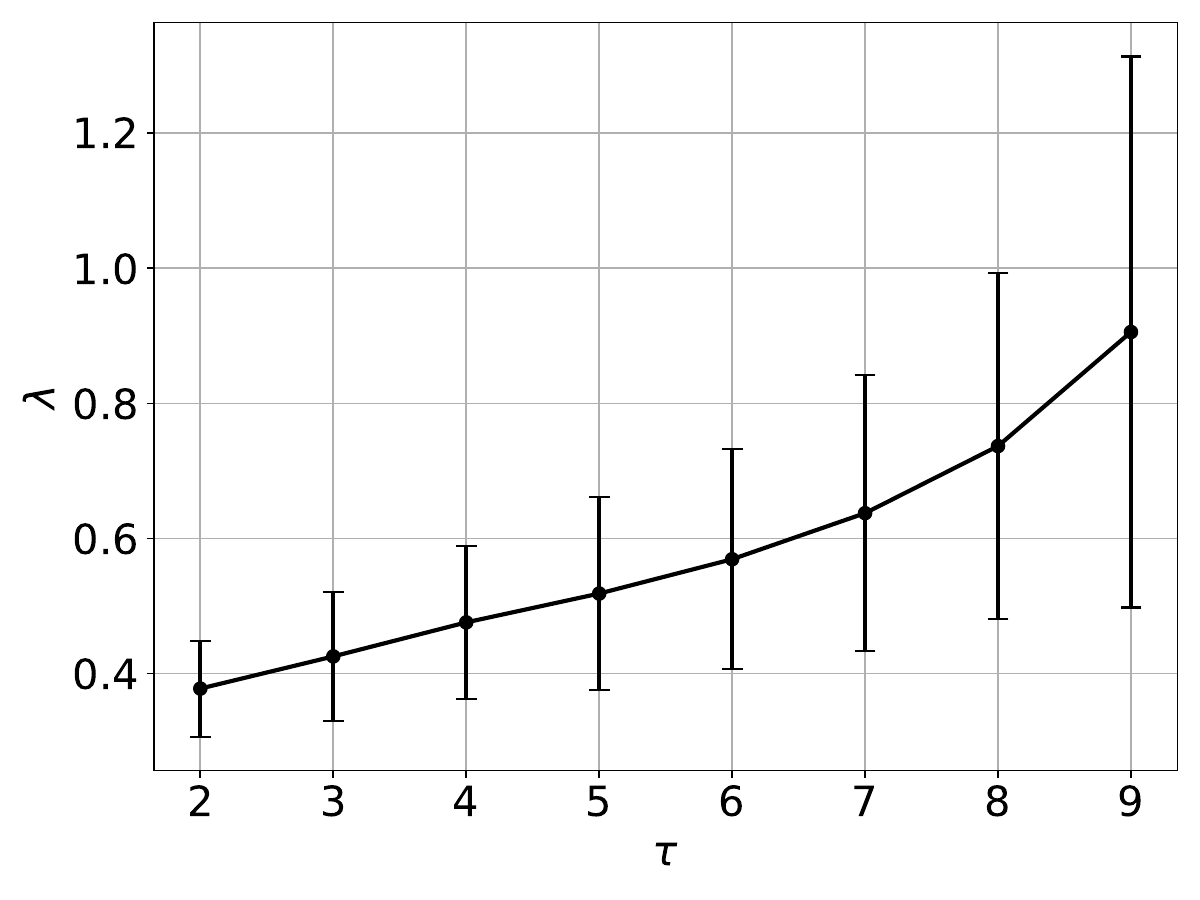}};
        \node at (0,-2.) {\small{$\tau$}};
        \node[rotate=90] at (-2.65,0) {\small{$\lambda$}};
        \end{tikzpicture}
        \label{fig tau lambda}
    }
    \hspace{-10pt}
    \subfigure[ $\tau$ and  \# of selected tokens relationship]{
        \begin{tikzpicture}
        \node at (0,0) {\includegraphics[height=.22\columnwidth, trim={1cm 1.3cm 0 0}, clip]{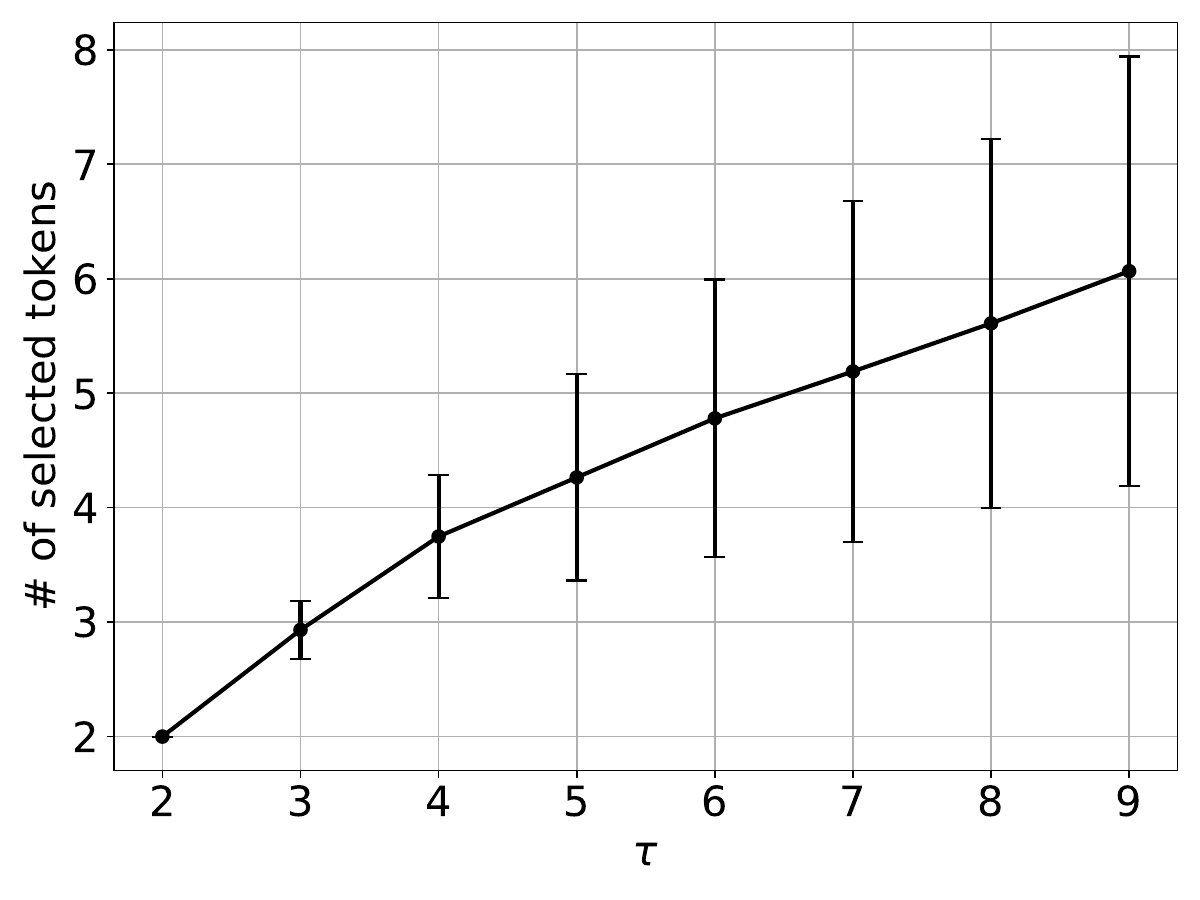}};
        \node at (0,-2.) {\small{$\tau$}};
        \node[rotate=90] at (-2.7,0) {\small{\# of selected tokens}};
        \end{tikzpicture}
        \label{fig tau ns}
    }
    \hspace{-10pt}
    \subfigure[Distribution of \# selected tokens over varying $\tau$]{
        \begin{tikzpicture}
        \node at (0,0) {\includegraphics[height=.22\columnwidth, trim={1.3cm 1.3cm 0 0}, clip]{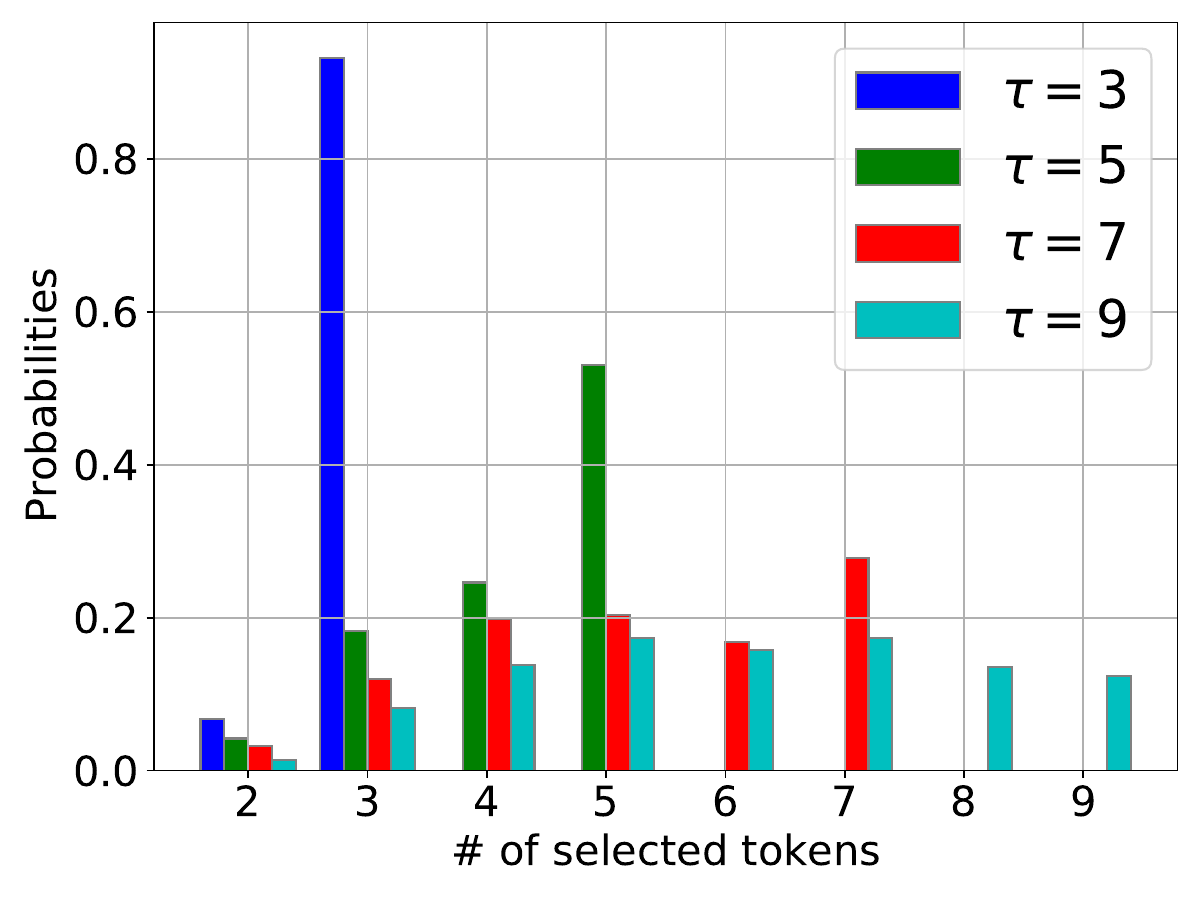}};
        \node[rotate=90] at (-2.7,0) {\small{Probabilities}};
        \node at (0,-2.){\small{\# of selected tokens}};
        \end{tikzpicture}
        \label{fig tau prob}
    }
    \caption{ Behavior of GD when selecting multiple tokens.} 
    \label{fig multi tau}
\end{figure}
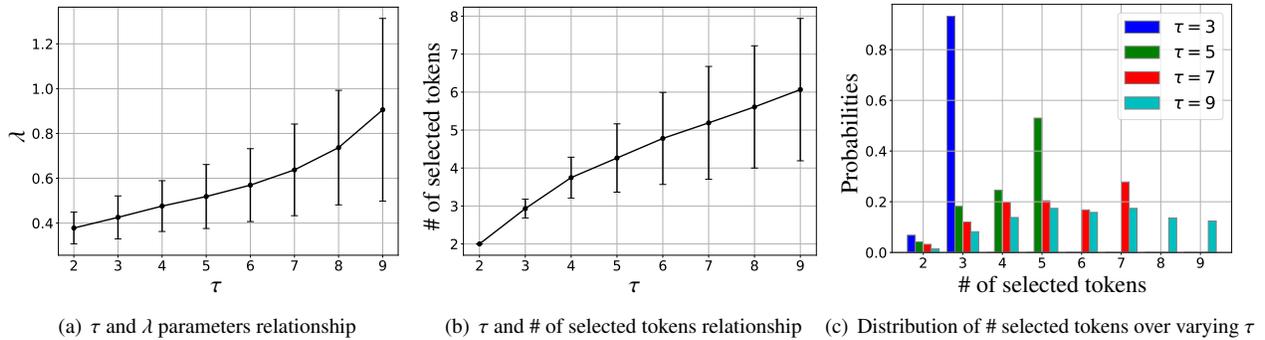
\paragraph{Investigation of Lemma~\ref{example dataset} over different $\tau$ selections (Figure~\ref{fig multi tau}).}   Consider the setting of Section~\ref{sec when} and Lemma~\ref{example dataset}.  Figure~\ref{fig multi corrs} explores the influence of $\lambda$ on the count of tokens selected by GD-derived attention weights. As $\lambda$ increases, the likelihood of selecting more tokens also increases. Shifting focus to Figure~\ref{fig multi tau}, we examine the effect of $\tau$. For each outcome, we generate random $\lambda$ values, retaining pairs $(\lambda,\X)$ satisfying $\tau$ constraints, with averages derived from $100$ successful trials. The results indicate a positive correlation among $\tau$, $\lambda$, and the number of selected tokens. Moreover, Figure~\ref{fig tau prob} provides a precise distribution of selected token counts across various $\tau$ values (specifically $\tau\in\{3,5,7,9\}$). The findings confirm that the number of selected tokens remains within the limit of $\tau$, thus validating the assertion made in Lemma~\ref{example dataset}.

\end{document}